\documentclass[twoside,11pt]{article}

\usepackage[abbrvbib, preprint]{jmlr2e}

\usepackage[english]{babel}
\usepackage{amsmath,amssymb,graphicx,enumerate,bbm,tabularx}

\usepackage[frozencache=true,cachedir=.]{minted}  

\usepackage{xparse}

\usepackage{etoc}

\newlength\tocrulewidth
\NewDocumentCommand{\myrule}{O{1pt} O{3pt}}{%
  \par\nobreak 
  \kern\the\prevdepth 
  \kern#2 
  {\hrule height #1 width\hsize} 
  \kern#2 
  \nointerlineskip 
}

\newcommand{\CC}{\mathcal{C}}

\newcommand{\fixpointofcmap}{c^{*}}
\newcommand{\qfixedpoint}{q^{*}}

\newcommand{\vardbtilde}[1]{\tilde{\raisebox{0pt}[0.85\height]{$\tilde{#1}$}}}

\newcommand{\nobracket}{}

\newcommand{\tmem}[1]{{\em #1\/}}
\newcommand{\tmmathbf}[1]{\ensuremath{\boldsymbol{#1}}}
\newcommand{\tmop}[1]{\ensuremath{\operatorname{#1}}}
\newcommand{\tmscript}[1]{\text{\scriptsize{$#1$}}}
\newcommand{\tmtextbf}[1]{\textbf{#1}}
\newcommand{\tmtextit}[1]{\textit{#1}}
\newenvironment{enumeratenumeric}{\begin{enumerate}[1.] }{\end{enumerate}}
\newenvironment{enumerateroman}{\begin{enumerate}[i.] }{\end{enumerate}}
\newenvironment{itemizedot}{\begin{itemize} }{\end{itemize}}
\newcommand{\nonconverted}[1]{\mbox{}}

\newcommand{\RR}{\mathbb{R}}

\newcommand{\vp}{\varphi}

\newcommand{\python}[1]{\mintinline{python}{#1}}

\newcommand{\poscite}[1]{\citeauthor{#1}'s (\citeyear{#1})}

\newcommand{\james}[1]{}
\newcommand{\js}[1]{}
\newcommand{\sam}[1]{}
\newcommand{\greg}[1]{}
\newcommand{\valentin}[1]{}

\ShortHeadings{Deep Kernel Shaping}{Martens et al.}
\firstpageno{1}

\begin{document}

\title{Rapid training of deep neural networks without skip connections or normalization layers using Deep Kernel Shaping}

\author{\name James Martens \email jamesmartens@deepmind.com \\
      \name Andy Ballard \email aybd@deepmind.com \\
      \name Guillaume Desjardins \email gdesjardins@deepmind.com \\
      \name Grzegorz Swirszcz \email swirszcz@deepmind.com \\
      \name Valentin Dalibard \email vdalibard@deepmind.com \\
      \addr DeepMind, London, UK\\
      \AND
      \name Jascha Sohl-Dickstein \email jaschasd@google.com \\
      \name Samuel S. Schoenholz \email schsam@google.com \\
      \addr Google, Mountainview, USA
}


\maketitle


\begin{abstract}
Using an extended and formalized version of the Q/C map analysis of \citet{poole2016exponential}, along with Neural Tangent Kernel theory, we identify the main pathologies present in deep networks that prevent them from training fast and generalizing to unseen data, and show how these can be avoided by carefully controlling the ``shape" of the network's initialization-time kernel function. We then develop a method called Deep Kernel Shaping (DKS), which accomplishes this using a combination of precise parameter initialization, activation function transformations, and small architectural tweaks, all of which preserve the model class. In our experiments we show that DKS enables SGD training of residual networks without normalization layers on Imagenet and CIFAR-10 classification tasks at speeds comparable to standard ResNetV2 and Wide-ResNet models, with only a small decrease in generalization performance. And when using K-FAC as the optimizer, we achieve similar results for networks \emph{without} skip connections. Our results apply for a large variety of activation functions, including those which traditionally perform very badly, such as the logistic sigmoid. In addition to DKS, we contribute a detailed analysis of skip connections, normalization layers, special activation functions like RELU and SELU, and various initialization schemes, explaining their effectiveness as alternative (and ultimately incomplete) ways of ``shaping" the network's initialization-time kernel.
\end{abstract}


\james{TODO: maybe add a table of notation with pointers to specific sections or equations where it is defined in the text?}

\james{make sure the notion of ``subnetwork'' is used consistently throughout. Need to be careful about calling the whole network an example of a subnetwork.}

\james{TODO: consider making some Colab code public with the paper. e.g. for the C map plotting}

\james{change code to use have a hyperparam for residual weight instead of shortcut weight in order to be consistent with paper}

\newpage

\section{Introduction}

\james{TODO: Should add a quick definition of extended Q/C maps with some math notation to make the discussion in these earlier sections easier to follow}

The current standard approach to deep learning relies on a combination of architectural elements including skip connections, normalization layers, and carefully chosen activation functions (such as RELU) to overcome the well-documented optimization difficulties present in traditional deep neural networks \citep{ioffe2015batch, he2016deep, szegedy2017inception}. While this approach has proven very successful, enabling many applications in diverse fields such as vision \citep[e.g.][]{he2016deep, tan2019efficientnet}, language \citep[e.g.][]{vaswani2017attention, gpt3}, protein folding \citep{jumper2021highly} and reinforcement learning \citep[e.g.][]{espeholt2018impala, silver2018general}, it is not entirely satisfying for at least several reasons.

First, the precise mechanism of action of these elements, as well as their interaction, is still not well understood, despite some recent progress in this area. This lack of understanding makes it difficult to design new network architectures, as architectural choices not only affect the network's expressivity, but also its trainability, and in ways that are hard to predict. Second, without competitive alternatives to compare to, it's not clear whether the current standard approach enables deep networks to reach their full potential, or whether it has unseen drawbacks and limitations. For example, while the use of skip connections helps very deep networks to train much faster, this might only be because it makes them behave like an ensemble of shallower networks \citep{veit2016residual}. Finally, the extra complexity introduced by these architectural elements, and their non-trivial interactions, makes theoretical analyses much more difficult, potentially holding us back from developing a more fundamental understanding of deep learning. And while existing theoretical analyses can (and often do) drop these elements, they do so at the risk of missing an essential piece of the picture.


In an ideal world, modelling and trainability would be decoupled, so that architectures could be designed with only modelling considerations in mind, and rapid training would be guaranteed as long as they conformed to a well-defined set of rules. One might also hope that the components of such a framework would each have a clear purpose, be theoretically well-understood, and interact with each other in simple and predictable ways. 

In the present work we take an important step towards this ``ideal world", while simultaneously providing a competitive alternative to the current standard approach to deep learning. We do so by developing a theoretically well-founded method for constructing deep networks which allows them to be rapidly trained without the use of skip connections, normalization layers, or standard activation functions. Our approach, which we call Deep Kernel Shaping (DKS), requires only minor model class-preserving modifications to the architecture and activation functions, and is fully compatible with existing analysis frameworks such as Neural Tangent Kernel (NTK) theory \citep{jacot2018neural}. 

As we show in experiments, DKS enables very deep residual networks without normalization layers to be trained using SGD on Imagenet and CIFAR-10 classification tasks at similar speeds to standard ResNetV2 \citep{he2016identity} and Wide-ResNet models \citep{zagoruyko2016wide}. It also achieves the same for networks \emph{without} skip connections or normalization layers when combined with stronger optimizers like K-FAC \citep{martens2015optimizing} or Shampoo \citep{gupta2018shampoo}. Moreover, it works well with a large variety of activation functions, including those that traditionally perform very poorly (such as the logistic sigmoid). As a caveat, we observe a small decrease in generalization performance compared to standard ResNets, which we believe can be addressed in future work.

While there have been some recently proposed methods for training very deep networks without skip connections and normalization layers \citep[e.g][]{schoenholz2016deep, balduzzi2017shattered, xiao2018dynamical}, to the best of our knowledge, DKS is the first to achieve training speeds competitive with standard ResNet models on a challenging dataset like Imagenet. And while our use of K-FAC plays an important role in these results, our experiments show that K-FAC alone is not enough, even when used in combination with the aforementioned methods.

The starting point for our development of DKS is the work of \citet{poole2016exponential}, who described the approximate initialization-time behavior of fully-connected combined layers (which we define here as an affine layer followed by an element-wise nonlinear layer) using special one-dimensional maps known as ``Q and C maps", and then used the fixed-point behavior of these maps to describe the depth-limiting behavior of a network composed of many such layers in sequence. We also take inspiration from \citet{schoenholz2016deep}, who applied this analysis framework to design an initialization method which modulated the fixed point behavior of each layer's C map to slow the loss of ``geometric information" with depth, and demonstrated encouraging results training very deep networks without skip networks or normalization layers.

\james{TODO: consider giving a more detailed account of Q/C maps here. Maybe with some equations like in the version of the intro text I wrote for the TAT paper?}

While originally derived within the semi-rigorous framework of ``mean field analysis", it turns out that Q/C maps also describe the approximate behavior of a combined layer's kernel function in wide networks, and as we will show, can be applied to convolutional layers if one uses a \textbf{\textit{Delta initialization}} for the filter banks. These maps can be further extended to describe entire networks with arbitrary topologies, where they provide useful information outside of the depth-limiting case considered by \citet{poole2016exponential}. In the case of a fully connected network $f$, its Q map approximates the mapping from $\|x\|^2 / \dim(x)$ to $\|f(x)\|^2 / \dim(f(x))$, and its C map approximate the mapping from $x^\top x' / (\|x\|\|x'\|)$ to $f(x)^\top f(x') / (\|f(x)\|\|f(x')\|)$. 

In deeper networks, the C map can easily become ``degenerate’’, mapping most of its input domain $[-1, 1]$ to a small point-like subset of its codomain, $[-1, 1]$. The implication of this is that the distance between any pair of output vectors from the network is effectively independent of the distance between the corresponding pair of input vectors. As we argue, both heuristically and using NTK theory, this behavior inevitably leads to very slow training and/or poor generalization under gradient descent.

We thus design DKS to prevent this problem, while also guarding against certain secondary pathologies such as a badly behaved Q map, high approximation error in the Q/C maps themselves, and network behavior that is ``too linear" (which limits network expressivity under gradient descent). To do this, we relate the overall ``shape" of the network's C map and its tendency to become degenerate, to its value and derivative at a couple of points, which we in turn relate to the values of derivatives of the C maps for the network's individual layers. We then control these properties (and a couple of additional ones to address the aforementioned secondary pathologies), by transforming each activation function using a model class preserving scale and shift operation its input and output. This transformation is the same for each nonlinear layer for a given activation function, but depends on the global structure of the network. In theory, we also require that sum operations in the network are ``normalized" in a certain way, that a special kind of data preprocessing is used, and that pooling layers are replaced with certain roughly equivalent alternatives, although we find the latter two of these to be non-essential in practice.

In addition to developing DKS, we also use Q/C maps to help explain the effectiveness of standard deep learning techniques such as normalization layers, skip connections, initialization methods, and common activation functions such as RELU and SELU \citep{klambauer2017self}, in terms of their effect on the network's initialization-time kernel. This is facilitated in part by the connections we establish between Q/C maps and alternative analysis frameworks such as ``variance propagation" and ``signal propagation" which underlie many of the said techniques.

\section{Outline}

\james{TODO in future version: Consider making this section into a more detailed overview of the entire paper. Almost like a mini-conference version.}

This manuscript is organized into five parts.

Part \ref{part:theory_prelim} gives our assumptions and establishes the theoretical concepts used in subsequent parts. We begin in Sections \ref{sec:notation-and-arch-assumptions} and \ref{sec:param_dist} by defining our notation and stating our initial assumptions on network architecture and initialization. In Section \ref{sec:kernel-approx-main} we discuss kernel functions for networks conforming to these assumptions, and how they can be approximated with much simpler functions at initialization time. In Sections \ref{sec:QC_map_combined} and \ref{sec:extended-maps} we show how these kernel approximations can be further broken down in terms of a generalized version of the Q/C maps originally proposed in \citet{poole2016exponential}. 
Derivative computations for Q/C maps are given in Section \ref{sec:QC_map_derivatives}, and Section \ref{sec:sum-ops} discusses how to handle sum operations when computing Q/C maps. In Section \ref{sec:uniform-q}, we show how C maps can be simplified down to one dimensional functions (from three dimensions) using a special type of data preprocessing which is designed to make two of their three inputs constant. And in Section \ref{sec:uniform-q-consequences} we discuss additional consequences of this, including that C maps become ``positive definite functions".

With the theoretical groundwork established, Part \ref{part:desirable_QCmap_properties} focuses on identifying desirable Q/C map properties and ways to achieve these. In Section \ref{sec:Cmaps_trainability} we discuss C map behavior in deep networks and how it can -- and usually does -- become ``degenerate", leading to slow training and/or poor generalization. We then set out to analyze C maps with the hope of controlling their properties so as to prevent this. To that end, in Section \ref{sec:Cmap-analysis} we use the positive definiteness of a C map to show how its deviation from the identity function (which is large in degenerate maps) can be predicted from its derivative at $1$ and value at $0$, implying that we can prevent degeneration by enforcing certain conditions on these quantities. In Section \ref{sec:linear-networks} we identify another way that a network can fail to be trainable: that its parameters must move very far from their initial values before the network can exhibit any significantly nonlinear behavior. We then show this failure mode can be avoided by enforcing a condition on the C map of each nonlinear layer. In Section \ref{sec:error-and-Qmaps} we identify the breakdown of our kernel approximations as a third problem that we must avoid, and propose several solutions to this, including a condition to enforce on the network's Q map.

Having identified three distinct ways that a network can fail to be trainable, and conditions to enforce on the Q and C maps to prevent or mitigate these failures, we proceed with the specification and derivation of DKS in Part \ref{part:spec_and_deriv_of_DKS}. In Section \ref{sec:QCmap-conditions} we list the four conditions on the Q and C maps of the network (or more precisely its ``subnetworks") which we will enforce. Then in Section \ref{sec:global-to-local} we show how these conditions can be reduced to ones on the Q/C maps of the individual layers of the network via a special translation mechanism called the ``maximal slope function" which encodes structural information about the network (including its depth). In Section \ref{sec:activation-transform} we describe our main mechanism of enforcement for these per-layer conditions: scaling and shifting operations applied to the input and output of each nonlinear layer's activation function (which preserve the model class). Finally, in in Sections \ref{sec:normalization-layers} and \ref{sec:pooling-layers} we discuss how to deal with normalization and pooling layers in DKS. With DKS fully derived, we give a step-wise summary of it in Section \ref{sec:arch-requirements-method}, and provide details for the more difficult aspects of its implementation in Section \ref{sec:implementation-details}. In Section \ref{sec:application-to-resnet} we demonstrate the application of DKS on the various modified ResNet and Wide-ResNet models which we use in our experiments, including ones with skip connections and/or normalization layers removed.

Before proceeding to experiments, in Part \ref{part:additional_analysis} we delve deeper into the theory underlying DKS, and analyze various related approaches from the perspective of kernel approximations and Q/C maps. In Section \ref{sec:NTK-analysis} we review Neural Tangent Kernel (NTK) theory, and give an elegant expression for the NTK using (extended) C maps. We then show how NTK theory predicts slow training and poor generalization for networks with degenerate C maps, and characterize the form of the NTK for networks constructed using DKS. In Section \ref{sec:var/sig-prop-relation-to-kernel} we review certain previously published methods for understanding the behavior of neural networks at initialization time (such as variance/signal propagation), show how they give rise to what are essentially Q and C maps (but different interpretations for what they actually compute), and advocate for the use of approximate kernel analysis as a more flexible and mathematically rigorous alternative. Exploiting these connections, we then review and analyze some prior methods for initializing and constructing neural networks in Section \ref{sec:review-and-analysis-of-related}, including standard techniques such as normalization layers and residual networks, as well as methods aimed at replacing them. In each case we argue that the method can interpreted as enforcing some set of conditions on the network's Q/C map, which is often a strict subset of those enforced by DKS.

Finally, in Part \ref{part:experiments_and_conclusions} we discuss experiments and conclude. This begins in Sections \ref{sec:experimental-setup} and \ref{sec:experimental-results}, where we describe the setup of our experiments, and discuss their results. Our experiments include comparisons of DKS to standard ResNets, the methods reviewed/analyzed in Section \ref{sec:review-and-analysis-of-related}, and various ``ablated"/modified versions of DKS. We then summarize our conclusions in Section \ref{sec:conclusions}, and in Section \ref{sec:limitations-and-future-directions} discuss the limitations of DKS and possible ways to address them in future work.

\newpage

\part*{Table of contents}

\myrule


\

\

\setcounter{tocdepth}{1}  

\begingroup
\etocsettocstyle{}{}
\tableofcontents 
\endgroup

\newpage

\part{Theoretical preliminaries} \label{part:theory_prelim}

\section{Neural network terminology and architectural assumptions} \label{sec:notation-and-arch-assumptions}

\subsection{Basic neural network terminology}

Throughout this work we will assume that the reader is already familiar with convolutional neural networks \citep{fukushima1982neocognitron, lecun1998gradient} for which many overviews and tutorials are available \citep[e.g.][Chapter 9]{goodfellow2016deep}. The purpose of this subsection won't be to define convolutional network concepts from scratch, but rather to lay out the specific terminology we will use when referring to them.

In this work will consider neural networks consisting of \tmtextbf{\tmtextit{affine layers}} of the standard fully-connected and convolutional types, and \tmtextbf{\tmtextit{nonlinear layers}} that compute element-wise \tmtextit{\tmtextbf{activation functions}} (that are typically nonlinear). We define a \tmtextbf{\tmtextit{combined layer}} to be an affine layer, immediately followed by a nonlinear layer. (Note that a combined layer is what was traditionally referred to as a ``layer'' in the neural network literature, before the modern trend of referring to the individual affine and nonlinear parts as their own separate ``layers''.)

The input and output of convolutional layers (or networks) are called \tmtextit{\tmtextbf{feature maps}}, and consist of an array of \tmtextit{\tmtextbf{locations vectors}}, with the entries of these vectors being called \tmtextit{\tmtextbf{channels}}. The parameters of an affine layer are its \tmtextbf{\tmtextit{weights}} (sometimes called a \tmtextit{\tmtextbf{filter bank}} in the convolutional case), and its \tmtextit{\tmtextbf{bias vector}}. So for example, in the fully-connected case, a combined layer would compute $\phi (Wz + b)$, where $z$ is its input vector, $W$ its matrix of weights, $b$ its bias vector, and $\phi$ its activation function (which is defined from $\mathbb{R}$ to $\mathbb{R}$, and applied element-wise for higher dimensional inputs). 

In this work, the discussion will center around a single neural network which we will refer to simply as \tmtextit{\tmtextbf{the network}} or sometimes \tmtextit{\tmtextbf{the entire network}}. We will define a \tmtextit{\tmtextbf{subnetwork}} as a neural network formed from a subset of the entire network's layers which preserves all dependency relationships and has a {\tmem{well-defined and singular input and output}} (unlike the entire network, which can have multiple inputs and outputs in general). Subnetworks can be thought of as performing part of the computation of the network. So for example, if the network consists of a sequence of five layers, then layers 2, 3 and 4 form a subnetwork whose input is the input to layer 2, and whose output is the output of layer 4. But layers 2, 4, and 5 do not form a subnetwork since the dependency of layer 4 on layer 2 is not preserved.

\subsection{Initial architectural assumptions}\label{sec:arch-assumptions}

We observe that a fully-connected layer is equivalent to an convolutional layer with a 1x1 feature map and 1x1 filter size, where the input/output data dimensions are just the input/output channel dimensions. Thus, going forward, we will restrict our analysis to the convolution case, which implicitly handles the fully-connected case via this reduction.

We will also assume, for now, that the network can be entirely built out of three components: combined layers (as define above), non-zero constant scalar multiplications operations applied to individual feature maps, and \tmtextbf{\tmtextit{concatenation operations}}, which concatenate two feature maps of compatible sizes along their channel dimensions. We will permit a given feature map to act as the input to multiple operations/layers in the network, thus allowing ``branching structures'' and multiple ``output heads''.

The restriction to combined layers isn't as severe as it might seem, as an isolated affine layer is equivalent to a combined layer with an identity activation function. And while sum operations are not \emph{explicitly} included among the allowed operations, under certain conditions they can be simulated via a simple construction whose details we will defer to Section \ref{sec:sum-ops}. This means that our analysis can apply to networks containing actual sum operations, under said conditions. Two or more consecutive nonlinear layers are also not allowed by our assumptions, however one can simply fuse two such layers into a single one by composing their activation functions.

For now we will assume that the network does not contain any pooling layers. We will (partially) relax this assumption later in Section \ref{sec:pooling-layers}.

\section{Parameter distributions}\label{sec:param_dist}

\subsection{Assumptions on the form of the parameter distributions}

In order to obtain a sufficiently simple characterization of the function computed by a neural network at initialization time, we will make certain assumptions about the distribution of its parameters at initialization.

Our first one will be that the bias vector is initialized to zero. While not strictly necessary to the derivation and viability of DKS, this assumption will simplify our presentation. Our second will be that if the input of one layer depends on the output of another, either directly or indirectly, then the parameters of these layers must be initialized independently from each other. This rules out recurrent neural networks, for example, since parameters are shared across time-steps.

Finally, except where stated otherwise, we will assume the use of a ``Delta initialization" \citep{balduzzi2017shattered, xiao2018dynamical}, which requires that filter bank tensors are initialized to zero everywhere except for their central location/offset (and have odd-sized filter dimensions to make this possible). As an example, if we have a $5 \times 5$ filter, then only the weights corresponding to entry $(3, 3)$ would be non-zero. Note that for fully-connected layers there is only one location, so that a Delta initialization becomes equivalent to a standard one.

The non-zero weights of a Delta-initialized filter bank form a $m \times k$ matrix
, where $k$ is the input channel dimension and $m$ is the output channel dimension. To initialize this matrix we have two options. First, we can use an entry-wise iid Gaussian distribution with mean 0 and variance $1 / k$, which gives rise to the \tmtextit{\tmtextbf{Gaussian Delta initialization}}. While it might seem restrictive to assume a variance of $1 / k$ (instead of $\sigma^2 / k$ for general $\sigma > 0$), this will simplify our presentation going forward, and other choices can be simulated by rescaling the network's activation functions (which will be part of DKS).

The second option is to use a \tmtextbf{\tmtextit{scaled-corrected uniform orthogonal (SUO) distribution}}, which is a special distribution of rescaled orthogonal matrices. When $m \leqslant k$, samples from this distribution can be generated as $(XX^{\top})^{- 1 / 2} X$, where $X$ is a $m \times k$ matrix with entries sampled iid from $N (0, 1)$. When $m > k$, we may apply the same procedure but with $k$ and $m$ reversed, and then transpose the result. The resulting distribution is given by the well-known \tmtextit{\tmtextbf{Haar measure}} on orthogonal matrices \citep[e.g.][]{meckes2019random}, and is also sometimes called the uniform distribution. To be consistent with the scaling characteristics of the Gaussian initialization, we further multiply by the scaling factor $\max \left( \sqrt{m / k}, 1 \right)$, which will have an effect only when $m > k$. We will call Delta initializations that use the SUO distribution \tmtextbf{\tmtextit{Orthogonal Delta initializations}}.

\james{TODO: Should add an intuitive explanation of where scaling factor comes from}

\subsection{A brief discussion about random orthogonal matrices and the SUO distribution}\label{sec:random-ortho-discussion}

\james{Could also put this one in the appendix too?}

The scaled-corrected uniform orthogonal distribution, as we have defined it, has the property that it is invariant to pre or post-multiplication of the matrix by a constant square orthonormal matrix \citep[][Chapter 7]{eaton1989group}. This implies that left-multiplying an input vector by an unobserved matrix sampled from this distribution erases all information about the vector's direction. The input vector's dimension-normalized squared norm (i.e.~$\frac{1}{\dim (x)} \| x \|^2$) can meanwhile be exactly recovered when $k \leq m$, and is equal to the output vector's dimension-normalized squared norm.

For the computations in the next section to be valid for a given orthogonal weight distribution, we require that the distribution satisfies these properties. However, many randomized procedures used in practice for sampling orthogonal matrices lack the directional invariance property. And even procedures whose distributions do possess it often don't include the $\max \left( \sqrt{m / k}, 1 \right)$ scale correction factor, which is required for the dimension-normalized squared norm to be preserved. Thus, we strongly recommend that anyone implementing DKS use the sampling procedure for orthogonal matrices that we have outlined, unless they are confident that their own procedure gives precisely the same distribution. Note that \citet{saxe2014exact} and \citet{xiao2018dynamical} have used distributions over orthogonal matrices to initialize neural networks. It turns out that the formulas they derive also require SUO-distributed weights to be correct, even though they did not state this explicitly. 

Finally, note that the entry-wise iid $\mathcal{N} (0, 1 / k)$ distribution for $m \times k$ matrices behaves very similarly to the SUO distribution with respect to multiplication by an input vector, and gives a distribution on the output vector which is identical up to a multiplication by a random scalar (which is distributed according to the chi distribution with $m$ degrees of freedom). The output vector's dimension-normalized squared norm is thus a random multiplicative perturbation of the input vector's (instead of being equal to it), where the perturbation's mean and variance are $1$ and $2 / m$ respectively. 
From these observations we can see that Gaussian initializations, like SUO ones, give rise to directional invariance, but only approximately preserve the dimension-normalized squared norm (and in a way that gets more precise as $m$ grows).

\section{Kernel function approximations for neural networks}\label{sec:kernel-approx-main}

The starting point for our analysis of the initialization-time behavior of neural networks will be kernel functions, and the approximations of these that hold at initialization-time when the channel dimensions are large. This type of analysis was originally pioneered by \citet{neal1996bayesian}, and developed further in various subsequent works \citep[e.g.][]{williams1997computing, rahimi2008weighted, cho2012kernel, mairal2014convolutional, anselmi2015deep, hazan2015steps, daniely2016toward, matthews2018gaussian, lee2018deep, garriga2018deep, novak2018bayesian, arora2019exact}. In this section we will review these concepts and establish our notation and terminology for the key quantities. We will depart from the index-heavy tensor notation of some previous works \citep[such as][]{novak2018bayesian} in favor of a more compact one based on matrices.

\subsection{Simplified version for the fully-connected case}
Before we launch into our full treatment of kernel function approximations for convolutional neural networks, in this subsection we will quickly give a simplified version for the fully-connected case, with the goal of building intuition. Note that the notation defined here is only a special case of the more general notation we will develop in subsequent subsections.

For a vector-valued function $f: \RR^k \to \RR^m$, we define its kernel function $\kappa_f$ by
\begin{equation*}
\kappa_f(z,z') = \frac{1}{m} f(z)^\top f(z') .
\end{equation*}
It turns out \citep[e.g.][]{daniely2016toward} that when $f$ is a sufficiently wide fully-connected combined layer with iid $\mathcal{N}(0, 1/k)$ weights and activation function $\phi$, $\kappa_f(z,z')$ is closely approximated with high probability by $\widetilde{\kappa_f} (\Sigma_{z,z'})$, where
\begin{align}
 \label{eqn:fully-connected-APKF}
  \widetilde{\kappa_f} (\Sigma_{z,z'}) = \mathbb{E}_{\tmscript{\left[\begin{array}{c}
    u_1\\
    u_2
  \end{array}\right] \sim \mathcal{N} (0, \Sigma_{z,z'})}} [\phi (u_1) \phi (u_2) ],
\end{align}
where
\begin{align*}
\Sigma_{z,z'} = \frac{1}{k} \left[\begin{array}{cc}
     \|z\|^2 & z^{\top} z'\\
     z^{\top} z' & \|z'\|^2
   \end{array}\right] \in \RR^{2\times 2}.
\end{align*}
This can be derived by observing that any two units in $f$'s nonlinear layer are Gaussian distributed (when conditioned on $z$ and $z'$) with mean zero and covariance matrix $\Sigma_{z,z'}$. And so if we consider enough of these units, their average statistics (given $\kappa_f(z,z')$) converge in probability to the expectation.

Using the notable fact that $\widetilde{\kappa_f}(z, z')$ only depends on $\Sigma_{z,z'}$ (and not the full details of $z$ and $z'$), we can then compose these layer-wise kernel approximations to form ones for networks consisting of many such layers.

\subsection{Notation for feature maps and subnetworks}

Throughout this work we will represent feature maps as matrices in $\mathbb{R}^{\text{\#channels} \: \times \: \text{\#locations}}$, where \#channels is the number of channels in the feature map and \#locations is the number of locations. Note that for fully-connected layers these matrices are just column vectors.

We will represent subnetworks (of which single layers are a special case) by symbols such as ``$f$'' or ``$g$''. Implicit in these representations is a dependence on all the structural details of the subnetwork, including its parameters, its activation functions, and anything else we need in order to construct our various approximations. At the same time, we will use standard functional notation such as $f (Z)$ when we want to treat $f$ as a function from its input to its output. 

\subsection{Inner product matrices (IPMs) and Pair-location kernel functions (PKFs)}\label{sec:PKF}

Suppose $X, Y \in \mathbb{R}^{k \times \ell}$ are feature maps with channel dimension $k$ and number of locations $\ell$. We will define the \tmtextit{\tmtextbf{inner product matrix (or IPM)}} of $X$ and $Y$, denoted as $\Sigma_{X, Y}$, by
\[ \Sigma_{X, Y} \equiv \frac{1}{k}  \left[\begin{array}{cc}
     X^{\top} X & X^{\top} Y\\
     Y^{\top} X & Y^{\top} Y
   \end{array}\right] \in \mathbb{R}^{2 \ell \times 2 \ell} . \]
The entries of an IPM are the (dimension-normalized) inner products between all pairs of column vectors from $X$ and $Y$, or in other words, the average (across channels) of the entry-wise products between pairs of location vectors from the feature maps $X$ and $Y$. 

\james{Maybe switch to calling these Gram matrices? The downside is that Gram matrix is somewhat too general.}

Now suppose $f$ is a subnetwork whose output feature map is in $\mathbb{R}^{m \times \ell}$. We define the \tmtextbf{\tmtextit{paired-location kernel function (or PKF)}} of $f$, denoted by $\kappa_f$, as
\[ \kappa_f (Z, Z') \equiv \Sigma_{f (Z), f (Z')} = \frac{1}{m}  \left[\begin{array}{cc}
     f (Z)^{\top} f (Z) & f (Z)^{\top} f (Z')\\
     f (Z')^{\top} f (Z) & f (Z')^{\top} f (Z')
   \end{array}\right] . \]
If $f$ is a fully-connected combined layer, then $\kappa_f (Z, Z')$ is just a 2x2 matrix, while for general convolutional combined layers it has a $2 \times 2$ block structure, with blocks of size $\text{} \ell \times \ell$. 
Note that PKFs are analogous to \poscite{novak2018bayesian} ``activation covariance matrices''.

$f$'s PKF $\kappa_f$ gives us a ``geometric view" of $f$'s input-output behavior. In particular, because $\kappa_f$ determines the inner-products between all pairs of output vectors (across the different locations and both inputs), it determines the distances between all such vectors via the formula $\|x - y' \|= \sqrt{x^{\top} x + y^{\top} y' - 2 x^{\top} y'}$.


\subsection{Initialization-time approximations to the PKF for combined layers}

In this subsection we will assume that $f$ is a combined layer with element-wise activation function $\phi$. We will also assume Gaussian-distributed weights, as part of either a Delta or non-Delta initialization scheme. (An extension to SUO-distributed weights given in Subsection \ref{sec:ortho-weights-APKF}).

We are interested in extracting a simple mathematical approximation of $\kappa_f$ that is valid at initialization time, which we can use in order to construct approximations of the PKF of larger subnetworks. To begin with, we will assume that the convolutional part of $f$ uses padding and has a stride of 1, which means that input and output locations will be in one to one correspondence. (This assumption will be relaxed in the next subsection.)

In general, computing $\kappa_f$ for combined layers $f$ boils down to direct evaluation of the defining formula, with no simplifications possible. But when $f$'s initial parameters are distributed as per Section \ref{sec:param_dist}, there exists a much simpler function $\widetilde{\kappa_f}$ that approximates $\kappa_f$ at initialization time with high probability, which we call the \tmtextit{\tmtextbf{approximate paired-location kernel function (or APKF)}} of $f$. $\widetilde{\kappa_f}$ is obtained from $\kappa_f$ by taking the limit as the output channel dimension go to infinity, and is a good approximation when the actual (finite) output channel dimension is sufficiently large.

As shown by \citet{garriga2018deep} and \citet{novak2018bayesian}, the APKF for convolutional combined layers initialized with a standard \tmtextit{\tmtextbf{Gaussian fan-in initialization}}\footnote{This initialization uses an entry-wise iid Gaussian distribution with mean 0 and variance $1 / d$, where $d$ is the filter size times the input channel dimension.} \citep{lecun1998efficient} is given by
\begin{equation}
  \widetilde{\kappa_f} (\Sigma_{Z, Z'}) =\mathbb{E}_{u \sim \mathcal{N} \left( 0, \: \mathcal{A} (\Sigma_{Z, Z'}) \right)} [\phi (u) \phi (u)^{\top}], \label{eqn:fanin-APKF}
\end{equation}
where $\mathcal{A}$ is the operator which maps $\Sigma_{Z, Z'}$ to $\Sigma_{P (Z), P (Z')}$, with $P (Z)$ denoting the matrix of patch vectors\footnote{A ``patch vector" is one formed by concatenating together the subset of columns of $Z$ corresponding to a particular location visited by the convolutional filter. They have dimension $k b^2$ for $b \times b$ convolutions.} generated from $Z$. A key property of $\widetilde{\kappa_f}$ is that it only depends on $Z$ and $Z'$ via the associated IPM $\Sigma_{Z, Z'}$. 

As discussed in Section \ref{sec:param_dist}, we are assuming the use of a Delta initialization scheme in this work. Intuitively, a Delta initialization makes a convolutional layer behave like a set of fully-connected layers that operate independently over locations in the feature map (and share parameters). This results in a simplified form for $\widetilde{\kappa_f}$ which is a directly analogous to the kernel approximation for fully-connected combined layers (i.e.~Equation \ref{eqn:fully-connected-APKF}). It is given by\footnote{This formula can be obtained from Equation \ref{eqn:fanin-APKF} by observing that a Delta-initialized filter bank behaves like a 1x1 filter, and that $\mathcal{A}$ is the identity operator in the case of a 1x1 filter (since $P (Z) = Z$).}
\[ \widetilde{\kappa_f} (\Sigma_{Z, Z'}) =\mathbb{E}_{u \sim \mathcal{N} (0, \Sigma_{Z, Z'})} [\phi (u) \phi (u)^{\top}] . \]


A minor technical point is that $\Sigma_{Z, Z'}$ may be singular, in which case $\mathcal{N} (0, \Sigma_{Z, Z'})$ will be ``degenerate'', and its density function technically undefined. The easiest way this can happen is if $Z = Z'$. However, one can still meaningfully define a distribution and sample from it using $(\Sigma_{Z, Z'})^{1 / 2} v$ for $v \sim \mathcal{N} (0, I)$, which is equivalent to adding $\epsilon I$ to $\Sigma_{Z, Z'}$ and then letting $\epsilon \rightarrow 0$. With this extended definition of $\mathcal{N} (0, \Sigma_{Z, Z'})$ our formulas remain valid.

\subsection{Padding, strides, and dropped locations}

If the stride of $f$'s convolution is not 1, or if it doesn't use padding and has a filter size larger than $1 \times 1$, then the locations in the input and output feature maps won't be in one to one correspondence. Instead, they will be related to each other via a projection function $s (u)$, which maps input locations (given by the entries of $u$) to their corresponding output locations (given by the entries of $s (u)$). This results in the following generalized formula for $\kappa_f$:
\begin{equation}
  \widetilde{\kappa_f} (\Sigma_{Z, Z'}) =\mathbb{E}_{u \sim \mathcal{N} (0, \Sigma_{Z, Z'})} [\phi (s (u)) \phi (s (u))^{\top}] . \label{eqn:APKF}
\end{equation}
When the input and output locations are in one to one correspondence, $s$ is just the identity function. Otherwise, $s$ essentially ``drops'' the input locations that are never visited by the center of the filter (i.e.~$s(u)$ will be independent of the entries of $u$ that are ``dropped"). We will refer to these as \tmtextit{\tmtextbf{dropped locations}}, and most of our discussions going forward will assume that the location under consideration has \emph{not} been dropped at the layer in question. When a location is dropped at some layer, both the exact and approximate PKFs of that layer (and all subsequent layers) will be effectively zero for that location.

\subsection{Deriving APKFs given Gaussian distributed weights}\label{sec:deriving-kernel-approx}

At a high level, the APKF formulas given above for a combined layer $f$ can be derived by observing that each pair of outputs from the affine part of $f$ are linear combinations of Gaussian random variables (i.e.~the filter weights) when conditioned on the two inputs $Z$ and $Z'$, and are thus are jointly Gaussian distributed with mean zero. A straightforward computation then shows that the covariance matrix $C$ of this distribution is $\Sigma_{Z, Z'} \otimes I_{m \times m}$ or $\mathcal{A} (\Sigma_{Z, Z'}) \otimes I_{m \times m}$, where $\otimes$ denotes the Kronecker product. Because units in different channels have zero covariance they are independent, and so $\kappa_f (Z, Z')$ is equal to an average over output channels of iid random variables, and thus {\tmem{converges in probability}} to its expectation as the number of output channels goes to infinity. We set $\widetilde{\kappa_f} (\Sigma_{Z, Z'})$ equal to this expectation, whose formula then follows from the one for $C$. Probabilistic bounds on the approximation error can then be obtained using concentration inequalities.


\subsection{The APKF Condition and network-level PKF approximations}\label{sec:network-level-PKF}

The main approximation which we will use going forward is that the PKF of each combined layer is equal to its associated APKF at initialization time. Or in other words, that
\[ \Sigma_{f (Z), f (Z')} = \kappa_f (Z, Z') \approx \widetilde{\kappa_f} (\Sigma_{Z, Z'}) \]
for each combined layer $f$ of the network. We will refer to this as the \tmtextit{\tmtextbf{APKF Condition}}.

Observe that a combined layer's APKF depends on $Z$ and $Z'$ only through the associated IPM $\Sigma_{Z, Z'}$. Thus, under the APKF Condition, we can compose APKFs for each combined layer to form an initialization-time approximations of the PKFs for arbitrary subnetworks, which we will call \tmtextbf{network-level PKF approximations}. Extending our notation from the combined layer case, we will denote these approximations by $\widetilde{\kappa_f}$ for arbitrary subnetworks $f$. (Note that we rely on the property that subnetworks have a single input and output feature maps for this definition and notation to make sense.)

An additional complication that we must deal with when constructing network-level PKF approximations is the presence of concatenation operations, where $Z$ is the concatenation of two feature maps $X$ and $Y$ along their channel dimensions. In this cases, we observe that
\begin{equation}
\label{eqn:concat-IPM}
\Sigma_{Z, Z'} = \frac{k_1 \Sigma_{X, X'} + k_2 \Sigma_{Y, Y'}}{k_1 + k_1} ,
\end{equation}
where $k_1$ and $k_2$ are the number of channels in $X$ and $Y$ respectively.

As we will see in the following sections, network-level PKF approximations are amenable to detailed analysis, and expose several key properties which end up being crucial determinants of network trainability (and which can be controlled through careful interventions).

\subsection{How accurate are these approximations?}\label{sec:how-accurate-approx}

As discussed above, the APKF for a combined layer $f$ is derived by observing that the entries of $\kappa_f (Z, Z')$ are empirical averages of iid variables that converge in probability to their expectations as the output channel dimension goes to infinity. Applying concentration inequalities then leads to statements of the form: ``for any $\epsilon > 0$ and $\delta > 0$ there exists an integer $m_0 (\epsilon, \delta)$ such that if the output channel dimension satisfies $m \geqslant m_0 (\epsilon, \delta)$ then $\| \kappa_f (Z, Z') - \widetilde{\kappa_f} (\Sigma_{Z, Z'}) \| < \epsilon$ with probability $1 - \delta$.''. The precise dependency of $m_0 (\epsilon, \delta)$ and on $\epsilon$ and $\delta$ is of practical interest, as the output channel dimension of real neural network layers is finite, and may not even be particularly large in some cases.

Ultimately, we are interested in bounding the kernel approximation error not just for single combined layers but for entire networks. In general, such bounds will be worse than anything provable for single layers, as approximation error will compound with depth (since the output of one approximation is fed as input into the next). The only work we are aware of that gives such bounds is that of \citet{daniely2016toward}. In that work, the authors analyze what are essentially networks of fully-connected combined layers arranged in arbitrary topologies, with certain technical conditions imposed on their input data and activation functions. Translating their main result into the language and assumptions of this work yields the following theorem:

\begin{theorem}[Adapted from Theorem 2 of \citet{daniely2016toward}]
  \label{thm:error-bound-daniely}Suppose that $f$ is a network containing only fully-connected combined layers and concatenation operations, the former of which are initialized independently of each other with a standard Gaussian fan-in initialization, and use the same activation function $\phi$. Suppose further that $\phi$ is twice continuously differentiable and satisfies \ $\mathbb{E}_{x \sim \mathcal{N} (0, 1)} [\phi (x)^2] = 1$ and $\| \phi \|_{\infty}, \| \phi' \|_{\infty}, \| \phi'' \|_{\infty} \leqslant C$ for some $C$ (with $\| \cdot \|_{\infty}$ denoting the supremal value), and that each layer has output dimension (aka{\hspace{1em}}``width'') greater than or equal to
  \[ \frac{(4 C^4)^D \log (8 L / \delta)}{\epsilon^2}, \]
  where $D$ is maximum number of nonlinear layers in any input-output path through the network (i.e.~its ``depth''), $L$ is its number of combined layers, and $\delta, \epsilon > 0$. Then at initialization time, for all input vectors $z$ and $z'$ to $f$ satisfying $\| z \|^2 = \| z' \|^2 = \dim (z)$, we have that
  \[ | [\kappa_f (z, z')]_{1, 2} - [\widetilde{\kappa_f} (\Sigma_{z, z'})]_{1, 2} | \leqslant \epsilon \]
  with probability at least $1 - \delta$.
\end{theorem}

\begin{remark}
  Note that in our notation, both $\kappa_f (z, z')$ and $\widetilde{\kappa_f} (\Sigma_{z, z'})$ are $2 \times 2$ matrices, and $[\cdot]_{1, 2}$ extracts the $(1, 2)$-th entry, or in other words, the value of $\frac{1}{\dim (f (z))} f (z)^{\top} f (z')$ and its approximation. One can estimate the error for diagonal entries simply by setting $z = z'$. 
\end{remark}

\begin{remark}
  Because $1 = \sqrt{\mathbb{E}_{u \sim \mathcal{N} (0, 1)} [\phi (u)^2]} \leqslant \sqrt{\mathbb{E}_{u \sim \mathcal{N} (0, 1)} [\| \phi \|_{\infty}^2]} = \| \phi \|_{\infty} \:$, it thus follows that $C \geqslant 1$ in the above theorem. And while the theorem assumes the use of the Gaussian fan-in initialization, we note that for fully-connected networks this is equivalent to the Gaussian Delta initialization.
\end{remark}

\begin{remark}
  This theorem statement differs from the one in \citet{daniely2016toward} by explicitly assuming that the activation function $\phi$ satisfies $\mathbb{E}_{x \sim \mathcal{N} (0, 1)} [\phi (x)^2] = 1$, or is in other words ``normalized". As far as we can tell, this assumption is implicit in the definitions made by \citet{daniely2016toward}.
\end{remark}

\begin{remark}
  The condition that $\mathbb{E}_{x \sim \mathcal{N} (0, 1)} [\phi (x)^2] = 1$ can be achieved by normalizing the output of the activation functions by an appropriate constant. And the condition that $\| z \|^2 = \| z' \|^2 = \dim (z)$ can be achieved through data pre-processing (as discussed in Section \ref{sec:PLN}). Both of these conditions will be enforced as part of DKS (although motivated differently).
\end{remark}

The bound in Theorem \ref{thm:error-bound-daniely} predicts an exponential dependence of the minimum required width and depth $D$, and a $1 / \epsilon^2$ dependence on the error tolerance $\epsilon$. The exponential dependence on $D$ means that this bound could never realistically be applied to a moderately deep network running on actual hardware, as the required width would be prohibitive. While it could easily be the case that {\tmem{some}} choices of $\phi$ give an exponential dependence as the bound predicts, we conjecture that with more carefully designed assumptions on the properties of $\phi$, a bound with better dependence could be proven. Indeed, \citet{daniely2016toward} themselves give a more specialized bound for networks with rescaled RELU activations (which technically violate the hypotheses of Theorem \ref{thm:error-bound-daniely} since they are unbounded and not differentiable everywhere), where the required width is only quadratic in $D$.

The main limitation of Theorem \ref{thm:error-bound-daniely} is that it applies only to networks of fully-connected combined layers that don't share weights. We conjecture that a similar result may also hold for networks with convolutional layers and a restricted type of inter-layer weight sharing. 

\subsection{The orthogonal initialization case (assuming SUO-distributed weights)}\label{sec:ortho-weights-APKF}

The kernel formulas and theory given so far in this section have all assumed the use of Gaussian Delta initializations. However, our assumptions also permit the use of Orthogonal Delta initializations, which as discussed in Section \ref{sec:param_dist}, use the SUO distribution instead of an iid Gaussian one to initialize the non-zero weights of the filter. While some previous works \citep[e.g.][]{xiao2018dynamical} have used these kinds of kernel approximation formulas in the orthogonal case, and have appealed to the vague notion that random orthogonal matrices ``look like'' Gaussian-distributed ones in high dimensions, there hasn't been any mathematically rigorous justification of this practice until the recent work of \citet{martens2021validity}.

The following theorem, which is adapted from \citet{martens2021validity}, establishes convergence in probability of the APKF to the associated PKF for a fully-connected combined layer with SUO-distributed weight matrix. Like Theorem \ref{thm:error-bound-daniely}, it provides an explicit and fairly reasonable convergence rate. An extension of this result to multi-layer networks would likely proceed along similar lines to the argument given in \citet{daniely2016toward} for the Gaussian case.

\begin{theorem}[Adapted from Theorem 2 of \citet{martens2021validity}]
  \label{thm:SUO-kernel-approx}Let $f$ be a fully-connected combined layer with an SUO distributed $m \times k$ weight matrix $W$, a bias vector equal to $0$, and an activation function $\phi$ satisfying $\| \phi \|_{\infty}, \| \phi' \|_{\infty} \leqslant C$ for some $C$ (with $\| \cdot \|_{\infty}$ denoting the supremal value). Denote $n = \max (k, m)$, and suppose that for $\delta, \epsilon \geqslant 0$ we have
  \[ \frac{m^{5 / 2}}{(n + 1)^2} \geqslant \log (2 / \delta) \text{\qquad and\qquad} \frac{n - 1}{m^{3 / 4}} \geqslant \frac{8 \sqrt{2} C^2}{\epsilon} . \]

  Then, at initialization time, for all pairs of vectors $z, z' \in \mathbb{R}^k$ satisfying $\| z \|^2 = \| z' \|^2 = k$, we have that
  \[ | [\kappa_f (z, z')]_{1, 2} - [\widetilde{\kappa_f} (\Sigma_{z, z'})]_{1, 2} | \leqslant \epsilon \]
  with probability at least $1 - \delta$.
\end{theorem}

\begin{remark}
  The conditions on $k$, $m$, and $n \equiv \max (k, m)$ in the theorem statement will be satisfied as long as $n$ is sufficiently large and $k$ is not too much larger than $m$. In the case where $m \geqslant k$, the LHS's of these bounds simplifies to approximately $m^{1 / 2}$ and $m^{1 / 4}$, respectively. It thus follows that the APKF converges in probability to the PKF as the output dimension $m$ goes to $\infty$.
\end{remark}

\begin{remark}
  In the case where $m \geqslant k$, the conditions imply that
  \[ m \gtrsim \frac{128 C^4 \log (2 / \delta)^2}{\epsilon^2}, \]
  which is similar to the width bound from Theorem \ref{thm:error-bound-daniely} for $D = 1$.
\end{remark}

\begin{remark}
  Note that while this theorem is stated only for fully-connected combined layers, it also applies to convolutional combined layers that use Orthogonal Delta initializations by taking $z$ and $z'$ to be any pair of vectors from the union of the columns $Z$ and $Z'$.
\end{remark}

\section{Q and C maps for combined layers}\label{sec:QC_map_combined}

Q maps and C maps are mathematical constructs introduced by \citet{saxe2014exact} and \citet{poole2016exponential} that describe the initialization time behavior of deep fully-connected networks. While original derived within the semi-rigorous ``signal propagation'' framework (which is discussed in Section \ref{sec:signal-prop}), they can also be applied under certain conditions within the more rigorous context of kernel function approximations. In that context, they provide a compact alternative representation of approximate kernel functions that is easier to work with. 

As will be discussed later in Part \ref{part:desirable_QCmap_properties}, the Q/C maps of a network tell us a lot about its trainability. Indeed, they have appeared either implicitly or explicitly, often in simplified forms, in much of the previous work on network design and initialization (as will be made clear in Sections \ref{sec:var/sig-prop-relation-to-kernel} and \ref{sec:review-and-analysis-of-related}). They are also central to the derivation of DKS, and over the next few sections we will develop the generalized version of them that we will use in this work.

In this section we will formally introduce Q/C maps maps and their associated notation, and give formulas to compute them for combined layers under our stated hypotheses. Note that while the connection between Q/C maps and approximate kernel functions has been previously observed \citep[e.g.][]{lee2018deep}, it hasn't before been carefully worked out, nor has it been generalized to convolutional layers (as we will do here). In the section that follows we will show how Q/C maps can be naturally extended beyond single combined layers to describe the behavior of network-level PKF approximations for arbitrary subnetworks with complex topologies.

\subsection{Q maps for combined layers}

Consider a combined layer $f$ with $\phi$ as its element-wise activation function, and $Z$ and $Z'$ as its two inputs. By Equation \ref{eqn:APKF} and basic properties of Gaussian expectations, any given diagonal entry $q_{\tmop{out}}$ of $\widetilde{\kappa_f} (\Sigma_{Z, Z'})$ depends only on the corresponding diagonal entry $q_{\tmop{in}}$ of $\Sigma_{Z, Z'}$, and can be computed as
\begin{equation}
  q_{\tmop{out}} = Q_f (q_{\tmop{in}}) =\mathbb{E}_{u \sim \mathcal{N} (0, q_{\tmop{in}})} [\phi (u)^2] \: = \: \mathbb{E}_{x \sim \mathcal{N} (0, 1)} \left[ \phi \left( \sqrt{q_{\tmop{in}}} x \right)^2 \right], \label{eqn:Q-map}
\end{equation}
where $Q_f$ is defined as the \tmtextit{\tmtextbf{Q map}} of $f$. We will call such diagonal entries \tmtextit{\tmtextbf{q values}}, and note that they are equal to the dimension-normalized squared norms of their associated location vectors under the APKF Condition. Notably, the form of the Q map is the same for each location, and so we may associate them with combined layers in a location-independent way.

\subsection{C maps for combined layers}

An off-diagonal entry $m_{\tmop{out}}$ of $\widetilde{\kappa_f} (\Sigma_{Z, Z'})$ has a slightly more complex dependence on $\Sigma_{Z, Z'}$ in Equation \ref{eqn:APKF}, as it depends on both the corresponding entry $m_{\tmop{in}}$ of $\Sigma_{Z, Z'}$, as well as the two associated diagonal entries ($q_1$ and $q_2$) that share a row or column. It is given by
\begin{equation}
  m_{\tmop{out}} =\mathbb{E}_{\tmscript{\left[\begin{array}{c}
    u_1\\
    u_2
  \end{array}\right] \sim \mathcal{N} \left( 0, \left[\begin{array}{cc}
    q_1 & m_{\tmop{in}}\\
    m_{\tmop{in}} & q_2
  \end{array}\right] \right)}} [\phi (u_1) \phi (u_2)] . \label{eqn:M-map}
\end{equation}
We call such off-diagonal entries \tmtextit{\tmtextbf{m values}}, and note that they are equal to the dimension-normalized inner product of their two associated location vectors under the APKF Condition.

Following \citet{poole2016exponential}, we focus on ``length-normalized'' versions of the m values called \tmtextit{c values}. A c value can be obtained from an m value by dividing it by the square root of the product of its two associated q values. (e.g.~$c_{\tmop{in}} = m_{\tmop{in}} / \sqrt{q_1 q_2}$ in the context of Equation \ref{eqn:M-map}.) Under the APKF Condition, c values 
are equal to the cosine similarity between their two associated location vectors.

c values are computed using \tmtextit{\tmtextbf{C maps}}, which for a combined layer $f$ are given by
\begin{eqnarray}
  c_{\tmop{out}} & = & C_f (c_{\tmop{in}}, q_1, q_2) \nonumber\\
  & \equiv & \frac{1}{\sqrt{Q_f (q_1) Q_f (q_2)}} \mathbb{E}_{\tmscript{\left[\begin{array}{c}
    u_1\\
    u_2
  \end{array}\right] \sim \mathcal{N} \left( 0, \left[\begin{array}{cc}
    q_1 & m_{\tmop{in}}\\
    m_{\tmop{in}} & q_2
  \end{array}\right] \right)}} [\phi (u_1) \phi (u_2)] \nonumber\\
  & = & \frac{1}{\sqrt{Q_f (q_1) Q_f (q_2)}} \mathbb{E}_{\tmscript{\left[\begin{array}{c}
    v_1\\
    v_2
  \end{array}\right] \sim \mathcal{N} \left( 0, \left[\begin{array}{cc}
    1 & c_{\tmop{in}}\\
    c_{\tmop{in}} & 1
  \end{array}\right] \right)}} \left[ \phi \left( \sqrt{q_1} v_1 \right) \phi \left( \sqrt{q_2} v_2 \right) \right] \nonumber\\
  & = & \frac{1}{\sqrt{Q_f (q_1) Q_f (q_2)}} \mathbb{E}_{x, y \sim \mathcal{N} (0, 1)} \left[ \phi \left( \sqrt{q_1} x \right) \phi \left( \sqrt{q_2}  \left( c_{\tmop{in}} x + \sqrt{1 - c_{\tmop{in}}^2} y \right) \right) \right],  \label{eqn:C-map}
\end{eqnarray}
where we have used the fact that $\sqrt{q_1} x$ and $\sqrt{q_2}  \left( c_{\tmop{in}} x + \sqrt{1 - c_{\tmop{in}}^2} y \right)$ are mean-zero Gaussian distributed with covariance matrix \ $\left[\begin{array}{cc}
  q_1 & \sqrt{q_1 q_2} c_{\tmop{in}}\\
  \sqrt{q_1 q_2} c_{\tmop{in}} & q_2
\end{array}\right] = \left[\begin{array}{cc}
  q_1 & m_{\tmop{in}}\\
  m_{\tmop{in}} & q_2
\end{array}\right]$. Like the Q map, the C map is the same for each location, and so we may associate a single C map to each combined layer.

We note that $C_f$, when considered as a function of $c$, maps from $[- 1, 1]$ to $[- 1, 1]$. This is immediate from the interpretation of c values as cosine similarities if we assume the APKF Condition, but is true more generally. Intuitively, it must be the case, since the APKF Condition becomes exact in the limit as the channel dimension grows, and thus c values are precisely equal to cosine similarities in infinite-dimensional spaces. To be more rigorous, one may apply H{\"o}lder's inequality within the Hilbert space of functions defined by the inner product $\langle g, h \rangle =\mathbb{E}_{x, y \sim \mathcal{N} (0, 1)} [g (x, y) h (x, y)]$, taking $g (x, y) = \phi \left( \sqrt{q_1} x \right)$ and $h (x, y) = \phi \left( \sqrt{q_2}  \left( c_{\tmop{in}} x + \sqrt{1 - c_{\tmop{in}}^2} y \right) \right)$.

\subsection{Q/C maps for more general combined layers?}

Note that the existence of Q/C maps, as we have defined them, depends on our stated hypotheses for combined layers. In particular, that they are convolutional (or fully-connected), and use a Delta initialization scheme. While APKFs do exist for certain other layer types and initialization schemes, they may not always give rise to low dimensional maps that fully describe their behavior. For example, if we use a conventional fan-in initialization instead of a Delta initialization for the filter weights, then the resulting APKF (given in Equation \ref{eqn:fanin-APKF}) implies a more complex dependence of the entries of the output IPM on the input IPM, where output q values will depend on (many) input c values.

\section{Extended Q and C maps}\label{sec:extended-maps}

\james{Could possible move some of the stuff about concats, unit averages, and sums to the appendix and just give the rules? Could maybe even move PLN there?}

In \citet{poole2016exponential} and  \citet{schoenholz2016deep}, the neural networks analyzed were assumed to be sequences of $D$ fully-connected combined layers, each with the same activation function. Thus, the network's initialization-time behavior could be approximated using a single per-layer Q/C map composed with itself $D$ times, and a dynamical systems analysis of this map could thus be performed. This analysis looked for the map's stable points and attractors, and characterized its asymptotic behavior as the number of self-compositions $D$ (i.e.~the network's depth) went to infinity.

In this work we consider architectures with a more general structure, and with layers that can be convolutional and employ a variety of activation functions. We are also interested in the given architecture's finite structure, instead of its depth-limiting behavior, as this will allow us to more carefully tailor our manipulations to the given network. To facilitate this, in this section we will extend the notion of Q maps and C maps to arbitrary subnetworks (consisting of potentially many layers) in the natural way.

Going forward, we will refer to Q maps and C maps defined specifically for combined layers as \tmtextbf{\tmtextit{local Q/C maps}}, and maps defined specifically for larger subnetworks, via the extension procedure defined in the next subsection, as \tmtextit{\tmtextbf{extended Q/C maps}}. Unqualified, ``Q/C maps'' will be a general term referring to both.

\subsection{Definition of extended Q/C maps}

The definition for extended Q/C maps is the natural generalization of the definition for local Q/C maps, where we replace APKF approximations for combined layers with network-level PKF approximations for subnetworks. In particular, given a subnetwork $f$, an extended Q map maps input q values, corresponding to the diagonal entries of the input IPM $\Sigma_{Z, Z'}$, to the associated output q values, corresponding to the diagonal entries of the associated output IPM $\widetilde{\kappa_f} (\Sigma_{Z, Z'})$, as computed by the network-level PKF approximation $\widetilde{\kappa_f}$. The definition for extended C maps is similar.

That these definitions can be made in a location-independent way (as with the definitions of local Q/C maps), follows from the fact that extended Q/C maps can be constructed from local ones via composition and weighted averaging (as will be detailed below), which are both operations that preserve the location-independence property.

\subsection{Computing extended maps}

Because Q maps compose with each other, and C maps compose with the combination of both, we can take the per-combined-layer maps and compose them in a way that mirrors the composition of the subnetwork's combined layers, analogously to how we assembled network-level PKF approximations from APKF approximations of each combined layer. For example, if we have two consecutive combined layers $f$ and $g$, and wish to compute the Q and C map for the subnetwork $h$ consisting of their composition, this is simply $Q_h (q) = Q_g (Q_f (q))$ and $C_h (c, q_1, q_2) = C_g (C_f (c, q_1, q_2), Q_f (q_1), Q_f (q_2))$. \sam{Just to point out, we use this compositional approach to compute the NNGP kernel / NTK in \href{github.com/google/neural-tangents}{Neural Tangents}. Incidentally, if you were interested, I'd love to pair program with you / implement DKS into neural tangents! Maybe you'd find this \href{https://colab.sandbox.google.com/github/google/neural-tangents/blob/main/notebooks/phase_diagram.ipynb}{colab} interesting.} 

The only complication is that we need to describe how q and c values can be computed when feature maps are concatenated along their channel dimensions, or when they are multiplied by a non-zero scalar constant. To handle the former situation, we recall from Equation \ref{eqn:concat-IPM} that concatenation leads to a weighted averaging of the feature maps' associated IPMs, with weights given by their respective number of channels. Thus, the q values, which are the diagonal entries of these matrices, average in the same way under concatenation. So given the channel dimensions $k_1$ and $k_2$, and the q values $q_1$ and $q_2$, we have that the associated q value of the concatenation is simply
\begin{equation}
  \frac{k_1 q_1 + k_2 q_2}{k_1 + k_2} . \label{eqn:q-concat-formula}
\end{equation}
c values are slightly more complicated to deal with, but still relatively straightforward. We note that m values, the unnormalized counterparts of c values, are the off-diagonal entries of the IPMs, and thus exhibit the same kind of averaging as q values. We can thus obtain the c values by first converting them to m values, performing the required weighted average, and then converting back to c values. This gives us the analogous formula
\begin{equation}
  \frac{k_1  \sqrt{q_{1, 1} q_{1, 2}} c_1 + k_2  \sqrt{q_{2, 1} q_{2, 2}} c_2}{k_1  \sqrt{q_{1, 1} q_{1, 2}} + k_2  \sqrt{q_{2, 1} q_{2, 2}}}, \label{eqn:c-concat-formula}
\end{equation}
where $q_{i, j}$ refers to the $j$-th q value associated with the c value from the $i$-th feature map being concatenated. (Recall that each c value is associated to {\tmem{two}} q values.)

Note that the property that local C maps send $[- 1, 1]$ to $[- 1, 1]$ carries over to extended C maps, as this clearly preserved under composition and weighted averages.

To handle multiplication of a feature map by a constant $\alpha \neq 0$, we note that the IPM of $\alpha Z$ and $\alpha Z'$ is equal to $\alpha^2$ times the IPM of $Z$ and $Z'$, or in other words: $\Sigma_{\alpha Z, \alpha Z'} = \alpha^2 \Sigma_{Z, Z'}$. We thus have that an output q value (or m value) for such an operation is simply $\alpha^2$ times the corresponding input q value (or m value). And an output c value is just equal to the corresponding input c value, since the constant $\alpha^2$ will cancel out when we divide by the geometric mean of the q values.

\subsection{Generalization to subnetworks with isolated affine and nonlinear layers}

Because it will simplify Q/C map computations for certain architectures (such as residual networks), we will also generalize Q/C maps to subnetworks that may contain affine or nonlinear layers in isolation (i.e.~separated from their parent combined layer). To do this, we will define local Q/C maps for isolated affine and nonlinear layers in a way that is consistent with our previous definitions (with one small proviso), and then use the previous composition argument to extend Q/C maps to larger subnetworks containing such layers.

The isolated affine layer case is trivial, as an affine layer is equivalent to a combined layer with an identity activation function, and so is covered under the previous discussion. It follows that \tmtextit{\tmtextbf{affine layers have local Q and C maps that are the identity function}} (which can easily be verified by setting $\phi (u) = u$ in Equations \ref{eqn:Q-map} and \ref{eqn:C-map}), and can thus be essentially ignored in the extended map computations.

The case of nonlinear layers is more subtle. APKFs, from which local Q and C maps are defined, don't actually exist for nonlinear layers in isolation. In particular, for arbitrary input vectors it is {\tmem{not}} the case that one can closely approximate the norm of the output vector given only the norm of the input vector (with high probability). However, when the layer is part of a larger network in which its input vector is always the output of some affine layer (with a suitable parameter distribution), such a prediction can be made, and is given by the APKF for the corresponding combined layer. (To see this, note affine layers have identity Q and C maps and thus the input to the nonlinear layer has the same q and c values as the input to the corresponding combined layer.) 

Thus, we can define \tmtextit{\tmtextbf{the local Q and C map for an isolated nonlinear layer to be equal to the local Q and C maps for its associated combined layer}}, with the proviso that it describes the layer's kernel behavior only for ``typical'' input vectors (i.e.~those that are produced with high probability by the previous layers' computation) and not arbitrary input vectors.

Note that these definitions are consistent with our definitions for combined layers, as the composition of the local Q/C map for an affine and nonlinear layer (as we have defined them here) does indeed recover the local Q/C map of the associated combined layer. Also, it should be emphasized that these arguments rely crucially on the fact that nonlinear layers may be ``isolated'' only from the point of the view of a given subnetwork. From the perspective of the entire network, it is still required that they are always part of a combined layer, or in other words, are always directly preceded by an affine layer.

\section{Q and C map derivative computations}\label{sec:QC_map_derivatives}

Central to our analysis of Q and C maps are their derivatives, which encode many of the properties that we will care about. In this section we show how to compute them, first for local maps, and then for extended maps of arbitrary subnetworks.

\subsection{Local map case}

A conceivable approach to computing the derivatives of local maps would be to derive a closed form expression for the required integrals, and then apply standard differentiation techniques. Unfortunately, closed form expressions for these integrals are not generally available for most the activation functions. Instead, following \citet{poole2016exponential}, we will give integral expressions for the derivatives which are similar to the original maps themselves, and which can be efficiently approximated using numerical integration (as discussed in Section \ref{sec:estimate-expectations}). 

\subsubsection{Local Q map derivative for combined layers (or isolated nonlinear layers)}

Let $f$ be a combined layer (or an isolated nonlinear layer) with element-wise activation function $\phi$.

The derivative for $Q_f (q)$ with respect to $q$, which we denote by $Q'_f (q)$, can be computed straightforwardly from Equation \ref{eqn:Q-map}, and is equal to
\[ Q'_f (q) = \frac{1}{\sqrt{q}} \mathbb{E}_{x \sim \mathcal{N} (0, 1)} \left[ \phi \left( \sqrt{q} x \right) \phi' \left( \sqrt{q} x \right) x \right], \]
where $\phi'$ is the derivative of $\phi$. Note that because $\phi$ is continuous, we are still able to compute this expectation, and similar ones to follow, when $\phi'$ is undefined on a finite set of inputs (which is permitted under our global assumptions).

\subsubsection{Local C map derivatives}\label{sec:local-C-map-derivatives}

The derivative of local C maps with respect to their c value argument has an especially nice form which we make use of later in Section \ref{sec:uniform-q-consequences}.

We begin by defining the following notation:
\begin{equation}
  \Gamma_{\phi} (c, q_1, q_2) \equiv \mathbb{E}_{x, y \sim \mathcal{N} (0, 1)} \left[ \phi \left( \sqrt{q_1} x \right) \phi \left( \sqrt{q_2}  \left( cx + \sqrt{1 - c^2} y \right) \right) \right] . \label{eqn:Gamma-def}
\end{equation}
This function is closely related to the local C map of $f$ (given by Equation \ref{eqn:C-map}) in the sense that $C_f (c, q_1, q_2) = \frac{1}{\sqrt{Q_f (q_1) Q_f (q_2)}} \Gamma_{\phi} (c, q_1, q_2)$. The derivative of $\Gamma_{\phi} (c, q_1, q_2)$ with respect to $c$, which we denote as $\Gamma'_{\phi} (c, q_1, q_2)$, is given by
\[ \Gamma'_{\phi} (c, q_1, q_2) = \sqrt{q_1 q_2} \Gamma_{\phi'} (c, q_1, q_2), \]
This elegant formula was stated in \citet{poole2016exponential}, although no explicit derivation of it was given. For completeness we provide one in Appendix \ref{app:Gamma-derivative}.

An immediate consequence of this result is that the $i$-th derivative of $\Gamma_{\phi} (c, q_1, q_2)$ with respect to $c$, which we denote by $\Gamma^{(i)}_{\phi} (c, q_1, q_2)$, is equal to
\[ \Gamma^{(i)}_{\phi} (c, q_1, q_2) = (q_1 q_2)^{i / 2} \Gamma_{\phi^{(i)}} (c, q_1, q_2), \]
where $\phi^{(i)}$ denotes the $i$-th derivative of $\phi$. From this it follows that the $i$-th derivative of $C_f (c, q_1, q_2)$ w.r.t. $c$ can be written as
\begin{equation}
  C^{(i)}_f (c, q_1, q_2) = \frac{(q_1 q_2)^{i / 2}}{\sqrt{Q_f (q_1) Q_f (q_2)}} \Gamma_{\phi^{(i)}} (c, q_1, q_2) . \label{eqn:C-map-gen-derivative}
\end{equation}
This formula is valid even when $i = 0$, where the $0$-th derivative is defined as the function itself (i.e.\,$\phi^{(0)} = \phi$), as is standard convention. When $\phi^{(i)} (u)$ isn't defined on a measure zero set of points, the formula may still be valid, provided that $\phi^{(i - 1)}$ is continuous. 

For example, if $\phi$ is the RELU function, $\phi (u)$ is continuous everywhere and has a derivative everywhere except at $u = 0$, so the formula is valid for $i = 1$. However, $\phi^{(1)} (u)$ is not continuous at $u = 0$, and one can use Equation \ref{eqn:C-map-RELU} to show that $C^{(2)}_f (c, 1, 1) \rightarrow \infty$ as $c \rightarrow 1$, while the formula would wrongly predict a value of $0$.

\james{I suspect that there may be an equally nice derivative formula that won't even involve computing derivatives of $\phi$. The advantage of such a formula would be that we wouldn't need to assume that the derivatives of $\phi$ exist.}

\subsection{Derivatives of extended maps}\label{sec:derivative_extended_maps}

Because extended maps can be expressed as compositions and weighted averages of local maps, their derivative computations can be performed straightforwardly using automatic differentiation. The resulting formulae will still depend on the derivatives of local maps, but these can be computed (or numerically approximated) as per the previous subsection.

In such a scheme, composition corresponds to multiplication, and weighted averages correspond to weighted averages (since differentiation is linear). Notably, because q values don't depend on c values, the derivative of extended C maps with respect to their input c values can be computed as if all the q values in the network are constant (although they still depend on the network's input in general). So for example, the C map derivative for a composition of many combined layers is just the product of the local C map derivatives for each layer, evaluated at the appropriate values of $c$ as per the forward evaluation.

\section{Handling weighted sum operations}\label{sec:sum-ops}

An operation commonly performed in neural network models is the (weighted) sum of two or more feature maps. For example, in the ResNet-V2 architecture (which is described in detail in Section \ref{sec:standard-ResNet}), the input to a ``residual block" is added to its output, using what is known as a ``residual connection''. Since sum operations are not among those listed as allowed in Section \ref{sec:arch-assumptions}, it would seem that our assumptions rule out such architectures.

However, for the purposes of our analysis, there is no requirement that a network be formally constructed the same way it would implemented in code or drawn in a diagram; it only matters that it {\tmem{can}} be constructed in a way that conforms to the assumptions outlined in Section \ref{sec:arch-assumptions}. With this in mind, we will now describe a way that a certain restricted class of weighted sum operations can be simulated using only directly supported operations. The consequence of this is that our analysis will in fact apply to architectures that contain such sum operations.

Typically, the feature maps that are summed in neural networks are the outputs of a set affine layers $f_1, f_2, \ldots, f_n$ that don't share parameters. (This is true in ResNet-V2 architectures, for example.) In such situations, we can replace the sum $\sum_{i = 1}^n f_i (Z_i)$ with a single affine layer $h \left( \left[\begin{array}{cccc}
  Z_1^{\top} & Z_2^{\top} & \cdots & Z_n^{\top}
\end{array}\right]^{\top} \right)$, which is obtained by concatenating the filter banks, the bias vectors, and the input feature maps (i.e.~the $Z_i$'s) together along their respective channel dimensions. (If $\sum_{i = 1}^n f_i (Z_i)$ is followed by a nonlinear layer in the network, then one simply forms a new combined layer consisting of this and $h$.)

While almost good enough, the issue with this construction is that the implied initial distribution of $h$'s filter bank parameters is not one of the ones described in Section \ref{sec:param_dist}, and in particular, the variance/scale is not correct. To account for this, we must renormalize by the new number of channels (after the concatenation), the effect of which is that $h$ will instead compute a weighted sum of the form
\[ \frac{\sum_{i = 1}^n \sqrt{k_i} f_i (Z_i)}{\sqrt{\sum_{i = 1}^n k_i}}, \]
where $k_i$ is the input channel dimension for $f_i$.

Fortunately, we can extend this construction to support a more general class of weighted sums (with weights $w_i$) by multiplying each $Z_i$ by a scalar $\alpha_i = w_i  \sqrt{\sum_{i = 1}^n k_i} / \sqrt{k_i}$ before concatenating them. Doing so gives
\[ \frac{\sum_{i = 1}^n \sqrt{k_i} f_i (\alpha_i Z_i)}{\sqrt{\sum_{i = 1}^n k_i}} = \frac{\sum_{i = 1}^n \sqrt{k_i} \alpha_i f_i (Z_i)}{\sqrt{\sum_{i = 1}^n k_i}} = \sum_{i = 1}^n w_i f_i (Z_i), \]
where we have used the fact that the affine $f_i$'s are in fact linear (given that the biases are initialized to 0).

If the layer $f_i$ is still in the network after this replacement is performed for some $i$ (e.g.~because its output is used in more than one place), this creates parameter sharing between $h$ and $f_i$. However, as long as the network with sum operations that we are trying to simulate doesn't violate our parameter independence assumptions from Section \ref{sec:param_dist}, neither will our simulating network.

The existence of this construction thus implies that {\tmem{weighted sums between the outputs of two or more affine layers (and directly followed by an optional nonlinear layer) {\tmem{are}} supported within our framework, provided that said affine layers don't share parameters}}. Note that the weighed sum operation can be performed directly in the model code, and the concatenation-based construction only needs to be referenced in the theoretical analysis.

To deal with sum operations in Q map computations, we observe that the q value of $\left[\begin{array}{cccc}
  \alpha_1 Z_1^{\top} & \alpha_2 Z_2^{\top} & \cdots & \alpha_n Z_n^{\top}
\end{array}\right]^{\top}$ is, according to Equation \ref{eqn:q-concat-formula}, equal to
\begin{equation}
  \frac{\sum_{i = 1}^n k_i \alpha_i^2 q_i}{\sum_{i = 1}^n k_i} = \frac{\sum_{i = 1}^n k_i w_i^2  \left( \sum_{i = 1}^n k_i \right) / k_i q_i}{\sum_{i = 1}^n k_i} = \sum_{i = 1}^n w_i^2 q_i , \label{eqn:sum-q-formula}
\end{equation}
where $q_i$ is the q value associated with $Z_i$, and we have used the fact that the q value for $\alpha_i Z_i$ is $\alpha_i^2 q_i$. From this it follows that the output q value from the sum is also $\sum_{i = 1}^n w_i^2 q_i$, since $Q_h$ is just the identity function.

Given uniform q values, a similar derivation based on Equation \ref{eqn:c-concat-formula} lets us compute the corresponding c value as
\begin{equation}
  \frac{\sum_{i = 1}^n w_i^2 q_i c_i}{\sum_{i = 1}^n w_i^2 q_i}, \label{eqn:sum-c-formula}
\end{equation}
where $c_i$ is the c value associated with $Z_i$. Note that unlike the formula for the q value, this is always a weighted average of the $c_i$'s, regardless of the values of the $w_i$'s. And in the case where all input q values are equal, it simplifies to $\left( \sum_{i = 1}^n w_i^2 c_i \right) / \sum_{i = 1}^n w_i^2$.

\section{Uniform q values}\label{sec:uniform-q}

In general, C maps are three dimensional functions that depend on an input c value and two associated q values. While simpler objects than a network's PKF (or even the network-level PKF approximation), they are not yet simple enough for our purposes. In particular, the behavior of C maps depends strongly on the two input q values, and q values can vary significantly between different network inputs and/or feature map locations. Finding a single scheme that controls the behavior of the C map for all conceivable input q pairs is likely impossible in general, and so we look to restrict the possible q values through some sort of active intervention.

The one we propose in this section is a form of input data preprocessing, which ensures that all q values for a given layer are equal (across all possible locations in the feature map and inputs to the network). (Note that this condition \tmtextit{does not} require that q values be the same across different layers.) We will call this condition \tmtextit{\tmtextbf{uniform q values}}. 

\subsection{A previous solution to this problem}\label{sec:edgeofchaos-q-solution}

In \citet{poole2016exponential} it was observed that local Q maps can have stable fixed points, and that if the network consists of a composition of many combined layers of the same type, then its q values will converge with depth to such a point. Thus, a reasonable approximation, especially for deeper layers, is to assume that this convergence has already taken place, and that the q values over the entire network are equal. (Note that this is strictly stronger condition that uniform q values.)

As discussed in Section \ref{sec:extended-maps}, our setting is different from \poscite{poole2016exponential} in that we consider more general architectures, and are interested in the precise behavior of a finite network architecture instead of its depth limiting behavior. Moreover, it may be a poor approximation to assume that q values are close to convergence in the earlier layers of the network, especially if there are no constraints placed on the initial q values (which are determined by the network's input).

\subsection{Uniform q values via Per-Location Normalization}\label{sec:PLN}

Our solution to the problem of unpredictable q values is a type of input data pre-processing which we call \tmtextbf{\tmtextit{Per-Location Normalization}} (\tmtextbf{\tmtextit{PLN}}). This is related to the data normalization done in \citet{daniely2016toward} for fully-connected networks, but generalized to convolutional networks. PLN ensures that each location vector in the network's input feature map has a dimension-normalized squared norm of 1, or in other words, that the q values for the network's input layer are all 1. Because subsequent q values are fully determined by previous q values via location-agnostic computations (i.e.~Q maps), it thus follows by induction that each layer will have uniform q values under PLN.

PLN can be easily realized through a number of different possible transformations of the network's input, although care must be taken not to destroy information. The naive approach of normalizing the vector at each location of the input feature map (and multiplying by the square root of the channel dimension) destroys information because the vector goes from having $k$ degrees of freedom to $k - 1$ degrees of freedom (where $k$ is the number of channels). This can be seen most starkly when $k = 1$, in which case all location-wise ``vectors" are reduced to $\pm 1$ scalar values.

The naive approach can however be repaired, by first adding an extra channel to the network's input. In our experiments we used the value $ \left( \frac{1}{k} \mathbb{E}_j [\| x_j \|^2] \right)^{\frac{1}{2}}$ for this extra channel, where the expectation is an average over location vectors $x_j$ for the given input feature map $X$. This results in a vector of the form
\[ \frac{(k + 1)^{\frac{1}{2}}}{\left( \| x_i \|^2 + \frac{1}{k} \mathbb{E}_j [\| x_j \|^2] \right)^{\frac{1}{2}}}  \left[\begin{array}{c}
     x_i\\
     \left( \frac{1}{k} \mathbb{E}_j [\| x_j \|^2] \right)^{\frac{1}{2}}
   \end{array}\right] \]
for each location $i$.

Note that this approach to PLN still destroys some information, although it's only one degree of freedom across $X$, which includes all locations and channels. This can be seen most clearly in the case of only one location vector $x$, in which case the formula becomes
\[ \frac{(k + 1)^{\frac{1}{2}}}{\left( \| x \|^2 + \frac{1}{k}  \| x \|^2 \right)^{\frac{1}{2}}}  \left[\begin{array}{c}
     x\\
     \| x \| / \sqrt{k}
   \end{array}\right] = \left[\begin{array}{c}
     \sqrt{k} x / \| x \|\\
     1
   \end{array}\right] , \]
from which we cannot recover the norm of $x$. Thus, it only makes sense to use this form of PLN when there are a large number of locations and/or channels.

For cases where there is only one location (i.e.~in a fully-connected network) and the channel dimension is small, one possible alternative is to use a data-independent constant value for the extra channel. Using a value of 1 gives
\[ \frac{(k + 1)^{\frac{1}{2}}}{(\| x \|^2 + 1)^{\frac{1}{2}}}  \left[\begin{array}{c}
     x\\
     1
   \end{array}\right] = \sqrt{k + 1}  \left[\begin{array}{c}
     x / \sqrt{\| x \|^2 + 1}\\
     1 / \sqrt{\| x \|^2 + 1}
   \end{array}\right], \]
which doesn't destroy any information about $x$ (since we can invert the last entry to get $\sqrt{(\| x \|^2 + 1) / (k + 1)}$, and then multiply that by the other entries to recover $x$). The disadvantage of this approach is that the scale of $x$ could differ very significantly from $1$, so that after normalization, the value of the extra channel could either dominate the overall vector, or be minuscule.

Performing PLN may not always be important in practice, as we demonstrate later in our ablation experiments. Moreover, the version we have proposed seems to slightly {\tmem{degrade}} optimization performance in our benchmarks, possibly because of the extra parameters it adds to the first layer (for the extra channel dimension), or because of the nonlinear warping it applies to the input space. On the other hand, our experiments also demonstrate that very badly scaled input data can sometimes cause DKS to perform poorly, unless PLN is applied as a corrective measure. (See Appendix \ref{app:ablation-PLN} for the relevant results.)

There are other ways we can produce normalized vectors without destroying information, such as those discussed in \citet{daniely2016toward} for fully-connected networks. Of all of the aspects of DKS, our method of realizing PLN is the least explored, and we wouldn't be surprised if there was a significantly better way of doing it.

\subsection{Assuming uniform q values going forward}

From this point forward we will assume that the uniform q value condition holds. This thus allows us to treat C maps as one dimensional functions, as the q value for each layer will be constant. As we will show in the next section, it also imbues C maps with a set of very useful properties which end up being crucial to our subsequent analysis of them in Section \ref{sec:Cmap-analysis}.

\section{Additional consequences of uniform q values for C maps}\label{sec:uniform-q-consequences}

\subsection{Local C maps are positive definite functions and map c values of 1 to 1}

For convenience, when we have uniform q values we will drop the formal dependence of the C map on its input q values, and instead treat these as known constants within the expression (which are both equal to the same value). This allows us to view C maps as essentially one dimensional functions, and we can use notation of the form ``$C_f (c)$'' for them going forward.

Under the assumption of uniform q values, several interesting and useful properties of local C maps emerge. Suppose $f$ is a combined layer (or an isolated nonlinear layer) with activation function $\phi$. We begin by setting $q_1 = q_2 = q$ in Equation \ref{eqn:C-map-gen-derivative}, which gives
\[ C^{(i)}_f (c) = \frac{q^i}{Q_f (q)} \mathbb{E}_{x, y \sim \mathcal{N} (0, 1)} \left[ \phi^{(i)} \left( \sqrt{q} x \right) \phi^{(i)} \left( \sqrt{q}  \left( cx + \sqrt{1 - c^2} y \right) \right) \right] . \]
From this expression we can deduce that
\[ C_f (1) = \frac{1}{Q_f (q)} \mathbb{E}_{x \sim \mathcal{N} (0, 1)} \left[ \phi \left( \sqrt{q} x \right)^2 \right] = \frac{Q_f (q)}{Q_f (q)} = 1 \]
and
\begin{eqnarray}
  C^{(i)}_f (0) & = & \frac{q^i}{Q_f (q)} \mathbb{E}_{x, y \sim \mathcal{N} (0, 1)} \left[ \phi^{(i)} \left( \sqrt{q} x \right) \phi^{(i)} \left( \sqrt{q} y \right) \right] \nonumber\\
  & = & \frac{q^i}{Q_f (q)} \mathbb{E}_{x \sim \mathcal{N} (0, 1)} \left[ \phi^{(i)} \left( \sqrt{q} x \right) \right]^2 \hspace{0.8em} \geq \: \: 0 .  \label{eqn:C-at-0}
\end{eqnarray}
The second consequence implies two interesting properties. First, that for local maps, $C_f (0) = 0$ if and only if $\mathbb{E}_{x \sim \mathcal{N} (0, 1)} \left[ \phi \left( \sqrt{q} x \right) \right] = 0$, where we note the interpretation of $\mathbb{E}_{x \sim \mathcal{N} (0, 1)} \left[ \phi \left( \sqrt{q} x \right) \right]$ from Appendix \ref{app:average-unit-approx}. And second, that the Taylor series expansion of $C_f (c)$ about $c = 0$, which is given by
\[ \sum_{i = 0}^{\infty} \frac{1}{i!} C^{(i)}_f (0) c^i, \]
will have all non-negative coefficients. Provided that the Taylor series converges and is equal to $C_f (c)$ (which it will under mild technical conditions), it thus follows that $C_f (c)$ is a \tmtextit{\tmtextbf{positive definite function}} \citep{schoenberg1988positive, daniely2016toward}, which is defined as a function from $[- 1, 1]$ to $\mathbb{R}$ that can be written as
\[ \sum_{i = 0}^{\infty} b_i c^i, \]
for non-negative coefficient $b_i$.

\subsection{Properties of positive definite functions}\label{sec:pd-functions}

Positive definite functions have many interesting and useful properties which thus carry over to C maps. These include:
\begin{enumeratenumeric}
  \item The set of positive definite functions is closed under differentiation\footnote{Closedness under differentiation can be easily verified by observing that the derivative of $\sum_{i = 0}^{\infty} b_i c^i$ with respect to $c$ is just $\sum_{i = 1}^{\infty} ib_i c^{i - 1}$, which is also positive definite since $ib_i \geqslant 0$ when $b_i \geqslant 0$.}.
  
  \item Positive definite functions are non-negative, non-decreasing, and convex on the non-negative part of their domain. (This follows from the fact that their derivatives are also positive definite functions, and thus non-negative for non-negative inputs.)
  
  \item The set of positive definite functions is closed under composition and weighted averages with non-negative weights\footnote{Closedness under composition can be easily verified by substituting one series into the other, expanding, and observing that the coefficients of the resulting series are non-negative combinations of coefficients from the two original series. Similarly, closedness under weighting averaging follows by observing that the coefficients of the series for the weighted average are just weighted averages of the corresponding coefficients from the original two series.}.
\end{enumeratenumeric}

\subsection{Extended C maps are positive definite functions and map c values of 1 to 1}\label{sec:extended-map-posdef}

Given that local C maps are positive definite functions and map c values of 1 to 1, it's easy to show that the same applies to extended C maps. First, the property that c values of 1 map to 1 is clearly preserved under composition and weighted averaging, and thus carries over to extended maps since they are constructed from local maps this way. Second, the property of C maps being positive definite functions also carries over, since positive definite functions are closed under composition and non-negative weighted averages as mentioned above.

Another way that one can show that extended C maps are positive definite is by observing that they describe the exact one-dimensional kernel function $[\kappa_f (z, z')]_{1, 2}$ of a fully-connected network $f$ in the limit of infinite width (where we substitute convolutional layers in the original network with fully-connected layers). Because this kernel depends only on the inner product of its inputs via the function $C_f$, it is thus invariant to orthogonal transformations of its input, and so by Schoenberg's Theorem \citep{schoenberg1988positive} it is a positive definite function of this inner product (i.e.~$C_f$ is positive definite). Note that this argument works even for non-smooth activation functions for which Equation \ref{eqn:C-map-gen-derivative} may not apply. See \citet{daniely2016toward} for more details.

\subsection{A complementary perspective based on ``dual activations functions'' }\label{sec:dual-activations}

Assuming input q values of 1, \tmtextit{local} C maps are essentially equivalent to the ``dual activation functions'' defined in \citet{daniely2016toward}, with the only difference being that dual activation functions aren't normalized by the output q values (as C maps are). Using the notation of Equation \ref{eqn:Gamma-def}, the dual activation function $\tilde{\phi}$ of $\phi$ can be written as
\[ \tilde{\phi} (c) = \Gamma_{\phi} (c, 1, 1) . \]
Assuming that we normalize each activation function so that its output q value is 1, these dual activation functions can be composed and averaged in order to form what are called ``compositional kernels'', which are approximations of the kernel function for the entire network, and are analogous to our network-level PKF approximation in the case where there is only one location (i.e.~the fully-connected case).

Given this connection, it may thus be an appealing prospect for us to adopt \poscite{daniely2016toward} framework instead of the one we've presented, as it's very carefully laid out and rigorously developed, and comes packaged with the best known error bounds for initialization-time kernel approximations of neural networks (one of which we adapt in Section \ref{sec:how-accurate-approx}). However, while their framework can deal with local receptive fields, it cannot directly deal with the weight sharing used in convolutional layers, and it would likely require significant work to extend it in that direction.

Indeed, to deal with convolutional layers in a way that allows kernel approximations for individual layers to be naturally composed, one seemingly must define something like our APKFs which keep track of approximations to entire IPMs (which contain inner products between every pair of locations in the feature maps of both inputs). Moreover, without assuming a Delta initialization, a decomposition of the kernel approximations into 1 dimensional functions (such as Q/C maps) becomes impossible, since APKFs for standard initializations involve non-trivial interactions between all the locations in the feature map (as seen in Equation \ref{eqn:fanin-APKF}).

While we do indeed restrict our attention to Delta initializations in this work, without the PKF/APKF formalism we would not be able to extend our analysis to mean pooling layers, since the kernel approximation for such layers also involves interactions between locations. As we will see later in Section \ref{sec:application-to-resnet}, this extension will be necessary later in order to understand how DKS can be applied to standard convolutional neural network architectures.

\part{Desirable Q/C map behavior and how to achieve it}
\label{part:desirable_QCmap_properties}

\section{C map behavior in deep networks and necessary requirements for trainability}\label{sec:Cmaps_trainability}

C maps approximate a network's PKF at initialization time. In this view, c values approximate the cosine similarity between pairs of vectors (corresponding to different locations/inputs), and their evolution via C maps thus describes how these cosine similarities evolve in the network. Given uniform q values, the norms of these vectors are approximately constant (for a given layer), and thus their relative distance is related to their cosine similarity $c$ via
\[ \frac{\|x - y\|}{\frac{1}{2}  (\| x \| + \| y \|)} = \frac{\|x - y\|}{\sqrt{\| x \|  \| y \|}} = \sqrt{2 (1 - c)} .  \]
A (sub)network's C map thus provides a complete description of how it warps the geometry of its input space at initialization time.

As we will argue in this section, the preservation of some amount of this geometric information through the network, as indicated by a ``well-behaved'' C map, is a necessary condition for the network to be trainable. When C maps ``degenerate'' in certain ways, as we will show they do for standard deep neural networks, it means that the relative distances between the network's (location-wise) input vectors are hard to infer from the network's outputs, making gradient-based training difficult.

\james{TODO: maybe call forward or give preview of the argument about why degen makes learning hard?}

\subsection{RELU networks}\label{sec:relu-network-degen}

The local C map of a combined layer $f$ with a RELU activation function is given by
\begin{equation}
  C_f (c) = \frac{\sqrt{1 - c^2} + (\pi - \cos^{- 1} (c)) c}{\pi} . \label{eqn:C-map-RELU}
\end{equation}
This formula is stated in \citet{daniely2016toward}, and is based on a derivation by \citet{cho2012kernel}, where it corresponds to a normalized version of the ``1st-order arc-cosine kernel function''. Note that while in general C maps depend on the input q value, this formula is valid for {\tmem{any}} q value, which is a consequence of the fact that RELUs are positively homogeneous (i.e.~$\tmop{RELU} (\lambda u) = \tmop{RELU} \phi (u)$ for all $\lambda \geqslant 0$).

One interesting fact about $C_f$ is that $C_f' (1) = 1$, which can be verified by taking the derivative of Equation \ref{eqn:C-map-RELU} and letting $c \rightarrow 1$. Moreover, because $C_f (1) = 1$ (which is true for general C maps), we have that a deep RELU network $g$ consisting of the composition of $D$ combined layers will also have the property that $C_g' (1) = 1^D = 1$.

The following is a plot of $C_f$:
\begin{center}
\resizebox{4.5in}{3in}{\includegraphics{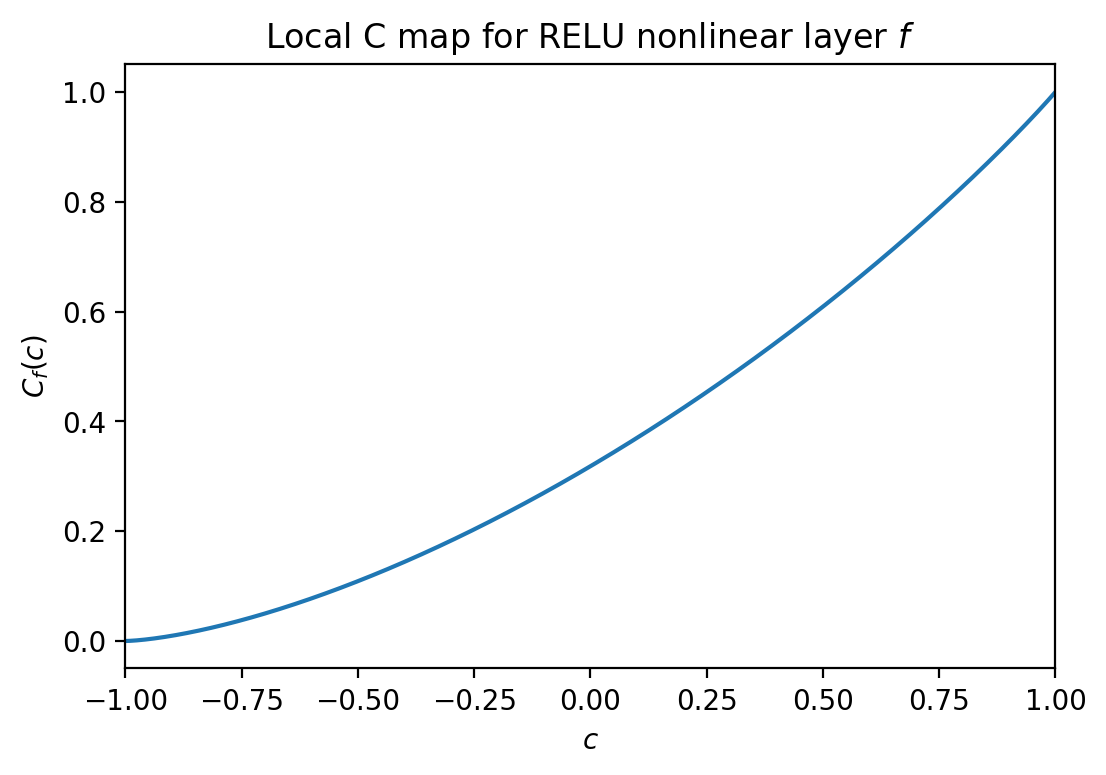}}
\end{center}

From this we can see that the entire domain $[- 1, 1]$ of input c values is compressed to the range $[0, 1]$, which makes intuitive sense since the RELU function is non-negative. Other than that, $C_f$ resembles a slightly shifted and rescaled identity function, and so is reasonably well-behaved.

However, if we build a deep network as the composition of many RELU combined layers, compression of the C map's output becomes much more extreme, with outputs rapidly concentrating around 1 as depth increases. This can be seen below in the plot of the C map for RELU networks of depths 5, 20 and 100 (which we obtain by iterating Equation \ref{eqn:C-map-RELU} the required number of times):
\begin{center}
\resizebox{4.5in}{3in}{\includegraphics{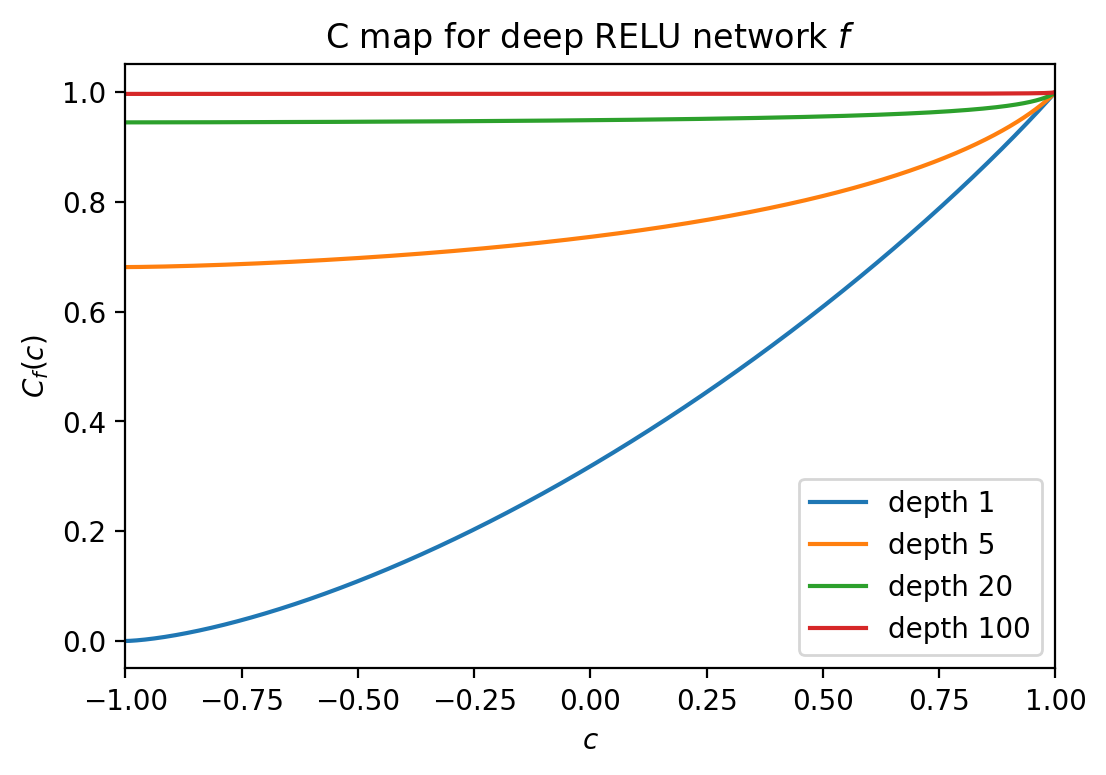}}
\end{center}

Here, the C map for depth 100 has the property that maps the entire interval $[- 1, 1]$ to [0.996, 1], which represents an extreme amount of compression.

\subsection{Sigmoidal networks (using the $\tmop{erf}$ activation)}\label{sec:erf-network-degen}

Another example of an activation function whose associated local Q and C maps have closed-form expressions is the classical ``error function'', which is given by $\tmop{erf} (u) = \frac{2}{\sqrt{\pi}}  \int_0^u \exp (- t^2) \, d \, t$. This function has a ``sigmoidal shape'', which makes it a reasonable stand-in for the more common sigmoidal activation functions like $\tanh$ and the logistic sigmoid.

The local Q map for a combined layer $f$ with an $\tmop{erf}$ activation function is given by
\[ Q_f (q) = \frac{2}{\pi} \sin^{- 1} \left( \frac{2 q}{1 + 2 q} \right), \]
and the local C map is given by
\[ C_f (c) = \frac{1}{Q_f (q)}  \frac{2}{\pi} \sin^{- 1} \left( \frac{2 cq}{1 + 2 q} \right) . \]
These formulas follow from equation 11 of \citet{williams1997computing}.

Unlike in the RELU case, the local C map depends on the input q value, and so to plot it we must make an assumption about this value. One natural choice is $q = 1$, which gives the following plot:
\begin{center}
\resizebox{4.5in}{3in}{\includegraphics{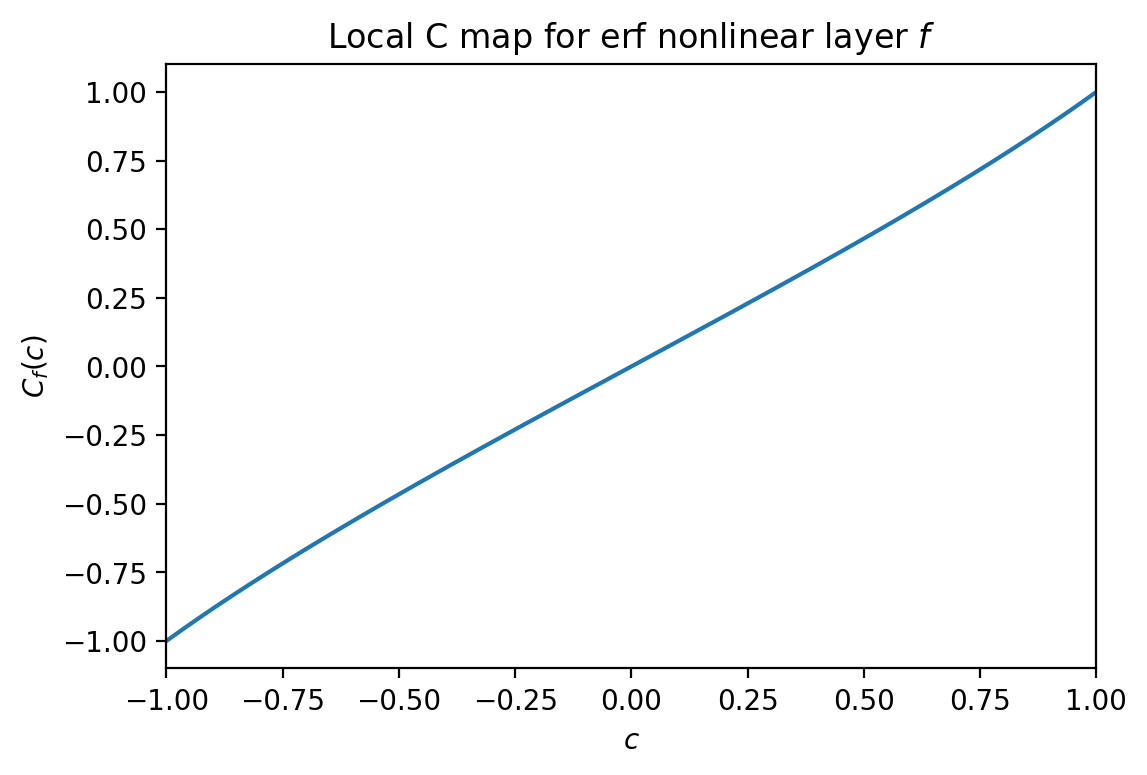}}
\end{center}

Visually, this function is almost indistinguishable from the identity function.

To compute the C maps for deeper RELU networks we need to track both the q and c values through each layer, using the previously stated equations. Doing so for depths 50, 150, and 500 gives the following plot:
\begin{center}
\resizebox{4.5in}{3in}{\includegraphics{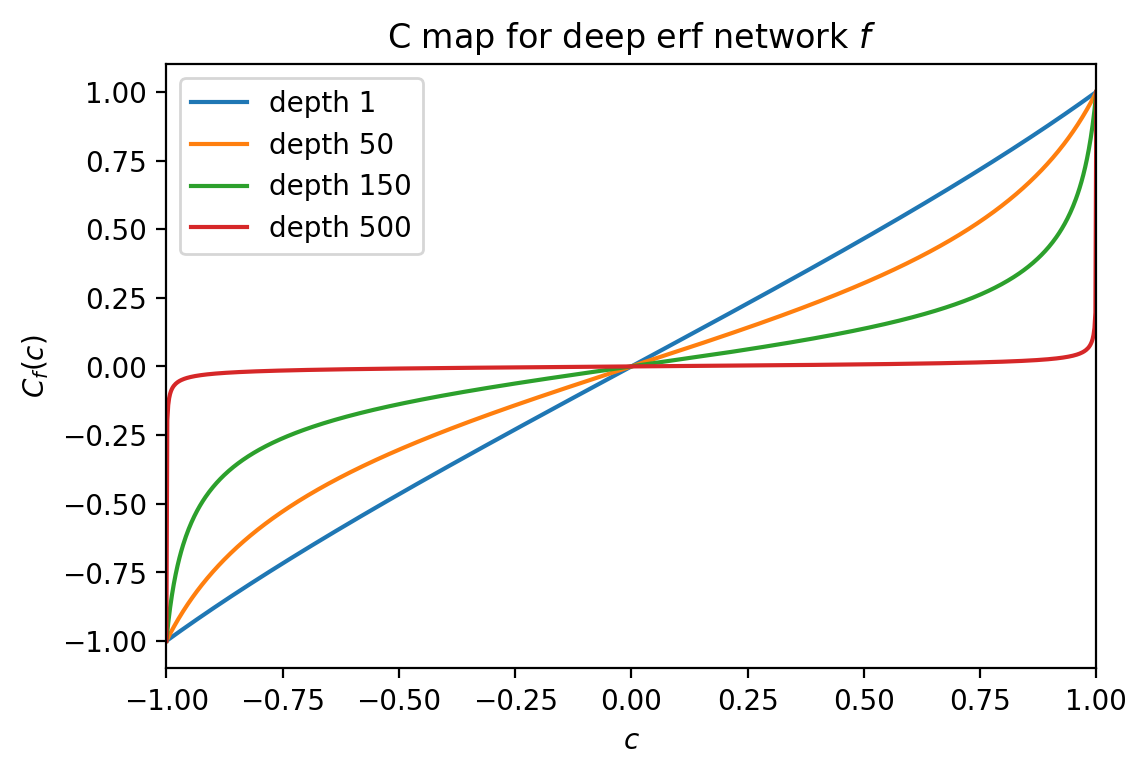}}
\end{center}

From this plot we can see that at high depths, the network's C map has a tendency to compress nearly all input c values to a small region around 0. Moreover, this behavior only becomes more extreme as the depth increases.

\subsection{C map degeneration in more general deep nonlinear networks}

The following proposition establishes that the C map degeneration we observed above for deep RELU and tanh networks happens for a larger class of deep networks. Moreover, the point $\fixpointofcmap$ towards which (nearly) all c values get mapped as the depth increases is unique.

\begin{proposition}
  \label{prop:C-map-fixed-point}Suppose $f$ is a deep network consisting of a composition of $D$ subnetworks, each with the same C map $\CC$. Then for all $c \in (- 1, 1)$ we have
  \[ \lim_{D \rightarrow \infty} C_f (c) = \fixpointofcmap, \]
  for some $\fixpointofcmap \in [0, 1]$.
\end{proposition}
The proof of this proposition is a straightforward generalization\footnote{While the statement of their claim assumes that each of the $D$ subnetworks is a combined layer, the only fact they use about $\CC$ in their proof is that it is positive definite, which holds for more general subnetworks by Section \ref{sec:extended-map-posdef}.} of the proof of ``Claim 1" from \citet{daniely2016toward}. 

While Proposition \ref{prop:C-map-fixed-point} describes the convergence of $C_f (c)$ in the limit of infinite depth, it is still informative about $C_f(c)$ at finite depths (which is the setting we actually care about). In particular, it essentially says that for any $\epsilon$ there is a constant $D_{\epsilon}$ so that for when $D \geqslant D_{\epsilon}$, nearly all input c values get compressed to a region of radius $\epsilon$ around $\fixpointofcmap$. We will call C maps exhibiting this compressive behavior (with a small $\epsilon$) \tmtextbf{\tmtextit{degenerate}}.

\begin{remark}
  Note that Proposition \ref{prop:C-map-fixed-point} assumes that $\CC$ is the same for all values of $D$. If, for example, we were to modify the network's activation functions based on the value of $D$ (which DKS will do), the convergence seen in the proposition may not occur.
\end{remark}

\begin{remark}
  Also note that the hypothesis in Proposition \ref{prop:C-map-fixed-point} that each subnetwork has the same C map $\CC$ will rule out many common cases. For example, it is violated for the deep $\tmop{erf}$ network we looked at before, because the q values are different for each layer (which leads to different local C maps). However, because the q values converge rapidly to a fixed point in such networks, convergence of the c values will still occur. As shown in Appendix \ref{app:resnet-map-analysis}, the ``residual blocks" of ResNets (which are repeated many times in sequence) also violate this hypothesis, but their q values do {\tmem{not}} converge to a fixed point.
\end{remark}

In Appendix \ref{app:math-for-Cmaps_trainability} (Theorems \ref{thm::small_derivative}, \ref{thm::fixedpoint-convergence-detailed_appversion}, and \ref{thm:odd-case}) we give a much more detailed analysis of the convergence of $C_f(c)$ in terms of the properties of $\CC$ and the location of $c$ in $[-1, 1]$. When $\CC'(1) \neq 1$ we prove exponential convergence (as a function of the depth $D$) with precise rates, thus establishing that degeneration can happen very quickly in deep networks. In contrast to the related analyses of \citet{poole2016exponential}, our results apply pre-asymptotically.

\subsection{Types of degeneration and their implications for trainability}\label{sec:degen-C-maps}

As we've seen above, degenerate C maps send a large range of input c values to a small (and sometimes point-like) region near some fixed point $\fixpointofcmap$. This means that the original geometric relationships between the corresponding input vectors are obscured by the action of the network, becoming essentially impossible to recover from its outputs. While it seems intuitively plausible that this would make gradient-based optimization of such networks difficult (as has been argued by \citet{schoenholz2016deep}), it's worth examining the situation in more detail. In this subsection we will give a detailed intuitive argument. A more rigorous argument which confirms these intuitions will appear later in Section \ref{sec:NTK-analysis}.

Suppose $f$ is some subnetwork of the overall network that we wish to train. There are two basic cases, corresponding to different possible values for $\fixpointofcmap$.

\subsubsection{$\fixpointofcmap = 1$: the collapsing case}

If $C_f$ sends nearly all input c values to a small neighborhood near $\fixpointofcmap = 1$, this means that regardless of the original distance of the two associated input vectors, their corresponding output vectors under $f$ will be nearly identical (i.e.~have a relative distance $\sqrt{2 (1 - C_f (c))} \approx \sqrt{2 (1 - \fixpointofcmap)} = 0$). And because this holds for \tmtextit{all} pairs of vectors, it means that $f$ is nearly constant, with only a \tmtextit{very} weak dependence on its input. This has a different set of consequences for layers in the network before $f$ versus layers after. 

For layers before $f$ there are two cases. If $f$'s Jacobian is non-negligible for most inputs, then it will have to vary wildly over $f$'s input space, as this is the only that a function can achieve a nearly constant output while having a non-negligible Jacobian. (\citet{balduzzi2017shattered} observed a similar phenomenon for early layers in deep RELU networks, likening the gradient function to a ``random noise process''.) This will make learning difficult, or at the very least unlikely to generalize, as similar pairs of training cases will produce very different gradients. If on the other hand $f$'s Jacobian is negligible, this means that gradient magnitudes for layers before $f$ will be very small compared to those for other layers. This makes simultaneous optimization of the network's layers with gradient descent very difficult, and even sophisticated 2nd-order methods may struggle in the more extreme cases. (This is arguably related to the well-known ``vanishing gradients'' phenomenon identified in \citet{hochreiter2001gradient}.)

Meanwhile, layers after $f$ won't be able to learn anything more than a constant output prediction, as their inputs will be nearly constant. 
And even if the output produced by $f$ has enough variance across the training data to overcome the limits of numerical precision, the part of the network after $f$ would need to have a {\tmem{very large}} Lipschitz constant in order to produce well-separated outputs for different training cases.

\subsubsection{$0 \leqslant \fixpointofcmap < 1$: the exploding case}\label{sec:exploding-case}

If $C_f$ sends nearly all input c values to a small neighborhood around $\fixpointofcmap$ with $0 \leqslant \fixpointofcmap < 1$, then any two input vectors (that aren't either almost identical or negations of each other) will be mapped by $f$ to output vectors that are nearly a constant relative distance $d = \sqrt{2 (1 - \fixpointofcmap)} > 0$ apart.

Given this condition, and that $f$ is differentiable, it follows that $f$'s Jacobian must be very large in certain regions of the input space (and possibly everywhere). This means that the gradient magnitudes for layers in the network before $f$ will be much larger than those for other layers, making it difficult to optimize them simultaneously. (This is arguably related to the well-known ``exploding gradient'' phenomenon discussed in \citet{hochreiter2001gradient}.)

A network containing such a subnetwork $f$ possibly stands a greater chance of being trainable than in the previous ``collapsing case", since $f$'s output vectors will still be distinguishable for different inputs vectors. Optimization of the layers after $f$ (or towards the end of $f$) could even conceivably learn a map from these vectors to their associated targets in the training set. However, it is highly unlikely that the resulting model would generalize well, as the similarity between two such output vectors would have no discernible relationship to the similarity of the associated two input vectors. 

\subsection{Are well-behaved C maps sufficient?}\label{sec:suff-cond?}

While a well-behaved (i.e.~non-degenerate) C map seems like a \tmtextit{necessary} condition for trainability (as argued above), it should be noted that without additional hypotheses, no condition on the C map can be a \tmtextit{sufficient}. For example, because the C map is invariant to the network's parameterization\footnote{This can be seen by observing that APKFs, from which local Q/C maps and ultimately extended Q/C maps are defined, only depend on the functional behavior of combined layers, and not which variables are formally considered ``parameters" from the standpoint of the optimizer. Indeed, one could reparameterize a layer's weights using any invertible function without changing what it computes at initialization time, or the form of its PKF/APKF.}, or even whether its parameters are considered {\tmem{trainable}} at all, it cannot completely predict the performance of a gradient-based optimizer. Even if we assume the standard network parameterization, and a model class which is equivalent to a standard deep nonlinear network, interesting counterexamples to sufficiency still exist, as we will show in Section \ref{sec:too-linear}.

One way to incorporate hypotheses about the network's parameterization, and the optimizer used, is to analyze gradient descent training from the perspective of Neural Tangent Kernel (NTK) theory \citep{jacot2018neural}. Later in Section \ref{sec:NTK-analysis} we will argue that for a deep fully-connected network, a degenerate C map leads to a form for the NTK which implies very poor generalization and/or slow optimization under gradient descent. Conversely, we will also show that a network constructed using DKS has an NTK which is suggestive of good generalization and fast optimization (although doesn't necessarily guarantee it). This NTK-based analysis can be viewed as a rigorous version of the intuitive argument given in the previous subsection.

\section{Mathematical analysis of C maps}\label{sec:Cmap-analysis}

As we saw in Section \ref{sec:Cmaps_trainability}, C maps can become degenerate in deep networks by mapping nearly all input c values to a point-like region around some $\fixpointofcmap \in [0,1]$, which leads to difficulty when training with standard optimization methods. In this section we will analyze C maps in closer detail, and show how their overall deviation from the identity function (which serves as a measure of degeneration) can be predicted from their slopes at $c = 0$ and/or $c = 1$. (We will ultimately design DKS to control these slopes in order to prevent degeneration.) We will also establish connections between the properties of a local C map and its associated activation function, and characterize the slope behavior of degenerate C maps over $[-1, 1]$.

\subsection{Measures of deviation from the identity}

Suppose $f$ is some subnetwork. The question of how to measure the deviation of $C_f$ from the identity function, which we will use as a measure of ``degeneracy'', is an interesting one. Since we ultimately want to forbid extreme behavior of $C_f$ for all input c values, it makes sense to look at the worst-case ones. This suggests the following two options, which compare $C_f$ to the identity function using either its values or its derivatives:
\begin{enumeratenumeric}
  \item $\max_{c \in [- 1, 1]} | c - C_f (c) |$ and
  
  \item $\max_{c \in [- 1, 1]} | 1 - C'_f (c) |$.
\end{enumeratenumeric}
The first of these, while a reasonable choice, could fail to detect small ranges of c values where geometric information is lost due to the slope of $C_f$ getting close to zero. The second option will detect such regions, but unlike the first measure, is insensitive to $C_f$ being shifted by an additive constant\sam{Since $C_f(1) = 1$, wouldn't a shifted C-map have to have some curvature that would get picked up by option 2? I see that this comes into theorem 15, but maybe amend the wording?}, and is only weakly sensitive to it being ``angled'' away from the identity function. Fortunately, we can avoid having to choose between the two, since as we will see next, they can be simultaneously bounded using a few easily-computed properties of $C_f$.

\subsection{Bounding deviation from the identity} \label{sec:bounding-deviation-from-id}

The following result relates the deviation of $C_f$ from the identity function to the values of $C_f (0)$, $C'_f (0)$, and $C'_f (1)$. Its proof, which makes strong use of the fact that $C_f$ is a positive definite function, is given in Appendix \ref{app:thm-dev-proof}.

\begin{theorem}
  \label{thm:deviation-bound}For any subnetwork $f$ we have
  \begin{eqnarray*}
    \frac{1}{4}  (1 - C'_f (0)) \hspace{0.8em} \leqslant & \max_{c \in [- 1, 1]} | C_f (c) - c | & \leqslant \hspace{0.8em} 2 (1 - C'_f (0))
  \end{eqnarray*}
  and
  \[ \max_{c \in [- 1, 1]} | C'_f (c) - 1 | \hspace{0.8em} \leqslant \hspace{0.8em} 2 (1 - C'_f (0)) + (C_f' (1) - 1) . \]
  If $C_f (0) = 0$, then we additionally have that
  \[ \max_{c \in [- 1, 1]} | C_f (c) - c | \hspace{0.8em} \leqslant \hspace{0.8em} 2 (C_f' (1) - 1) \]
  and
  \[ \max_{c \in [- 1, 1]} | C'_f (c) - 1 | \hspace{0.8em} \leqslant \hspace{0.8em} 3 (C_f' (1) - 1) . \]
\end{theorem}

From this result we see that the first measure of deviation is within a factor 4 of $1 - C'_f (0)$ (and also upper bounded by $2 [C_f' (1) - 1]$ when $C_f (0) = 0$), and the second measure is within a factor 3 of $C_f' (1) - 1$ when $C_f (0) = 0$. This suggests that we can control the deviation of $C_f$ from the identity function by simply controlling the distance of $C_f' (0)$ and/or $C_f' (1)$ from 1. 

\begin{remark}
  \label{rem:simple-Cmap-bounds}Note that a simple consequence of this result is that $C_f' (1) \geqslant 1$ for subnetworks $f$ satisfying $C_f (0) = 0$. It is also true more generally that $C'_f (0) \leqslant 1$ (as shown in Appendix \ref{app:thm-dev-proof}).
\end{remark}

\begin{remark}
  The lower bounds in Theorem \ref{thm:deviation-bound} do not imply that $C_f$ will look globally nonlinear when $C_f' (1)$ is large (in contrast to the theorem's upper bounds which \emph{do} imply it will look globally linear when $C_f' (1)$ is small). For example, the function $\left( 1 - 1 / \sqrt{j} \right) c + \left( 1 / \sqrt{j} \right) c^j$ is a valid C map and has a derivative $\geqslant \sqrt{j}$ at $c = 1$, and yet is very close to linear everywhere except near $c = 1$ and $c = -1$ when $j$ is large. 
\end{remark}

\subsection{The relationship between identity C maps and linear activations }\label{sec:ident-Cmap-lin-act}

The local C map for a nonlinear layer with a linear activation function is the identity (which can be easily verified from Equation \ref{eqn:C-map} by taking $\phi (u) = \lambda u$ with $\lambda \neq 0$). However, from this observation it's not immediately obvious that a local C map will converge to the identity function as its associated activation function becomes ``more linear'', or vice versa. In this subsection we will show that this is indeed the case for certain carefully chosen measures on both function spaces, and we will give the rate of this convergence.

We will assume that the activation functions live in a Hilbert $H$ space with inner product given by $\langle \phi, \psi \rangle_H =\mathbb{E}_{x \sim \mathcal{N} (0, 1)} [\phi (x) \psi (x)]$, whose associated norm is $\| \phi \|_H = \sqrt{\langle \phi, \phi \rangle}$. (By definition, elements of this Hilbert space are those functions $\phi$ for which $\| \phi \|_H$ exists.).

The standard measure on $H$, defined by
\[ \| \phi - \psi \|_H = \sqrt{\mathbb{E}_{x \sim \mathcal{N} (0, 1)} [(\phi (x) - \psi (x))^2]}, \]
is arguably the most natural one to use, as it employs a weighting over $x$ that precisely reflects the input distribution we would expect given an input q value of 1. This measure is also closely related to Q/C maps in the sense that $\| \phi \|_H^2 = \Gamma_{\phi} (1, 1, 1) = Q_f (1)$ (with $\Gamma_{\phi}$ defined as in Equation \ref{eqn:Gamma-def}), where $f$ is a nonlinear/combined layer with $\phi$ as its activation function.

$\phi$ is linear insofar as it's close to a multiple of the identity function $h_1$. With this in mind, and given the fact that $\| h_1 \|_H = 1$, we will measure the level of nonlinearity of $\phi$ according to
\[ \tmop{nl} (\phi) \equiv \frac{\| \phi - \langle \phi, h_1 \rangle_H h_1 \|_H}{\| \phi \|_H}, \]
where we normalize by $\| \phi \|_H$ to keep $\tmop{nl} (\phi)$ invariant to changes in the overall scale of $\phi$. Note that with this definition, $\phi$ is perfectly linear (i.e.~a multiple of $h_1$) if and only if $\tmop{nl} (\phi) = 0$.

It turns out that we can relate $\tmop{nl} (\phi)$ to properties of $C_f$, as is established in the following proposition whose proof is given in Appendix \ref{app:prop-nonlin-proof}:
\begin{proposition}
  \label{prop:nonlinear-measure}Suppose $f$ is a nonlinear/combined layer with activation function $\phi$ and input q value 1. Then we have
  \[ \tmop{nl} (\phi)^2 = 1 - C'_f (0) . \]
  Moreover,
  \[ \begin{array}{lll}
       \frac{1}{4} \tmop{nl} (\phi)^2 \hspace{0.8em} \leqslant & \max_{c \in [- 1, 1]} | C_f (c) - c | & \leqslant \hspace{0.8em} 2 \tmop{nl} (\phi)^2 .
     \end{array} \]
\end{proposition}

This result shows a strong relationship between the level of linearity of $\phi$, and the distance between $C_f$ and the identity function (as measured by the infinity norm). Moreover, one converges to 0 (as we vary $\phi$) if and only if the other one does.

\begin{remark}
  Proposition \ref{prop:nonlinear-measure} can be straightforwardly extended to the case of general input $q$ values by modifying the definition of the inner product used for $H$.
\end{remark}

\subsection{The relationship between affine C maps and affine activations}\label{sec:Cmap-ratio-affine-act}

An affine function is, by definition, a linear function plus a constant term. Equivalently, it is a function whose derivative is constant (i.e.~is a multiple of the constant function $h_0 (x) = 1$). From this second characterization, we can measure the ``non-affineness'' of $\phi$ by
\[ \tmop{na} (\phi) \equiv \frac{\| \phi' - \langle \phi', h_0 \rangle_H h_0 \|_H}{\| \phi' \|_H} . \]
Note that with this definition, $\phi$ is perfectly affine if and only if $\tmop{nl} (\phi) = 0$.

It turns out that we can relate $\tmop{na} (\phi)$ to properties of $C_f$, as is established in the following proposition (whose proof is given in Appendix \ref{app:prop-nonaff-proof}).

\begin{proposition}
  \label{prop:nonaffine-measure}Suppose $f$ is a nonlinear/combined layer with $\phi$ as its activation function. Then
  \[ \tmop{na} (\phi)^2 = 1 - \frac{C'_f (0)}{C'_f (1)} . \]
\end{proposition}

From the above expression we can see that $\phi$ becomes more affine as the ratio $C'_f (0) / C'_f (1)$ approaches 1. Moreover, $\phi$ approaches an affine function if and only if $C_f$ itself does, as $C'_f (0) / C'_f (1)$ is a measure of how affine $C_f$ is. (To see this, note that $C'_f (0) \leqslant C'_f (c) \leqslant C'_f (1)$ for all $c \in [0, 1]$ since $C_f$ is convex on $[0, 1]$ by Section \ref{sec:pd-functions}, and thus $C_f'$ approaches a constant function on $[0, 1]$ as $C'_f (0) / C'_f (1) \rightarrow 1$, which extends to all of $[- 1, 1]$ since $C_f$ is analytic.)

\subsection{Slope properties of degenerate C maps}\label{sec:degen-C-map-deriv}

In subsection \ref{sec:bounding-deviation-from-id} we saw how certain conditions on the slope of a C map at $c=0$ and/or $c=1$ ensure that it is well behaved (i.e.~not degenerate). In this subsection we will establish the converse: that degenerate C maps necessarily have extreme values for these slopes. 

As shown in Section \ref{sec:Cmaps_trainability}, C maps in deep networks can become degenerate in the sense that they map almost their entire input domain (except points very close to $\pm 1$) to a small region around some limiting c value $\fixpointofcmap$. One way to quantify this behavior is to look at how ``flat'' the function is up to some c value $c$ s.t.~$| c | < 1$, which we can measure using
\[ F_f (c) \equiv \frac{C_f (| c |) - C_f (0)}{| c |} \geqslant 0 \]
for $c \neq 0$. When $C_f$ is degenerate, or in other words very flat, $F_f (c)$ will be very small (for values of $c$ not too close to $\pm 1$).

While the interpretation of $F_f (c)$ is clear for $c > 0$, it is less clear for $c < 0$. In the following proposition (proved in Appendix \ref{app:F-neg-c-interp-proof}) we show that $F_f (c)$ does indeed work as a measure of flatness for values of $c$ less than $0$.

\begin{proposition}
  \label{prop:F-neg-c-interp}Suppose $f$ is a subnetwork, and $c \in [- 1, 1]$ with $c \neq 0$. Then
  \[ \left| \frac{C_f (c) - C_f (0)}{c} \right| \leqslant F_f (c) . \]
\end{proposition}

Note that because $C_f$ is convex and non-decreasing on $[0, 1]$ (by Section \ref{sec:pd-functions}) and $C_f (1) = 1$, we have that $F_f (c') \leqslant F_f (c) \leqslant 1$ for any valid $c, c'$ s.t.~$| c' | \leqslant | c |$. Thus, $F_f (c)$ being small implies flatness over the entire domain $[- c, c]$, and not just at $c$.

A more ``analytic'' way to measure flatness is to look $| C_f' (c) |$, which intuitively should be small in flat regions of $C_f$. It turns out that this intuition is basically correct, as we establish in the following proposition (whose proof is given in Appendix \ref{app:degen-deriv-bound-proof}).

\begin{proposition}
  \label{prop:degen-deriv-bound}Suppose $f$ is a subnetwork, and $c \in (- 1, 1)$ with $c \neq 0$. Then
  \[ | C_f' (c) | \hspace{0.8em} \leqslant \hspace{0.8em} \frac{F_f (c) \log F_f (c)}{| c | \log | c |}  (1 + | c |) . \]
\end{proposition}
This bound establishes that $| C_f' (c) |$ will be small whenever $F_f (c)$ is (which is to say, when $C_f$ is degenerate), provided that $| c |$ is not too close to 1.

\begin{remark}
  \label{rem:degen-deriv-bound}While Proposition \ref{prop:degen-deriv-bound} doesn't address the value of $| C_f' (0) |$ directly, we can still bound it by applying Proposition \ref{prop:degen-deriv-bound} with $c = \epsilon$ for some $0 < \varepsilon < 1$ and then use the fact that $0 \leqslant C_f' (0) \leqslant C_f' (\varepsilon) = | C_f' (\varepsilon) |$ for any $0 < \varepsilon < 1$ (which is true because $C_f$ is non-decreasing and convex on $[0, 1]$ by Section \ref{sec:pd-functions}).
\end{remark}

In general, we cannot say that much about $C_f' (1)$ or $C_f' (- 1)$ when $C_f$ is degenerate. However, the following two propositions (proved in Appendices \ref{app:degen-C-map-deriv-proof} and \ref{app:degen-C-map-deriv-2-proof}) give us some basic information about these values in certain special cases.

\begin{proposition}
  \label{prop:degen-C-map-deriv}Suppose $f$ is a composition of $D$ subnetworks each having the C map $\CC$, and that $\fixpointofcmap = 1$. Then we have that $C'_f (1) = \mathcal{C}' (1)^D$ with $0 \leqslant \mathcal{C}' (1) \leqslant 1$, and either $C'_f (- 1) = - C'_f (1)$ or $\lim_{D \rightarrow \infty} C'_f (- 1) = 0$.
\end{proposition}

\begin{remark}
  Note that the claim made in Proposition \ref{prop:degen-C-map-deriv} does not hold for more general types of networks. For example, if we have a sequence of networks $(f_n)_{n = 1}^{\infty}$ such that $C_{f_n} (c) = 1 - 1 / n + c^{n^2} / n$, then for all $c \in [- 1, 1]$ we have $C_{f_n} (c) \rightarrow 1$, $F_{f_n} (c) \rightarrow 0$ and $C'_{f_n} (1) = n \rightarrow \infty$ as $n \rightarrow \infty$.
\end{remark}

\begin{remark}
  For RELU combined layers we have $\mathcal{C}' (1) = 1$ (which follows from Equation \ref{eqn:C-map-RELU} by taking the derivative and letting $c \rightarrow 1$), and thus the bound $\mathcal{C}' (1) \leqslant 1$ in Proposition \ref{prop:degen-C-map-deriv} is tight. 
\end{remark}

\begin{proposition}
  \label{prop:degen-C-map-deriv-2}Suppose $f$ is a subnetwork. For all $0 < \epsilon < 1$ we have
  \[ C_f' (1) \hspace{0.8em} \geqslant \hspace{0.8em} \frac{1 - C_f (0) - F_f (1 - \epsilon)}{\epsilon} . \]
\end{proposition}

By taking a small value for $\epsilon$, Proposition \ref{prop:degen-C-map-deriv-2} tells us that for degenerate C maps with $\fixpointofcmap < 1$ (so that $C_f (0) \approx \fixpointofcmap < 1$), $C_f' (1)$ will be large provided that the flatness measure $F_f (1 - \epsilon)$ is small. See Section \ref{sec:erf-network-degen} for an example of degenerate C map where $C_f' (1)$ is indeed very large.

\section{C map behavior in linear networks and the problem of being ``too linear"}\label{sec:linear-networks}

\subsection{Linear networks have identity C maps and are easy to train}

In Section \ref{sec:Cmaps_trainability} we saw that deep nonlinear networks can easily have degenerate C maps, which makes them very hard to train with gradient-based methods. One might wonder if this pathology is reserved to nonlinear networks, or if deep {\tmem{linear}} networks\footnote{Here, a linear network is defined as one whose activation functions are a constant multiple of the identity function.} also suffer from it.

Given our assumption that the initial biases are zero, it turns out that the answer is no. The local C map for a combined layer with a non-zero linear activation function is equal to the identity. Because identity functions are preserved under composition and weighted averages, it thus follows that the extended C map for any subnetwork is also the identity function, and is therefore well-behaved.

Does this mean that linear networks are easy to train? Well, as discussed in Section \ref{sec:suff-cond?}, more hypotheses are required to say anything about trainability. But if one adopts the standard parameterization, very deep linear networks are surprisingly easy to train both in theory and practice using standard techniques \citep{saxe2014exact}, provided that they are initialized using orthogonal weights. Linear networks thus represent an interesting example of where our necessary condition for trainability (i.e.~having a well-behaved C map) also appears to be sufficient.

Despite how easy they are to train, we obviously can't use linear networks in practice, as their expressivity is fundamentally limited. But these observations do suggest a possible strategy to address problem of degenerate C maps in \tmtextit{nonlinear} networks: we can transform the network's activation functions so that they appear ``sufficiently linear" at initialization time. However, as we will see in the next subsections, going overboard on this idea will lead to a special type of untrainability that exists only in networks with very well-behaved C maps.

\subsection{The problem of being ``too linear''}\label{sec:too-linear}

As suggested in the previous subsection, one way to achieve a well-behaved C map would be to transform the activation functions in a network to resemble the identity function (or a multiple thereof). We can do this for the RELU activation function (defined by $\tmop{RELU} (u) = \max (0, u)$) by adding a large constant $a$ to its input and subtracting the same constant from its output. In other words, we set
\[ \phi (u) = \tmop{RELU} (u + a) - a = \max (0, u + a) - a, \]
which is equivalent to the identity function for all inputs $u \geqslant - a$. If $a$ is extremely large, say $10^{100}$, this means that all practically sized inputs $\phi$ will satisfy this constraint, and thus the function can be treated as the identity for all practical purposes. Moreover, the expectation formulas for local Q and C maps given in Section \ref{sec:QC_map_combined} will produce practically identical results to the identity function case, since the probability mass associated with inputs $u < - a$ to $\phi$ will be vanishingly small. (See Section \ref{sec:ident-Cmap-lin-act} for a formal justification of this.) Thus, the C map for networks consisting of the composition of many combined layers with these transformed RELUs will be equal to the identity function up to a vanishingly small approximation error.

Meanwhile, because nonlinear layers are always preceded and followed by linear layers with learnable biases, the network can in principle learn to undo these transformations and thus simulate a standard deep RELU network. The model class is thus {\tmem{technically}} no different from a standard RELU network, assuming a perfect optimizer. But despite this, these transformed networks will never actually learn nonlinear behavior via standard gradient-based methods in a reasonable amount of time, and so their hypothetical expressive power will fail to be properly utilized. Indeed, unless the optimizer manages to change the parameters by a factor on the order of $10^{100}$, the network will behave nearly identically to the corresponding linear network (which computes only affine functions of its input) throughout the entire course of optimization, both in terms of its function values and its gradient/curvature estimates.

The basic problem here is that the transformed network has become ``too linear'' in the sense that we require a very large change in its parameters to see any significant nonlinear behavior. While such a network may be readily trainable within the class of linear functions (as linear networks are), it will be severely limited compared to a standard RELU network in terms of its effective expressive power under gradient descent optimization. Thus, to achieve trainable networks that are also expressive, one must avoid this failure mode in addition to requiring a well-behaved C map.

\subsection{How to avoid networks that are ``too linear''}\label{sec:avoid-too-linear}

If a subnetwork $f$ has an C map which is very close to the identity function, this will usually mean the local C maps of its nonlinear layers are also very close to the identity function (and perhaps much more so). By Section \ref{sec:ident-Cmap-lin-act}, this implies that the activation functions must therefore be very close to linear, so that the network is at risk of being ``too linear'' (as defined above). We may thus hope to prevent this by insisting that the network's C map isn't too close to the identity function. However, this alone won't be good enough, as shown in the following example.

Consider modifying the transformed RELU activations in the previous example by adding 1 to their output and dividing the result by $\sqrt{2}$, so that they compute $\phi (x) = (x + 1) / \sqrt{2}$ over their high-probability range of inputs (which is an affine function of $x$). With this change, gradient-based learning will still be effectively restricted to the class of linear networks (which compute affine functions). Meanwhile, a straightforward calculation via Equations \ref{eqn:Q-map} and \ref{eqn:C-map} shows that $Q_f (q) \approx q$ and $C_f (c) \approx (c + 1) / 2$ for nonlinear/combined layers $f$ with activation function $\phi$, and so $C_f (c)$ differs significantly from the identity.

Fortunately, by leveraging the previous analysis, there is a simple way we can use C map properties to provably avoid networks that are ``too linear''. From Section \ref{sec:Cmap-ratio-affine-act}, the degree of ``non-affineness'' of $\phi$, denoted $\tmop{na} (\phi)$, is given by
\[ \tmop{na} (\phi)^2 = 1 - \frac{C'_f (0)}{C'_f (1)} . \]
As we have $C'_f (0) \leqslant 1$ by Remark \ref{rem:simple-Cmap-bounds}, it thus follows that
\[ \tmop{na} (\phi)^2 \geqslant 1 - \frac{1}{C'_f (1)} . \]
This shows that we can avoid activation functions that are too affine (which is sufficient to avoid networks that are ``too linear'') by requiring that $C'_f (1)$ be sufficiently greater than 1 for every nonlinear layer $f$.

\section{Mitigating kernel approximation error}\label{sec:error-and-Qmaps}

Our analysis of the initialization-time behavior of deep neural networks via Q/C maps relies on the assumption that the APKF approximation, when applied over multiple layers in a nested fashion, is a reasonable one to make. If this isn't true, then Q/C maps will cease to describe the network's PKF at initialization time, and our attempts to make the network trainable by controlling their properties will be doomed to failure.

In Section \ref{sec:how-accurate-approx} we discussed the error bounds from \citet{daniely2016toward} in order to help justify our use of nested APKF approximations in deep networks. These bounds make high-probability statements about the error of initialization-time kernel approximations of neural networks, and give a maximum value which shrinks with the square root of the width and grows exponentially with depth. While they represent the best rigorous account of neural network kernel approximations, they are still too pessimistic to be useful in practical settings, either for predicting the approximation error or controlling it.

In this section we will propose a heuristic way of looking at how approximation error originates and evolves across multiple layers which we have found to be quite predictive in practice, and which implies certain error mitigation strategies that we can incorporate into DKS.

\subsection{Minimizing propagation of errors by controlling Q map derivatives}

Suppose $f$ is a subnetwork consisting of a composition of many combined layers. A perturbation to the input q value to $Q_f$, representing the error from approximations made at previous layers, will manifest as a perturbation of $Q_f$'s output. Up to first order, the size of the latter will be approximated by that of the former, multiplied by $Q_f$'s derivative.

As discussed in Section \ref{sec:derivative_extended_maps}, the derivative of $Q_f$ is equal to the product of the derivatives for its constituent local Q maps (i.e.~those for each of $f$'s combined layers), and thus will grow or shrink in an exponential fashion as a function of the depth. Thus, it can easily be the case that deep networks will have very large Q map derivatives, which suggests a very large amplification of error through successive layers network. One way we can avoid this issue is by requiring that derivative of each local Q map, when evaluated at its expected input, is less than or equal to 1.

A closely related perspective, which applies to networks consisting of a sequence of combined layers each with local Q map given by $\mathcal{Q}$, is that a fixed point $q^*$ of $\mathcal{Q}$ will be attractive if $\mathcal{Q}'(q^*) < 0$. Attractive fixed points have the property that $\mathcal{Q}^D(q)$ will converge to $q^*$ as $D \to \infty$ for all values of $q$ sufficiently close to $q^*$, and are thus naturally ``robust" to reasonably sized errors in $q$.

See Appendix \ref{app:Q-map-error} for empirical evidence of the relationship between Q map derivatives and kernel approximation error.


\subsection{Minimizing errors using large width and SUO-distributed weights}

In addition to controlling the propagation of errors across layers, another way to mitigate error is to increase quality of each of the layer-wise APKF approximations from which the errors first originate. In the case of Gaussian-distributed weights, APKFs use analytic expectations to approximate finite averages (over unit outputs), where each element being averaged is an iid unbiased estimator of the expectation. Increasing the width/channel dimension $m$ will thus reduce the variance of the overall estimator, and thus reduce error (as predicted by Theorems \ref{thm:error-bound-daniely} and \ref{thm:SUO-kernel-approx}).

While not originally conceptualized as such, the use of SUO-distributed weights (as defined in Section \ref{sec:param_dist}) provides another way to mitigate the kernel approximation error that is essentially free. When using the SUO distribution, the weights are no longer statistically independent or Gaussian distributed, and so the unit outputs being averaged across are neither iid nor unbiased estimators of the kernel approximation formula. Nonetheless, their average is a {\tmem{consistent}} (but biased) estimator, whose variance goes to zero as $m$ increases (as is established in Theorem \ref{thm:SUO-kernel-approx}).

It is well-known that biased estimators can sometimes have lower variance than unbiased ones, and this does seem to be the case here. Recall the discussion from the end of Section \ref{sec:random-ortho-discussion}, where it was observed that the distribution of an output vector from a linear layer is identical for the Gaussian and SUO-distributed cases, except that the former introduces a random multiplicative perturbation (with mean 1 and variance $1 / m$) on the vector's dimension-normalized squared length (which is an estimator of the associated q value). While this perturbation is required for the implied estimator (after the nonlinearity) to be unbiased, it leads to additional variance. We conjecture that this extra variance is more significant than the bias, and thus SUO-distributed weights yield an overall lower approximation error for APKFs.

Note that the upper bounds given in Theorems \ref{thm:error-bound-daniely} and Theorem \ref{thm:SUO-kernel-approx} do not reflect these intuitions, as they suggest an overall lower approximation error for Gaussian-distributed weights. However, because these are only upper bounds without matching lower bounds, and are likely quite loose/pessimistic, one cannot draw any conclusions. Indeed, we conjecture that tighter bounds could be obtained for SUO-distributed weights given a more careful analysis than the one in \citet{martens2021validity}.

\part{Specification and derivation of Deep Kernel Shaping}
\label{part:spec_and_deriv_of_DKS}

\section{Conditions on Q/C maps that we will enforce}\label{sec:QCmap-conditions}

Having established a detailed understanding of Q and C maps, and how their properties relate to network trainability, we are now in a position to state and justify the specific conditions which we will attempt to enforce with DKS. Our particular mechanism for doing this, which involves certain transformations of the network's activation functions, will be described in later sections, and isn't important for the present discussion.

In the follow series of subsections we describe each of the conditions. Note that because we want the entire network's capacity to be utilized, and not just the subnetwork corresponding to the most direct input-output path, we will enforce these conditions for {\tmem{all}} subnetworks of the network.

\subsection{$Q_f (1) = 1$ for all subnetworks $f$}

The uniform q condition ensures that the q values for a given layer are independent of location (in the feature map) and the network's input. While this alone is a sufficient hypothesis to derive an approach similar to DKS, we will go a step further and standardize to a q value of 1, which will allow us to reuse local Q/C map computations across the entire network. Note that given our use of PLN (which ensures initial q values are 1), this is equivalent to the condition that $Q_f (1) = 1$ for all subnetworks $f$.

The choice to standardize to a q value of $1$ (as opposed to some other positive constant) is somewhat arbitrary and not particularly important. 2 would have worked equally well, for example. However, the choice of 1 does lead to somewhat simpler expressions for the local Q/C map and their derivatives, and corresponds to an output scale which is in the range of ``interesting behavior'' for most typical loss functions (such as the commonly used softmax cross-entropy error).

\subsection{$Q'_f (1) = 1$ for all subnetworks $f$}\label{sec:qslope-condition}

As discussed in Section \ref{sec:error-and-Qmaps}, the size of the error in our kernel approximations can be roughly predicted from the size of the derivatives of the Q maps. Thus, in order for Q/C maps to be an accurate description of the network's true kernel function, we must keep the size of these derivatives under control. To this end, we will require that $Q'_f (1) = 1$ for all subnetworks $f$.

We look at the derivative at $q = 1$ in particular since this is the input q value we expect in the absence of error, due to the condition $Q_f (1) = 1$. In principle, we also care about the value of $Q_f' (q)$ for $q$'s close to 1, which would give us a more complete picture of the approximation error (and perhaps let us to establish rigorous bounds on it). Unfortunately, we don't yet have a powerful theory for the global properties Q maps like we do for C maps, and so the best we can do is to look at their properties at particular points. That being said, because we know that Q maps are smooth, it's likely that $Q_f' (q)$ will be reasonably well approximated by $Q_f' (1)$ for values of $q$ close to 1.

The choice to make $Q_f' (1)$ equal 1, which corresponds to the error neither growing nor shrinking, is somewhat arbitrary, and other choices for this value are possible. For example, we could minimize $Q'_f (1)$ instead of setting it to 1, thus suppressing the growth of errors as much as possible. Minimizing $Q'_f (1)$ did seem to work in our experiments, however we found that for certain activation functions (such as $\tanh$) it resulted in slower optimization compared to using $Q'_f (1) = 1$. (See Appendix \ref{app:min-qslope-experiment} for these experiments.)

We are currently not sure why setting $Q'_f (1) = 1$ works better than minimizing it for some activation functions. One possible explanation is that $Q'_f (1) = 1$ allows $f$ to transmit information about the overall scale of its input vector as a roughly linear function of $q$. Meanwhile, networks where $Q'_f (1)$ is minimized will tend to ``squash'' the range around $q = 1$, making it harder to recover the original input $q$ value from the network's output. One can perhaps draw a rough analogy between this and the preservation of ``geometric information" by C maps as discussed in Section \ref{sec:Cmaps_trainability}.

\subsection{$C_f (0) = 0$ for all subnetworks $f$}

In order to apply the analysis of Section \ref{sec:Cmap-analysis} we require that $C_f (0) = 0$ for all subnetworks $f$. While this might seem like an overly stringent requirement, it is worth noting that arbitrary deviation of $C_f$ from the identity function is possible if we don't place any restrictions on the value of $C_f (0)$. This is the case even when $C_f' (1) = 1$ is enforced, as can be seen in Section \ref{sec:relu-network-degen} for deep RELU networks.

\subsection{$C'_f (1) \leq \zeta$ for all subnetworks $f$}\label{sec:slope-ub-cond}

As discussed in Section \ref{sec:Cmaps_trainability}, degenerate C maps correspond to networks that are difficult to train with gradient-based methods. Avoiding this degeneration is the central aim of DKS. As argued in Section \ref{sec:Cmap-analysis}, we can do this for a given $C_f$ by bounding its maximum deviation from identity function (which is the canonical non-degenerate C map).

Given the condition $C_f (0) = 0$, Theorem \ref{thm:deviation-bound} says that this deviation is roughly equal to $C_f' (1) - 1$. Thus, we will enforce the condition $C'_f (1) \leq \zeta$ for all subnetworks $f$, where $\zeta > 1$ is a hyper-parameter which we will sometimes refer to as the \tmtextit{\tmtextbf{global slope bound}}. (Note that $C'_f (1) \geqslant 1$ is true automatically as consequence of Theorem \ref{thm:deviation-bound}.) This condition can be thought of as imposing a limit on how ``non-linear'' any given subnetwork is allowed to look. 

\subsection{$\min_f [C'_f (1)]$ is maximized}

Even assuming that our kernel approximations are exact, a well-behaved C map is not a \tmtextit{sufficient} condition for a nonlinear network to be trainable. As discussed in Section \ref{sec:too-linear}, one way such a network can fail to be trainable is if it's too far away in parameter-space from a significantly nonlinear function, or in other words is ``too linear''. In such cases, a gradient-based optimizer will struggle to utilize the full expressive power of the network.

For neural networks with standard parameterizations, this issue will manifest as nonlinear layers with activation functions that behave too much like affine functions. As argued in Section \ref{sec:avoid-too-linear}, this can be avoided by requiring that $C'_f (1)$ be sufficiently larger than 1 for such layers $f$. Thus, it makes sense to balance the condition in Subsection \ref{sec:slope-ub-cond} with one requiring that $\min_f [C'_f (1)]$ is maximized, where the minimum is taken over all nonlinear layers $f$ in the network.

\subsection{Choosing the global slope bound $\zeta$}\label{sec:choosing-zeta}

Given the above two conditions, the global slope bound $\zeta$ corresponds to the maximum value of $C'_f (1)$ over all subnetworks (and must be $\geqslant 1$). Heuristically, the degree to which $\zeta$ is greater than 1 tells us how nonlinear the network's functional mapping is at initialization time. If $\zeta$ is too large, then the C map for the network (or one of its subnetworks) will experience the ``exploding'' type of degeneration discussed in Section \ref{sec:exploding-case}, where c values are squashed towards $c_0 = 0$. If it's too close to 1, then the C map will be very close to the identity, and we run the risk of making the network ``too linear'' (as per Section \ref{sec:too-linear}).

In our experiments on 100 layer networks we tried only a few values of $\zeta$ before settling on $\zeta = 1.5$, and in general we found that DKS is reasonably robust to significant variations in $\zeta$ (or more precisely, $\log(\zeta-1)$). For example, $\zeta = 1.01$ and $\zeta = 100$ both produced depth 100 networks that trained at competitive speeds, being only somewhat outperformed by networks that used $\zeta = 1.5$. More extreme choices like $\zeta = 1.001$ and $\zeta = 10000$ meanwhile produced significantly slower training. See Appendix \ref{app:sweep-zeta-1} for these results. For depths 200 or greater we found that it was sometimes necessary to use a value of $\zeta$ less than 1.5 (such as 1.1) to achieve stable training. We speculate that this is because the kernel approximations underlying our Q/C maps tend to break down at very high depths (given our modest layer widths), but that this can be mitigated by making the network ``more linear".

An interesting systematic trend we observed is that larger $\zeta$ values (up to a certain limit) tended to produce slightly faster optimization, whereas smaller ones led to slightly improved generalization, possibly because this made the inductive bias of the model (plus optimizer) favor a more linear solution. Relevant experimental data is presented in Appendix \ref{app:sweep-zeta-2}.

\section{From global map conditions to local ones}\label{sec:global-to-local}

In this section we will describe how the ``global" map conditions given in the previous section can be achieved by enforcing an equivalent set of conditions on the local Q/C maps of the network. The particular mechanism we will use to enforce these ``local" conditions will be discussed later in Section \ref{sec:activation-transform}.

\subsection{Slope polynomials and maximal slope functions}\label{sec:slope-poly-and-maximal}

Before we can write down the local map conditions we will define a special construction that allows us to relate the slope at 1 of extended C maps to the slope at 1 of local C maps.

Let $f$ be an arbitrary subnetwork. As discussed in Section \ref{sec:derivative_extended_maps}, we may apply automatic differentiation to compute $C'_f (1)$ from the derivatives of the local C maps of $f$'s constituent layers, where composition corresponds to multiplication and weighted averages (due to concatenations or sum operations) correspond to weighted averages (with the same weights). Since c values of 1 always map to 1 (as argued in Section \ref{sec:uniform-q-consequences}), the expression for the derivative will be a polynomial function of the local C map derivatives at $c = 1$.

If we further assume that there is a constant $\psi$ such that $C'_g (1) = \psi$ for each nonlinear layer $g$ in $f$, then we can express $C'_f (1)$ as a polynomial function of $\psi$, as the local maps for all other layers are the identity function. We will call this function the \tmtextbf{\tmtextit{slope polynomial}} of $f$ and denote it by $p_f (\psi)$. Note that since $C'_g (1) \geqslant 1$ whenever $C_g (0) = 0$ (which we are enforcing), we may thus assume $\psi \geqslant 1$ without loss of generality.

Because slope polynomials can be constructed from products and weighted averages of lower degree slope polynomials, and the value of 1 is preserved under multiplication and weighted averages, it follows that $p_f (1) = 1$ for any subnetwork $f$. And since $\psi \mapsto \psi$ and $\psi \mapsto 1$ are positive definition functions of $\psi$ (trivially), and positive definite functions are closed under multiplication and non-negative weighted averaging (as discussed in Section \ref{sec:pd-functions}), it also follows that slope polynomials are positive definition functions, just like C maps. They are thus non-decreasing for $\psi \geqslant 0$, and indeed strictly increasing provided that the subnetwork contains a nonlinear layer. From this it also follows that $p_f (\psi) \geqslant 1$ for all $\psi \geqslant 1$.

As we will be interested in computing the most extreme slope over a network, we will define a related function called the \tmtextit{\tmtextbf{maximal slope function}}, which is given by $\mu (\psi) = \max_f [p_f (\psi)]$, where the maximum is taken over all subnetworks $f$ of the entire network. Because subnetworks with no nonlinear layers won't influence the maximum, and the maximum of a set of strictly increasing functions is strictly increasing, we have that the maximal slope function is strictly increasing provided that the network contains at least one nonlinear layer. And because it is the maximum over a set of continuous functions, the maximal slope functions is also continuous, and therefore invertible, which is a fact we will make use of later.

\subsection{Computing maximal slope functions}\label{sec:comp-max-slope-func}

The number of distinct subnetworks in a network can be very large, and so computing the maximal slope function naively from the definition can be laborious. Fortunately, we can eliminate most of these subnetworks from consideration immediately.

Observe that if a subnetwork is formed by feeding the output of one subnetwork into the input of another, i.e.~$h = f \circ g$, then we have $p_h (\psi) = p_f (\psi) p_g (\psi)$ by the chain rule. And because $p_f (\psi) \geqslant 1$ and $p_g (\psi) \geqslant 1$ for all $\psi \geqslant 1$, it thus follows that $p_h (\psi) \geqslant p_f (\psi)$ and $p_h (\psi) \geqslant p_g (\psi)$ for $\psi \geqslant 1$. Therefore, any subnetwork that is part of another subnetwork in this particular sense can be ignored when computing the maximum. Moreover, without assuming any relationship between $f$, $g$, and $h$, if $p_f (\psi)$ is a factor of $p_h (\psi)$, then $p_h (\psi) / p_f (\psi)$ is also a valid slope polynomial, and therefore $p_h (\psi) \geqslant p_f (\psi)$ for all $\psi \geqslant 1$, thus allowing us to ignore $f$ in the maximum.

Note this does {\tmem{not}} therefore imply that the maximal slope function is always the slope polynomial of the entire network\footnote{For the entire network to even have a slope polynomial requires that it be a valid subnetwork of itself, which is only the case for networks with a singular input and output.}, as not every subnetwork can be related to the entire network in this way. For example, if we have a very deep nonlinear network with $D$ nonlinear layers and a skip connection from the initial input to the final output, so that the final output is $1 / \sqrt{2}$ \ times the initial input plus $1 / \sqrt{2}$ \ times the output of the nonlinear subnetwork, then the slope polynomial for the nonlinear subnetwork is $\psi^D$, while the slope polynomial for the entire network is
\[ \left( \frac{1}{\sqrt{2}} \right)^2 1 + \left( \frac{1}{\sqrt{2}} \right)^2 \psi^D = \frac{1}{2}  (1 + \psi^D), \]
which is strictly smaller than $\psi^D$ for all $\psi \geqslant 1$. (This formula can be derived by following the recipe given in Section \ref{sec:slope-recipes}.) For this network, the maximal slope function is in fact $\psi^D$.

An even more interesting example is the same network, but with additional nonlinear layer added to the end, after the skip connection. The maximal slope function of this network is $\max \{ \psi^D, \psi (1 + \psi^D) / 2 \}$, which cannot be reduced to a polynomial as there are settings of $\psi$ for which either input to the $\max$ is larger.

\subsection{The equivalent local map conditions}\label{sec:local_map_conditions_list}

Having defined the maximal slope function, we are now in a position to derive the \tmtextit{\tmtextbf{equivalent local map conditions}} to the global ones given in Section \ref{sec:QCmap-conditions}.

First, observe that local Q/C maps for affine layers are identity functions, and can essentially be ignored when computing extended Q/C maps. What remains are nonlinear layers and weighted sum operations, and so we will concentrate on these.

If we have that $Q_f (1) = 1$ for all nonlinear layers $f$, then the analogous property automatically holds for all subnetworks that don't contain weighted sums or constant scalar multiplications, as it is clearly preserved under composition and weighted averages (arising due to concatenations). The same reasoning also applies to the condition $C_f (0) = 0$.

While weighted sum operations are constructed from concatenation operations (which are accounted for in the above argument), the construction also introduces scalar multipliers which can affect the q values. To account for this, we must ensure that the output q value of each weighted sum operation is 1. By Equation \ref{eqn:sum-q-formula}, this is equivalent to requiring that the squares of the weights sum to 1, assuming that the inputs to the sum have q values of 1. We will call weighted sums satisfying this condition \tmtextbf{\tmtextit{``normalized sums''}}. Given that a weighted sum is normalized, and that its input q values are 1, it additionally follows (by Equations \ref{eqn:sum-q-formula} and \ref{eqn:sum-c-formula}) that the corresponding output q and c values will be weighted averages of the input q and c values, with weights given by the {\tmem{squares}} of the weights of the sum itself.

To finally achieve $Q_f (1) = 1$ for all subnetworks $f$ we must remove any constant scalar multiplication operations from the network, except those that are part of the above normalized sums. With this done, a simple inductive argument then establishes that $Q_f (1) = 1$ for all subnetworks $f$.

Given constant q values of 1, and weighted sums that are all normalized, we may compute $Q'_g (1)$ using the same slope polynomials used to compute $C'_g (1)$, provided that $Q'_f (1)$ is the same for all nonlinear layers $f$. Thus if we impose the condition $Q'_f (1) = 1$ for all nonlinear layers $f$, it will follow that $Q'_g (1) = p_g (1) = 1$ for all subnetworks $g$.

Finally, maximizing $\min_f [C'_f (1)]$, while requiring that $C'_f (1) \leq \zeta$ for all subnetworks $f$, is equivalent to setting $C'_f (1) = \psi$ for all nonlinear layers $f$ (since a single nonlinear layer is a subnetwork), where $\psi = \mu^{- 1} (\zeta)$ and $\mu^{- 1}$ is the inverse of the maximal slope function for the network (which exists as long as the network has at least one nonlinear layer).

Summarizing, the equivalent local map conditions are:
\begin{enumeratenumeric}
  \item $Q_f (1) = 1$,
  
  \item $Q'_f (1) = 1$,
  
  \item $C_f (0) = 0$, and
  
  \item $C'_f (1) = \mu^{- 1} (\zeta)$,
\end{enumeratenumeric}
for all nonlinear layers $f$, with the additional requirement that all weighted sum operations in the network are normalized (i.e.~that the squares of their weights sum to 1).

\section{Activation function transformations}\label{sec:activation-transform}

In addition to PLN and the use of Delta initializations, our main mechanism of control over the initialization-time behavior of neural networks will be to apply transformations to their activation functions. In particular, we will apply constant scalar multiplications and shifts to both their inputs and outputs. For most typical activation functions this will give us sufficient control over a combined/nonlinear layer's local Q/C maps to enforce the ``equivalent local map conditions" from the previous section.

\subsection{Basic definition}

Suppose $\phi$ is some element-wise activation function in the network. We propose to make the following replacement:
\[ \phi (u) \quad \longrightarrow \quad \hat{\phi} (u) \equiv \gamma (\phi (\alpha u + \beta) + \delta) \]
where $\alpha$, $\beta$, $\gamma$, and $\delta$ are static scalar constants (that we do not train). Note that these constants are the same for all channels and feature map locations within a given layer, but can differ between layers.

\subsection{Equivalent parameters and preservation of the model class}

Provided that each nonlinear layer is both preceded by and followed by an affine layer (which is true for most architectures), this way of transforming the activation functions has the property that it \tmtextit{preserves the model class} of the original network. By this we mean that for any network with transformed activation functions, there exists an equivalent network with untransformed activations that has precisely the same functional behavior. We will call the filter weights and biases of this second network the \tmtextit{\tmtextbf{equivalent parameters}}.

Computing the equivalent parameters is relatively straightforward. Because each nonlinear layer is both preceded by and followed by an affine layer, we can essentially absorb the input/output scale and shift operations into these layers. To be more explicit, in the case of a standard fully-connected layer with weight matrix $W$ and bias vector $b$, $W (\gamma x + \delta \mathbbm{1}) + b$ becomes $W' x + b'$ with $W' = \gamma W$ and $b' = \delta W \mathbbm{1} + b$ (where $\mathbbm{1}$ denotes the vector of ones). Similarly, $\alpha (Wx + b) + \beta \mathbbm{1}$ becomes $W' x + b'$ with $W' = \alpha W$ and $b' = b + \beta \mathbbm{1}$. The construction for convolutional layers is similar, and relies on the fact that the scaling and shifting constants are the same for each location (just as the filter weights and biases are).

\subsection{Our method for transforming activation functions viewed as an initialization scheme} \label{sec:method-as-pure-init}

The existence of equivalent parameters, and their relatively straightforward computation, makes it possible to turn our method for transforming activation functions into an initialization scheme for the network's parameters. One simply computes the constants needed to appropriately transform the activation functions, and then uses them to instead compute the equivalent parameters, starting from a network initialized as per Section \ref{sec:param_dist}. If we view this process as a sampling procedure for the network's parameters, then it corresponds to a distribution with non-trivial correlations between the weights and biases of each affine layer.

Note that while a transformed network and an untransformed network (with equivalent parameters) compute the same function, they correspond to \tmtextit{different parameterizations} of the same model class, and thus may give rise to different optimization dynamics. Stochastic gradient descent for example, is not invariant to reparameterizations of this type, and so we would expect it to behave differently on either network. The K-FAC optimizer \citep{martens2015optimizing} on the other hand is approximately invariant reparameterizations involving affine transformations of layer inputs and outputs \citep{martens2015optimizing, grosse2016kronecker, luk2018coordinate}.

Our experimental results indicate that the transformed networks are easier to optimize with stochastic gradient descent than networks with equivalent parameters. Meanwhile, as predicted by the theory, the optimization performance with K-FAC is roughly the same for both versions. See Appendix \ref{app:equiv-params-experiments} for the relevant results.

\subsection{Achieving local map conditions with activation function transformations}\label{sec:achieving-local-conds}

Suppose $f$ is a nonlinear layer with activation function $\phi$ which we propose to replace by $\hat{\phi} (u) \equiv \gamma (\phi (\alpha u + \beta) + \delta)$. The four equivalent local map conditions (from Section \ref{sec:local_map_conditions_list}) give rise to a system of four nonlinear equations, with the four scalar constants ($\alpha$, $\beta$, $\delta$, and $\gamma$) as its unknowns. In this subsection we will show how to solve for these constants, assuming that a solution exists.

By Equation \ref{eqn:C-at-0}, the condition $C_f (0) = 0$ holds if and only if $\mathbb{E}_{x \sim \mathcal{N} (0, 1)} [ \hat{\phi} (x)] = 0$. Noting that
\[ \mathbb{E}_{x \sim \mathcal{N} (0, 1)} [ \hat{\phi} (x)] = \gamma (\mathbb{E}_{x \sim \mathcal{N} (0, 1)} [\phi (\alpha x + \beta)] + \delta), \]
this becomes equivalent to
\[ \delta = -\mathbb{E}_{x \sim \mathcal{N} (0, 1)} [\phi (\alpha x + \beta)] . \]
Thus $\delta$ is fully determined by $\alpha$ and $\beta$, which eliminates a single degree of freedom.

From Equation \ref{eqn:Q-map}, and basic properties of expectations, we have that
\begin{equation}
  Q_f (1) =\mathbb{E}_{x \sim \mathcal{N} (0, 1)} [ \hat{\phi} (x)^2] = \gamma^2 \mathbb{E}_{x \sim \mathcal{N} (0, 1)} [ (\phi (\alpha x + \beta) + \delta)^2] = \gamma^2 \tmop{Var}_{x \sim \mathcal{N} (0, 1)}  [\phi (\alpha x + \beta) ] . \label{eqn:special-Q1-formula}
\end{equation}
Thus the condition $Q_f (1) = 1$ is equivalent to
\[ \gamma = (\mathbb{E}_{x \sim \mathcal{N} (0, 1)} [ (\phi (\alpha x + \beta) + \delta)^2])^{- \frac{1}{2}} \: = \: \tmop{Var}_{x \sim \mathcal{N} (0, 1)} [\phi (\alpha x + \beta)]^{- \frac{1}{2}} . \]
This fully determines the value of $\gamma$ in terms of the other constants, thus eliminating another degree of freedom.

Given the above solutions for $\gamma$ and $\delta$, which we will treat as functions $\gamma(\alpha, \beta)$ and $\delta(\alpha, \beta)$ of $\alpha$ and $\beta$, it remains to solve for the values of $\alpha$ and $\beta$ which satisfy the final two conditions $Q'_f (1) = 1$ and $C'_f (1) = \mu^{- 1} (\zeta)$. From the fact that $Q_f (1) = 1$ (for our choice of $\gamma$), these two conditions can be written as:
\begin{enumerateroman}
  \item $\mathbb{E}_{x \sim \mathcal{N} (0, 1)} [ \hat{\phi} (x)  \hat{\phi}' (x) x] = 1$, and
  
  \item $\mathbb{E}_{x \sim \mathcal{N} (0, 1)} [\hat{\phi}' (x)^2] = \mu^{- 1} (\zeta)$ ,
\end{enumerateroman}
where the dependence on $\alpha$ and $\beta$ is implicit in $\hat{\phi}(x) = \gamma(\alpha, \beta) (\phi (\alpha x + \beta) + \delta(\alpha, \beta))$ and $\hat{\phi}'(x) = \alpha \gamma(\alpha, \beta) \phi' (\alpha x + \beta)$.

We are not aware of any closed-form solution for this two dimensional system. However, because it's only two dimensional, and the expectations required to evaluate it are one dimensional (including those needed to compute $\delta$ and $\gamma$), we can readily solve it using black-box numerical software, assuming a solution exists. And because the system of equations only depends on the functional form of $\phi$ and no other details about $f$, we only need to solve it once for each distinct activation function in the network. Implementation details are given in Section \ref{sec:implementation-details}.

\subsection{When will solutions exist?}\label{sec:when-solutions-exist}

While we found in our experiments that solutions for $\alpha$ and $\beta$ exist for nearly all commonly used nonlinear activation functions, the popular RELU is a notable exception (which we will examine in the next subsection). Thus, it is worth delving deeper into the question of the existence of these solutions.

Noting that $\mu^{- 1} (\zeta)$ will typically be quite close to 1, if we can show that $\lim_{\alpha \rightarrow 0} \: C_f' (1) = 1$, this will suggest that $C_f' (1) = \mu^{- 1} (\zeta)$ is achievable by choosing a sufficiently small value of $\alpha$. Intuitively speaking, shrinking $\alpha$ allows us to effectively narrow the interval of typical inputs to $\phi$, meaning that $\phi$ starts to resemble an affine function over this interval (since differentiable functions are, by definition, closely approximated by their 1st-order Taylor approximations within any sufficiently small neighborhood). As discussed in Section \ref{sec:Cmap-ratio-affine-act}, this means that $C'_f (0) / C'_f (1) \rightarrow 1$ as $\alpha \to 0$, which in turn implies that $C'_f (1) \rightarrow 1$ (as we have by Remark \ref{rem:simple-Cmap-bounds} that $C'_f (0) \leqslant 1 \leqslant C'_f (1)$ when $C_f (0) = 0$).

The following proposition formalizes this intuition, although is proved (in Appendix \ref{app:alpha-exist-proof}) using a different technique.
\begin{proposition}
  \label{prop:alpha-exist}Let $f$ be a nonlinear layer with transformed activation function $\hat{\phi}$ defined as above, with $\delta$ and $\gamma$ chosen as per Section \ref{sec:achieving-local-conds}. If $\phi' (\beta) \neq 0$ then we have
  \[ \lim_{\alpha \rightarrow 0} \: C_f' (1) = 1 . \]
\end{proposition}

The hypothesis that $\phi' (\beta) \neq 0$ is required here since otherwise $\hat{\phi}$ will tend to the zero function as $\alpha \rightarrow 0$ (whose C map is undefined). Apart from this restriction, there is no obvious requirement on $\beta$ for the condition $C'_f (1) = \mu^{- 1} (\zeta)$ to hold, and indeed in our preliminary tests we found that we could satisfy this for nearly all reasonable choices of $\beta$ for most activation functions. The role of $\beta$ can thus be thought of selecting the position in $\phi$'s graph to ``zoom in on'', and gives us the extra flexibility needed to control the value of $Q'_f (1)$.

\james{can we also explain why this will let us get the value of $Q'_f (1)$ that we want?}

\subsection{The problem with positively homogeneous activation functions}\label{sec:positively-homogeneous-problem}

A positively homogeneous activation function $\phi (u)$ of degree $k$ is one where $\phi (\lambda u) = \lambda^k \phi (u)$ for all non-negative scalars $\lambda$. A well-known example for $k = 1$ is the RELU activation function, which is given by $\phi (u) = \max (u, 0)$.

Due to their defining property, positively homogeneous activation functions yield at most three effective degrees of freedom under our parameterized transformation, instead of the typical four. This can be seen by observing that
\begin{eqnarray*}
  \hat{\phi} (u) = \gamma (\phi (\alpha u + \beta) + \delta) & = & \gamma (\phi (| \alpha |  (\tmop{sign} (\alpha) u + \beta / | \alpha |)) + \delta)\\
  & = & | \alpha |^k \gamma (\phi (\tmop{sign} (\alpha) u + \beta / | \alpha |) + \delta / | \alpha |^k)\\
  & = & \tilde{\gamma}  (\phi (\tmop{sign} (\alpha) u + \tilde{\beta}) + \tilde{\delta}),
\end{eqnarray*}
where we have defined $\tilde{\gamma} = | \alpha |^k \gamma$, $\tilde{\beta} = \beta / | \alpha |$, and $\tilde{\delta} = \delta / | \alpha |^k$. Apart from the sign of $\alpha$, which can take only two discrete values, we effectively have only three free real-valued random variables: $\hat{\gamma}$, $\hat{\beta}$, and $\hat{\delta}$.

Because of this reduction in the degrees of freedom for positively homogeneous activation functions, we can only enforce at most three of our four local map conditions. The only one which is arguably optional is the condition that $Q'_f (1) = 1$ for all combined layers $f$, which corresponds to the equation $\mathbb{E}_{x \sim \mathcal{N} (0, 1)} [ \hat{\phi} (x)  \hat{\phi}' (x) x] = 1$. Thus, in all of our experiments with RELU networks we dropped this condition, and while it did produce networks that trained reasonably well, optimization performance was still slower compared to all other activation functions we tested, at least for skip connection-free networks trained with K-FAC. (See Section \ref{sec:different-act-experiment} for these results.)

\subsection{Examples of transformed activation functions}

In this subsection we will give some examples of transformed activation functions produced by DKS. Our examples will assume a basic feedforward network of 100 combined layers, and a global slope bound $\zeta = 1.5$. We will consider the standard tanh and RELU activation functions, as well as Swish \citep{prajit2017swish}, SELU \citep{klambauer2017self}, and a commonly used smooth substitute for RELU called ``softplus'' (which is given by $\phi (x) = \log (1 + \exp (x))$).

The following table gives the approximate values for the activation function parameters found by DKS:
\begin{center}
  \begin{tabular}{|c|c|c|c|c|}
    \hline
    \textbf{Activation function} & \textbf{$\alpha$ value} & \textbf{$\beta$ value} & \textbf{$\delta$ value} & \textbf{$\gamma$ value} \\
    \hline
    tanh & 0.090438 & -0.56011 & 0.50500 & 14.9025\\
    \hline
    softplus & 0.22802 & 0.40751 & -0.92372 & 7.30325\\
    \hline
    relu & 0.387604 & 1.0000 & -1.0006 & 2.5916\\
    \hline
    swish & 0.12945 & 0.349475 & -0.20889 & 11.50455\\
    \hline
    selu & 0.088294 & -0.25244 & 0.38694 & 8.25434\\
    \hline
  \end{tabular}
\end{center}

In the following plots we compare the default and transformed activation functions over the input interval $[- 10, 10]$ for tanh, softplus, and RELU. Assuming uniform q values of 1, and that the error in our kernel approximations is relatively low, this interval contains all the inputs that our nonlinear units will see at initialization time with overwhelming probability.

\begin{center}
    \resizebox{4.5in}{3in}{\includegraphics{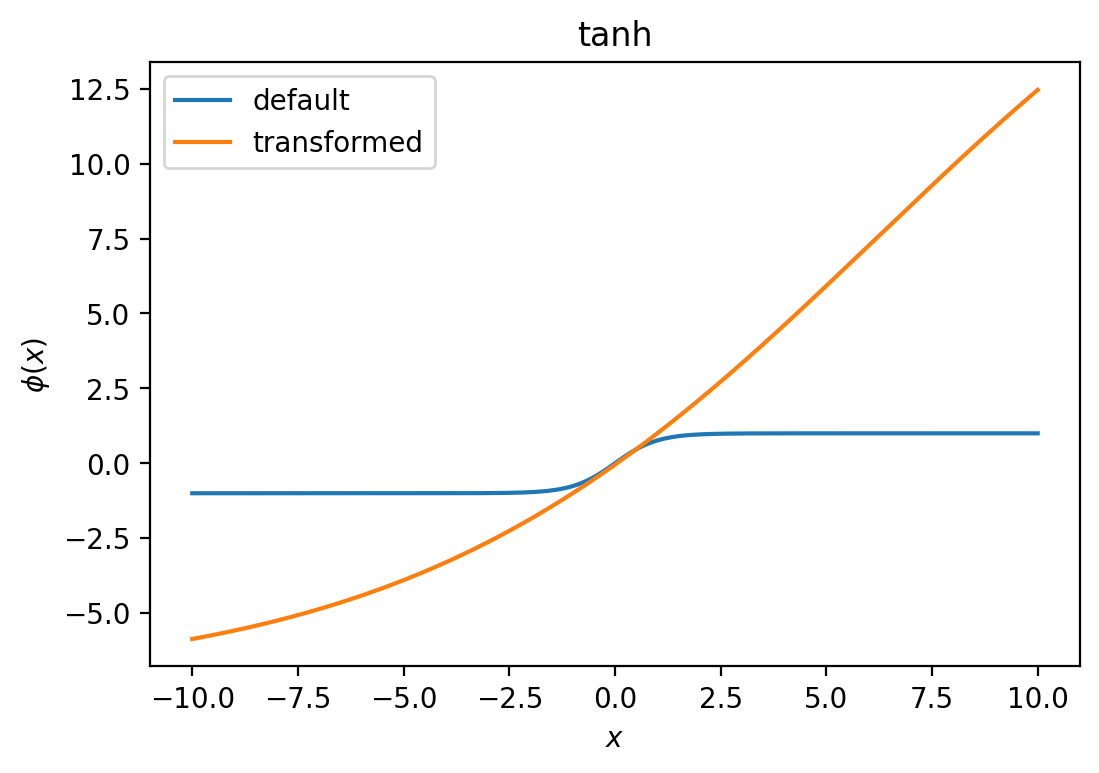}}
    
    \resizebox{4.5in}{3in}{\includegraphics{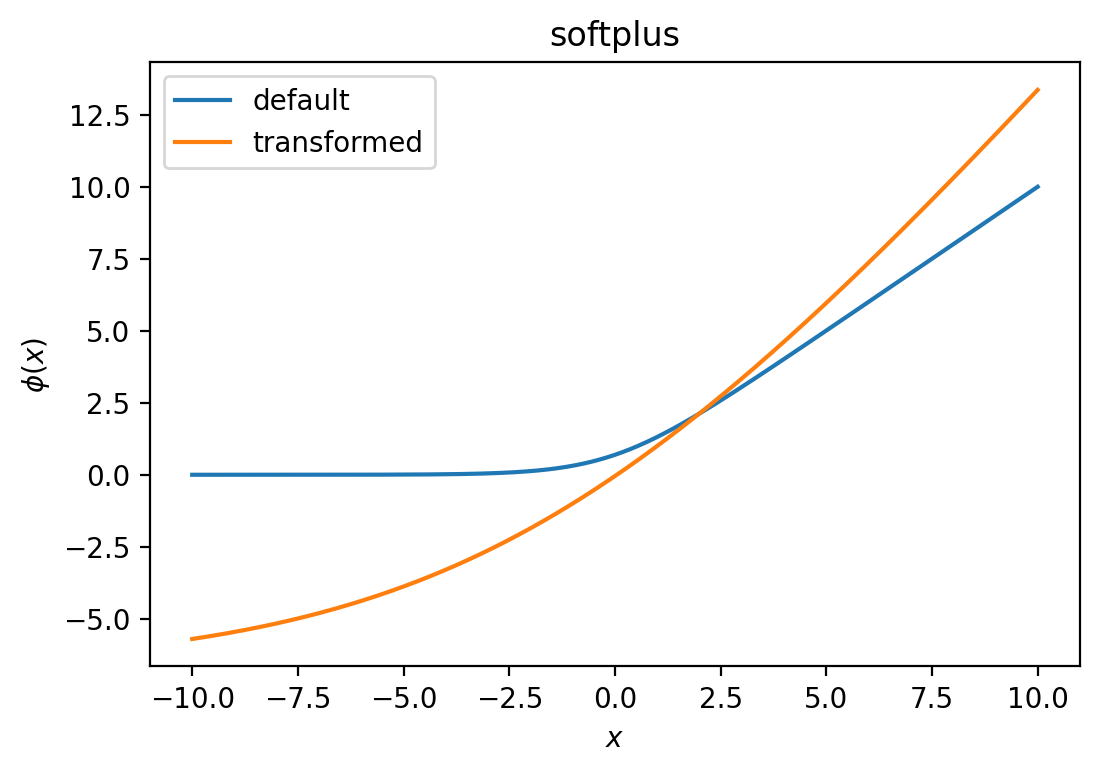}}
    
    \resizebox{4.5in}{3in}{\includegraphics{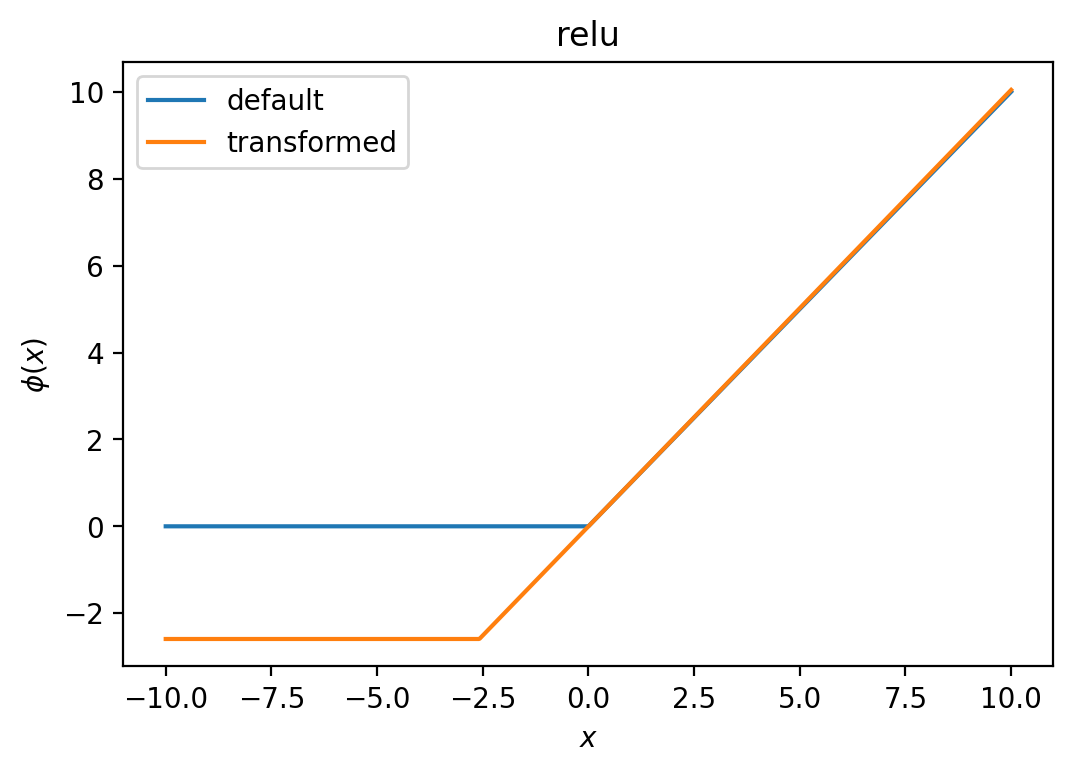}}
\end{center}

We can see from these plots that the transformed activation functions tend to look more like the identity functions than the defaults ones do (over the relevant range of inputs). In fact, they all bare a resemblance to each other (especially tanh, softplus and swish), as can be seen in the following plot:
\begin{center}
\resizebox{4.5in}{3in}{\includegraphics{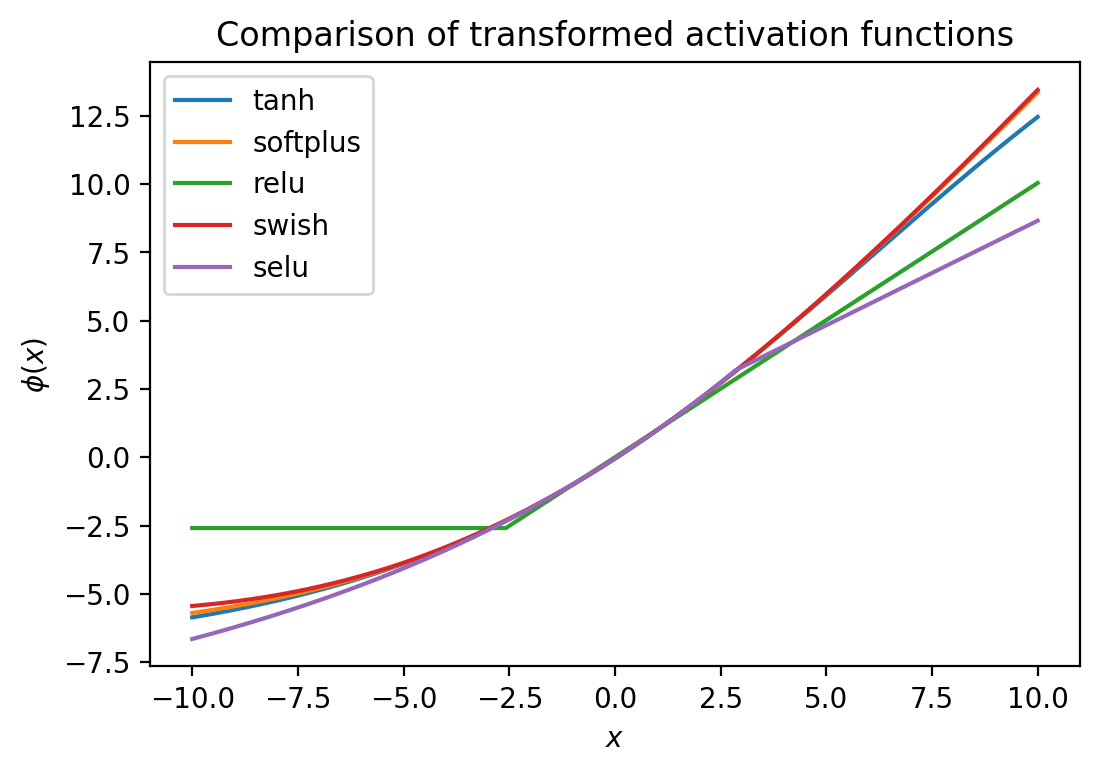}}
\end{center}

\james{Maybe add examples of C maps for these different transformed and untransformed activations?}

\section{Addressing normalization layers}\label{sec:normalization-layers}

\subsection{Batch Normalization layers}\label{sec:batchnorm}

Batch Normalization (BN) layers \citep{ioffe2015batch} are an important component in many neural network architectures, especially convolutional networks. For each unit scalar $u$ in their input, BN layers compute a mean $\mu$ and variance $\sigma^2$ of $u$ over the training mini-batch, and then output a ``normalized'' version $(u - \mu) / \sqrt{\sigma^2 + \epsilon}$, where $\epsilon$ is a small constant. This is this sometimes followed by the application of per-channel learnable bias parameters, which are initialized to zero.

Because they use statistics computed over the mini-batch, BN layers cannot really be described in Q/C map framework we have presented, and are therefore incompatible with DKS. In particular, our formalism assumes that the network's computation for a single training input depends only on that input, and not on other elements of the mini-batch. To account for such interactions, one would have to introduce hypotheses on the size of the mini-batch and the statistical distribution of its vectors, as the behavior of BN layers are highly dependent on these factors. Moreover, the evolution of q and c values would not happen independently across the mini-batch, which would likely preclude a simple one-dimensional description like Q and C maps.

\subsection{Layer Normalization layers}\label{sec:layer-norm}

Layer Normalization (LN) layers \citep{ba2016layer} are a popular ingredient in neural network architectures such as Transformers, and are sometimes used as an alternative to BN layers. For each location vector $z$ in its input feature map, an LN layer computes the scalar mean $\mu = \frac{1}{k}  \mathbbm{1^{\top}} z_i$ and variance $\sigma^2 = \frac{1}{k}  \| z - \mu \mathbbm{1} \|^2$ over the $k$ entries of $z$ (where $\mathbbm{1}$ denotes the vector of 1's), and outputs a ``normalized'' version $(z - \mu \mathbbm{1}) / \sqrt{\sigma^2 + \epsilon}$, where $\epsilon$ is a small constant. This is this sometimes followed by the application of learnable per-channel gain and bias parameters, which are initialized to 1 and 0 respectively. 

Note that LN layers were not explicitly defined for convolutional networks in the original paper. Thus, one could also conceivably define them as computing a mean $\mu$ and variance $\sigma^2$ over both locations and channels, instead of individually per location. In this work we will assume our previous definition, and anything we say regarding LN layers from this point will apply {\tmem{only to that definition}}.

Unlike BN layers, LN layers perform their computations and transformations individually per training case, and do not involve any computations across the mini-batch. Averaging of statistics instead occurs over entries (i.e.~channels) of the location vectors, and the same scale and shift is applied to all entries. In general, $\mu$ and $\sigma^2$ will be different for each input to the network, so that the learnable gain and bias cannot ever actually ``undo'' the normalization for all training cases cases simultaneously. This means that introducing LN layers into a network will fundamentally change its model class.

As we will show next, LN layers can be understood within our Q/C map framework, and are thus compatible with DKS. The formulas for their local Q/C maps are given below.

\subsubsection{Q/C map computations for Layer Normalization layers}\label{sec:layernorm-maps}

As we are concerned with the network's initialization-time behavior when computing Q/C maps, we will assume going forward that the LN layers $f$'s learnable gain and bias parameters, if they are indeed used, are set to their initial values (1 and 0). Given this assumption, and taking $\epsilon = 0$, the output of $f$ will always have a dimension-normalized squared length of 1, as $\sigma^2 = \frac{1}{k}  \| z - \mu \mathbbm{1} \|^2$ by definition. As this is precisely the quantity approximated by q values, we can thus define $Q_f (q) = 1$.

To understand how $f$ will affect c values, it suffices to analyze it as a mapping from $z$ to $z - \mu \mathbbm{1}$, since c values are invariant to scalar multiplications of the underlying vectors. 

Suppose $z_1$ and $z_2$ are two different vector inputs to $f$ (for a particular location), with q values $q_1, q_2$ and c values and $c_1$, $c_2$ (respectively), and define $\mu_i = \frac{1}{k} \mathbbm{1^{\top}} z_i$ for $i = 1, 2$. Then we have $\frac{1}{k}  (z_i - \mu_i  \mathbbm{1})^{\top}  (z_j - \mu_j  \mathbbm{1}) = \frac{1}{k} z_i^{\top} z_j - \mu_i \mu_j$ for $i, j \in \{ 1, 2 \}$. $C_f (c, q_1, q_2)$ approximates the cosine similarity of $z_1 - \mu_1  \mathbbm{1}$ and $z_2 - \mu_2  \mathbbm{1}$, which can therefore be written as
\[ \frac{\frac{1}{k} z_1^{\top} z_2 - \mu_1 \mu_2}{\sqrt{\left( \frac{1}{k}  \| z_1 \|^2 - \mu_1^2 \right)  \left( \frac{1}{k}  \| z_2 \|^2 - \mu_2^2 \right)}} \approx \frac{\sqrt{q_1 q_2} c - \mu_1 \mu_2}{\sqrt{(q_1 - \mu_1^2)  (q_2 - \mu_2^2)}} . \]
For a network initialized as per Section \ref{sec:param_dist}, we have by Equation \ref{eqn:average-unit-approx} that $\mu_i \approx \mathbb{E}_{x \sim \mathcal{N} (0, 1)} \left[ \phi \left( \sqrt{q_i} x \right) \right]$ for $i = 1, 2$, where $\phi$ is the activation function of the immediately preceding combined layer $g$ (with $\phi$ being the identity if $g$ is affine).

If we have uniform q values (so that $q_1 = q_2 = q$), then by Equation \ref{eqn:C-at-0} this implies $\mu_1 = \mu_2 = \sqrt{qC_g (0)}$, so that the above expression for $f$'s C map simplifies to
\begin{equation}
  C_f (c) = \frac{qc - qC_g (0)}{q - qC_g (0)} = \frac{c - C_g (0)}{1 - C_g (0)} . \label{eqn:C-map-LN}
\end{equation}
When $g$ is an affine layer (or a sum over multiple affine layers), or is a combined/nonlinear layer transformed via DKS, we have $C_g (0) = 0$. In this case, the above expression for $f$'s C map reduces to the {\tmem{identity function}}.

More generally, we note that
\[ C_{f \circ g} (0) = C_f (C_g (0))  = \frac{C_g (0) - C_g (0)}{1 - C_g (0)} = 0, \]
and so the application of the LN layer $f$ after $g$ thus has the effect of ensuring that $C_{f \circ g} (0) = 0$ even when $C_g (0) \neq 0$.

\section{Addressing pooling layers}\label{sec:pooling-layers}

Pooling layers are a type of layer used in certain convolutional network architectures to compress information from a larger feature map into a smaller one (with fewer locations). In this section we will discuss why standard pooling layers aren't compatible with our Q/C map framework, and describe potential replacements for them which are. We will also give mathematical arguments and empirical evidence suggesting that it may nonetheless be okay to use them with DKS in practice.

\subsection{(Global) mean-pooling layers}

Mean-pooling layers function similarly to convolutional layers, except that instead of computing a (learnable) affine function of each ``patch'' of activation vectors, they simply compute the average of those vectors. Typically these patches don't overlap, and thus a mean pooling layer reduces the number of locations (while preserving the channels).

In order to simplify the discussion we will restrict our attention to ``global'' mean-pooling layers, which average over {\tmem{all}} locations, and are the most common type used in practice. The same basic conclusions will apply to general mean-pooling layers, with somewhat more complicated formulas for the associated kernel functions.

Formally, a global mean-pooling layer $f$ computes
\begin{equation}
  f (Z) = \frac{1}{\ell} Z \mathbbm{1}, \label{eqn:global-mean-pool}
\end{equation}
where $\ell$ is the number of locations of the feature map $Z$, and $\mathbbm{1}$ denotes a vector of 1's of the appropriate dimension (which will change based on context).

\subsubsection{The PKF for mean-pooling layers and associated difficulties}

From the above equation, the PKF associated with $f$ is a $(2 \times 2)$-matrix-valued function given by
\begin{equation}
  \kappa_f (Z, Z') = \frac{1}{k \ell^2}  \left[\begin{array}{cc}
    (Z \mathbbm{1})^{\top} Z \mathbbm{1} & (Z \mathbbm{1})^{\top} Z'  \mathbbm{1}\\
    (Z'  \mathbbm{1})^{\top} Z \mathbbm{1} & (Z'  \mathbbm{1})^{\top} Z'  \mathbbm{1}
  \end{array}\right] = \frac{1}{\ell^2}  \left[\begin{array}{cc}
    \mathbbm{1}^{\top}  \left( \frac{1}{k} Z^{\top} Z \right)  \mathbbm{1} & \mathbbm{1}^{\top}  \left( \frac{1}{k} Z^{\top} Z' \right)  \mathbbm{1}\\
    \mathbbm{1}^{\top}  \left( \frac{1}{k} {Z'}^{\top} Z \right)  \mathbbm{1} & \mathbbm{1}^{\top}  \left( \frac{1}{k} {Z'}^{\top} Z' \right)  \mathbbm{1}
  \end{array}\right], \label{eqn:mean-pool-PKF}
\end{equation}
where $k$ is the output channel dimension.

Noting that the input and output channel dimension are equal for mean-pooling layers, we have
\[ \Sigma_{Z, Z'} = \left[\begin{array}{cc}
     \frac{1}{k} Z^{\top} Z & \frac{1}{k} Z^{\top} Z'\\
     \frac{1}{k} {Z'}^{\top} Z & \frac{1}{k} {Z'}^{\top} Z'
   \end{array}\right], \]
and so $\kappa_f (Z, Z')$ only depends on the inputs $Z$ and $Z'$ via their IPM $\Sigma_{Z, Z'}$. Thus, $\kappa_f$ can be composed with APKFs to form a network-level PKF approximation.

Unfortunately, while $\kappa_f$ depends only on input q and c values (i.e.~the entries of $\Sigma_{Z, Z'}$), it cannot be broken down in terms of Q and C maps that operate independently across different locations. For example, the output q value is given by $\frac{1}{\ell^2}  \mathbbm{1}^{\top} \left( \frac{1}{k} Z^{\top} Z \right)  \mathbbm{1}$, which would be computed as an average of multiple input q and c values. Even if we assume that the {\tmem{input}} q values to $\kappa_f$ are uniform, the {\tmem{output}} q values will differ for each input to the network and each location due to their dependence on the c values. This invalidates our C map analysis for subsequent layers, which is predicated on uniform q values.

\subsubsection{A possible replacement: weighted mean-pooling layers}\label{sec:weighted-mean-pools}

A possible solution to the issues associated with mean-pooling layers is to replace them with layers that can be more easily handled within our framework, and which ideally don't shrink the model class. (Expansion of the model class is less objectionable, provided that it doesn't significantly harm generalization performance.)

One natural option is to use convolutional layers whose filter size equal is equal to that of the entire feature map. This won't shrink the model class, as such layers can easily simulate global mean-pooling layers by setting all filter weights to $\left( \ell \sqrt{k} \right)^{- 1}$. Unfortunately, such a layer will most likely have a very large filter bank matrix, and this may significantly increase the total number of parameters in the network.

Another option is something we call \tmtextit{\tmtextbf{weighted mean-pooling layers}}, which are defined similarly to regular mean-pooling layers, except that the vector of 1's in Equation \ref{eqn:global-mean-pool} is replaced by a learnable vector of weights $w$, giving
\[ f (Z) = Zw. \]
These can clearly also simulate regular mean-pooling layers (by setting $w = (1 / \ell)  \mathbbm{1}$). And because they don't introduce as many new parameters as the previous option, they have a better chance of preserving the model's generalization characteristics.

Suppose $f$ is a weighted mean-pooling layer. As with combined layers, we can compute an APKF $\widetilde{\kappa_f}$ which approximates $f$'s PKF $\kappa_f$ at initialization-time (with high probability), under the assumption that $w \sim \mathcal{N} (0, (1 / \ell) I)$. As shown in Appendix \ref{app:weighted-mean-pool}, this is given by
\[ \widetilde{\kappa_f} (\Sigma_{Z, Z'}) = \frac{1}{\ell}  \left[\begin{array}{cc}
     \tmop{tr} \left( \frac{1}{k} Z^{\top} Z \right) & \tmop{tr} \left( \frac{1}{k} Z^{\top} Z' \right)\\
     \tmop{tr} \left( \frac{1}{k} {Z'}^{\top} Z \right) & \tmop{tr} \left( \frac{1}{k} {Z'}^{\top} Z' \right)
   \end{array}\right] , \]
and becomes a more precise approximation as $\ell$ grows, provided that the average absolute input c value across all pairs of locations simultaneously goes to zero. This latter condition could occur if the feature vectors for nearly all pairs of locations in the network's input image have a small absolute cosine similarity, since c values always decrease with depth under DKS. It could also occur for more general input images if DKS is used with a large $\zeta$ parameter, and $f$ is sufficient deep into the network (so that the C map up to $f$ maps most inputs to a relatively small region near zero).

$\widetilde{\kappa_f}$ has more favorable properties than the PKF for mean pooling layers given in Equation \ref{eqn:mean-pool-PKF}. In particular, since the output q/c value is just the average across locations of the input q/c values, the property of uniform q values of 1 will be preserved, thus enabling our C map analysis to be valid for subsequent layers.

Complicating the story somewhat is the fact that the c values for different locations are averaged together, as our analysis up to this point has assumed them to be separate and independently evolving. This means that geometric information about each individual location is no longer strictly preserved, as the averaging operation makes recovery of the individual c values impossible. It is true however that each location still has a proportional effect on the output, and thus the degeneration discussed in Section \ref{sec:degen-C-maps} can still be avoided, as long as the C map of the subnetwork up to $f$ is sufficiently well-behaved.

Because the input c values are given by $C_g (c_i)$ for locations $i = 1, 2, \ldots, \ell$ (for some $c_i$'s) where $g$ is the subnetwork up to $f$, and $C_g$ is a convex and increasing function on $[0, 1]$ (by Section \ref{sec:pd-functions}), we have that
\[ C_g \left( \frac{1}{n}  \sum_i c_i \right) \leqslant \frac{1}{\ell}  \sum_i C_g (c_i) \leqslant C_g (\max_i \{ c_i \}), \]
for $c_i \in [0, 1]$. The output c value associated $\widetilde{\kappa_f}$, which is given by $\frac{1}{\ell}  \sum_i C_g (c_i)$, is thus closely related to $C_g$ as applied to a single input. It is on this basis that we will treat $f$ as having an identity C map in our computations which, for lack of a richer multi-dimensional theory, seems like a reasonable heuristic.

\james{can someone think of a better argument than the one given above?}

In some of our experiments on convolutional networks we tried using a weighted mean-pooling layer in place of the usual global mean-pooling operation near the end of the network. While this worked well, we found that it didn't provide any optimization benefit. (See Appendix \ref{app:weighted-mean-pool-experiment} for these experiments.) Thus in our main set of experiments with DKS we continued to use standard global mean-pooling layers, despite their apparent incompatibility with our theoretical framework.

\subsection{Max-pooling layers}\label{sec:max-pooling}

A max-pooling layer is similar to a mean-pooling layer, but instead of taking the mean of a set of location vectors, it takes the coordinate-wise maximum.

\subsubsection{PKFs for max-pooling layers and approximate map properties}

In the previous subsection we saw that the PKF for a mean-pooling layer, despite having a simple form that depended only on the IPM ($\Sigma_{Z, Z'}$) of its input, had unfavorable properties that made it impossible to properly analyze within our Q/C map framework. The situation with max-pooling layers is arguably even worse, as its PKF has a more general dependence on its input, and thus cannot be composed with APKFs of combined layers to form a network-level PKF approximation. But despite this, we can still make some non-trivial statements about a max-pooling layer's PKF that will be useful in understanding how DKS may possibly still apply to networks containing such layers.

Consider a patch of locations over which the max operation is applied. If all the location vectors in the patch are nearly equal to each other, then the max operation simply outputs a close approximation of the vector in the center of the patch, and thus has a PKF approximated by the identity (for non-dropped locations). It is therefore reasonable to approximate max-pooling layers as having local Q and C maps equal to the identity in this case. This situation is fairly common when max-pooling layers are used very early in the network, since nearby pixels tend to be similar to each other in natural image data, which, assuming well-behaved C maps, means that the corresponding vectors for subsequent layers will be similar too (as measured by their cosine similarly).

Analogously, if the pixels within a patch fall into two tight clusters, which can happen if the patch overlaps an edge or the corner of an object, then the subsequent vectors will also fall into two tight clusters. If this is the situation, and we assume uniform q values and wide layers, then it can be shown that the output q value of a max pooling layer will be closely approximated by its input q value (so that we can treat the layer as having an identity Q map). This is shown in Appendix \ref{app:max-pool}, and relies on the somewhat surprising fact that
\[ \mathbb{E}_{\tmscript{\left[\begin{array}{c}
     u_1\\
     u_2
   \end{array}\right] \sim \mathcal{N} \left( 0, \left[\begin{array}{cc}
     1 & c\\
     c & 1
   \end{array}\right] \right)}} [\max \{ u_1, u_2 \}^2] = 1 \]
for all $c \in [- 1, 1]$, along with the mild assumption that the max-pooling layer in question is directly preceded by a convolutional layer. Note that for clusters of 3 or more pixel values this approximation doesn't work, although the output q value will only deviate from the input q value by a factor that grows slowly with the number of clusters.

\subsubsection{Possible replacements for max-pooling layers}

As with mean-pooling layers, we could consider replacing max-pooling layers with ones that are handled within our framework. However, unlike the mean operation, the max operation is difficult to elegantly simulate using our standard layer types, and so there are no obvious substitutions that would preserve the model class.

In some architectures, max-pooling layers are used merely to reduce the size of a feature map, with the particular choice of pooling operation (max or mean) being unimportant from a modeling perspective. In such cases it may thus be quite reasonable to replace max-pooling layers with weighed mean-pooling layers.


\section{Summary of our method}

\subsection{Architectural requirements}\label{sec:arch-requirements-method}

In order to apply DKS we must observe certain architectural requirements on the network. These are summarized below:
\begin{enumeratenumeric}
  \item The network must be constructed from combined layers (defined as an affine layer followed optionally by a nonlinear layer), weighted sums between the output of two or more affine layers (followed optionally by a nonlinear layer), concatenations of two or more feature maps along their channel dimension, mean-pooling layers, and max-pooling layers (although the latter should be used with caution as discussed in Section \ref{sec:max-pooling}).\label{item:layers_assumptions}
  
  \item Batch Normalization layers must not be used. However, Layer Normalization layers are allowed, provided that their associated gain and bias parameters are initialized as per Item \ref{item:extra-params-assumptions} of this list. (See Section \ref{sec:normalization-layers} for additional discussion of normalization layers.)
  
  \item Nonlinear layers must use element-wise activation functions. Positively homogeneous activation, such as RELU, are allowed but not recommended as they lead to a limited version of DKS (as discussed in Section \ref{sec:positively-homogeneous-problem}).
  
  \item Multiplication operations, such as those used in attention mechanisms, are also {\tmem{not}} allowed (although we hypothesize that DKS can be extended to handle these in the future).
  
  \item The network should not contain any extraneous trainable parameters such as scalar multiplications or shift, unless these parameters have no effect at initialization time (e.g.~a shift that is initialized to 0). Constant scalar multiplications are allowed, although these will typically be removed as part of the application of DKS.\label{item:extra-params-assumptions}
  
  \item Similarly, constant multiplications and shifts are not allowed, except as part of weighted sum operations. (Note that if such constants are normally required for the network to be trainable with standard optimization methods, it's likely that DKS will render them unnecessary/obsolete.)
\end{enumeratenumeric}

\subsection{Execution steps}\label{sec:method-steps}

To apply DKS to a given network one performs the following steps:
\begin{enumeratenumeric}
  \item Initialize each bias vector to 0, and each weight matrix/filter bank using either a Gaussian Delta initialization, or an Orthogonal Delta initialization (which are both defined in Section \ref{sec:param_dist}).
  
  \item Choose a value larger than 1 for the scalar hyperparameter $\zeta$ (such as 1.5 or 1.1). Note that $\zeta$ roughly corresponding to the ``degree of nonlinearity'' of the network. See Section \ref{sec:choosing-zeta} for additional discussion of this. As observed in Section \ref{sec:choosing-zeta}, lower values of $\zeta$ tend to be associated with slightly better generalization, at the cost of somewhat slower optimization.
  
  \item (optional) Apply some version of Per-Location Normalization (PLN) to the input data. Note that this can be done entirely online, as it only requires the current example, and not any aggregate statistics over the entire training set. (See Section \ref{sec:PLN} for more details.)
  
  \item Remove any constant scalar multiply operations from the network.
  
  \item Replace any weighted sums between features maps $Y_1, \ldots, Y_n$ with ``normalized sums'' of the form $\sum_{i = 1}^n w_i Y_i$, for weights $w_i$ satisfying $\sum_{i = 1}^n w_i^2 = 1$ (which may be chosen freely). Note that if $w_i = w_j$ for all $i$ and $j$ this simplifies to $\frac{1}{\sqrt{n}}  \sum_{i = 1}^n Y_i$, although other choices are permitted and may indeed be preferable (as demonstrated in Section \ref{sec:modifications-made-by-method}).
  
  \item (optional) Replace any mean-pooling layers with ``weighted mean-pooling layers'' as defined in Section \ref{sec:weighted-mean-pools}.
  
  \item Compute the network's maximal slope function $\mu (\psi)$ (or some approximation of this). One can use the recipe given in Section \ref{sec:slope-recipes}.
  
  \item Using the fact that $\mu$ is a 1D monotonically increasing function, compute $\mu^{- 1} (\zeta)$ using binary search (or a similar such method).
  
  \item For each distinct activation function $\phi$ in the network, do the following:\label{item:replace_activation}
  \begin{enumerateroman}
    \item Given $\mu^{- 1} (\zeta)$, solve for $\alpha$, $\beta$, $\gamma$, and $\delta$ as per Section \ref{sec:achieving-local-conds}, using the numerical methods described in Section \ref{sec:implementation-details} (or some alternative).
    
    \item Replace all instances of $\phi$ in the network with $\hat{\phi} (u) \equiv \gamma (\phi (\alpha u + \beta) + \delta)$.
  \end{enumerateroman}
\end{enumeratenumeric}
\subsection{Recipe for computing slope polynomials and maximal slope functions}\label{sec:slope-recipes}

As per Sections \ref{sec:slope-poly-and-maximal} and \ref{sec:comp-max-slope-func}, the maximal slope function $\mu (\psi)$ is computed as
\[ \mu (\psi) = \max_f [p_f (\psi)], \]
where $p_f (\psi)$ is the network polynomial of $f$ (whose computation we will describe below), and the maximum is taken over all subnetworks $f$ of the entire network.

Note while the number of distinct subnetworks may be quadratic (or worse) in the depth, when computing the maximum we may ignore any subnetwork that can be composed with another one to form a strictly larger subnetwork, or more generally, any subnetwork whose slope polynomial is a factor of the slope polynomial of another subnetwork.

For a given subnetwork $f$, the computational graph of the network polynomial $p_f (\psi)$ may be obtained from the computational graph of $f$ by recursively applying the following rules (which are essentially just the result of applying automatic differentiation to the graph of $C_f (c)$ and then evaluating the result at $c = 1$ to obtain $p_f (\psi) = C'_f (1)$):
\begin{enumeratenumeric}
  \item Composition $g \circ h$ of two subnetworks $g$ and $h$ maps to $p_g (\psi) p_h (\psi)$.
  
  \item Affine layers map to the constant 1.
  
  \item Nonlinear layers map to $\psi$.
  
  \item Concatenation operations (over the channel dimension) between the outputs of subnetworks $g_1, g_2, \ldots, g_n$  map to
  \begin{equation*}
    \frac{1}{\sum_{i = 1}^n k_i} (k_1 p_{g_1} (\psi) + k_2 p_{g_2} (\psi) + \cdots + k_n p_{g_n} (\psi)) ,
  \end{equation*}
  where $k_i$ is the number of output channels of $g_i$.
  
  \item Normalized sums with weights $w_1, w_2, \ldots, w_n$ over the outputs of subnetworks $g_1, g_2, \ldots, g_n$ map\footnote{For reference, when computing the slope polynomial for a network whose q values may vary between layers (which won't come up when applying DKS), normalized sums instead map to
  \[ \frac{1}{\sum_{i = 1}^n w_i^2 q_i}  (w_1^2 q_i p_{g_1} (\psi) + w_2^2 q_i p_{g_2} (\psi) + \cdots + w_n^2 q_i p_{g_n} (\psi)), \]
  where $q_i$ is the output q value associated with $g_i$.} to
  \[ w_1^2 p_{g_1} (\psi) + w_2^2 p_{g_2} (\psi) + \cdots + w_n^2 p_{g_n} (\psi) . \]

    \item Layer normalization layers map to the constant 1.
  
  \item Max-pooling and weighted mean-pooling layers map to the constant 1. (Standard mean-pooling layers can be heuristically mapped to 1, although they technically break our network polynomial formalism.)
  
  \item The network's input maps to the constant 1.
\end{enumeratenumeric}
Note that this recipe for computing can be generalized to compute $C_f' (1)$ for networks in which $C'_g (1)$ may be different for each nonlinear layer (i.e.~not equal to some common $\psi$) by mapping nonlinear layers $g$ to $C'_g (1)$ instead of $\psi$.

Provided that the network architecture is compatible with DKS, a quick way to compute slope polynomials is to count the number $k$ of nonlinear layers in a given sequence of layers (to get a slope polynomial of $\psi^k$ for that subnetwork), and then apply the rule for normalized sums where appropriate. See Sections \ref{sec:comp-max-slope-func} and \ref{sec:maximal-slope-mod-resnet} for instructive examples of how to the compute maximal slope function for certain architectures.

\james{is it possible that we could develop a set of rules that gets the maximal slope function more directly without bothering with this max over slope polynomials?}

\section{Some implementation details}\label{sec:implementation-details}

Through careful optimization and engineering, and a lot of trial and error, we were able to get the runtime of DKS down to a few seconds for typical large-scale networks. In this section we describe the aspects of this that were the most challenging and non-obvious.

\subsection{Solving for the $\alpha$ and $\beta$ constants}

As we saw in Section \ref{sec:achieving-local-conds}, finding the appropriate $\alpha$ and $\beta$ constants for a particular activation function $\phi$ amounts to solving a system of two nonlinear equations. Since we don't have a closed form solution for this, we must resort to numerical methods.

After trying several possibilities, we got the best results using \python{scipy.optimize.root()}, which is part of the popular SciPy Python package \citep{scipy2001}. We call this with the arguments \python{method="hybr"} and \python{jac=False}, and leave all other options at their defaults. This invokes an implementation of the modified Powell algorithm \citep{powell1964efficient}.

Because the implicit regression loss of the system is non-convex in general, the solver sometimes fails to find a solution from its initial guess. Our solution to this is simply to call it repeatedly with different initial guesses until it returns successfully. In our experiments it never took more than 4 calls to find a solution for any of the eleven activation functions we tried. We took our first six initial guesses for $(\alpha, \beta)$ from the list
\begin{equation*}
[(1, 0), (1, 1), (1, -1), (0.1, 0), (0.1, 1), (0.1, -1)]
\end{equation*}
and then generated subsequent ones randomly using \python{numpy.random.uniform(low=0.0, high=2.0)} for $\alpha$, and \python{numpy.random.uniform(low=-3.0, high=3.0)} for $\beta$.

\subsection{High-quality estimates of the expectations}\label{sec:estimate-expectations}

In order to guarantee fast and reliable convergence, the solver \python{scipy.optimize.root} requires the LHS values of the nonlinear equations to be computed to a very high precision. Moreover, the system we need to solve may be numerically sensitive in general, regardless of the particular solver algorithm used. Thus it is important that we compute high quality estimates of the four Gaussian expectations that determine these LHS values.

A naive estimate based on sampling $x$'s from $\mathcal{N} (0, 1)$ will perform very poorly, as its variance scales as $1 / n$, where $n$ is the number of sample points. Fortunately, all four of the Gaussian expectations are one dimensional, which opens the door to heavy-duty numerical integration methods capable of achieving near numerical precision in a reasonable amount of time. After experimenting with several such methods built into SciPy, we found that the best performing one by far was \python{scipy.integrate.fixed_quad}, which implements fixed-order Gauss-Legendre quadrature to compute one dimensional definite integrals. We used this method to approximate Gaussian expectations as
\[ \mathbb{E}_{x \sim \mathcal{N} (0, 1)} [h (x)] = \frac{1}{\sqrt{2 \pi}}  \int_{- \infty}^{\infty} h (x) \exp (- x^2 / 2) d x \: \approx \: \frac{1}{\sqrt{2 \pi}}  \int_{- 10}^{10} h (x) \exp (- x^2 / 2) d x, \]
which is justified by the fact that the integrand is negligible for values of $x$ outside of $[- 10, 10]$ for all the $h(x)$'s we care about. To ensure high quality estimates, we set the ``order'' parameter $n$ to $10^5$ (which is considered very high).

By far the most expensive part of \python{fixed_quad}'s computation is the calculation of the sample points and weights (via the \python{roots_legendre} function), which can take around 15 minutes on a modern CPU when $n = 10^5$. Fortunately, because this part of the computation only depends on the order parameter $n$, it is cached by \python{fixed_quad} during the first call and reused in subsequent calls, allowing these calls to execute almost instantly. In our codebase we went a step further and stored the results of \python{roots_legendre} in a file which was then loaded and monkey-patched into the SciPy library before the first call to \python{fixed_quad}, thus eliminating this 15 minute overhead completely.

\subsection{Computing and inverting the maximal slope function in software}\label{sec:compute-maximal-slope-in-practice}

As they are determined by a network's architecture, maximal slope functions can in principle be computed automatically and efficiently, according to recipe in Section \ref{sec:slope-recipes}. Automating this in software requires access to a high-level description of network's structure, which could possibly be extracted from the API calls made to the neural network library. In our experiments we just computed them by hand, as we only experimented with a small handful of architectures.

An alternative to manual computation or automation is to approximate the maximal slope function by a reasonable surrogate. For networks whose ``deepest path'' has $D$ nonlinear layers, a natural approximation to use is $\psi^D$. However, for network architectures that involve extensive use of skip connections such as ResNets, this approximation may be very poor, as it fails to account for how the skip connections make the network's computation ``more linear''. (See Section \ref{sec:maximal-slope-mod-resnet} for more details.)

As discussed in Section \ref{sec:slope-poly-and-maximal}, maximal slope functions are continuous 1-dimensional functions of $\psi$ that are strictly increasing (as long as the network has at least one nonlinear layer), and thus they can be inverted up to a fixed tolerance using a simple binary search. Our implementation started with an interval of $[1, 2]$ for the solution, and kept doubling the maximum if the solution was determined to lie outside of it. We used a convergence tolerance of $10^{- 6}$ on the function value, and used full precision 64-bit floating point numbers in all computations (which are cheap and can be done on the CPU). Note that for certain simple special cases, such as when $m (\psi) = \psi^D$, closed form solutions can be used if desired.

\section{Application to various modified ResNets}\label{sec:application-to-resnet}

In this section we will discuss the very commonly used ResNet architecture \citep{he2016deep, he2016identity} and certain modified versions of it, and then go over the application of DKS to these different versions. In addition to being instructive in the application of DKS, these example will be the primary focus of our later experiments.

\subsection{Standard ResNet-V2 architectures and terminology}\label{sec:standard-ResNet}

In this work we will only consider the ``V2'' version of the ResNet architecture \citep{he2016identity} as the opposed to the ``V1'' version \citep{he2016deep}, as the former is conceptually simpler and is usually preferred by practitioners. We will also concentrate on the version of ResNet-V2 designed specifically for use in 224x224 Imagenet classification, noting that versions of the architecture for other problems and datasets can differ, especially in terms of their first and last few layers.

We will denote by $D$ the ``depth parameter" of the ResNet architecture, which corresponds to the total number of nonlinear layers plus 1. The standard values for $D$ are 50, 101, and 152. The input is assumed to be 224x224 features maps with 3 dimensional pixel features (possibly extended to 4 dimensions if PLN is used as per Section \ref{sec:PLN}). This is then fed into a 7x7 convolutional layer with 64 output channels and a stride of 2. Following this is a 3x3 max-pooling layer with a stride of 2. These two early layers are particular to the Imagenet version of ResNet-V2, and have the purpose of reducing the dimension of the feature representation to a smaller size for processing by subsequent layers.

Following this is a long sequence of \tmtextbf{\tmtextit{residual blocks}} that form the large bulk of the network. Each of these is parameterized by an associated output channel dimension, a ``bottleneck'' channel dimension, and a stride, which can differ from block to block. The particular values for these parameters are determined by $D$.

Let $k$ be the input channel dimension, $d$ the output channel dimension, $b$ the bottleneck channel dimension, and $s$ the stride associated with a particular residual block. The residual block contains two ``branches'' from its input that get summed together at the output. The first is called the \tmtextbf{\tmtextit{residual branch}}, and consists of the following sequence of layers: a Batch Normalization (BN) layer, a RELU nonlinear layer, a 1x1 convolutional layer with stride 1 and output channel dimension $b$, a BN layer, a RELU nonlinear layer, a 3x3 convolutional layer with stride $s$ and output channel dimension $b$, a BN layer, a RELU nonlinear layer, and finally a 1x1 convolutional layer with stride 1 and output channel dimension $d$. The second branch is called the \tmtextbf{\tmtextit{shortcut branch}}, and consists of the identity map if $k = d$ and $s = 1$, or a 2x2 max-pooling layer if $k = d$ and $s > 1$. Otherwise, if $k \neq d$ (which is the case for \tmtextbf{\tmtextit{transition blocks}}), it consists of the following sequence of layers: a BN layer, a RELU nonlinear layer\footnote{This nonlinear layer in the shortcut branch does not contribute to the total number of nonlinear layers for the purposes of computing $D$. Moreover, it can be identified with the first nonlinear layer of the residual branch (as they compute the same thing), in which case both the shortcut branch and residual branches can be seen as ``branching off'' from this layer's output (instead of from the block's original input). }, and a 1x1 convolutional layer with stride 1 and output channel dimension $d$.


\james{diagram a ResBlock and possibly the whole network? Guillaume possibly interested in adding this?}

The shortcut branch is meant to act as an identity function or a reasonable approximation to one, except when its input and output channel dimensions differ (which is only the case for transition blocks). Meanwhile, the residual branch, which always contains 3 nonlinear layers, is what performs the interesting nonlinear computation in the network.

After the sequence of residual blocks, there is a BN layer and a RELU nonlinear layer, followed by a ``global'' mean-pooling layer which reduces the number of locations down to 1. The final layer of the network is a 1x1 convolutional layer operating on this single location (or equivalently a fully-connected layer), whose output channel dimension is the number of classes.

Convolutional layers in ResNets typically do not have bias parameters, since these are made pointless by the mean-subtraction done by the BN layers that always immediately follow them. To compensate for this, BN layers will sometimes include trainable gain and/or bias parameters applied after their centering and normalization operations.

\subsection{Modified ResNet architecture}

In this subsection we will describe the particular changes we made to the ResNet-V2 architecture in order to conform to the requirements listed in Section \ref{sec:arch-requirements-method} and thus achieve compatibility with DKS. These changes don't perfectly preserve the model class, although we tried to make them as innocuous as possible in order to facilitate the fairest comparison to standard ResNets in our experiments. Note that other modification schemes are possible, and the one we present here is not meant to be in any way ``canonical'' for this or any other architecture.

As BN layers are incompatible with DKS we elect to remove them, while adding learnable bias parameters (initialized at zero) back into the convolutional layers. 
Another option would be to replace BN layers with Layer Normalization layers, as the latter are compatible with DKS.

As discussed in Section \ref{sec:positively-homogeneous-problem}, while technically supported, RELU activations force us to use a diminished version of DKS. Thus in our main experiments we often used alternative activation functions instead, including ones with a ``RELU-like" shape, such as softplus.

Max-pooling layers are provisionally supported by DKS, especially if they occur early in the network and have a relatively small kernel size. (See Section \ref{sec:positively-homogeneous-problem} for more details about this.) The 3x3 max-pooling layer near the beginning of the network meets these criteria, and so we elect to leave it in.


\subsection{Further modifications made {\tmem{by}} DKS}\label{sec:modifications-made-by-method}

Having achieved compatibility by making the above changes, we can now apply DKS to the resulting modified ResNet architecture. In this subsection we will describe the subsequent changes made to the network {\tmem{as part of the execution of DKS itself}}, whose steps are outlined in Section \ref{sec:method-steps}. Note that some of these steps are optional, or involve degrees of freedom, and are all designed to preserve (or slightly expand) the model class.

First, the mean-pooling layer near the end of the network can optionally be replaced by a weighted mean-pooling layer (as described in Section \ref{sec:weighted-mean-pools}). While this replacement is necessary for the Q/C map computations to make sense, we found that it didn't significantly improve optimization performance in our preliminary experiments, and so we didn't do it in our main ones. One possible explanation for this finding is that because there is only a single nonlinear layer after the mean-pooling layer, the non-uniform q values produced by the latter can have only a limited effect on the network's overall C map.

Next, we must replace the sum operations, which occur in ResNets at the end of each residual block (where the residual and shortcut branches are combined together), with normalized sums. Each normalized sum involves two weights and one constraint (that the squares of the weights sum to 1), and so has one degree of freedom. A natural choice is to set both weights to $1 / \sqrt{2}$, which naively seems like the best option for reproducing the behavior of an unmodified ResNet. However, as we will discuss in Section \ref{sec:resnet-related-discussion}, for networks that forgo normalization layers and/or use bounded activation functions (as our modified ResNets do), placing more weight on the shortcut branch will result in better behavior that more closely matches that of a standard ResNet. This is confirmed in our experiments in Appendix \ref{app:residual-weight-sweep}.

The final modification made to the network as part of DKS is to replace all of the activation functions with their transformed versions, as described in Step \ref{item:replace_activation} of Section \ref{sec:method-steps}.

\subsection{Computing the maximal slope function for modified ResNets}\label{sec:maximal-slope-mod-resnet}

Having described the modifications we made to ResNets to achieve compatibility with DKS, and the further ones made by DKS itself, we are now in a position to compute the maximal slope function following the recipe in Section \ref{sec:slope-recipes}. 

For simplicity, we will assume that the normalized sums at the end of the residual blocks each have a weight of $w$ on their residual branches. (A weight of $\sqrt{1-w^2}$ on the shortcut branches is then implied.)

The subnetwork before the sequence of residual blocks is just a affine and max-pooling layer and so has a slope polynomial of 1.

Consider any non-transition block. The slope polynomial for the residual branch is $\psi^3$, as it has 3 nonlinear layers, and is 1 for the shortcut branch. Thus, the overall slope polynomial for the block is $w^2 \psi^3 + (1 - w^2)$. Similarly, the slope polynomial for a transition block (which has a single combined layer in its shortcut branch) is $w^2 \psi^3 + (1 - w^2) \psi$.

The subnetwork after the sequence of blocks consists of nonlinear layer, a (possibly weighted) mean-pooling layer, and then an affine layer, and so has a slope polynomial of $\psi$.

Noting that the total number of residual blocks is $(D - 2) / 3$, and the number transition blocks is 4 for all values of $D$, the overall slope polynomial for the network $f$ (which has a single input and output and so is a subnetwork of itself) is
\[ p_f (\psi) = (w^2 \psi^3 + (1 - w^2))^{(D - 2) / 3 - 4}  (w^2 \psi^3 + (1 - w^2) \psi)^4 \psi , \]
which can be simplified to
\begin{equation}
p_f (\psi) = (w^2 \psi^3 + (1 - w^2))^{(D - 14) / 3}  (w^2 \psi^2 + (1 - w^2))^4 \psi^5 . \label{eqn:mod-resnet-slope-poly}
\end{equation}

The only subnetworks of $f$ that don't compose with other subnetworks to form larger ones are the residual branches, and so their slope polynomials are the only other ones to consider when computing the maximal slope function $\mu (\psi)$. As they are simple compositions of layers with 3 (or 2) nonlinear layers total, their slope polynomials are $\psi^3$ (or $\psi^2$). Noting that $\psi^3$ (or $\psi^2$) is a factor of $p_f (\psi)$, we may ignore them when computing the maximum, and thus conclude that $\mu (\psi) = p_f (\psi)$.

It is worthwhile to note the dependency of $\mu (\psi)$ on the value of the residual branch weight $w$. For $w = 0$ we have $\mu (\psi) = \psi^5$, and for $w = 1$ we have $\mu (\psi) = \psi^{D - 1}$. More generally, $\mu (\psi)$ will be a degree $D - 1$ polynomial in $\psi$, where the coefficients (which must sum to 1) will more heavily favor high order terms as $w$ approaches 1, and low-order terms as $w$ approaches 0. Thus, much like $\psi$, $w$ can be thought of as controlling the overall ``degree of nonlinearity'' of the network $f$ (as quantified by $C_f' (1)$).

\subsection{Equivalent standard convolutional networks}

To help demonstrate the power of DKS in our experiments, we will consider a skip-connection-free convolutional network architecture obtained from the above modified ResNet-V2 architecture by simple removal of the shortcut branches. The resulting architecture retains the channel dimensions, strides, etc., of the original ResNet architecture, including it use of ``bottleneck'' layers in the residual branches, but is otherwise a standard deep convolutional network.

Given the straightforward sequential structure of this architecture, with its $D - 1$ nonlinear layers, its network polynomial and maximal slope function are simply $\psi^{D - 1}$ (which corresponds to the $w = 0$ case above).

\subsection{``Wide'' ResNet variants for CIFAR-10}

For our experiments involving the CIFAR-10 dataset \citep{krizhevsky2009learning} we will make use of Wide Residual Networks \citep{zagoruyko2016wide}, which are a well-known variant of the standard ResNet architecture. The Wide-ResNets we used in our experiments differ from standard ResNets in the following ways:
\begin{itemizedot}
  \item The initial subnetwork before the sequence of residual blocks consists of just a 3x3 convolutional layer with 16 output channels and a stride of 1. There is no max-pooling layer.
  
  \item Given per-block parameters $s$ and $d$, a residual branch consist of the following sequence: a BN layer, a RELU nonlinear layer, a 3x3 convolutional layer with stride $s$ and output channel dimension $d$, a BN layer, and RELU nonlinear layer, and finally a 3x3 convolutional layer with stride 1 and output channel dimension $d$. Note that there are only 2 nonlinear layers instead of the 3 normally present in standard ResNets.
  
  
  \item There is a global ``width'' parameter which acts as multiplier on the output channel dimensions of all the residual blocks. In our experiments this was set to 2.
  
  \item The scheme for mapping $D$ to a configuration for the residual blocks is generalized to work with any value of $D$ such that $D - 4$ is divisible by 6. Here, $D$ represents the number of nonlinear layers plus 3, so that there are $(D - 4) / 2$ total blocks, 3 of which are transition blocks.
\end{itemizedot}
As we did for standard ResNets, to achieve compatibility of Wide-ResNets with DKS we will modify the architecture by removing the BN layers, adding back in learnable biases to the convolutional layers, and (possibly) replacing the RELU activation functions with various alternatives. 

Following a similar derivation to the one in Subsection \ref{sec:maximal-slope-mod-resnet}, the maximal slope function for these networks is given by
\[ \mu (\psi) = (w^2 \psi^2 + (1 - w^2))^{(D - 10) / 2}  (w^2 \psi + (1 - w^2))^3 \psi^4, \]
where like before we have assumed weights $w$ and $\sqrt{1 - w^2}$ for the weighted sum operation at the end of each residual block.

We can also define a skip-connection-free version of this architecture by removing the shortcut branches. The maximal slope polynomial associated with such a network is $\psi^{D-3}$, as there are $D-3$ total nonlinear layers.

\part{Additional analysis of DKS and related methods}
\label{part:additional_analysis}

\section{Neural Tangent Kernel analysis}\label{sec:NTK-analysis}

Recent advances in the theoretical understanding of neural network training have shown that highly overparameterized networks behave like linear functions of their parameters over the entire course of training by gradient descent \citep{jacot2018neural, li2018learning, du2018gradient, du2019gradient, allen2019convergence, arora2019exact}. This analysis works by approximating the network function by its own 1st-order Taylor series with respect to its parameters (centered at their initial values), and then showing that the parameters remain close enough to their initial values throughout training that the approximation remains a good one. Under this approximation, which becomes exact as the width of each layer goes to infinity, training a neural network with gradient descent resembles kernel regression, with a kernel\footnote{Note that the Neural Tangent Kernel is related to but distinct from the kernels we have been analyzing in this work so far.} known as the \tmtextbf{\tmtextit{Neural Tangent Kernel (NTK)}} that is computed from the network's Jacobian at initialization time. This enables one to accurately predict the functional form of the trained network, and precisely characterize the rate of convergence to this solution by gradient descent. While this type of analysis has been extended to exact and approximate natural gradient descent methods \citep{zhang2019fast, cai2019gram, karakida2020understanding}, we will only consider the gradient descent version in this work.

Even though real networks trained on challenging datasets like Imagenet are typically not wide enough to satisfy the formal requirements of NTK theory (especially when random dataset transformations are employed), the setting where this 1st-order Taylor approximation works well -- known colloquially as the \tmtextit{\tmtextbf{``NTK regime''}} -- may still serve as a rough analogy to more  realistic training. It is thus interesting to consider what effect the network's architecture, activation functions, and initialization has on the NTK, and what this says about training in the NTK regime.

In this section we will review the basics of NTK theory, characterize the NTK in terms of the properties of the network's C map, and show how the C map degeneration which happens naturally in deep networks (as shown in Section \ref{sec:Cmaps_trainability}) leads to a form for the NTK which implies very slow optimization and/or very poor generalization. We will then show how the form of the NTK under DKS, assuming a reasonable choice for the global slope bound $\zeta$, is much nicer, and leaves open the possibility of fast optimization and good generalization. (Although actually proving that it {\tmem{necessarily}} leads to these things would require assumptions on the dataset, and is beyond the scope of this work.)

\subsection{Assumptions of this analysis}\label{sec:NTK-assumptions}

For the remainder of this section we will assume that the network in question is a standard feed-forward MLP comprised of $D$ fully-connected combined layers, where the last such layer has an identity activation function. We will represent the network as the function $f (x, \theta)$ for input vector $x$ and parameter vector $\theta$. For notational simplicity we will assume that the network's output dimension is 1. 

The parameter vector $\theta$ will be split across layers into $D$ segments denoted $\theta_i$ for $i = 1, 2, \ldots, D$. Each $\theta_i$ corresponds to $\sqrt{d_i} W_i$ (as opposed to $W_i$ itself), where is $W_i$ is the weight matrix for layer $i$, and $d_i$ its input dimension. This non-standard parameterization, which is known as the\tmtextbf{ \tmtextit{NTK parameterization}}, is what we apply gradient descent on, and is required for the NTK to have its desired properties. We will not consider bias parameters in this analysis.

The training set will consist of $n$ input-output pairs $(x_i, y_i)_{i = 1}^n$ satisfying $\| x_i \|^2 = d_0$ for all $i$ (where $d_0$ is the network's input dimension), and the objective function used to train the network will be the standard mean squared error: $\frac{1}{2}  \sum_{i = 1}^n (y_i - f (x_i, \theta))^2$. We will denote by $\theta (t)$ (or $\theta_i (t)$) the parameters at iteration $t$ of optimization. $\theta (0)$ (or $\theta_i (0)$) will denote their random initial value, which is determined by a Gaussian fan-in initialization applied to the standard parameters (i.e.~the original $W_i$'s).

In addition to our global assumption that the activation functions are infinitely differentiable everywhere except for a finite set of points, we will also assume that they are Lipshitz continuous, which is required in order to apply the results in \citet{jacot2018neural}.

\subsection{NTK definition}

The \tmtextbf{\tmtextit{Neural Tangent Kernel (NTK)}} is given by $\Theta (x, x') = \sum_{i = 1}^D \Theta_i (x, x')$, where $\Theta_i (x, x') \in \mathbb{R}$ denotes the inner-product
\[ \left\langle \left. \frac{\partial f (x, \theta)}{\partial \, \theta_i} \right|_{\theta = \theta (0)}, \: \left. \frac{\partial f (x', \theta)}{\partial \, \theta_i} \right|_{\theta = \theta (0)} \right\rangle . \]
Given the NTK $\Theta (x, x')$ and training dataset $(x_i, y_i)_{i = 1}^n$, the \tmtextit{\tmtextbf{NTK matrix}} $K \in \mathbb{R}^{n \times n}$ is defined by
\[ K_i = \left[\begin{array}{cccc}
     \Theta (x_1, x_1) & \Theta (x_1, x_2) & \cdots & \Theta (x_1, x_n)\\
     \Theta (x_2, x_1) & \Theta (x_2, x_2) &  & \vdots\\
     \vdots &  & \ddots & \\
     \Theta (x_n, x_1) & \cdots &  & \Theta (x_n, x_n)
   \end{array}\right] . \]
We can similarly define the \tmtextbf{\tmtextit{per-layer NTK matrix}} $K_i \in \mathbb{R}^{n \times n}$ for layer $i$ by
\[ K_i = \left[\begin{array}{cccc}
     \Theta_i (x_1, x_1) & \Theta_i (x_1, x_2) & \cdots & \Theta_i (x_1, x_n)\\
     \Theta_i (x_2, x_1) & \Theta_i (x_2, x_2) &  & \vdots\\
     \vdots &  & \ddots & \\
     \Theta_i (x_n, x_1) & \cdots &  & \Theta_i (x_n, x_n)
   \end{array}\right], \]
noting that $K = \sum_{i = 1}^D K_i$.

\subsection{Training in the NTK regime: a brief review}

There are various NTK-type results that bound the convergence rate of gradient descent in the case of finite width layers \citep[e.g][]{du2018gradient}. However, such results are complicated to prove, and seem to be fairly pessimistic in terms of the rate of convergence\footnote{While these results predict exponential convergence, the associated rate constants are close enough to 1 that convergence requires a prohibitively large number of iterations.} they predict, and the width they require.

On the other hand, the situation simplifies considerably in the limit of infinite width (for all layers but input and output ones, whose width is fixed), and quite simple and elegant expressions exist for both the convergence rate of gradient descent, and the function computed at the converged solution \citep{jacot2018neural}. While the infinite width limit is unrealistic, and totally ignores how kernel approximation error affects the theoretical predictions, we will nonetheless use it in our analysis for the sake of simplicity and clarity.

\subsubsection{Notation and basic results}

Before we begin we must define some additional notation. Let
\[ \tmmathbf{y} = \left[\begin{array}{c}
     y_1\\
     y_2\\
     \vdots\\
     y_n
   \end{array}\right] \in \mathbb{R}^n \text{, \qquad} \tmmathbf{f} (t) = \left[\begin{array}{c}
     f (x_1, \theta (t))\\
     f (x_2, \theta (t))\\
     \vdots\\
     f (x_n, \theta (t))
   \end{array}\right] \in \mathbb{R}^n \text{, \qquad} \tmmathbf{k}_i (x) = \left[\begin{array}{c}
     \Theta_i (x, x_1)\\
     \Theta_i (x, x_2)\\
     \vdots\\
     \Theta_i (x, x_n)
   \end{array}\right] \in \mathbb{R}^n, \]
and define $\tmmathbf{k} (x) = \sum_{i = 1}^D \tmmathbf{k}_i (x)$.

In the infinite width limit, provided that $K$ is positive definite (i.e.~non-singular), a standard result of NTK theory is that the $t$-th iterate $\theta (t)$ produced by gradient descent (with learning rate $\eta$) satisfies \
\begin{equation}
  f (x, \theta (t)) = \tmmathbf{k} (x)^{\top} K^{- 1}  (I - (I - \eta K)^t)  (\tmmathbf{y} - \tmmathbf{f} (0)) + f (x, \theta (0)) \label{eqn:NTK-iter-t}
\end{equation}
for all valid $x$. If $0 < \eta \leqslant 1 / \lambda_1 (K)$, where $\lambda_i (K)$ denotes the $i$-th {\tmem{largest}} eigenvalue of $K$, then $\theta (t)$ converges to some ${\theta^{\star}}^{}$. At this solution, the form of $f$ is given by taking $t \rightarrow \infty$ in the above equation, yielding:
\begin{equation}
  f \left( {x, \theta^{\star}}^{} \right) = \tmmathbf{k} (x)^{\top} K^{- 1}  (\tmmathbf{y} - \tmmathbf{f} (0)) + f (x, \theta (0)) . \label{eqn:NTK-predict-new}
\end{equation}

\subsubsection{Convergence behavior on the training set}

Observing that $\tmmathbf{k} (x_i)^{\top}$ is the $i$-th row of $K$, we can ``stack'' both sides of Equation \ref{eqn:NTK-iter-t} to obtain
\begin{eqnarray*}
  \tmmathbf{f} (t) & = & KK^{- 1}  (I - (I - \eta K)^t)  (\tmmathbf{y} - \tmmathbf{f} (0)) + \tmmathbf{f} (0)\\
  & = & \tmmathbf{y} - \tmmathbf{f} (0) - (I - \eta K)^t  (\tmmathbf{y} - \tmmathbf{f} (0)) + \tmmathbf{f} (0)\\
  & = & \tmmathbf{y} - (I - \eta K)^t  (\tmmathbf{y} - \tmmathbf{f} (0)) \\
  & = & \tmmathbf{y} - \sum_{i = 1}^n (1 - \eta \lambda_i (K))^t  (v_i^{\top} (\tmmathbf{y} - \tmmathbf{f} (0))) v_i,
\end{eqnarray*}
where $v_i$ denotes the eigenvector of $K$ corresponding to the eigenvalue $\lambda_i (K)$.

Plugging this expression into the objective function and using the fact that the $v_i$'s are mutually orthogonal gives the following expression for the training loss:
\[ \frac{1}{2}  \sum_{i = 1}^n (y_i - f (x_i, \theta_t))^2 =  \frac{1}{2}  \| \tmmathbf{y} - \tmmathbf{f} (t) \|^2  =  \frac{1}{2}  \sum_{i = 1}^n (1 - \eta \lambda_i (K))^{2 t}  (v_i^{\top} (\tmmathbf{y} - \tmmathbf{f} (0)))^2 . \]
When $0 < \eta \leqslant 1 / \lambda_1 (K)$, this expression converges to 0 which implies that ${\theta^{\star}}^{}$ is indeed a global minimizer of the objective. Moreover, if we employ early stopping, then directions in function space corresponding to eigenvectors with with smaller eigenvalues in $K$ will have converged less than the others. As observed by \cite{jacot2018neural}, this may help explain early stopping's regularization benefits.

A complete picture of the convergence of the objective requires us to know the entire spectrum of $K$ and the coefficients $v_i^{\top} (\tmmathbf{y} - \tmmathbf{f} (0))$. However, assuming that $(v_n^{\top} (\tmmathbf{y} - \tmmathbf{f} (0)))^2$ is significantly large, the convergence speed will tend to $(1 - \eta \lambda_n (K))^{2 t}$ asymptotically. With the optimal learning rate of $\eta = 1 / \lambda_1 (K)$ this becomes $(1 - 1 / \tmop{cond} (K))^{2 t}$, where $\tmop{cond} (K) = \lambda_1 (K) / \lambda_n (K)$ is the condition number of $K$.

\subsubsection{Training only certain layers}

If we only optimize layer $i$, we may replace $K$ by $K_i$ and $\tmmathbf{k} (x)$ by $\tmmathbf{k}_i (x)$ in the above formulas (provided that it is positive definite) in order to obtain a description of the resulting convergence. As before, the training error will converge to zero at a speed determined by the eigenvalues of $K_i$. An analogous statement is also true if we optimize an arbitrary subset $S$ of the layers, in which case we replace $K$ by $\sum_{i \in S} K_i$ and $\tmmathbf{k} (x)$ by $\sum_{i \in S} \tmmathbf{k}_i (x)$. Note that because we are assuming infinitely wide layers there is no paradox here; each layer has enough capacity to memorize the training data entirely by itself.

The form of the NTK matrix allows us to gain insight into the relative contribution of each layer to the overall solution. A layer whose per-layer NTK matrix is much smaller\footnote{By ``smaller'' we mean that a PSD matrix $A$ is smaller than $B$, written $A \prec B$, if $B - A$ is positive definite.} than the other layers will have much smaller gradients, and the changes to its weights made during training will have a much smaller effect on the overall solution. While training any single layer is sufficient to achieve zero error in the infinite width case, what this arguably means for realistically sized networks is that layers with very small per-layer NTKs will train much slower than other layers.

\subsection{An elegant expression for the limiting NTK using C maps}\label{sec:NTK-via-C-map}

The inner-product which defines the per-layer NTK $\Theta_i (x, x')$ is a random variable that depends on the random initial value $\theta (0)$ of the parameters $\theta$. In the limit as the width of the network goes to infinity, $\Theta_i (x, x')$ converges in probability to a deterministic function in much the same way that the network's kernel function does. Indeed, an approximation result directly analogous to Theorem \ref{thm:error-bound-daniely} exists for the NTK \citep[][Theorem 3.1]{arora2019exact}. As we are performing our analysis in the infinite width limit, we will take $\Theta_i (x, x')$ to be this limiting value going forward.

Let $g_i$ be the subnetwork that maps the network's input to the input of the $i$-th combined layer (which is the output of the $(i - 1)$-th combined layer when $i \geqslant 2$). \citet{jacot2018neural} show that
\[ \Theta_i (x, x') = \left[ \widetilde{\kappa_{g_i}} (\Sigma_{x, x'}) \right]_{1, 2}  \prod^D_{j = i} \left[ \mathbb{E}_{u \sim \mathcal{N} \left( 0, \widetilde{\kappa_{g_j}} (\Sigma_{x, x'}) \right)} [\phi_j' (u) \phi_j' (u)^{\top}] \right]_{1, 2}, \]
where we note that $\widetilde{\kappa_{g_1}} (\Sigma_{x, x'}) = x^{\top} x' / d_0$ (since $g_1$ is the identity), and that the quantities inside of $[\cdot]_{1, 2}$ are $2 \times 2$ matrices (so that $[\cdot]_{1, 2}$ extracts their top corner entry).

Let $f_i$ represent the $i$-th combined layer of the network, $q_i$ its output q value (with $q_0 = 1$ being the q value for the network's input), and $\phi_i$ its activation function. By Equations \ref{eqn:C-map}, \ref{eqn:Gamma-def}, and \ref{eqn:C-map-gen-derivative}, we can write the above expression for the NTK as
\begin{eqnarray*}
  \Theta_i (x, x') & = & q_{i - 1} C_{g_i} (c_0)  \prod^D_{j = i} \Gamma_{\phi_j'} (C_{g_j} (c_0), q_{j - 1}, q_{j - 1})\\
  & = & q_{i - 1} C_{g_i} (c_0)  \prod^D_{j = i} \frac{q_j}{q_{j - 1}} C'_{f_j} (C_{g_j} (c_0))\\
  & = & q_{D + 1} C_{g_i} (c_0)  \prod^D_{j = i} C'_{f_j} (C_{g_j} (c_0)),
\end{eqnarray*}
where $c_0 \equiv x^{\top} x' / d_0$ is the c value for the network's input (recalling that $\| x \|^2 = \| x' \|^2 = d_0$ by assumption).

Denote by $h_i$ the subnetwork that maps the input of $f_i$ to the network's final output. Since we have $h_i = f_D \circ f_{D - 1} \circ \cdots \circ f_i$ it follows that $C_{h_i} = C_{f_{_D}} \circ C_{f_{_{D - 1}}} {\circ \cdots \circ C_{f_i}} $, and so by the chain rule we have $C'_{h_i} (C_{g_i} (c_0)) = \prod^D_{j = i} C'_{f_j} (C_{g_j} (c_0))$. Plugging this into the above equation we arrive at the elegant formula
\begin{equation}
  \Theta_i (x, x') = q_D C_{g_i} (c_0) C'_{h_i} (C_{g_i} (c_0)) . \label{eqn:NTK-elegant}
\end{equation}
While this formula has only been proven for deep MLPs (consisting of a composition of a sequence combined layers), we conjecture that it holds for more general architectures.

\subsection{The form of the NTK matrix given a degenerate C map and implications for gradient descent training}\label{sec:NTK-degen-C-map}

In this subsection we will consider the situation where a deep network $f$ has a ``degenerate'' C map $C_f$ that sends nearly all input $c$ values to a small region around some value $\fixpointofcmap$ (in the sense of Section \ref{sec:Cmaps_trainability}), and argue that this implies slow optimization and/or poor generalization in the NTK regime. This analysis can be seen as a more rigorous version of the intuitive argument given in Section \ref{sec:degen-C-maps}, and overlaps with the results of \citet{xiao2020disentangling}.

\subsubsection{Additional assumptions of this analysis}

To simplify the discussion, we will assume that each combined layer has the same activation function (except the last one, which is required to be linear), which means that the network's C map $C_f$ is just the composition of $D - 1$ copies of some local C map $\mathcal{C}$ . We will further assume that $\mathcal{C}$ is itself ``well-behaved'' in the sense that $\mathcal{C}' (1)$ is reasonably close to 1, so that the overall C map $C_f$ is degenerate only because $D$ is large. Moreover, any sufficiently ``deep'' subnetwork of $f$ is also degenerate, and any sufficiently ``shallow'' subnetwork is well-behaved.

Additionally, we will assume that $q_D = 1$ (without loss of generality), and that there are no two distinct inputs $x$ and $x'$, from either the training or test set, for which $x^{\top} x' / d_0$ is very close to $1$ or $- 1$ (which would imply that either $x \approx x'$ or $x \approx - x'$ given our previous assumption that $\| x \|^2 = \| x' \|^2 = d_0$).


\subsubsection{NTK matrix estimates}

As in Section \ref{sec:degen-C-maps} there are two main cases to consider for $\fixpointofcmap$: the ``collapsing case'', where $\fixpointofcmap = 1$, and the ``exploding case'', where $0 \leqslant \fixpointofcmap < 1$ with $\fixpointofcmap \not\approx 1$. For the collapsing case we must have $\mathcal{C}' (1) \leqslant 1$ since $\fixpointofcmap = 1$ is an {\tmem{attractive}} fixed point of $\mathcal{C}$. And for the exploding case we must have that $1$ is a non-attractive fixed point (since $\mathcal{C}$ can only have one such point by Proposition \ref{prop:C-map-fixed-point}), and so $\mathcal{C}' (1) > 1$. For the layer index $i$ there are three cases to consider: $i$ is small so that the layer is ``early'' in $f$, $D  - i$ is small so that the layer is ``late'' in $f$, and the default case where neither $i$ nor $D - i$ are small, so that the layer is the ``middle'' of the network.

The following table gives estimates of the layer-wise NTK matrix for each combination of cases. These estimates are computed in Appendix \ref{app:est-per-layer-NTK-mat} using a semi-rigorous style of argument. The results of these computations have been checked numerically for the case of RELU and Erf activation functions (whose C maps have convenient analytic forms). Here, the symbol $E$ denotes the matrix of 1's.

\begin{center}
{\noindent}\begin{tabularx}{1.0\textwidth}{|X|X|X|X|}
  \hline
  \ \textbf{Type of degeneration} & \textbf{Early layers} & \textbf{Middle layers} & \textbf{Later layers} \\
  \hline
  Collapsing case ($\fixpointofcmap = 1$) w/ $\mathcal{C}' (1) < 1$ & $K_i \approx 0$ & $K_i \approx 0$ & $K_i \approx \mathcal{C}' (1)^{D - i} E$\\
  \hline
  Collapsing case ($\fixpointofcmap = 1$) w/ $\mathcal{C}' (1) = 1$ & $K_i \approx I$ & $K_i \approx I + \alpha_i  (E - I)$ where $0 = \alpha_1 \leqslant \alpha_2 \leqslant \cdots \leqslant \alpha_D = 1$
  
  (Observed empirically, and conjectured to be true in general.)  & $K_i \approx E$\\
  \hline
  Exploding case ($0 \leqslant \fixpointofcmap < 1$, $\fixpointofcmap \not\approx 1$) w/ $\mathcal{C}' (1) > 1$ & $K_i \approx \mathcal{C}' (1)^{D - i} I$
  
  (very large) & $K_i \approx \mathcal{C}' (1)^{D - i} I$
  
  (very large, but still much smaller than for early layers) & $K_i \approx \mathcal{C}' (1)^{D - i} I + \fixpointofcmap  \mathcal{C}' (\fixpointofcmap)^{D - i}  (E - I)$
  
  ({\tmem{not}} very large)\\
  \hline
\end{tabularx}
\end{center}
From the above values we can compute an estimate of the overall NTK matrix. This is given the following table, which is computed in Appendix \ref{app:est-overall-NTK-mat}:

\begin{center}
{\noindent}\begin{tabularx}{1.0\textwidth}{|X|X|}
  \hline
  \  \textbf{Type of degeneration} & \textbf{Overall NTK matrix} \\
  \hline
  Collapsing case ($\fixpointofcmap = 1$) w/ $\mathcal{C}' (1) < 1$ & $K \approx \frac{1}{1 - \mathcal{C}' (1)} E$\\
  \hline
  Collapsing case ($\fixpointofcmap = 1$) w/ $\mathcal{C}' (1) = 1$ & $K \approx D (I + \bar{\alpha}  (E - I))$ for some $0 \leqslant \bar{\alpha} \leqslant 1$
  
  (Observed empirically for deep RELU networks with $\bar{\alpha} = 1 / 4$, and for other networks with $\bar{\alpha} = 1 / 3$. Conjectured to be true in general.)\\
  \hline
  Exploding case ($0 \leqslant \fixpointofcmap < 1$, $\fixpointofcmap \not\approx 1$) w/ $\mathcal{C}' (1) > 1$ & $K \approx \frac{1 - \mathcal{C}' (1)^D}{1 - \mathcal{C}' (1)} I + \frac{\fixpointofcmap}{1 - \mathcal{C}' (\fixpointofcmap)}  (E - I)$\\
  \hline
\end{tabularx}
\end{center}

\subsubsection{Implications for speed and generalization of gradient descent training}

There are several implications for gradient descent training that we can infer from the above estimates, all of which are bad.

Firstly, in the collapsing case with $\mathcal{C}' (1) < 1$, and in the exploding case, the magnitude of the per-layer NTK matrices differ substantially over the network. This means that the layers whose per-layer NTKs are not amoung the largest will train very slowly. Given that such layers are only a small fraction of the total, this implies that only a few layers of the network will have the potential to train quickly. While this is technically sufficient to minimize the training loss in the NTK regime, in practice, our networks often won't be highly overparameterized to the extent required by NTK theory, and so we actually will need to train all of the layers in order to fit the dataset. Insofar as the NTK regime is an analogy to this more realistic setting, this analysis thus predicts slow training.

Secondly, in the collapsing case with $\mathcal{C}' (1) < 1$, we have that the per-layer and overall NTK matrices are approximately rank 1, which implies that they have a very high condition number. This means neither the individual layers, nor the overall network, will train quickly, and so the training loss will take a very long time to be minimized no matter what subset of layers we elect to train.


Finally, in all cases we have that the approximate form of the per-layer and overall NTK matrices does not depend on the input training data. Additionally, we have that the vector $\tmmathbf{k} (x)$ does not depend on the training data, since by the derivations in Appendices \ref{app:est-per-layer-NTK-mat} and \ref{app:est-overall-NTK-mat}, $\Theta_i (x, x')$ doesn't depend on $x$ or $x'$ (except to detect when $x \approx x'$ or $x \approx - x'$). And $\tmmathbf{f} (0)$ also won't depend on the training data, since it will either look like a multiple of the ones vector in the collapsing case, or a completely random vector in the exploding case. It thus follows from Equation \ref{eqn:NTK-predict-new} that the predictions made by the fully trained network for a test point $x$ will not actually depend on the input training data in any significant way, making it impossible for the network to generalize. 

\subsubsection{Comparison to the results of \citet{xiao2020disentangling}}

The results of this subsection overlap with those of \citet{xiao2020disentangling}, who derive approximations to the overall NTK matrix for deep networks (although not for individual layers) using a different style of argument. Their results mostly agree with ours, except that for the case $\mathcal{C}' (1) = 1$ they predict a universal value of $\bar{\alpha} = 1 / 3$ (whereas we observe $\bar{\alpha} = 1 / 4$ for deep RELU networks), and for $\mathcal{C}' (1) = 1$ they estimate the second (and less significant) term of $K$ to be $\frac{1}{1 - \mathcal{C}' (\fixpointofcmap)}  (E - I)$, whereas we predict $\frac{\fixpointofcmap}{1 - \mathcal{C}' (\fixpointofcmap)}  (E - I)$. Numerical studies we performed on the Q/C maps of deep RELU/$\tmop{erf}$ networks seem to confirm our predictions in these cases.

\subsection{The form of the NTK under DKS}\label{sec:NTK-form-our-approach}

The following theorem is proved in Appendix \ref{app:NTK-our-approach-proof} using Theorem \ref{thm:deviation-bound} and Equation \ref{eqn:NTK-elegant}.

\begin{theorem}
  \label{thm:NTK-our-approach}Suppose that $\Theta_i$ is the per-layer NTK (for layer $i$) of a network conforming to the assumptions of Section \ref{sec:NTK-assumptions} which has been transformed using DKS with global slope bound $\zeta$. Then we have
  \[ \left| \Theta_i (x, x') - \frac{1}{d_0} x^{\top} x' \right| \leqslant 11 (\zeta - 1) . \]
\end{theorem}

The bound in this theorem establishes that each per-layer NTK matrix $K_i$ converges to the training data Gram matrix $X^{\top} X / d_0$ as $\zeta$ approaches 1, where $X = \left[\begin{array}{cccc}
  x_1 & x_2 & \cdots & x_n
\end{array}\right]$. It also allows us to reason about larger values of $\zeta$ to a limited extent, although it arguably only becomes interesting when $\zeta < 1 + \frac{1}{11}$. (We suspect that with a tighter and/or more detailed analysis, interesting statements about the relationship of $\zeta$ and the layer-wise NTK could be made for larger values of $\zeta$.)

If $X^{\top} X / d_0$ is low rank, which it will be in the common case that $\dim (x) < n$, this means that the $K$ will approach a low-rank matrix as $\zeta$ approaches 1, which corresponds to slow/impossible training under gradient descent. Intuitively this makes sense, since a value of $\zeta$ very close (or equal) to $1$ corresponds to a network that looks almost perfectly linear at initialization time (by Theorem \ref{thm:deviation-bound}), and thus could fail to properly train as per the discussion in Section \ref{sec:too-linear}. Indeed, the foundational works on NTK only predict that the NTK will be positive definite (i.e.~full-rank) when the activation functions are non-polynomial (and thus nonlinear) functions, and DKS makes them approach linear functions as $\zeta \rightarrow 1$.

So while a value of $\zeta$ very close to 1 is clearly a bad choice, a value somewhat close to $1$ (such as 1.5; which we use in most of our experiments) will allow $K$ to retain some of the structure of $X^{\top} X / d_0 $, thereby ensuring that the network's prediction (given in Equation \ref{eqn:NTK-predict-new}) depends on the training data and thus has the potential to generalize. It will also allow $K$ to deviate enough from $X^{\top} X / d_0$ to be full rank with a potentially small condition number (which would imply fast training). Unfortunately, $\tmop{cond} (K)$ is difficult to accurately estimate without full knowledge of both $X^{\top} X / d_0$ and the behavior of the C map over its entire domain, and existing methods to bound $\tmop{cond} (K)$ \citep[e.g.][]{du2018gradient} seem unlikely to produce useful results in our context. We leave the problem of accurately estimating the value of $\tmop{cond} (K)$ under DKS to future work.

\section{Variance propagation, signal propagation, and their relationship to approximate kernel analysis} \label{sec:var/sig-prop-relation-to-kernel}

Many previous methods for constructing and initializing neural networks are justified using analysis frameworks which attempt to characterize the initialization-time behavior of neural networks. The two most prominent examples of such frameworks are ``variance propagation'' and ``signal propagation''. 

In this section we review these frameworks, highlight certain mathematical issues with them, and provide counterexamples to their general claims where possible. We also relate them and their predictions to the kernel approximation framework underlying DKS, and advocate for the latter as a more powerful and mathematically rigorous alternative. 

\james{TOOD: maybe add more equations when describing var/sig prop?}

\subsection{The original variance propagation analysis of \citet{lecun1998efficient}}\label{sec:orig-var-prop}

The earliest such analysis that we are aware of appeared in \citet[][Section 4.6]{lecun1998efficient}, which we will call ``variance propagation''. It is based on the idea of computing the per-unit variance for each layer as a function of the per-unit variance of the previous layer, where the underlying distribution is over training cases. Typically, all units within a layer will have the same variance, and so only one scalar needs to be ``propagated''.

For fully-connected layers, \citet{lecun1998efficient} argue that the variance for a particular output unit is equal to the $\ell_2$-norm of that unit's vector of weights, multiplied by the per-unit variance of the input. For this to hold, they require that the input units are {\tmem{uncorrelated}} with each other and have the same variance. For nonlinear layers they assume the use of the activation functions that approximately preserve the mean and variance of their inputs. As an example, they give a transformed $\tanh$ activation function which resembles the identity function within a prescribed range of ``typical'' inputs around 0.

Assuming that each weight vector has a norm of approximately 1, and that the training input data is whitened, one might try to apply this single-layer argument recursively to all layers of the network, starting from the input. Unfortunately, this doesn't seem to work. While a whitening transform applied to the training data ensures that the input units of the first layer are uncorrelated with variance 1, uncorrelatedness will fail to hold for the output of this layer, making it impossible to apply the same argument for subsequent layers.

As no assumption is made about the network's parameters, beyond that the weight vectors must have a norm of $\sim\!\!1$, one can design a counterexample where these variance computations break down after the first fully-connected layer. For example, consider a linear neural network with a 1-dimensional input $x$, a 2-dimensional hidden layer, and a 1-dimensional output $y$, defined by the equations
\[ h_1 = x, \quad h_2 = - x, \quad \text{and} \quad y = \frac{1}{\sqrt{2}} h_1 + \frac{1}{\sqrt{2}} h_2 . \]
The input weight vector for each unit has norm 1, and yet $y = 0$ for all $x$, so that the network will {\tmem{not}} preserve the per-unit variance of $x$. (Intuitively, this is because $h_1$ and $h_2$ have strong negative correlation.) Note that this example can easily be generalized to arbitrarily wide layers.

Another more subtle issue with the analysis in \citet{lecun1998efficient} is that the approximation errors arising from the analysis of nonlinear layers can easily accumulate with depth, and may push the activation functions out of their assumed range of inputs.

\subsection{\poscite{klambauer2017self} modified variance propagation}\label{sec:SELU-var-prop}

\citet{klambauer2017self} present a modified version of \poscite{lecun1998efficient} variance propagation analysis, which would seem to address the issues we've highlighted. They do this by arguing that as long as the inputs to a fully-connected layer are independent, its outputs (which are fixed linear combinations of its inputs as determined by the weights) will be approximately Gaussian distributed, thanks to the Central Limit Theorem (CLT) and the assumption of wide layers. Using this approximation they then compute the moments of the subsequent nonlinear layer using Gaussian integrals (similar to those that define Q maps), without having to make any strong assumptions on its activation function. 

Unfortunately, CLT is not actually applicable to arbitrary weighted sums of variables, even when those variables are perfectly iid. For example, if the weights of the sum are $(1, 0, \ldots, 0)$, then the output will have the same distribution as the first input unit, which won't be Gaussian in general. Even if we somehow ruled out such weight vectors as having ``low probability'', and focused only on weight vectors for which CLT would apply, there would still be major difficulties to overcome.

Firstly, since we need to show that the output units of a fully-connected layer are approximately independent (in order to recursively apply the same analysis to subsequent layers), we would need to show that they are {\tmem{jointly}} Gaussian distributed with a diagonal covariance matrix. This would require the use of one of the multi-dimensional versions of CLT, all of which require significant additional hypotheses compared to the standard one-dimensional versions. Secondly, the approximate independence provided by CLT would not be sufficient to recursively apply the same argument to subsequent layers, as CLT typically requires {\tmem{exact}} independence of the variables under summation\footnote{In an attempt to preempt this criticism, \cite{klambauer2017self} refer to \citet{bradley1981central}, which proves a version of CLT that relaxes the independence assumption. However, this result assumes a very specific type of weak dependence which is unlikely to satisfied in this setting, and also only applies for one-dimensional variables.}. Thirdly, because CLT requires that the number of variables under summation is large, it usually won't be applicable to the first layer of the network (where the input dimension is a fixed property of the training data).

\subsection{The version of variance propagation in \citet{glorot2010understanding}}

\citet{glorot2010understanding} present a modified version of \poscite{lecun1998efficient} variance propagation analysis, which has formed the basis of many subsequent analyses over the years. The first change they make is to compute per-unit variances with respect to the {\tmem{joint distribution on network inputs and parameters}} (as opposed to just the inputs). Their second modification is to directly assume that the network's activation functions behave like the identity function over typical inputs, thus implying that they preserve the mean and variance of their inputs.

Assuming that the weights of a given fully-connected layer are iid with mean zero and variance $\sigma^2$, and that its input units are mean zero with variance $v$, they show that the per-unit variance of the layer's output is simply $k \sigma^2 v$, where $k$ is the input dimension.

Notably, by computing variances with respect to training cases \emph{and} parameters, they do not require the input units to a layer to be uncorrelated. (Intuitively, this is because the multiplication by the independent mean-zero random weights causes any two random variable to become decorrelated.) This addresses one of the main problems of the original variance propagation analysis, and allows it to be recursively applied over the entire network without issue. Unfortunately, the modification also introduces a new issue not present in prior analysis: {\tmem{the variances no longer refer to any single network (with a particular parameter setting), but rather to a distribution over networks.}} This makes the interpretation of these variances unclear, and represents a subtle but serious issue in their analysis.

One could possibly argue that, with high probability, a single network sampled from this distribution would have variances similar to those computed over the whole distribution. However, this would require additional hypotheses, since otherwise there are simple counterexamples to the general claim. For example, consider a linear network with $D \gg 1$ fully-connected layers of width 1, where the biases are zero and the weights are sampled iid from $\mathcal{N} (0, 1)$. The function computed by this network amounts to just multiplying its input by the product of $D$ scalar weights drawn independently from $\mathcal{N} (0, 1)$. Variance propagation would predict that such a network will exactly preserve the variance of its input, or in other words, that the product of these $D$ weights would be approximately 1. However, the distribution of the product of $D$ independent samples from $\mathcal{N} (0, 1)$ is {\tmem{highly}} concentrated around zero for even moderate large values of $D$ (which can be seen via Monte Carlo simulation), despite the fact that the variance of this product is 1.

These counterexamples are not restricted to narrow networks either. If, for example, the weights are drawn iid from a heavy-tailed distribution that is highly concentrated around zero and has variance $1 / k$ (where $k$ is the width), then even for large $k$ there will be an overwhelming probability that all the weights will be close to zero, leading to a network which ``squashes'' its input. This contradicts the prediction made by variance propagation, which is that such a network would approximately preserve the variance of its input.

\subsection{Extension of variance propagation to RELUs}\label{sec:He-init}

\citet{he2015delving} extend the version of variance propagation in \citet{glorot2010understanding} to deal specifically with RELU activation functions, as there is no zero-centered range of inputs for which RELUs resemble the identity function. To do this, they introduce the additional hypotheses that the weights have a symmetric distribution around zero, that the biases are initialized to zero, and that each RELU layer is directly preceded by a fully-connected layer. Given these hypotheses, it follows that the input to each RELU layer is distributed symmetrically around zero, and thus the expected squared valued of a RELU unit will be exactly $1 / 2$ times its input variance. One can then use this expected squared value in place of the input variance for the variance propagation calculation at the next layer, since multiplication by the mean-zero weights of said layer will restore a mean of zero.

While it deals with nonlinear layers in a cleaner fashion (at least for RELU networks), this analysis retains the central issue present in \poscite{glorot2010understanding} analysis, which is that the variances do not necessarily describe the behavior of a single network. Moreover, while the variance propagation formulas were originally derived for fully-connected networks, \citet{he2015delving} also applies them to convolutional networks without any additional justification. (This is problematic since the weight sharing violates the iid weights assumption.) 

\subsection{Extensions of variance propagation to normalizer-free residual networks}

\citet{zhang2019fixup}\footnote{\citet{zhang2019fixup} claim that their analysis actually describes the case of a constant network input, where only the network's parameters are random variables. However, all of their variance propagation equations express the per-unit output variance of a layer as a function of its per-unit input variance. As this variance will be zero for the first layer when the network's input is constant, this claim appears to be unsupported.}, \citet{de2020batch}, and \citet{shao2020normalization}\footnote{Technically, \citet{shao2020normalization} never actually specify what distribution they compute variances over. However, the only interpretation which makes sense, given the majority of their derivations, is that this is the distribution over network inputs and parameters. Despite this, they treat the parameters as fixed in some of their discussions.} adopt \poscite{glorot2010understanding} variance propagation framework to perform an analysis of ResNets (which are described in detail in Section \ref{sec:standard-ResNet}) without normalization layers. While a mostly straightforward application of the existing variance propagation formulas, they require an additional one which says that the output variance of a residual block is the sum of the output variances of its two branches.

While the formula itself is correct (given the conceits of variance propagation), as far as we can tell it has never been properly justified. \citet{zhang2019fixup} attempt to justify it using $\tmop{Var} (x + y) =\mathbb{E} [\tmop{Var} (y | x \nobracket)] + \tmop{Var} (x)$, although this formula is shown to be incorrect by taking $x = y$ \james{could someone confirm my interpretation here?}. \citet{shao2020normalization} give a different argument which assumes that the input units to a residual block are uncorrelated (which won't be true in general), and which treats the weights as fixed instead of random variables (which violates one of the core premises of variance propagation).

For reference, we will give an argument here for fully-connected networks. Let $x$ be the input to the residual block and $z$ be the input to the final fully-connected layer of the residual branch, which has weight matrix $W$. Given that $\mathbb{E} [W] = 0$ and $W$ is independent of $x$ and $z$ we have
\[ \tmop{Cov} (Wz, x) =\mathbb{E} [Wz (x -\mathbb{E} [x])^{\top}] =\mathbb{E} [W] \mathbb{E} [z (x -\mathbb{E} [x])^{\top}] = 0, \]
where we have used $\mathbb{E} [Wz] =\mathbb{E} [W] \mathbb{E} [z] = 0$. From this it follows that $\tmop{Var} (Wz + x) = \tmop{Var} (Wz) + \tmop{Var} (x)$.

\subsection{Signal propagation (aka mean field analysis)}\label{sec:signal-prop}

Closely related to variance propagation is an approach for understanding the initialization-time behavior of neural networks commonly referred to as ``signal propagation''\footnote{Note that some works \citep[e.g.][]{de2020batch} use the term ``signal propagation'' to refer to certain versions of what we have been calling ``variance propagation''. In such works elements of both types of analysis often appear, and the precise distinction between them becomes a bit blurry.} or ``mean field analysis'' \citep{poole2016exponential}. In this approach, instead of propagating variances, one propagates per-unit expected squared values 
, or expected products between corresponding units from two copies of the same network (each fed different inputs). Here, expectations are taken with respect to the distribution on network parameters, and on the two network inputs (which may be correlated). In order to propagate through nonlinear layers, one approximates their input as being Gaussian distributed with mean zero and covariance matrix determined by the expectations from the previous layer. As will be explained in the next subsection, the expectations computed under signal propagation end up being equal to q and m values (or c values, after suitable normalization) as we have defined them in this work. Indeed, Q/C maps were originally derived by \citet{poole2016exponential} in the context of signal propagation.

The mathematical justification of signal propagation given by \citet{poole2016exponential} in the case of a single fully-connected combined layer is roughly as follows. One starts from the assumption that the $k$ entries of the input vector are iid random variables with expected squared values given by $q$. Then, multiplication by an $m \times k$ random matrix with mean-zero iid entries of variance $\sigma^2 / k$ produces $m$ outputs, each of which has bounded variance $\sigma^2 q$, and is a sum of $k$ iid terms. For large $k$ one applies the Central Limit Theorem (CLT) to get that these sums will be approximately iid Gaussian distributed, with mean zero and variance $\sigma^2 q$. It then follows that the entry-wise outputs of the nonlinear layer are approximately iid, with expected squared values given by Gaussian integrals. Expected products are then handled using a straightforward generalisation of this argument.

In principle, this single layer argument can be applied recursively to a composition of combined layers, always starting from the hypothesis that the entry-wise inputs to a given layer are iid with some known expected squared value. Unfortunately, this recursive approach runs into the same problems with CLT discussed in the last paragraph of Section \ref{sec:SELU-var-prop}, which cannot be easily repaired\footnote{To the best of our knowledge, the only mathematically rigorous CLT-based treatment of the width-limiting behavior of random networks is that of \cite{matthews2018gaussian}, which is given in the context of approximate kernel analysis. It's not immediately obvious if/how \poscite{matthews2018gaussian} arguments can be used to rigorously justify signal propagation.}. 

As with variance propagation, the interpretation of the expectations computed under signal propagation isn't clear. In particular, there is no obvious relationship between these expectations, and the properties of a single randomly initialized network.

Signal propagation's two main advantages over variance propagation are that it handles nonlinearities in a much more general and precise way (via Gaussian integrals), and that it also describes the propagation of expected unit products for correlated network inputs. These features make it a much more powerful framework for understanding the understanding the initialization-time behavior of neural networks, and for designing initialization schemes. However, as we will discuss next, approximate kernel analysis has the same advantages while also being mathematically rigorous and more clearly interpretable. 

\subsection{Relationship of variance/signal propagation to approximate kernel analysis}\label{sec:var/sig-prop-apk-relationship}

When $\sigma^2 = 1$ (in the notation of the previous subsection), signal propagation's defining equations for fully-connected combined layers are precisely equivalent the local Q and C maps computed under approximate kernel analysis. We may thus interpret the quantities propagated by signal propagation as q and m values, and their normalized versions as c values. And for other values of $\sigma^2$, a similar statement holds for a slightly generalized notion of Q/C maps (as given in \citet{poole2016exponential}).

In some sense, this equivalence acts as a mathematical justification of signal propagation's equations, although with a different meaning for the quantities being propagated. 
In particular, the expected squared values computed by signal propagation correspond to q values, and can thus be thought of as approximations of dimension-normalized squared norms of the associated feature map's vectors. Similarly, the expected products computed by signal propagation correspond to m values, and can thus be viewed as approximations of the dimension-normalized inner-product between two such vectors (or the same vector for two different network inputs).

Given this relationship between approximate kernel analysis and signal propagation, we can also relate approximate kernel analysis to \poscite{glorot2010understanding} version of variance propagation (and its extensions). To so see this, note that insofar as the units in each layer have mean zero (under variance/signal propagation's assumed distribution), their expected squared values are equal to their variances, in which case variance propagation also computes q values. Moreover, even when the means are not zero, as is the case for RELU networks, one can modify variance propagation to deal directly with expected squared values in a manner similar to \citet{he2015delving}.

While approximate kernel analysis provides the same level of description as signal propagation, it has several advantages. The first is that it is based on a rigorous mathematical theory with clearly defined hypotheses and probabilistic error estimates. This allows one to be confident in determining which architectures it can be applied to, and to have a rigorous pathway for extending it to new architectures (which we exploited in our treatment of normalization and pooling layers). The second advantage is that the quantities it computes have a clear relationship to the (high probability) initialization-time behavior of actual randomly initialized networks with definite inputs and weights. The third is that it applies to networks with low dimensional inputs, for which the CLT-based arguments commonly used to justify signal propagation are inapplicable. And while these advantages come at the cost of additional/stronger hypotheses (such as Gaussian or SUO-distributed weights), such hypotheses are likely required in order for the predictions made by the equations to be accurate in general. 

\subsection{Extensions of variance/signal propagation to networks with Batch Normalization layers}\label{sec:var-sig-prop-for-bn}

\citet{de2020batch} propose an extension of variance propagation to networks with Batch Normalization (BN) layers, in order to analyze standard ResNets. To do this, they argue that for large mini-batches, BN layers will compute a per-unit empirical variance which closely matches the per-unit variance computed under variance propagation. Thus, after normalization by the square root of this variance, the per-unit output variance of a BN layer will be always be 1, regardless of its per-unit input variance. 

There appears to be a subtle issue with this argument. As discussed above, variance propagation is a faithful description of a single randomly initialized network (with definite inputs) only insofar as the variances it computes correspond to q values. q values in turn are approximations of dimension-normalized squared norms of entire activation vectors, and have no clear relationship to the properties of individual units within a layer of such a network. So in general, the empirical unit-wise variances computed by a BN layer will {\tmem{not}} correspond to the variances computed by variance propagation, even approximately. It is conceivable that with additional hypotheses on the batch size and distribution, the network, and the initialization, the empirical distribution of the values of each input unit to a BN layer (taken across the mini-batch, for {\tmem{fixed parameters}}) might all have roughly the same variance with high probability, in which case the approximation in \citet{de2020batch} would be a valid one. However, formalizing this would likely be quite difficult.

\citet{yang2019mean} propose an extension of signal propagation/mean field analysis to networks where BN layers are inserted between affine and nonlinear layers. To facilitate this, they propagate $B \times B$ matrices representing the expected products between different copies of the same unit for each of $B$ possible inputs to the network, where $B$ is the batch size. For networks without BN layers the propagation equations decomposes nicely in terms of low-dimensional Q/C maps, while for networks with BN layers no such decomposition exists, due to the way different elements of the mini-batch interact in BN layers. \citet{yang2019mean} are nonetheless able to analyze the resulting high-dimensional propagation equations using various sophisticated approximations and characterize their asymptotic fixed point behavior.

In their approach, the batch size, as well as the distribution used to generate the mini-batch, are encoded via the initial $B \times B$ matrix of expectations to the first layer. Thus, their analysis is not dependent on $B$ being large, or on any strong distributional assumptions about the mini-batch. However, unlike for networks with element-wise nonlinearities (where approximate kernel analysis gives rise to the same equations as signal propagation), there is no mathematically rigorous derivation of their generalized equations for BN layers. Thus, it remains an open question as to whether these equations are accurate approximations in the sense of Section \ref{sec:how-accurate-approx}. \citet{yang2019mean} provide empirical evidence that they are, at least for fairly wide networks with some commonly used activation functions.

\section{Review and analysis of related approaches for constructing and initializing deep neural networks}\label{sec:review-and-analysis-of-related}

In this section we will review some existing techniques, both standard and otherwise, for constructing and initializing neural networks in order to make them easier to train. We will further analyze these techniques from the perspective of approximate kernel theory by exploiting the latter's connections with variance/signal propagation established in Section \ref{sec:var/sig-prop-relation-to-kernel}.

\subsection{The fan-in initialization and related approaches derived from variance propagation}\label{sec:fan-in-and-friends}

The classical fan-in initialization \citep{lecun1998efficient} for fully-connected neural networks samples filter weights iid with mean zero and variance $1 / k$, where $k$ is the total input dimension. Here, $1 / k$ is precisely the value required for their version of variance propagation to predict constant per-unit variances throughout the entire network, with other values leading to an exponential increase or decrease with depth. However, as \poscite{lecun1998efficient} variance propagation analysis is only a reasonable approximation for activation functions that preserve the mean and variance of their input, their initialization will tend to fail in more realistic settings, especially as the network's depth increases \citep{he2015delving}.

\citet{glorot2010understanding} use their own version of variance propagation to motivate a similar initialization scheme, where the weight variance is $2 / (k + m)$, with $m$ being the output dimension. This choice is made as a ``compromise'' between the following two competing constraints: that the per-unit variances should be uniform across layers, and that the variances of the per-layer gradients should also be similarly uniform. 

We would argue that $2 / (k + m)$ is not a good choice in general compared to $1/k$. For example, if the layer widths alternate between between $n$ and $2 n$ for some $n \geqslant 1$, running variance propagation across two consecutive combined layers would predict a decrease in the variance by a factor $n (2 n)  (2 / (n + 2 n))^2 = 8 / 9$. This will lead to an exponential convergence of the variance towards zero as depth increases. Meanwhile, for the choice $1/k$, variance propagation (or approximate kernel analysis) predicts no such exponential increase or decrease for {\tmem{any}} choice of widths.

\citet{he2015delving} propose to use a weight variance of $2 / k$ specifically in RELU networks, which compensates for how RELU nonlinear layers decrease the variance by a factor of $1 / 2$ instead of preserving it. This is based on their expanded version of variance propagation that handles RELU activation functions.

Setting aside issues of mathematical rigor and the interpretation of the quantities being propagated, variance propagation and approximate kernel analysis involve similar calculations (as discussed in Section \ref{sec:var/sig-prop-apk-relationship}), and so these three initialization schemes can all be viewed as methods to control the q values of the network. When combined with the normalization of the input vectors (as per Section \ref{sec:PLN}), and applied to standard feed-forward fully-connected networks with suitable activation functions\footnote{Here, ``suitable" means (approximately) mean and variance preserving for the standard fan-in initialization, or RELU for \poscite{he2015delving} modified version. Note that the large majority of activation functions do not fall into the former category.}, the fan-in initialization method and its extensions achieve q values of $\sim\!\!1$ throughout the network, which is one of the four constraints enforced by DKS.

Having q values of 1 ensures that local C maps are the same for each combined layer (assuming they all use the same activation function), and that the final output of the network falls within a reasonable range. If this is not done, q values can grow very large or small with increasing depth, leading to various problems. In particular, very large values can cause bounded monotonic activation functions like $\tanh$ to ``saturate'', so that local C maps become increasing degenerate with depth. And very small values can cause most activation functions to behave in a way that is ``too linear'', which may limit the effective expressivity of the network (as per the discussion in Section \ref{sec:too-linear}). Notably, the RELU activation function is immune both of these issues due to it being positively homogeneous, which perhaps explains its popularity.

However, by not enforcing the other three conditions of DKS, networks using these initializations can still have degenerate network-level C maps, and can experience an exponential accumulation of kernel approximation errors with depth (so that the q values won't actually be constant in practice). As a concrete example of the former problem, consider the example from Section \ref{sec:relu-network-degen} of a standard deep RELU network. This network's C map doesn't depend on the input q value at all (as long its uniform), but still develops degenerate behavior at very high depths, leading to a network that is essentially untrainable.

Another more subtle issue with these initializations is that q values of 1 will work very badly for certain activation functions. For example, consider the activation function defined by $\phi (x) = \tanh (\alpha x)$. As $\alpha$ increases, an input q value of 1 becomes arbitrarily bad, leading to increasing levels of saturation and consequent C map degeneration. The reason that q values of one work reasonable well in practice is that the most commonly used activation functions in the literature happen to work well it, or have local C map behavior that is insensitive to q values (as is the case for RELUs). Notably, DKS does {\tmem{not}} suffer from this issue (despite also enforcing q values of 1), as its use of a multiplier on the input of each activation function ensures that 1 will always be optimal (since any other value can be effectively ``simulated"). 


\subsection{Layer-Sequential Unit-Variance initialization and Within-Layer initialization}\label{sec:LSUV}

\citet{mishkin2015all} proposed an initialization method called Layer-Sequential Unit-Variance (LSUV), which uses an iterative procedure that starts from a standard random initialization and adjusts the scale of each weight matrix/filter bank to achieve the condition that the variance of the output of each affine layer -- taken over the channels, locations and training cases -- is approximately equal to 1. These variances are computed by evaluating the network empirically on random mini-batches of training data.

By taking $\phi$ in Equation \ref{eqn:average-unit-approx} to be the identity function, we have that the average value (across channels) for each location-wise output vector of an affine layer is approximately zero with high probability. The variances computed by LSUV can therefore be interpreted as estimates of the length-normalized squared norms of the location vectors for each layer, except that they are also averaged over locations and network inputs. Thus, we can think of LSUV as enforcing the condition that the ``average q value'' for the output of each affine layer is equal to 1. This condition is similar to the one that the fan-in initialization (and its variants) are trying to achieve, and thus our discussion and critique of those methods (in Section \ref{sec:fan-in-and-friends}) also applies to LSUV. 
In particular, a network initialized with LSUV can still have degenerate C maps, with all of their consequent problems.

Because LSUV uses empirically computed statistics instead of canned formulas, it takes into account the given architecture and network topology, as well as the properties of the input training vectors. By contrast, variants of the fan-in initialization are usually only valid for the particular activation function and network topology they were derived for, despite often being applied more generally. They also implicitly assume that the training input vectors are appropriately normalized, which isn't always the case in practice. (Note that DKS, while it is also based on formulas instead of empirical evaluations of the network, takes into account the activation functions and network topology, and is packaged with a data preprocessing technique for the input vectors.)

The ``Within-Layer Initialization'' (WLI) of \citet{krahenbuhl2016data} can be viewed as a modification of LSUV where one enforces the condition that the mean and variance of the output of each layer \tmtextit{and each channel} is 0 and 1 respectively (as opposed to LSUV, which considers the average variance across channels). This is done by rescaling the filter bank weights separately for each output channel, and setting the bias appropriately. 
Because it enforces conditions per-channel instead of averaging across channels, this modification is harder to compare directly to fan-in initializations or DKS. It is perhaps more closely related to Batch Normalization (which is discussed later in this section), as it achieves the same conditions that BN does before the first step of optimization.

\subsection{Self-normalizing neural networks}\label{sec:SELU-nets}

Self-normalizing neural networks \citep{klambauer2017self} use Scaled Exponential Linear Unit (SELU) activation functions to achieve a per-unit mean of 0 and variance of 1 asymptotically with depth, as computed under variance propagation. Due to the relationship between variance propagation and Q/C maps (discussed in Section \ref{sec:var/sig-prop-apk-relationship}), this is equivalent to $Q_g(g)$ having an attractive fixed point at $q=1$, and $C_g(c)$ having an attractive fixed point at $c=0$, where $g$ is a SELU nonlinear layer. Assuming the use of PLN, and a standard feed-forward architecture (or normalized sums in more general architectures), these conditions imply that $Q_f(1)=1$ and $C_f(0)=0$ for all subnetworks $f$, which is two of the four conditions enforced by DKS.

As previously discussed, $Q_f(1)=1$ for all subnetworks $f$ is a good condition to have, and will prevent extreme q values from developing in deep networks (which can adversely affect C maps). However, it won't in general guarantee a well-behaved C map in deep networks, even when combined with the condition $C_f(0)=0$. 

For a SELU nonlinear layer $g$ we have $C_g' (1) \approx 1.0716$, which can be computed numerically using Equation \ref{sec:estimate-expectations} and the methods described in Section \ref{sec:estimate-expectations}. Along with the condition $C_g (0) = 0$, this guarantees a well-behaved C map up to a modest depth. For example, given a standard 100 layer network $f$ we have $C_f (0) = 0$ and $C'_f (1) = \mathcal{C}' (1)^{100} \approx 1.005 \cdot 10^3$, so that $C_f$ is reasonably well behaved according to Theorem \ref{thm:deviation-bound}. However, if $f$ has 300 layers then we have $C'_f (1) \approx 1.0157 \cdot 10^9$, which indicates degenerate behavior with $\fixpointofcmap = 0$ (in the sense of Section \ref{sec:degen-C-maps}).

\subsection{The ``Edge of Chaos'' (EOC) method}\label{sec:EOC-method}

Consider a network $f$ defined by a composition of $D$ combined layers, each with the same nonlinear activation function $\phi$. Every combined layer will have the same Q map, which we denote by $\mathcal{Q}$. As discussed in Section \ref{sec:edgeofchaos-q-solution}, $\mathcal{Q}$ will typically have a fixed point $\qfixedpoint$ which is rapidly converged to under repeated applications, and thus one may approximate the q values as uniform and constant across layers. This allows one to define a local C map $\mathcal{C}$ which only depends on the input c value, and which is the same for each combined layer.

As argued by \citet{schoenholz2016deep}, $\mathcal{C}' (1)$ will describe the asymptotic dynamics of the c values as they evolve through the layers of the network in the limit as $D \rightarrow \infty$. $\mathcal{C}' (1) > 1$ indicates rapid convergence of the c values to $1$, while $\mathcal{C}' (1) < 1$ indicates rapid convergence to a value $c_0 < 1$. Because of its close proximity to these two undesirable depth-limiting asymptotic behaviors, a network with $\mathcal{C}' (1) = 1$ is said to be ``on the edge of chaos'', and will have its c values converge slowly towards 1 at an asymptotic rate which is sub-exponential.

In the initialization method proposed by \citet{schoenholz2016deep}, which we will call the \tmtextbf{\tmtextit{``Edge of Chaos'' method (EOC)}}, one initializes the weights using a standard Gaussian fan-in method, with the variance of the weights and biases chosen so that $\mathcal{C}' (1) = 1$. (Note that $\qfixedpoint$ also depends on these variances, which is taken into account when computing $\mathcal{C}' (1)$.) As observed by \citet{schoenholz2016deep}, there are typically infinitely many combinations for these two variances which achieve $\mathcal{C}' (1) = 1$ (assuming any exist for the given activation function), and so one is chosen arbitrarily. Notably, the condition $\mathcal{C}' (1) = 1$ is based entirely on the properties of $\mathcal{C}$, and as such does not depend on the depth of the network.

In \citet{xiao2018dynamical}, a version of EOC was proposed that used an Orthogonal Delta initialization technique for convolutional layers, with variances chosen so that $\mathcal{C}' (1) = 1$. In their experiments \citet{xiao2018dynamical} showed that basic convolutional neural networks (without skip-connections or batch normalization) can be successfully trained on CIFAR-10 with the resulting initialization at depths of up to 10,000. This was a remarkable result, as such networks are considered essentially impossible to train at even modestly large depths when initialized using standard methods.

\subsubsection{Relationship to DKS}

DKS is in many ways a spiritual successor to EOC, and is derived using an extended version of the same basic Q/C map analysis that underlies the latter. Like EOC, DKS also makes use of the Orthogonal Delta initialization technique. But despite these similarities, there are many important differences between the two methods, which we will discuss in sequence below.

Firstly, DKS enforces uniform q values via data pre-processing instead of relying on the (presumed) convergent fixed-point behavior of the Q maps to achieve this asymptotically. See Section \ref{sec:edgeofchaos-q-solution} for a more detailed discussion of this point.

Secondly, instead of targeting the condition $\mathcal{C}' (1) = 1$ for each combined layer $f$, DKS targets $\mathcal{C}' (1) = \mu^{- 1} (\zeta)$, so that the ``degree of nonlinearity'' is calibrated to the given architecture. This is motivated by looking at the overall C map behavior of the network, instead of the fixed point convergence behavior of its local C maps. From the perspective of fixed point convergence, DKS achieves exponential rate towards $\fixpointofcmap = 0$, while EOC achieves sub-exponential rate towards $\fixpointofcmap = 1$. While this might seem like a point against DKS, one must remember that the precise rate of its exponential convergence will depend\footnote{In particular, the worst-case rate of convergence to the fixed point will be given by $\mathcal{C}' (0)$, which by Equation \ref{eqn:C-slope-ineq} satisfies $\mathcal{C}' (0) \geqslant 2 - \mathcal{C}' (1) = 2 - \mu^{- 1} (\zeta)$. For a $D$ layer convolutional network this is $2 - \zeta^{1 / D}$, which will be just slightly below $1.0$ when $D$ is large (given typical choices for $\zeta$).} on the overall depth $D$, the effect of which is that the c values will be far from converged even by the $D$-th layer.

The third difference between DKS and EOC is that DKS manipulate the network by transforming the input and output of the activation functions, as opposed to changing the variances of the weights and biases. If we consider the ``equivalent parameters'' (as per Section \ref{sec:method-as-pure-init}), DKS is implicitly searching over a space of distributions with more degrees of freedom than the two used by EOC, and one which allows for non-zero correlations between the weights and biases.

Fourthly, despite having more degrees of freedom with which to manipulate the network, DKS makes use of all of them in order to enforce a total of four conditions on the Q and C maps of the network (which are listed in Section \ref{sec:local_map_conditions_list}). EOC meanwhile only enforces a single condition ($\mathcal{C}' (1) = 1$), which leaves one of its two degrees of freedom unconstrained. (The effect of this is that EOC has a manifold of possible weight and biases variances from which to choose, and is therefore under-determined.)

Fifthly and finally, DKS is applicable to a diverse set of architectures, thanks to our generalized notion of Q/C maps, analysis of pooling layers, and network polynomial construction. As it was originally developed, EOC assumes a strictly feedforward network consisting of a sequence of fully-connected or convolutional combined layers.


\subsubsection{Possible failure modes of EOC}

It is worth pointing out that the condition $\mathcal{C}' (1) = 1$ enforced by EOC does not necessarily imply that the network will look perfectly linear at initialization time, as $\mathcal{C} (0) = 0$ is not enforced in EOC as it is in DKS. Nonetheless, depending on the activation function, there may be choices for the weight and bias variances which can make the network look ``too linear'', thus leading to very slow training as per the discussion in Section \ref{sec:too-linear}. As such choices are not explicitly forbidden in EOC, and are indeed compatible with the condition $\mathcal{C}' (1) = 1$, this may represent a failure mode of the method.

Conversely, without $\mathcal{C} (0) = 0$, the condition $\mathcal{C}' (1) = 1$ may not sufficient to ensure that the entire network's C map $C_f$ is well-behaved. For example, given a choice of 1 and 0 for the variances of the weights and biases respectively, we have $\mathcal{C}' (1) = 1$ for unmodified RELU activation functions (by Section \ref{sec:relu-network-degen}), and yet $C_f$ quickly degenerates as $D$ grows, as can be seen from the figures in Section \ref{sec:relu-network-degen}. Moreover, our experiments in Section \ref{sec:main-experiments} confirm that standard deep RELU networks (without BN layers or skip connections) are not readily trained at high depths, even with a high-powered optimizer like K-FAC.

\subsection{The Looks Linear method}\label{sec:looks-linear-method}

\citet{balduzzi2017shattered} use path-weight analysis (which we review in Appendix \ref{app:path-weight-analysis}) to argue that gradients with respect to the network's input will decorrelate or ``shatter'' in deep RELU networks, leading to difficulties when training with gradient descent\footnote{Note that this prediction agrees with our NTK analysis in the sense that the per-layer NTK matrix of the first layer of a deep RELU network is approximately the identity, although the implications for training are somewhat different; see Appendix \ref{app:path-weight-analysis}}. They also argue that this happens to a much lesser extent in ResNets.

Motivated by these observations, and by the fact that this effect doesn't occur in a purely linear network, \citet{balduzzi2017shattered} propose the Looks Linear (LL) method for initializing/constructing RELU networks. This method exploits the fact that $\phi (x) - \phi (- x) = x$ when $\phi$ is the RELU function in order to construct a network that behaves exactly like a linear one at initialization. In particular, one replaces $\phi$ by $(\phi (x), \phi (- x))$ for each RELU nonlinear layer (which effectively doubles its output channel dimension), and initializes the weights of the affine layer after each RELU layer according to $(W, - W)$, where $W$ is sampled according to a Delta Orthogonal initialization. For fully-connected combined layers this produces an overall computation $W \phi (x) - W \phi (- x) = W (\phi (x) - \phi (- x)) = Wx$, while for convolutional combined layers the computation is similarly linear (although harder to express in standard matrix notation).

From the perspective of our analysis, perfectly linear-looking networks have (very) well-behaved C maps, and thus satisfy one of the main necessary conditions for trainability. With such networks there is always the danger that they may be ``too linear'' in the sense of Section \ref{sec:too-linear}, but LL-initialized networks avoid this because $W_1 \phi (x) + W_2 \phi (- x)$ will become highly nonlinear given a relatively small perturbations of $W_1$ away from $- W_2$.

The main two obvious disadvantages to the LL approach are that it only works for RELU networks, and that it doubles the widths of RELU layers (without proportionally increasing the network's expressivity/capacity). Beyond these things, the main difference between LL and DKS is the precise mechanism used to make the network ``look linear'', and the implications that this has for optimization (which is not well understood in either case). In DKS, the degree of nonlinearity of a nonlinear layer, as measured by properties of its C map (such as the slope at $c = 1$), varies smoothly as a function of the parameters of the transformed activation functions, and so one could argue that it should also vary smoothly as a function of the network's parameters, resulting in easier optimization. By contrast, the linearity property of the LL method depends on a delicate mirrored symmetry of the weights in each layer, so that relatively small perturbations in these could lead to large changes in the degree of nonlinearity of the network. This sensitivity may make optimization more difficult, and may explain the optimization difficulties we observed in our experiments with the LL approach. (See Section \ref{sec:looks-linear-experiments} for more details.)

\subsection{Residual connections}\label{sec:resnet-related-discussion}

Residual Networks aka ResNets \citep{he2016deep, he2016identity}, which are described in detail in Section \ref{sec:standard-ResNet}, have become the dominant neural network architecture for computer vision problems. What makes ResNets so successful isn't that they are more powerful or expressive than other more traditional deep convolutional architectures like VGG \citep{simonyan2015very}, but rather that they are easier to train with stochastic gradient descent at very high depths \citep{he2016deep, szegedy2017inception}. This easier training is owed to their use of skip connections (aka shortcut connections; which have been a feature of network architectures since the 1990s), Batch Normalization (BN) layers \citep{ioffe2015batch}, RELU nonlinearities, and the surprising interplay between all three of these components \citep{de2020batch}. Moreover, popular new architectures such as Efficient Nets \citep{tan2019efficientnet} and Transformers \citep{vaswani2017attention} are based on the same high-level residual block structure, and differ only in terms the layers contained in their residual branches.

ResNets, and their generalizations, thus represent a solution to the problem of how to achieve fast and stable training of very deep neural networks. And while the nature of this solution is still not totally understood, there has been progress in this direction (of which we will cover only a small subset).

\citet{veit2016residual} argued that residual networks behave like a ensemble of shallow networks of varying depth throughout training. They gave evidence for this by showing that deep residual networks are highly robust to ``lesion" operations which remove or rearrange layers, and that the network's gradient is dominated by contributions made by paths through the network with fewer nonlinear layers

\citet{zhang2019fixup} observed that if one removes the BN layers from a ResNet-V2 network and initializes the last convolutional layer of each residual branch to zero (along with a few other smaller tweaks to the architecture and its initialization), the resulting network achieves training speed comparable to a standard ResNet, at least at modest batch sizes. In such networks, the residual blocks act as identity functions at initialization tune, only becoming nonlinear as training progresses. Subsequent work showed that one could achieve similar results in BN-free networks simply by using learnable weights on the residual branches that are initialized to zero \citep{de2020batch, bachlechner2020rezero}. More recently, it was found by \citet{shao2020normalization} that the branch sum can use static (non-learnable) weights, where the relative size of the weight on the residual branch is set to a small value (that can vary between blocks).

To help explain these findings for BN-free networks, \citet{de2020batch} applied a version of variance propagation to argue that the per-unit output variance of a residual block will be roughly 1 plus its per-unit input variance, so that the $i$-th residual block has a variance proportional to $i$. Then, because the output variance of each residual {\tmem{branch}} is constant (due to the use of BN), it follows that the {\tmem{relative}} contribution to the block's output made by the residual branch shrinks as $1 / i$. This, they argue, leads to a network which behaves more like an linear function than it otherwise would. In Appendix \ref{app:resnet-map-analysis} we make this argument more rigorous by computing q values in a (nearly) standard ResNet, and showing that their growth over layers leads to a better behaved C map. We also show that an identical C map can be obtained in a network without normalization layers via careful selection of weights on the residual and shortcut branches.


\subsection{Normalization layers}

Normalization layers (the two most common types of which are defined in Section \ref{sec:normalization-layers}) have become a standard component in neural networks, since the introduction of Batch Normalization (BN) by \citet{ioffe2015batch}. In addition to the important and specific role they play in ResNet-style architectures (as discussed in Section \ref{sec:resnet-related-discussion} and Appendix \ref{app:resnet-map-analysis}), these layers have been observed to make deep neural networks easier to train on their own. In this subsection we will discuss possible explanations for this, with a particular focus on ones arising from Q/C map analysis. We will also give some arguments for why normalization layers alone are insufficient to enable fast training of deep networks.

\subsubsection{Layer Normalization layers}

As discussed in Section \ref{sec:layernorm-maps}, a Layer Normalization (LN) layer $f$ has the property that $Q_f (q) = 1$ for all $q$, provided that its learnable gain and bias are set to their initial values. Additionally, when applied after a combined/nonlinear layer $g$, the C map of the composition has the property that $C_{f \circ g} (0) = 0$, regardless of C map behavior of $g$. Thus, when used after each nonlinear layer in a network initialized as per Section \ref{sec:param_dist}, LN layers achieve uniform q values of 1 throughout the network, and also $C_h (0) = 0$ for all subnetworks $h$, which are two of the four conditions enforced by DKS. (Note that if LN layers are instead inserted {\tmem{before}} each nonlinear layer, they will not achieve the latter condition.)

The way LN layers achieve these conditions differs from DKS in at least two ways. First, they perform direct calculation of the relevant quantities instead of using q and c values as approximations. This allows them to work with arbitrary initializations of the parameters (including badly scaled ones), poorly scaled input data, and without any explicit knowledge of the network's structure or activation functions. Second, LN layers continue to enforce a version of these conditions throughout training, or at least as long as their learnable gain and bias remain close to their initial values of 1 and 0.

As discussed previously (e.g.~Subsection \ref{sec:fan-in-and-friends}), uniform q values of 1 is a useful property to have, but is far from sufficient to ensure trainability. The condition $C_h (0) = 0$ for all subnetworks $h$ is meanwhile only one half of the two conditions required by Theorem \ref{thm:deviation-bound} to ensure well-behaved C maps, and arguably the less important of the two.

To make this discussion more concrete, we will consider how putting LN layers after each nonlinear layer will effect the C map of a standard deep RELU network. For a RELU nonlinear layer $g$ we have by Equation \ref{eqn:C-map-RELU} that $C_g (0) = 1 / \pi$ and $C_g' (1) = 1$ (which follows from Equation \ref{eqn:C-map-RELU} by taking the derivative and letting $c \rightarrow 1$). Taking the derivative in Equation \ref{eqn:C-map-LN} we have $C_f' (c) = 1 / (1 - C_g (0)) = \pi / (\pi - 1)$, and so by the chain rule $C'_{f \circ g} (1) = C_f' (C_g (1)) C_g' (1) = C_f' (1) C_g' (1) = \pi / (\pi - 1) \approx 1.467$. Thus we see that while the use of an LN layer after a RELU layer gives us $C_{f \circ g} (0) = 0$, it comes at the price of increasing the C map slope from 1 to $\sim\!\!1.467$.

The following plot shows the extended C map for a RELU network $h$ with 20 combined layers, with and without LN layers used after each nonlinear layer:
\begin{center}
\resizebox{4.5in}{3in}{\includegraphics{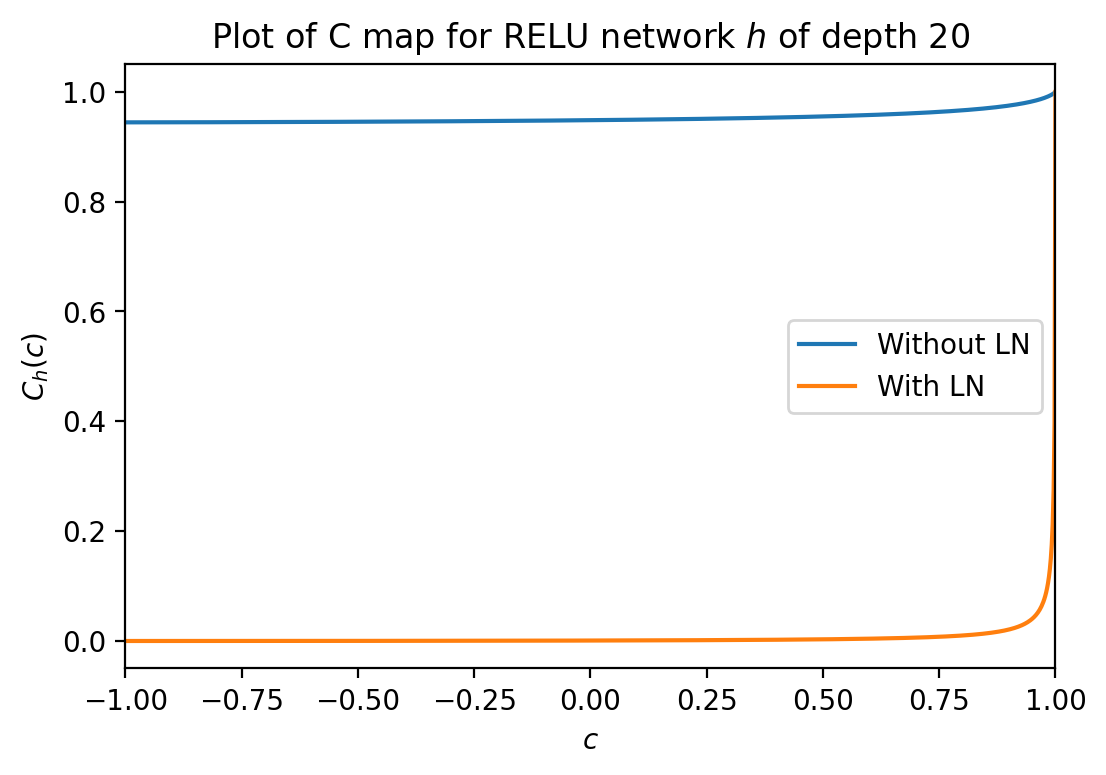}}
\end{center}

From this plot we can see that the C map for the network with LN layers has a much larger output range. However, for the vast majority of its input domain, the output is restricted to a small region around 0, and it is still highly degenerate in the sense of Section \ref{sec:degen-C-maps} (and thus suggestive of poor training).

Beyond their affect on the initial behavior of the network, LN layers may also have an independent and possibly beneficial effect on optimization, as they change the relationship of the loss and the parameters. \citet{ba2016layer} argue that LN layers lead to a Fisher information matrix with more favorable properties for optimization. Another intuition is that an LN layer decouples the scale and direction of the weights of its immediately preceding affine layer, which may encourage faster optimization with gradient descent. Networks with LN layers may also be also be smoother when considered as functions of either their parameters or their inputs, since the change in the output of an LN layer is always bounded. Despite these intuitions, as far as we know there has yet to be strong theoretical or empirical evidence in favor of a specific optimization benefit to LN layers beyond their affect on the network's initial behavior.

\subsubsection{Batch Normalization layers}

As discussed in Section \ref{sec:batchnorm}, Batch Normalization (BN) layers cannot be analyzed within the Q/C map framework we have presented. Despite this, we can still make some observations regarding their effect on network behavior in the context of our previous discussions.

As shown in Section \ref{sec:degen-C-maps}, one common way that a deep network $f$ can become difficult/impossible to train is when all input vectors map to approximately the same output vector (as measured by cosine similarity) at deeper layers of the network. This happens naturally in deep RELU networks, where BN is typically applied. Placing BN layers throughout the network may mitigate this particular pathology by ensuring that the empirical distribution of each unit over the mini-batch has a large variance compared to its mean. However, this won't obviously do anything to help with the opposite problem discussed in Section \ref{sec:degen-C-maps}, where output vectors appear ``random'', and in particular fail to reflect the geometric relationships between the original input vectors.

These intuitions are confirmed by \poscite{yang2019mean} signal propagation analysis of RELU networks with BN layers (which we discuss in Section \ref{sec:var-sig-prop-for-bn}). In particular, \citet{yang2019mean} predict that the distances between the output vectors (generated from different inputs) will converge to a constant as depth increases, and that this leads to an exponential increase in the norm of the gradient.

Like with LN layers, placing BN layers after each affine layer makes the network insensitive to the scale of its weight parameters, which can thus correct for badly scaled initial weights. One can perhaps also view BN layers as ensuring that the ``average q value'' across the mini-batch is 1, although this is an imperfect analogy since BN layers operate on a per-channel basis instead of averaging over channels.

Similar to LN layers, BN layers may also have an effect on optimization which is independent from their effect on the network's initial behavior. Evidence for this includes the fact that various methods which modify networks and their initializations in an attempt to eliminate the need for BN (such as those we've previously discussed) fail to achieve the same optimization performance under SGD, except perhaps at small mini-batches sizes (where classical optimization considerations like curvature matter a lot less, as argued in \citet{zhang2019algorithmic}). 

In support of the optimization-effect hypothesis, \citet{santurkar2018does} argue that BN layers make a network's output a smoother function of its parameters, and that this helps improves the performance of gradient descent. \citet{li2019exponential} argue that gradient descent applied to networks with BN layers behaves similarly to gradient descent applied to a normalizer-free network with a decaying learning rate, thus allowing gradient descent with a constant learning rate to converge in the stochastic setting (where it otherwise might not). Finally, \citet{roger-notes} argues that placing a BN layer after an affine layer $g$ will make the network invariant to scaling and shifting of $g$'s input, and that this leads to a curvature matrix for $f$'s parameters which is better conditioned.



\part{Experiments and conclusions}
\label{part:experiments_and_conclusions}

\section{Experimental setup}\label{sec:experimental-setup}

In this section we will describe and justify the setup we will use in our experiments, which will depart somewhat from common practice.

\subsection{Training problem and datasets}

The benchmark training problem we use in all of our experiments is image classification, on either the Imagenet \citep{deng2009imagenet} and CIFAR-10 \citep{krizhevsky2009learning} datasets.

The training objective is the average loss over the training set, with the loss given by the cross-entropy error between network's output (interpreted as ``logits'' of a softmax) and the dataset labels. We also measure top-1 classification accuracy, and report this instead of the loss in our plots due to its higher interpretability.

For Imagenet, we use an image preprocessing and random augmentation pipeline similar to the one from \citet{szegedy2015going} to obtain images of size $224 \times 224$. The training set is obtained from the standard Imagenet training set, minus the last 10000 cases (which is used as a new validation set), and the test set is obtained from the usual Imagenet validation set. Training accuracy is reported using the examples actually used during training, which are subject to random augmentation. Test accuracy is meanwhile reported using examples from the test set without random augmentation.

For CIFAR-10, we apply the standard preprocessing consisting of mean subtraction and normalization of each color channel. The training and test sets are their standard versions.

For both datasets we apply Per-Location Normalization (as described in Section \ref{sec:PLN}) as a final stage of processing before feeding the inputs to the network. This is done for all approaches and experiments unless stated otherwise, in the interest of fairness.

\subsection{Focusing on optimization speed}

While we will report test set accuracy in many of our experiments, our primary focus will be on optimization speed, as measured using training accuracy. Moreover, the decisions we make while designing out experiments will be in the interest of obtaining the cleanest and fairest comparison for optimization speed, and we will tune various components (like the learning rate schedule, regularization, etc) with this in mind. In this subsection we will explain our rationale for this decision.

The current \tmtextbf{\tmtextit{standard approach to deep learning}} is to train normalized residual architectures with RELU nonlinearities, such as ResNets or Transformers, with basic optimizers like SGD or Adam. Alternative approaches (such as normalization-free networks using Fix-up \citep{zhang2019fixup}, standard deep convolutional networks initialized with EOC, or DKS) can underperform the standard approach in one of two basic ways. First, they can yield networks whose training plateaus earlier, resulting in underfitting, or whose training is just much slower overall. And second, they can yield a worse inductive bias for typical training problems (like Imagenet classification), resulting in increased overfitting.

In our initial experiments we found that while alternative deep learning approaches are typically affected by both of these problems, slower training is by far the more significant one, particularly for networks without skip connections. Moreover, the resulting underfitting problem (given a finite optimization step budget) led to a commensurate degradation in test set performance. (These findings echo those of \citet{he2016deep} and \citet{szegedy2017inception}, who observed that the main benefit of adding skip connections was faster optimization.)  Thus, by focusing on training speed, we are isolating what is the more serious problem currently affecting alternative methods, and the one which arguably should be addressed before attempting to close the generalization gap.

We believe that the increased generalization gap we observed on Imagenet for alternative approaches such as ours is small enough that it can be overcome through the use of additional regularization strategies, dataset augmentation, scheduling of the optimizer hyperparameters, architectural tweaks, etc. We will leave this to future work. This position is echoed by \citet{zhang2019fixup}, and was arguably validated in the recent work of \citet{brock2021high} on ``Normalizer-Free Networks'', which used a combination of these techniques to close the generalization gap for one such alternative approach. It is also our view that overfitting may become less of a concern as the machine learning community moves beyond supervised benchmark problems like Imagenet classification, and towards giant/streaming datasets and unsupervised methods.

\subsection{Network architectures and regularization}\label{sec:exper-setup-nets-and-reg}

In our experiments we train standard and modified ResNets and Wide-ResNet models for Imagenet and CIFAR-10 image classification. (A detailed description of all the relevant architectures is given in Section \ref{sec:application-to-resnet}.) 

We will place particular emphasis on ``ablated'' versions of ResNets, where Batch Normalization (BN) layers and/or the skip connections are removed, leaving everything else unchanged. The motivation for doing this is that we want to facilitate the fairest possible comparison to the standard deep learning approach. In particular, since we are focused mostly on optimization speed in our comparisons, we want to use models that are provably no more powerful than standard ResNets in terms of the class of functions they can express, so that the fundamental data fitting problem doesn't become any ``easier''.

In the interest of making our experiments fair we also didn't include the standard L2 regularization that is often used when training ResNets. This is because the effect of L2 regularization on the effective capacity of a model is highly dependent on the model's parameterization, and this will vary significantly across the different approaches we consider. For example, due to the way BN layers are invariant to scalar multiplication of their inputs, one can rescale the weights of any affine layer that precedes a BN layer without changing the overall output of the network. Thus, networks with BN layers can effectively ``cheat'' the L2 regularization penalty in a way that networks without BN layers cannot. In our experiments we found that the removal of L2 regularization did have a small but still significant effect on the test set performance of standard ResNets, which is reflected in our reported results.

Note that our purpose in experimenting with these modified ResNets is not to show they are a good replacement for standard ResNets in practice. Rather, our purpose is to determine the extent to which we can replace the ingredients of the standard deep learning approach with various alternatives (that preserve the model class), while retaining its fast training capabilities. If we were primarily interested in maximizing test set performance in our evaluations, then we would be free to design a network architecture best suited to DKS, and to include whatever regularization scheme we found to be most effective. And while this does seem like an interesting direction to explore, as has been done in the context of other alternative approaches to deep learning \citep{brock2021high}, it is beyond the scope of the present work.

\subsection{Automatic learning rate schedules with Fire PBT}\label{sec:FIRE PBT}

Achieving a near optimal rate of convergence for standard ResNet training with SGD seems to require a carefully designed learning rate schedule (and not just a fixed value), especially for more difficult datasets like Imagenet. Through extensive and costly trial and error, the community has produced learning rate schedules which seem to work well on certain standard problems, such as Imagenet classification with ResNets. These typically involve a quick ``warm-up'' of the learning rate from a moderate starting value to a larger one, followed by a decay or step-wise descent towards zero. 

In our experiments we consider a large variety of approaches for training deep networks, most of which depart from the standard one along directions such as architectural choices, optimizers, initialization, etc. There is no reason to think that a learning rate schedule tuned for standard ResNet training with SGD should perform well for all such approaches, and this was borne out in our initial experiments. (By contrast, it seemed like the momentum hyperparameter was much less important.) Thus, in order to conduct fair experiments, which are minimally confounded by hyperparameter tuning, we need a way of determining a near-optimal learning rate schedule for each approach. And this should ideally be done in an automatic way in order to reduce the role of experimenter bias.

Recently, \citet{firepbt} proposed an alternative version of Population Based Training \citep[PBT,][]{jaderberg2017population} called FIRE PBT, which is designed specifically for the dynamic adjustment of optimizer hyperparameters. Like many other methods for automatically tuning the learning rate, \emph{standard} PBT falls into the trap of being too greedy, and tends to lower the learning rate too quickly for the sake of short-term improvements in the loss \citep{wu2018understanding}. FIRE PBT is designed to tackle this issue, using a strategy which we will now briefly explain.

Both PBT and FIRE PBT work by having many workers independently train neural networks, each with their own values for the hyperparameters. Both methods also associate a fitness to each of their workers which guides an evolutionary procedure. In PBT, this fitness is simply the current value of the objective function (which can be defined on the training or test sets). In FIRE PBT these fitnesses are altered in order to promote population members which may have a worse objective but are promising in other ways. In particular, a separate class of workers, called evaluators, periodically copy the model parameters of other workers, change the hyperparameters (e.g.~decay the learning rate), and measure the \emph{rate} at which the objective function improves while training with the new hyperparameters. The higher the rate of improvement as measured by the evaluator, the higher the fitness FIRE PBT will associate to the original worker whose model parameters were copied. This approach encourages workers to use ``non-greedy" hyperparameters (such as high learning rates), if it is shown that doing so leads to better performance after training with different hyperparameters (such as lower learning rates) in the long run.

In their experiments, \citet{firepbt} showed that FIRE PBT worked very well at automatically generating learning rate schedules on the fly for standard ResNet training with SGD, matching or exceeding the performance of the previously mentioned community-tuned schedules. In our initial experiments we found that this capability carried over nicely to non-standard deep learning approaches as well, and so we decided to use it in all of our subsequent experiments.

We now discuss the technical settings related to our use of FIRE PBT. We follow the presentation of \citet{firepbt}. Each experiment uses 36 workers. We divide them into three sub-populations $\mathcal{P}_1, \mathcal{P}_2, \mathcal{P}_3$, each of size 8, and the evaluator set $\mathcal{H}$ which includes the remaining 12 workers. We train for a maximum of 200,000 steps when training on ImageNet and for a maximum of 25,000 when training on CIFAR-10.

\textbf{Hyperparameters} We optimise the learning rate hyperparameter. When using SGD or Adam, the learning rate is initially sampled log-uniformly in the range $[10^{-5}, 1]$. When using K-FAC, we instead use the range $[10^{-7}, 10^{-3}]$.

\textbf{Objective function} We evaluate the current model by reporting the current negated training loss. 

\textbf{Ready} A member of the population is deemed ready to exploit and explore every 500 steps when training on ImageNet and 50 steps when training on CIFAR-10. 

\textbf{Exploit} We use a truncation selector: If a population member has a fitness in the bottom 25\% of the population, it copies the neural network weights and hyperparameters of a random member in the top 25\% of the population.

\textbf{Explore} We multiply the learning rate by a value sampled at uniform random between the following two value: [0.8, 1.25]. 

We further set the FIRE PBT hyperparameter of \python{max_eval_steps} to 7200 when training on ImageNet, and 360 when training on CIFAR-10. We set the hyperparameter \python{min_steps_before_eval} to 5000 and 10000 for $\mathcal{P}_2$ and $\mathcal{P}_3$ respectively when training on ImageNet, and to 250 and 500 when training on CIFAR-10.

The training curves plotted in our results section use the values recorded by the sequence of workers that led to the best eventual objective value (negative loss).

In Appendix \ref{app:PBT-LR-example} we plot the learning rate schedule found by FIRE PBT for some of our main experiments. We note that apart from some small fluctuations, these schedules are fairly simple and natural looking, and typically involve an initial rapid increase in the learning rate, followed by a gradual decay. Thus, we don't believe that the qualitative nature of our results is highly dependent on our use of FIRE PBT versus a simpler approach for learning rate tuning.

\subsection{Optimizers}

In our experiments we used SGD (with momentum), K-FAC \citep{martens2015optimizing}, Adam \citep{kingma2014adam}, and Shampoo \citep{gupta2018shampoo, anil2020scalable} as optimizers, with the majority just using SGD and K-FAC. Our motivation for considering stronger optimizers is that alternative deep learning approaches such as ours seem to benefit substantially from using them.

For all optimizers we used a momentum parameter of 0.9, and adjusted the learning rate dynamically throughout training using FIRE PBT. For Imagenet experiments we used a batch size of 512 with all optimizers, and for CIFAR-10 we used a batch size of 1024.

For Adam we used a value of $10^{-5}$ for the ``$\epsilon$'' parameter, which performed slightly better than the default value of $10^{-8}$.

For K-FAC, we used a 0.99 exponential decay of the curvature matrix, and computed its inverse every 50 iterations. We initialized K-FAC's damping parameter $\lambda$ to $10^{-3}$, and exponentially decayed it at the rate 0.98 every 50 iterations to a minimum value of $10^{-6}$. Finally, we enforced a maximum norm of $10^{-2}$ on all updates, with the norm computed using K-FAC's approximate curvature matrix (as in \citet{ba2017distributed}).

For Shampoo we used an epsilon parameter of $10^{-5}$ and an exponential decay factor if 0.99 for the second moments. In order to achieve optimization performance which was competitive with K-FAC, we used an ``exponent multiplier'' of 3 (which increases the exponent of all factors of the preconditioner by a factor of 3, with 1 being the default value), and enabled ``grafting'' (which uses Adagrad \citep{duchi2011adaptive} to compute the magnitude of the update for each parameter tensor, and the usual Shampoo formula to compute its direction).

\subsection{Hardware and implementation details}

All of our experiments were implemented in TensorFlow \citep{tensorflow2015-whitepaper}. Each of the 36 workers used by FIRE PBT ran on an 16 chip 32 core Cloud TPU v3 Pod \citep{tpus}. For multi-core TPU Pods, each core ran a ``replica" of the entire gradient computation on its assigned subset of the training mini-batch, with gradients and other key optimization quantities being averaged across the cores to simulate a single core computation. 

As long as training cases are independent of each other in the forward pass, this simulation is exact. However, this independence is slightly violated for networks with BN layers, and the resulting simulation is thus imperfect. Handling BN in this way in the multi-core setting has nonetheless become standard practice, and is even thought to be beneficial as it increases the ``noise" originating from BN layers, which is thought to have a regularizing effect.

\section{Experimental results} \label{sec:experimental-results}




In this section we will present our main experimental results as a series of plots of training/test top-1 accuracy vs iteration number, with some discussion. Most of our experiments will use the standard RELU, tanh, and softplus activation functions, the latter of which is a smooth analogue of the RELU function defined by $\phi (x) = \log (1 + \exp (x))$.


\setcounter{tocdepth}{3}

\localtableofcontents 

\setcounter{tocdepth}{1}

\james{TODO: add an overview/roadmap of experiments for this section with a summary of conclusions?}

\subsection{Default hyperparameters and network configurations assumed for each experiment}\label{sec:default-experiment-settings}

We use a value of $1.5$ in all experiments for DKS's global slope bound parameter $\zeta$, unless stated otherwise.

Except for DKS and the Looks Linear method, whenever using RELU activation functions we multiply the network's initial weights by $\sqrt{2}$. This has become standard practice in the literature following \citet{he2015delving}, and can be interpreted as making the local Q map of combined RELU layers equal to the identity. We don't do this for DKS or the Looks Linear method since those methods achieve identity local Q maps through other means.

{\tmem{Unless otherwise indicated, all result will be given for a skip connection-free BN-free modified ResNet-101 architecture trained on Imagenet.}} ``Standard ResNet'' will refer to a standard unmodified ResNet with RELU activation functions, initialized with the standard Gaussian Fan-in initialization (with a $\sqrt{2}$ multiplier). For networks trained on CIFAR-10 we will use a modified Wide-Resnet with 250 layers and a width multiplier of 2.

\subsection{DKS with skip-free nets vs standard baselines}\label{sec:main-experiments}

In this subsection we present our main results, in which we compare DKS networks without skip connections or BN layers, to both standard ResNets, and various ``ablated'' ResNets that are missing skip connections or BN layers (or both).

\

\resizebox{0.85\columnwidth}{!}{\includegraphics{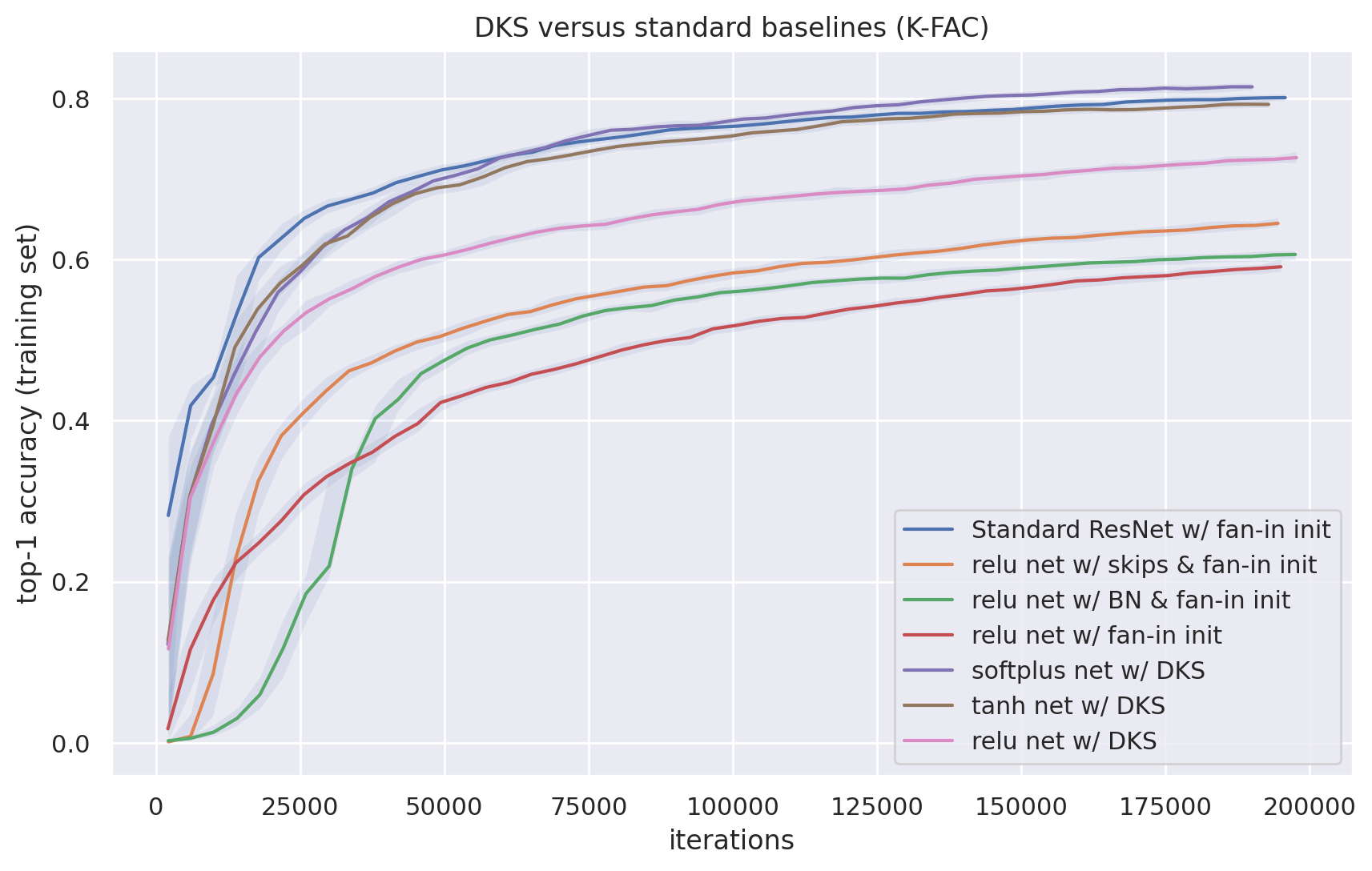}}

From this first plot we can see that, with K-FAC, DKS enables skip-free BN-free networks to train as fast as a standard ResNet on Imagenet, which is the first time this has been demonstrated to the best of our knowledge. Meanwhile, the ablated ResNets exhibit significantly slower optimization or underfitting. We also see that DKS underperforms for RELU compared to other activation functions, perhaps for the reasons discussed in Section \ref{sec:positively-homogeneous-problem}.

\

\resizebox{0.85\columnwidth}{!}{\includegraphics{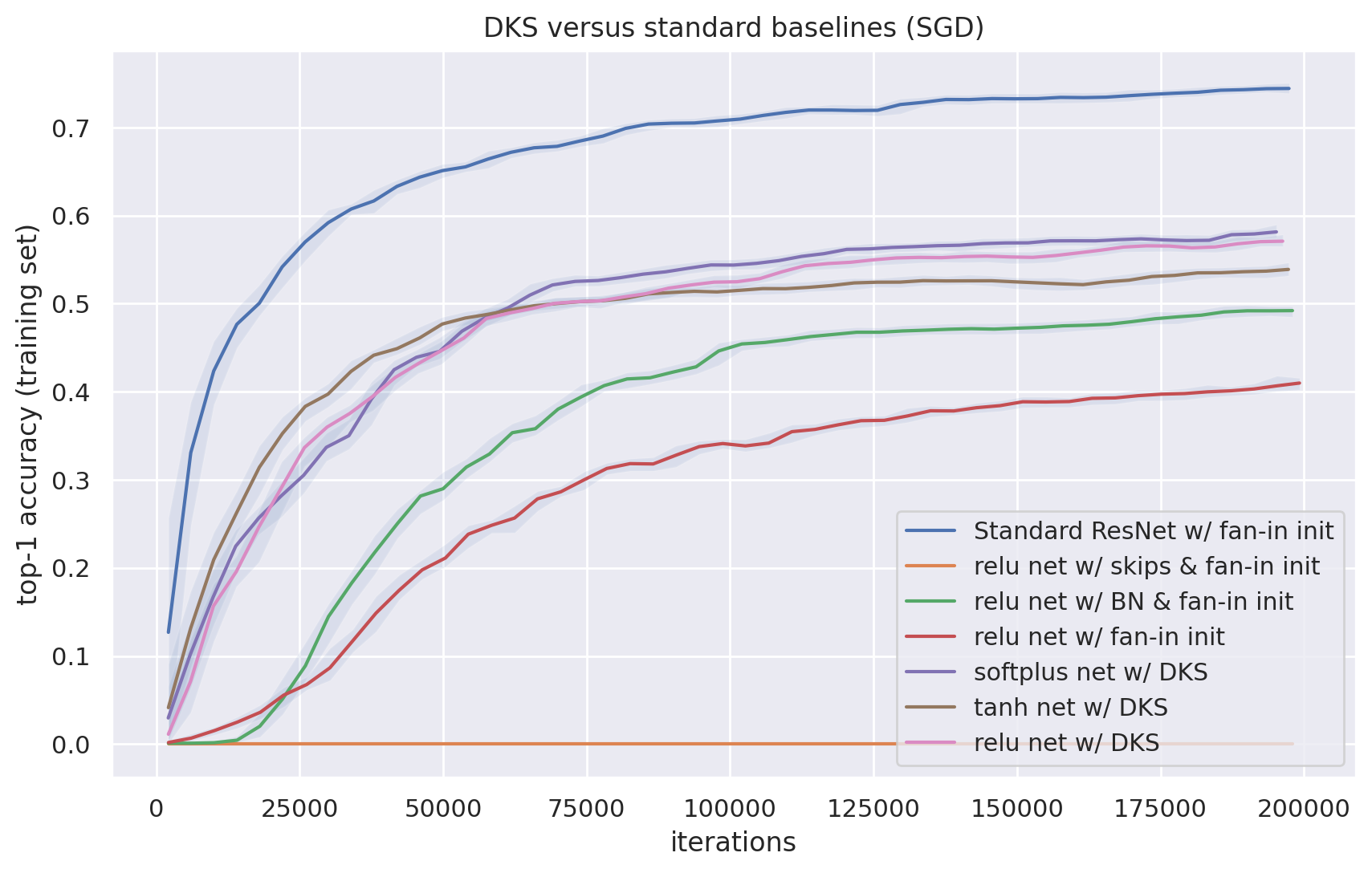}}

The story is somewhat different for SGD training. With SGD and no skip connections, DKS networks fail to match the training speed of standard ResNets, although they still outperform the ablated ResNets. Interestingly, RELUs give the same performance with DKS as the other activation functions do in this setting.

\

\resizebox{0.85\columnwidth}{!}{\includegraphics{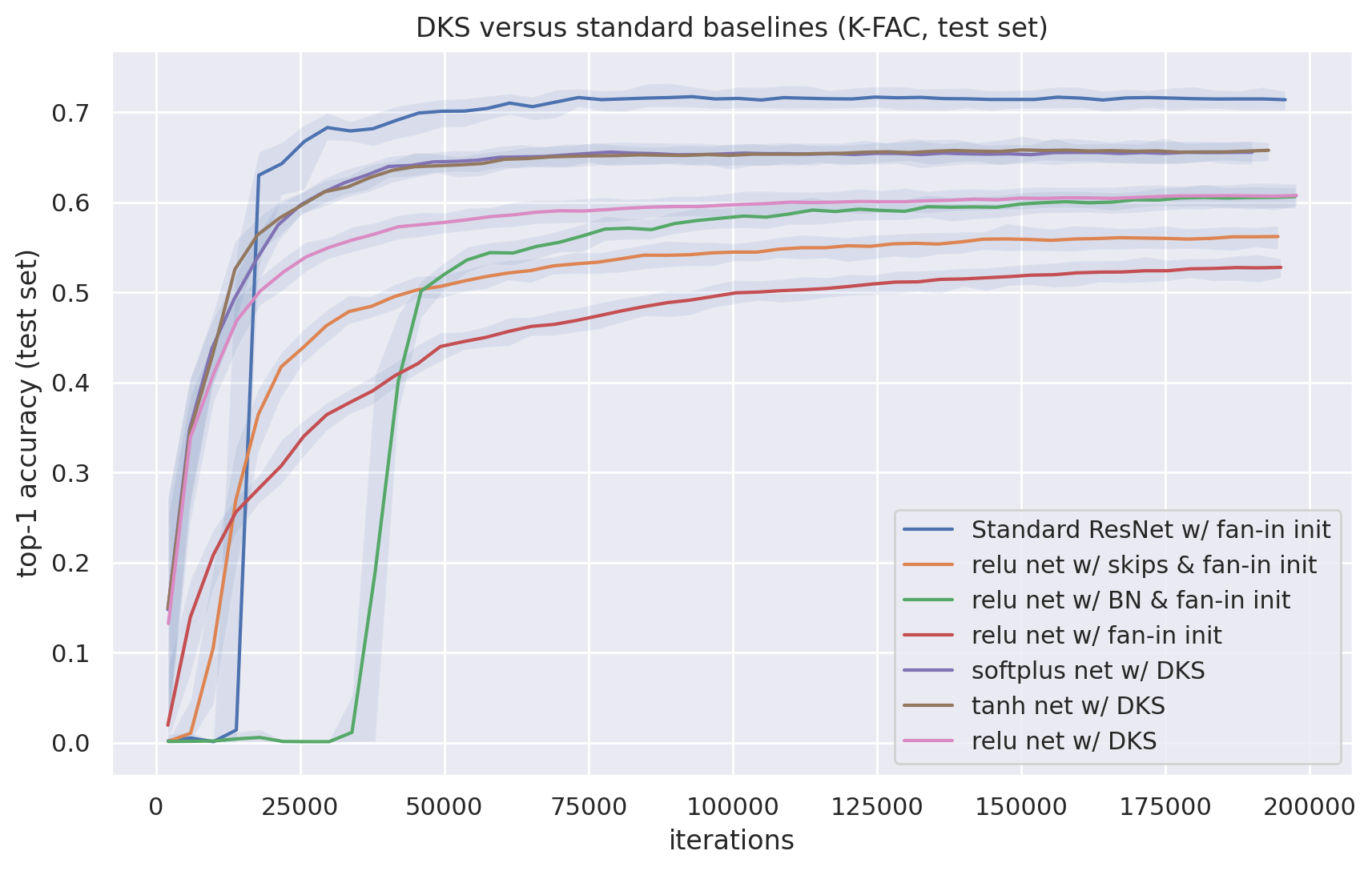}}

For test set performance with K-FAC training we observe increased overfitting with DKS compared to standard ResNets, resulting in an overall lower test accuracy. Notably however, test accuracy is still higher than for the ablated ResNets.

\

\resizebox{0.85\columnwidth}{!}{\includegraphics{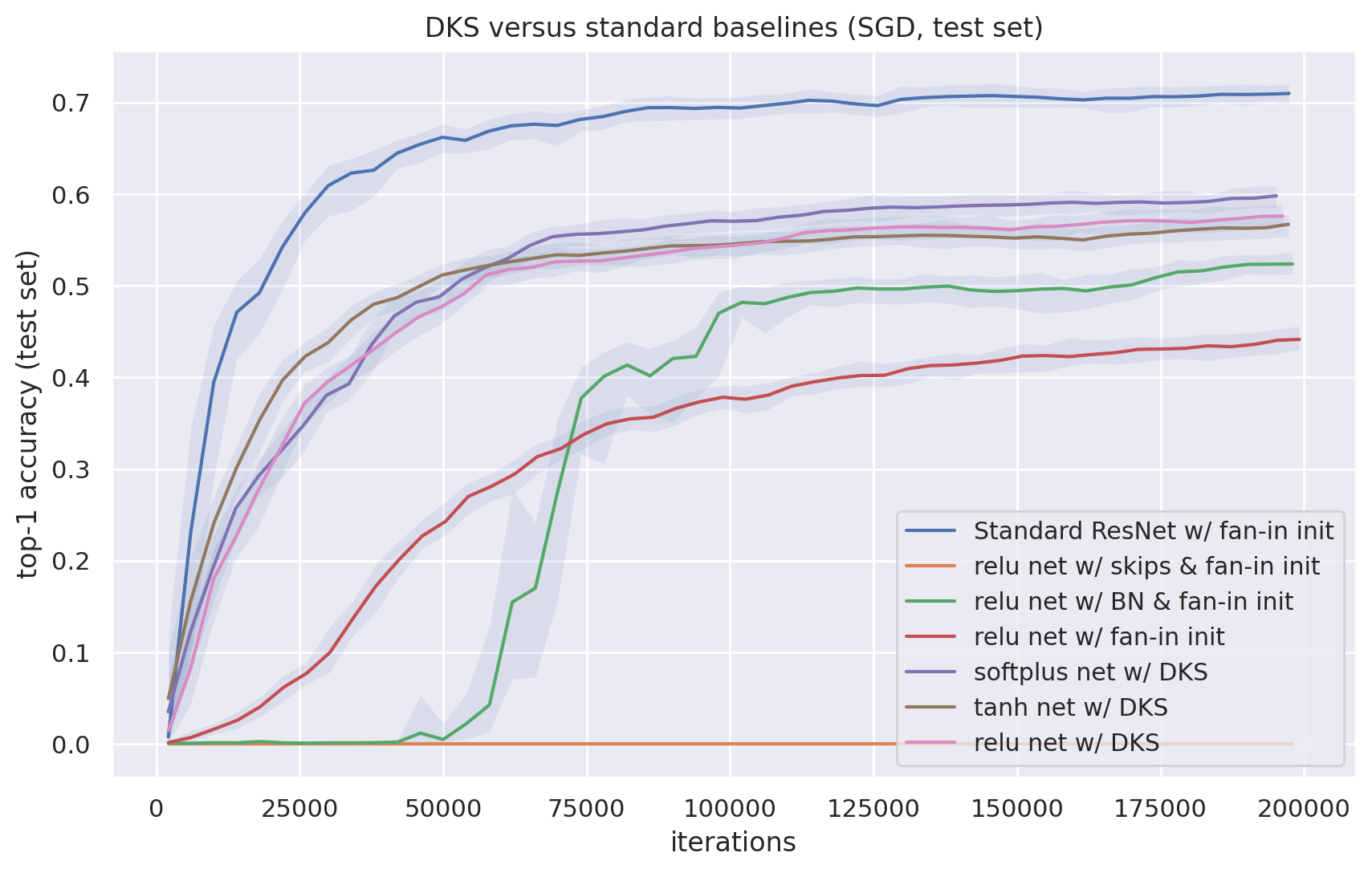}}

Once again the story is somewhat similar for SGD training, although with a larger performance gap vs standard ResNets due to the additional effect of underfitting from using SGD (without skip connections) instead of K-FAC.

Note that the test error numbers for standard ResNet training with SGD are a few percentage points the commonly reported values. This is for a number of reasons, including the fact that we don't include L2 regularization (as discussed in Section \ref{sec:exper-setup-nets-and-reg}), that we configured FIRE PBT to maximize training speed and not test set performance, and that we use PLN to process the data. (Because these things affect DKS as well, we believe the comparison to still be fair.)

The remaining results in this subsection are analogous to the previous ones, but use CIFAR-10 with modified/ablated Wide-ResNets models. The observations from these results are similar, although we note that the performance gap between the DKS networks and ablated Wide-ResNets is considerably larger, likely due to the higher depth (250 vs 100) used in these experiments.

\

\resizebox{0.85\columnwidth}{!}{\includegraphics{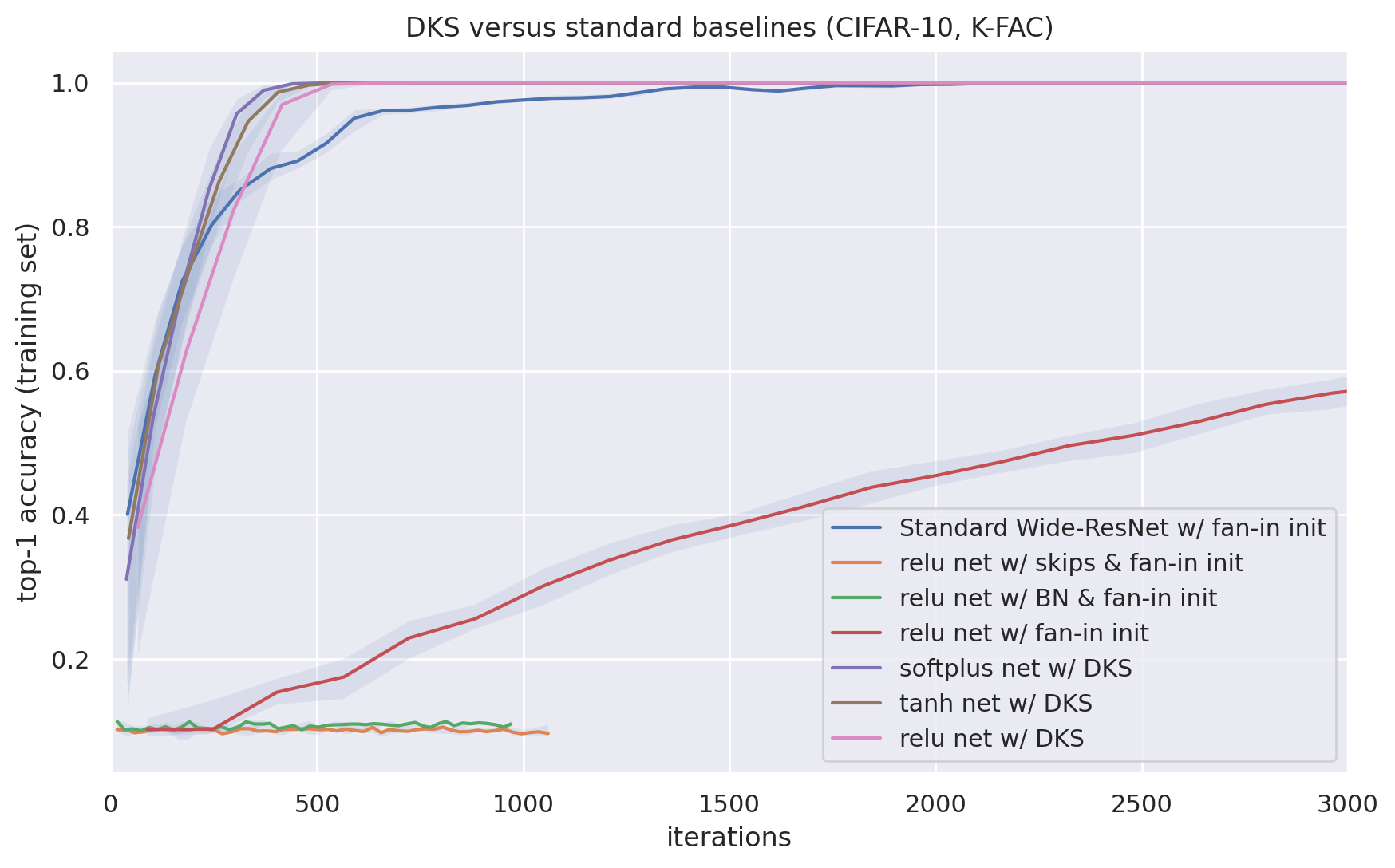}}

\resizebox{0.85\columnwidth}{!}{\includegraphics{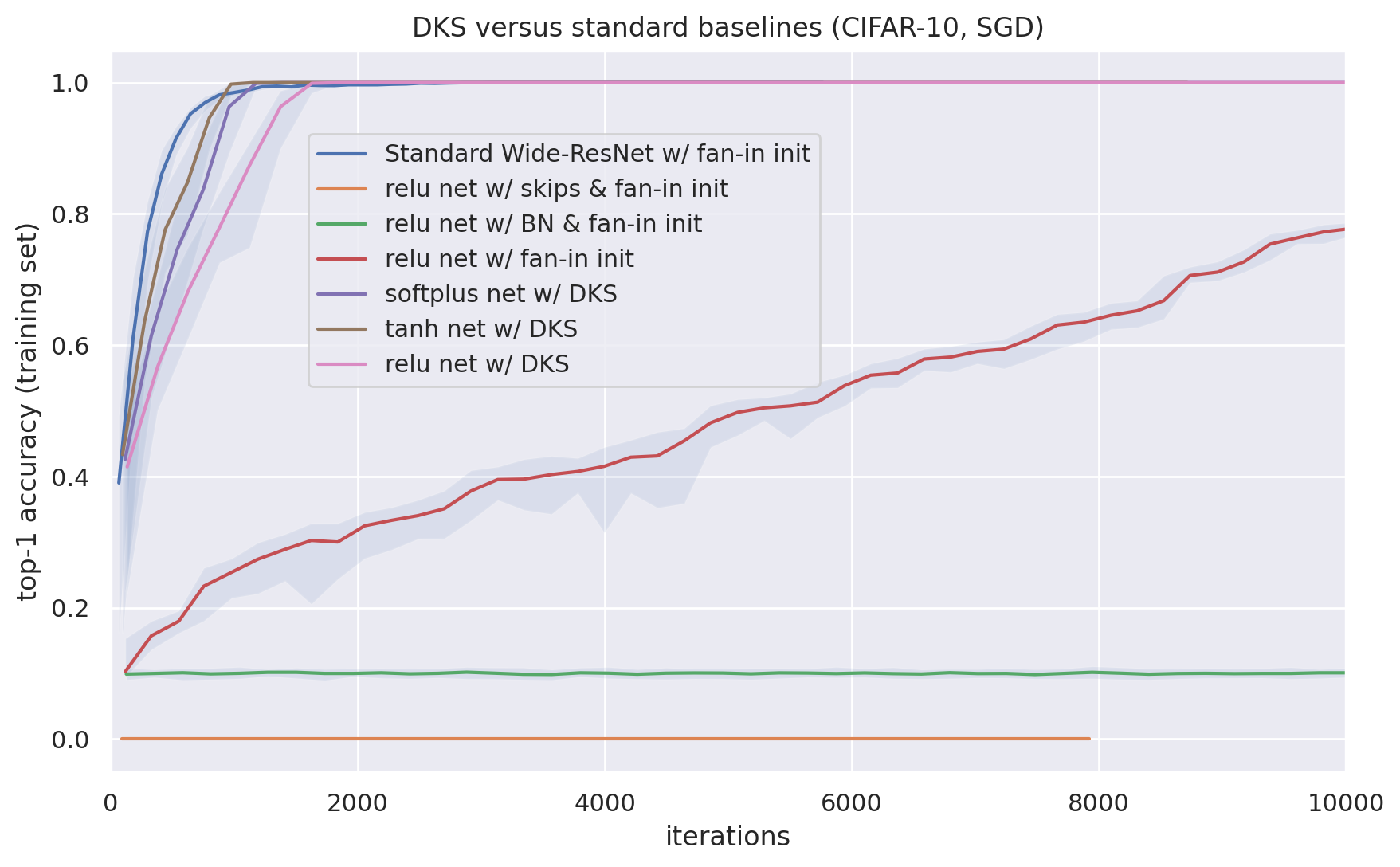}}

\resizebox{0.85\columnwidth}{!}{\includegraphics{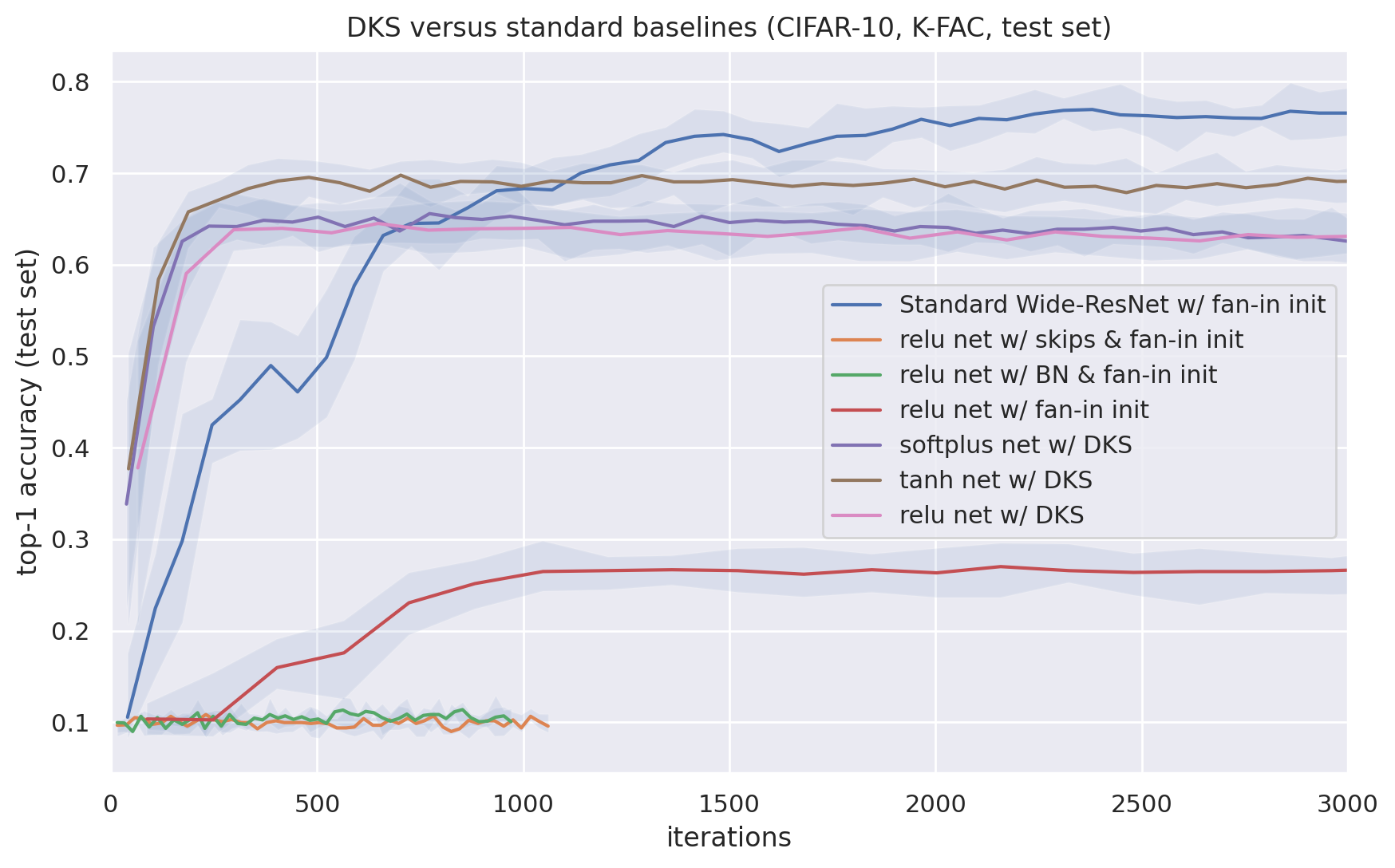}}

\resizebox{0.85\columnwidth}{!}{\includegraphics{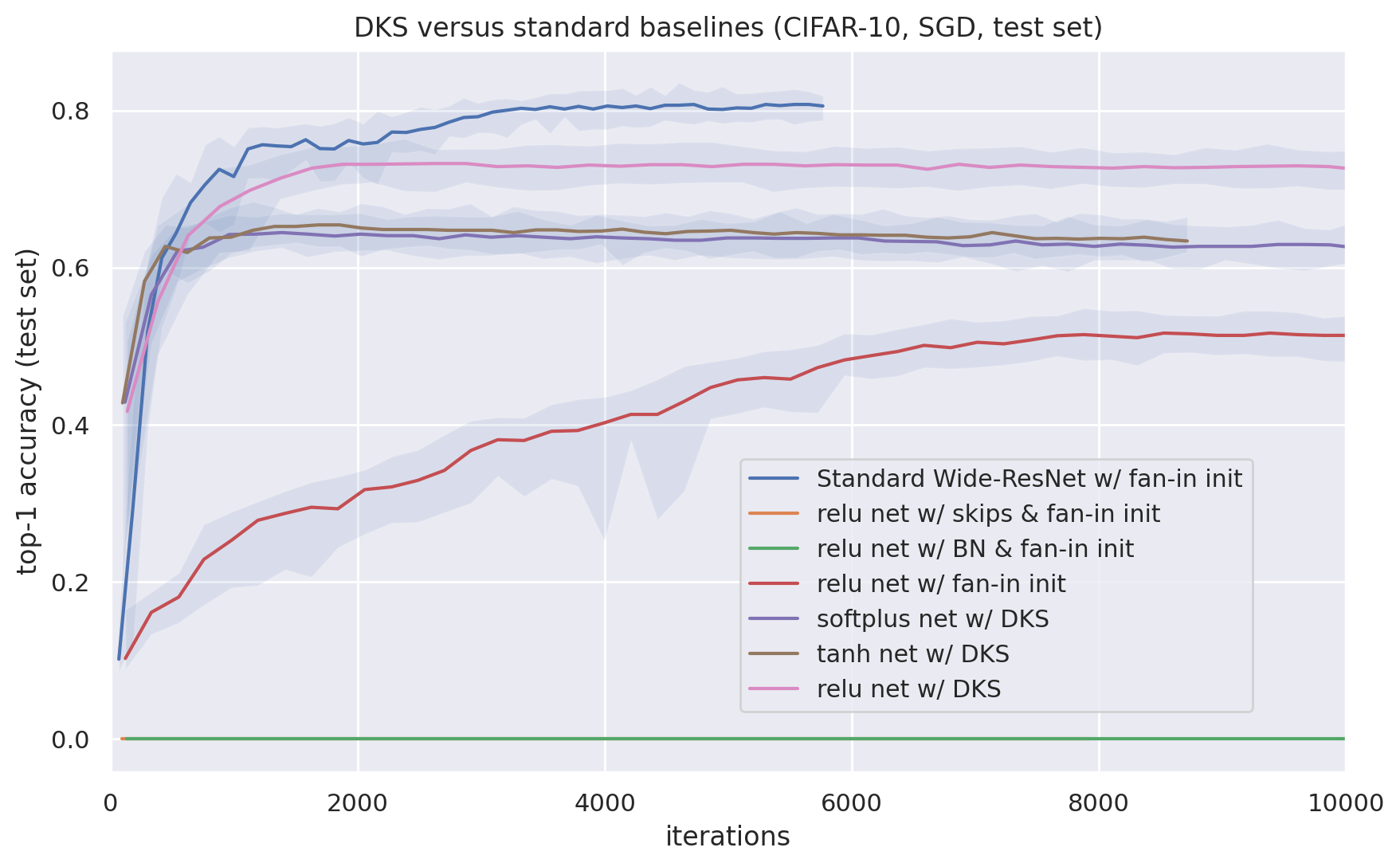}}

\subsection{DKS with and without skip connections} \label{sec:experiment-DKS-skips}

In this subsection we compare the performance, with and without skip connections, of BN-free networks constructed with DKS. We use weights of $\sqrt{0.05}$ and $\sqrt{0.95}$ for the residual and shortcut branches respectively (so that the all sums in the network are normalized as per Section \ref{sec:method-steps}). The value $\sqrt{0.05}$ was selected from several candidate options in order to maximize training speed, as shown in Appendix \ref{app:residual-weight-sweep}.

\

\resizebox{0.85\columnwidth}{!}{\includegraphics{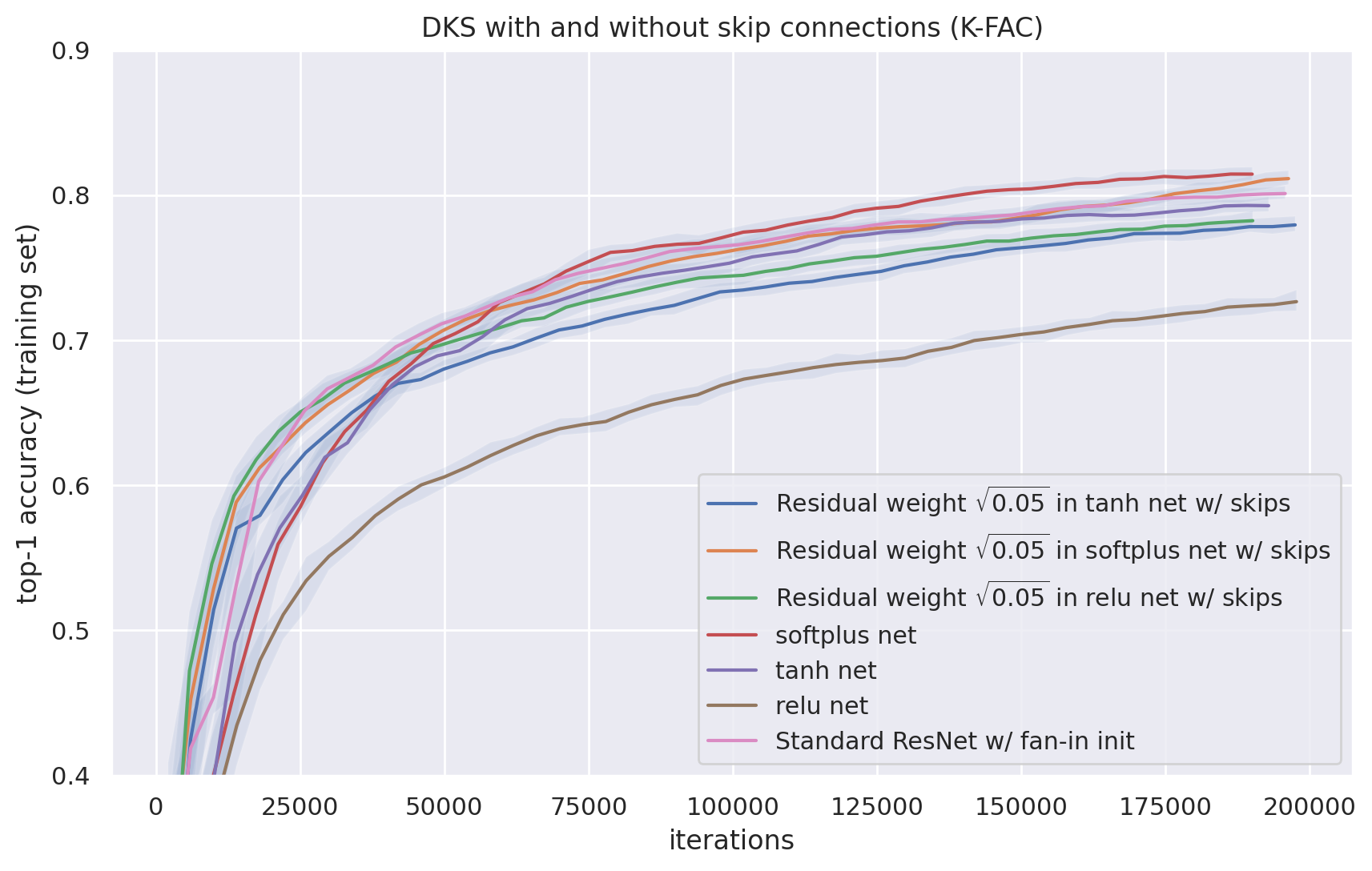}}

When using K-FAC we see that the training speed remains the same whether or not we use skip connections, except in the case of RELU activation functions. For RELUs, skip connections seem to help significantly, closing the performance gap with the other activation functions.

\

\resizebox{0.85\columnwidth}{!}{\includegraphics{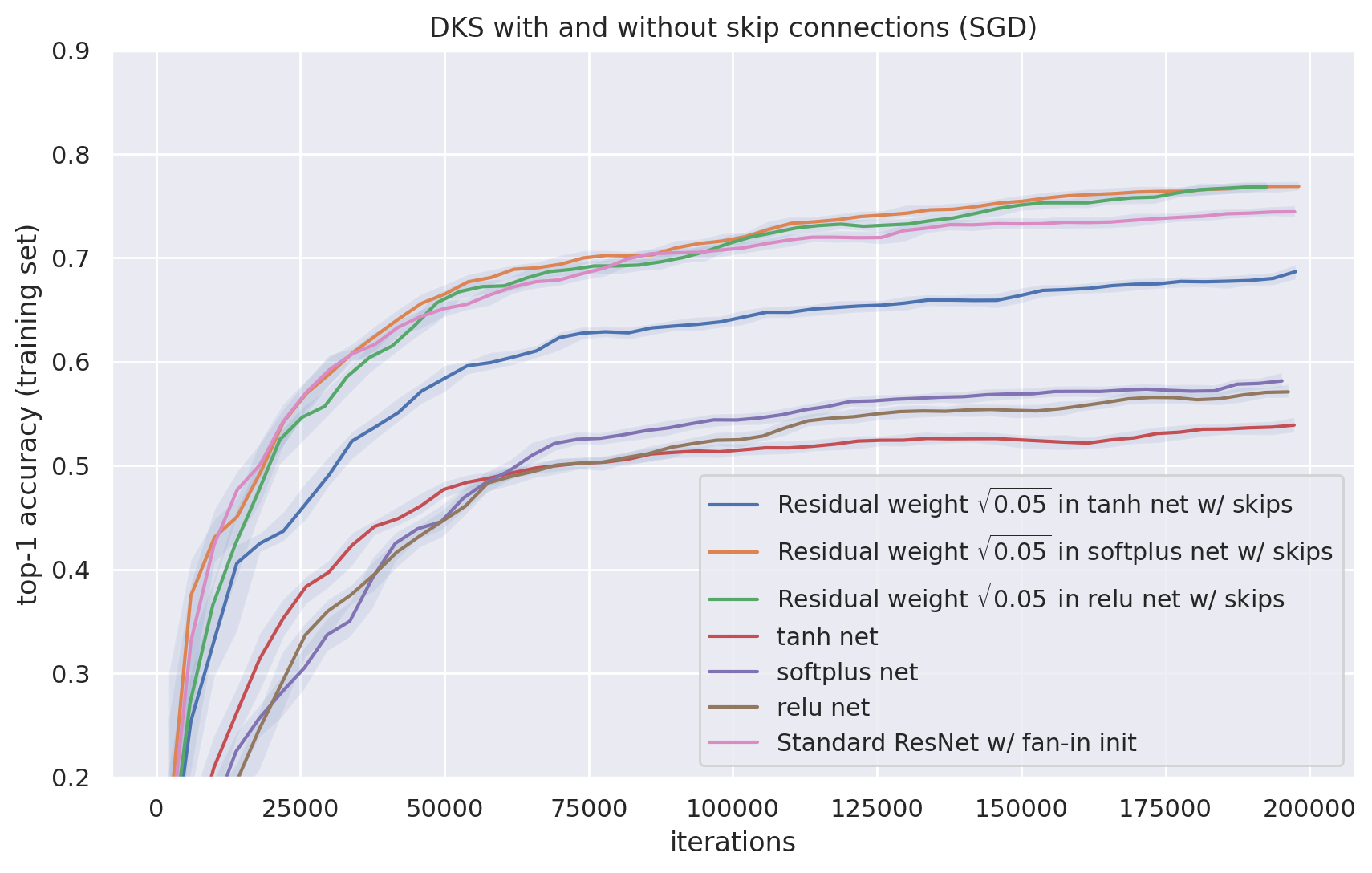}}

With SGD the story is different, and skip connections allow us to match the training speed of standard ResNets with DKS, at least when using softplus or RELU activation functions.

\

\resizebox{0.85\columnwidth}{!}{\includegraphics{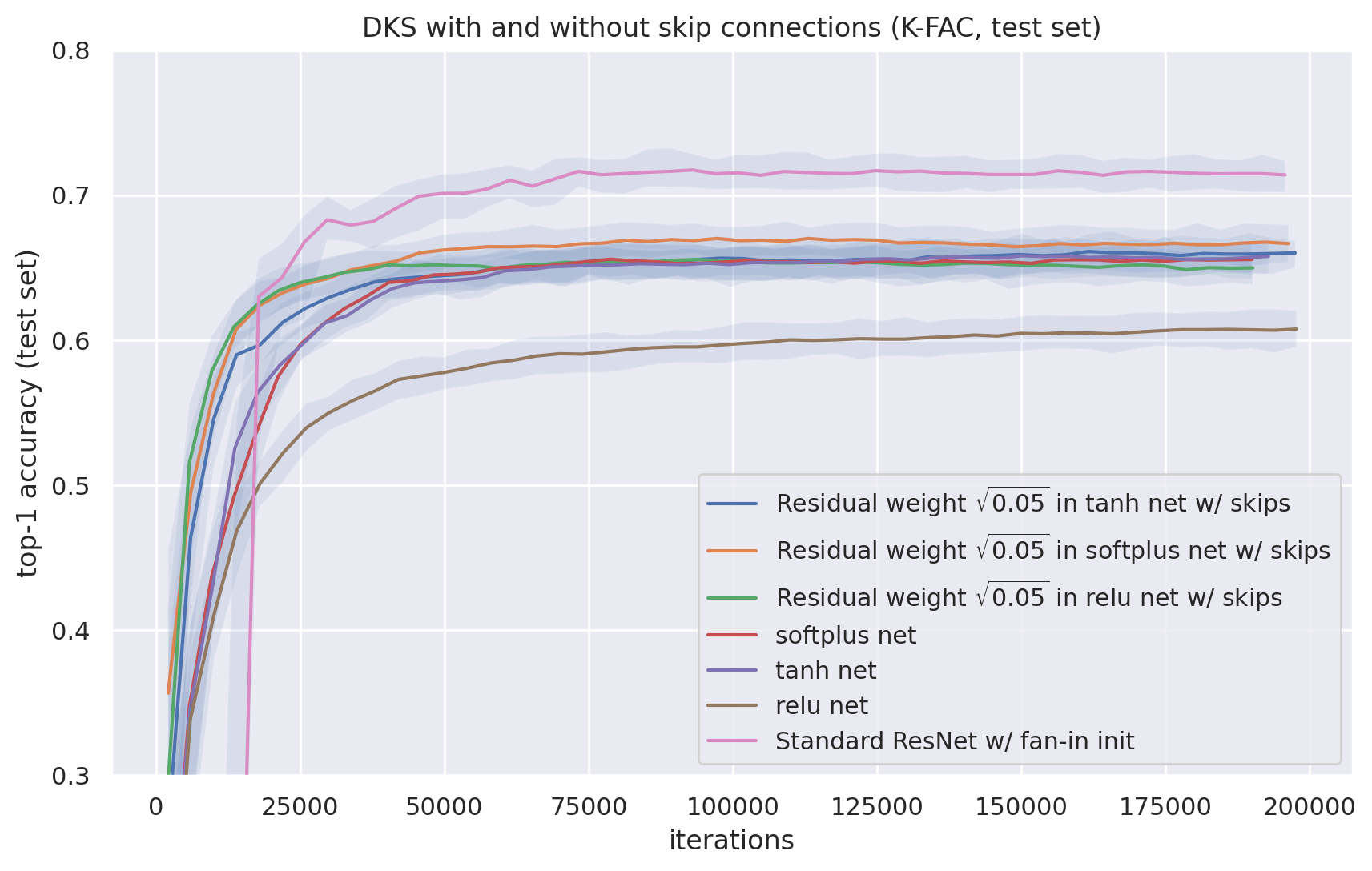}}

For K-FAC, the improvement to test set accuracy from using skip connections with DKS appears to be minimal, with the notable exception of RELU networks (where the improvement is likely due to improved fitting/optimization, as opposed to improved generalization).

\

\resizebox{0.85\columnwidth}{!}{\includegraphics{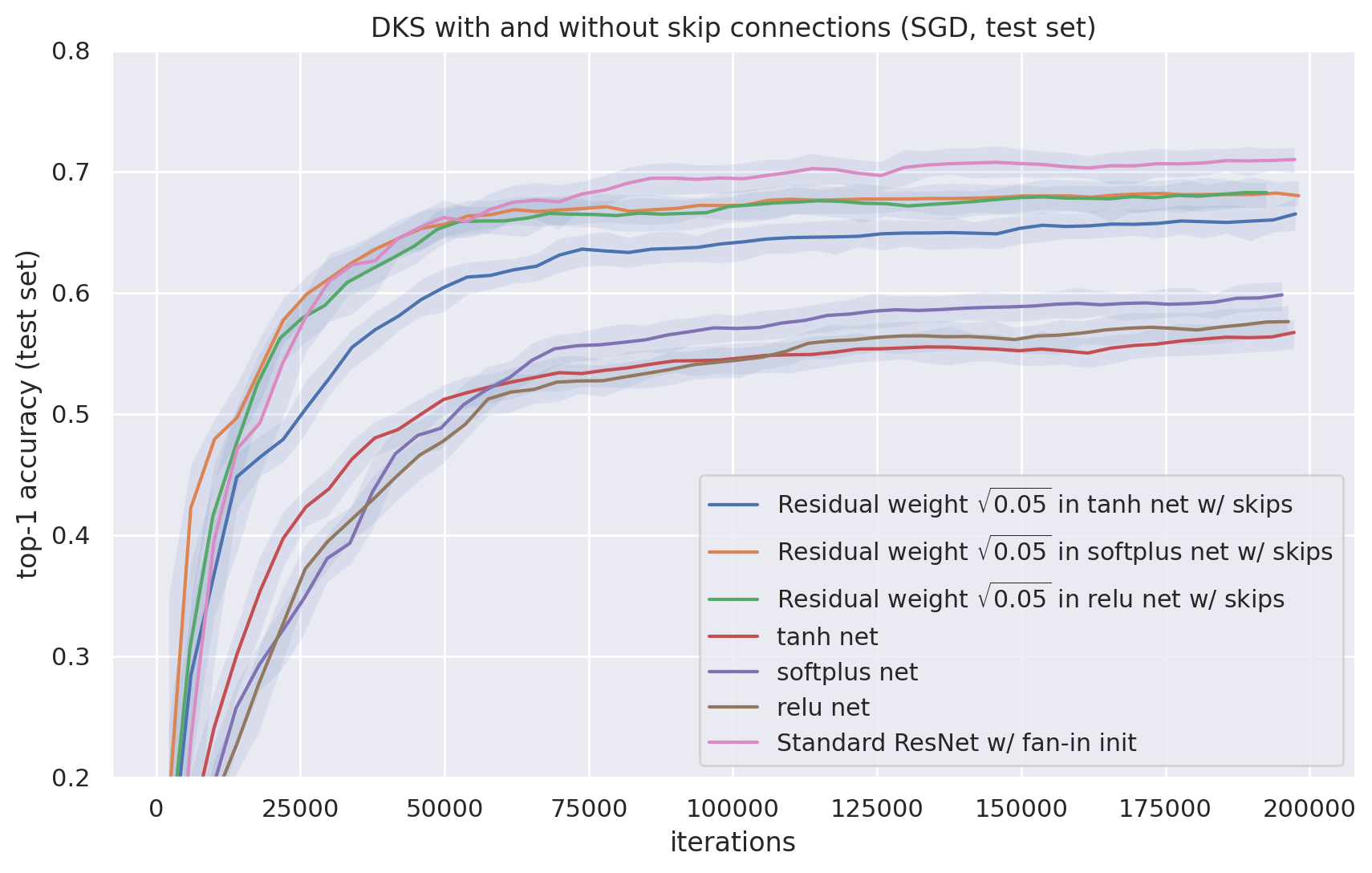}}

By contrast, in the context of SGD training we see a significant improvement to the test set accuracy from using skip connections with DKS. Although again, this is likely due to improved fitting enabled by the use of skip connections with SGD, rather than improved generalization.

\subsection{DKS with different activation functions}\label{sec:different-act-experiment}

In this subsection we compare performance of DKS with twelve different activation functions. In addition to certain well-known mathematical functions, we also include SELU \citep{klambauer2017self}, Softsign \citep{bergstra2009quadratic}, Swish \citep{ramachandran2017swish, elfwing2018sigmoid}, Elu \citep{clevert2016fast}, and BentId (defined by $\phi(x) = x + (\sqrt{x^2 + 1} - 1) / 2 $).

\resizebox{0.85\columnwidth}{!}{\includegraphics{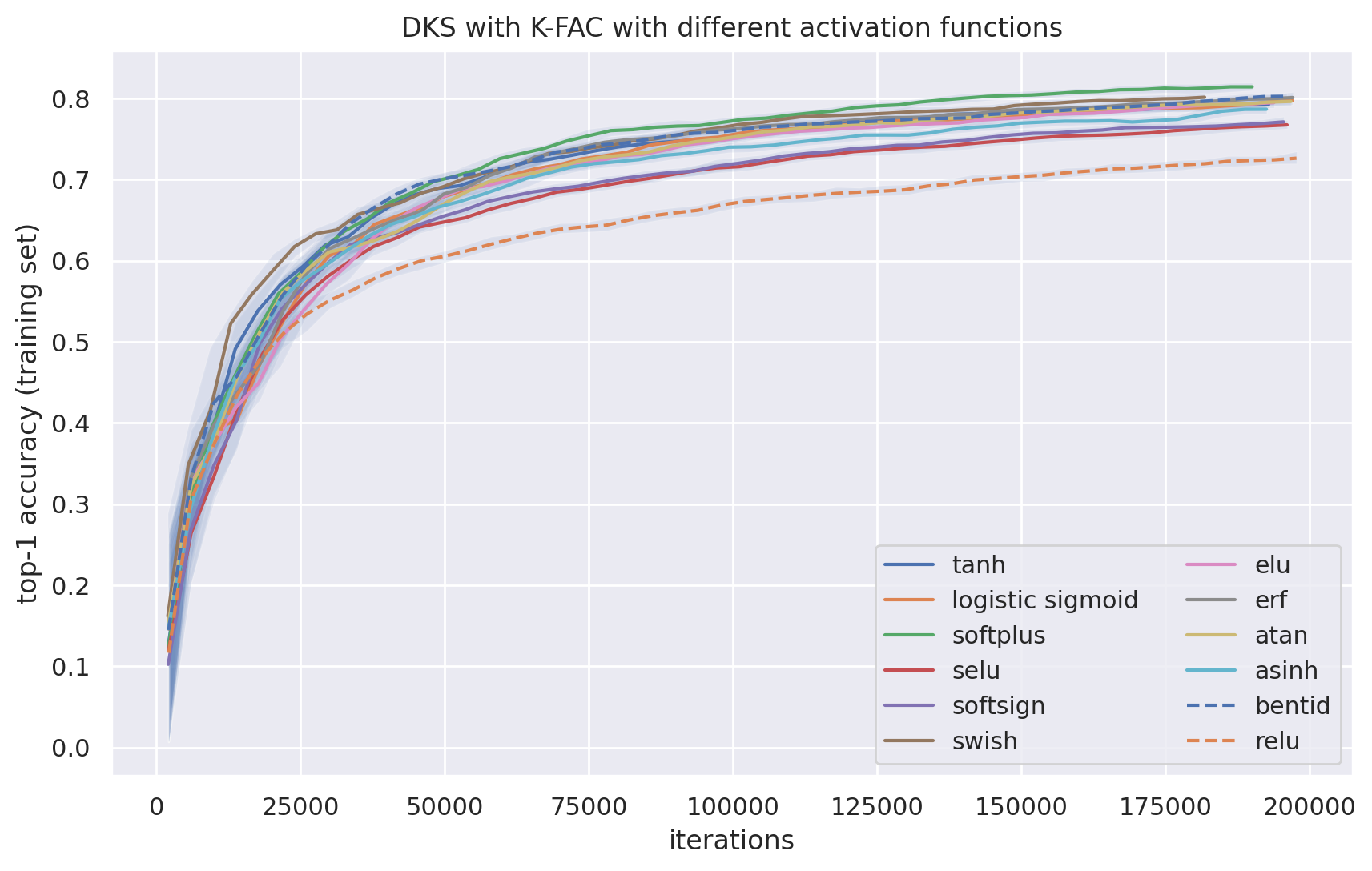}}

For K-FAC we see fairly similar training speeds for each of the twelve activation functions, with RELU being the notable outlier.

\

\resizebox{0.85\columnwidth}{!}{\includegraphics{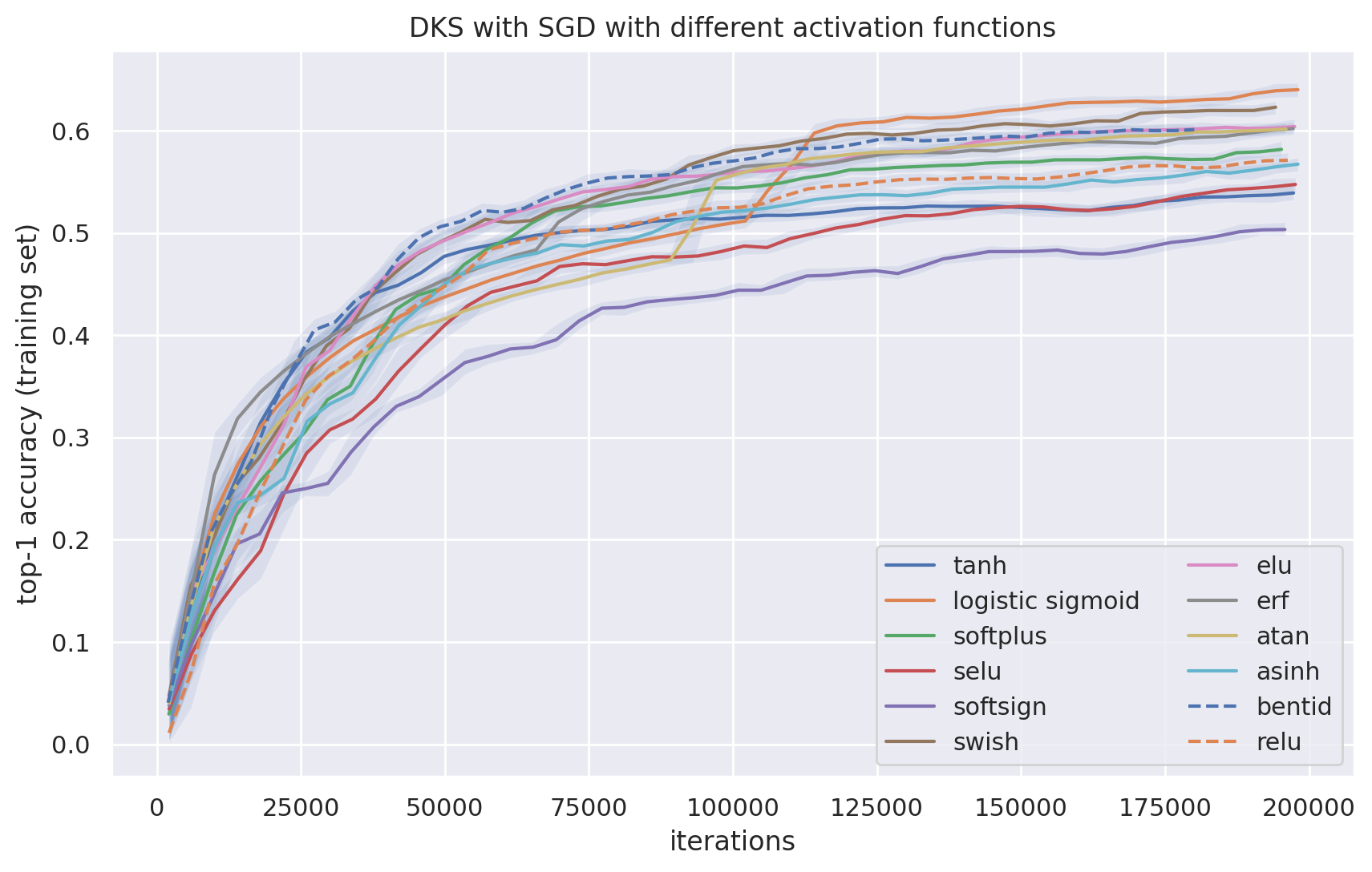}}

For SGD, there is a larger deviation in performance observed for the different options, and RELU is notably no longer an outlier.

Results for test set accuracy were qualitatively very similar, and so we won't report them here.

\subsection{Comparisons to other approaches}

In this subsection we compare DKS to various other approaches for initializing and constructing neural networks. We will focus primarily on skip connection-free BN-free networks, except when comparing to Fix-up (which requires the use of skip connections).

We will omit test set accuracy in these comparisons, as we found that it gave qualitatively similar results to training accuracy. (This is likely because nearly all competing methods yield significant underfitting for skip-free BN-free networks, which overwhelms any possible advantage they might have in terms of generalization.)

\subsubsection{Gaussian fan-in initialization}

The Gaussian fan-in initialization (aka ``variance scaling initialization'' or ``Lecun initialization''), which is discussed in Section \ref{sec:fan-in-and-friends}, is the default initialization method used in many modern neural network frameworks, and is the first method we compare to.

\

\resizebox{0.85\columnwidth}{!}{\includegraphics{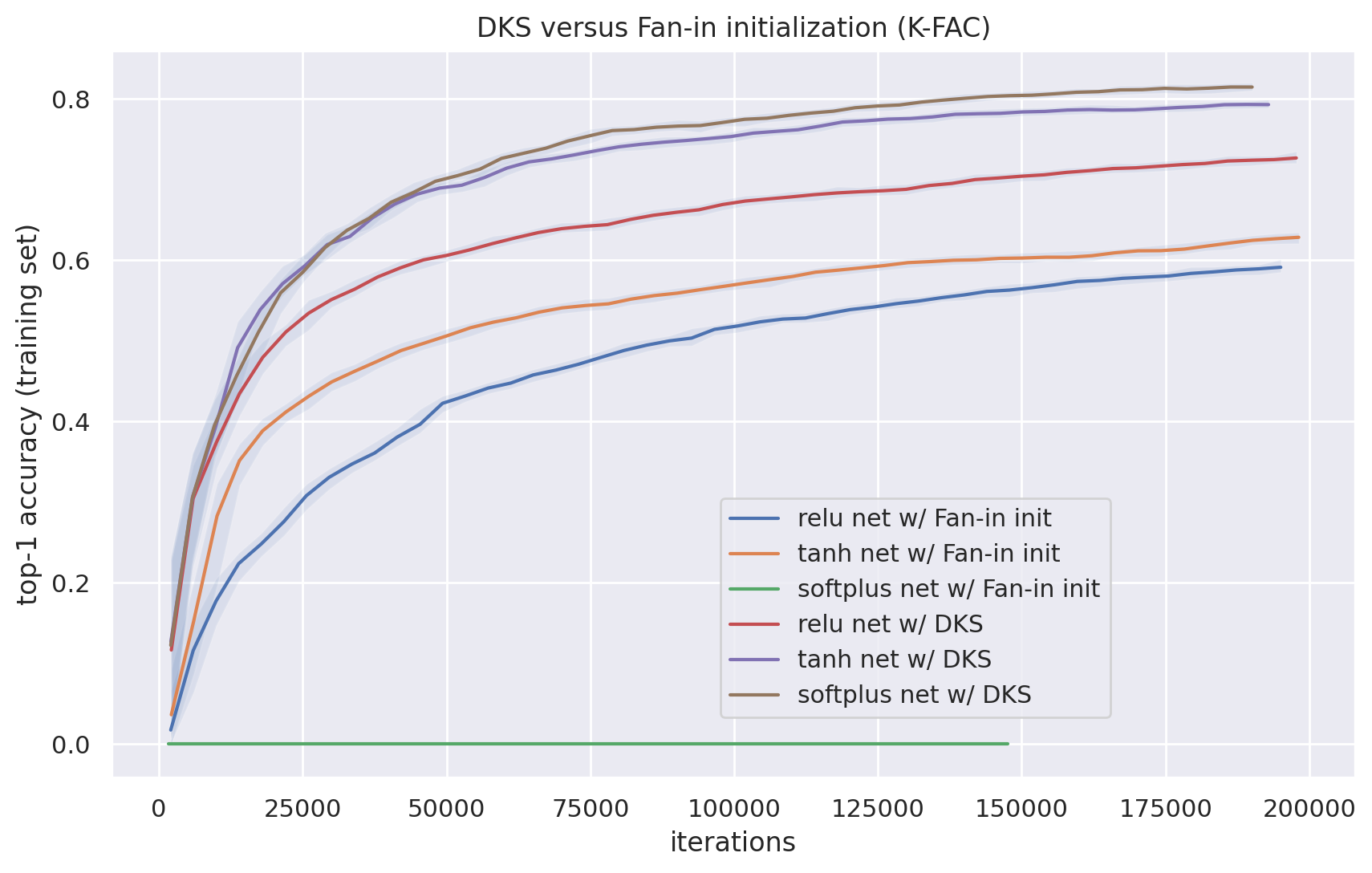}}

\resizebox{0.85\columnwidth}{!}{\includegraphics{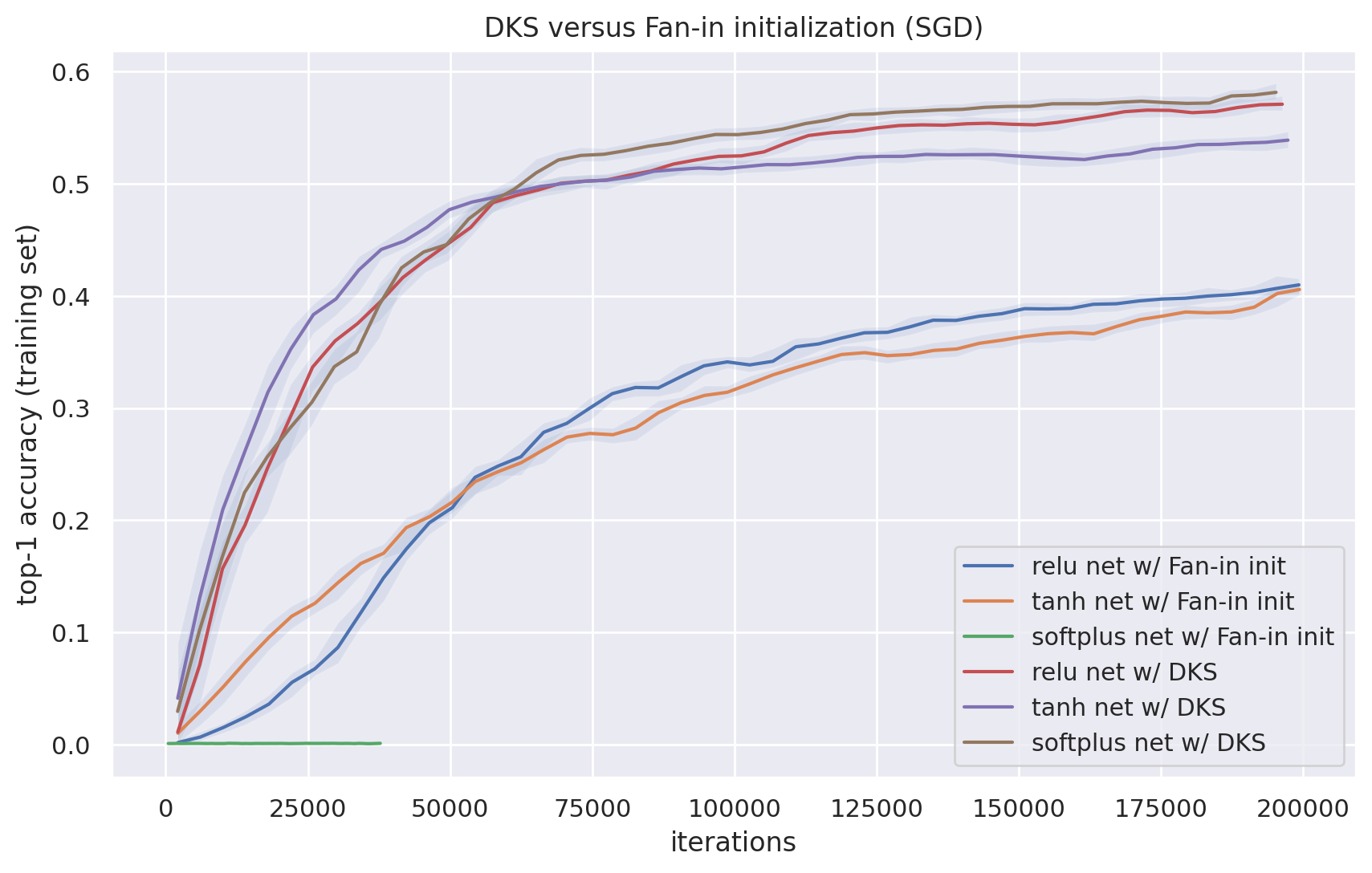}}

From these results we can see that DKS significantly outperforms this canonical approach, whose poor performance in this setting is not surprising given the analysis of Section \ref{sec:Cmaps_trainability}.

Note that it is common in practice to use a truncated Gaussian distribution or uniform distribution to sample the weights in a fan-in initialization, instead of the usual Gaussian distribution. When used with an appropriate rescaling term, these distributions produce weights with the same variance as the standard Gaussian distribution, although they won't necessarily give rise to the same approximate kernel functions. We ran additional experiments using these distributions, and found that they gave similar results to those presented above.

\subsubsection{Glorot uniform initialization}

Glorot initialization (aka ``Xavier initialization'') is a commonly used modification of the Gaussian fan-in initialization which we discuss in Section \ref{sec:fan-in-and-friends}. As with the Gaussian fan-in method, it is also often used with a truncated Gaussian or uniform distribution, the latter of which we will present results for. (We also performed experiments using truncated and non-truncated Gaussian distributions for the weights, which yielded similar findings.)

\

\resizebox{0.85\columnwidth}{!}{\includegraphics{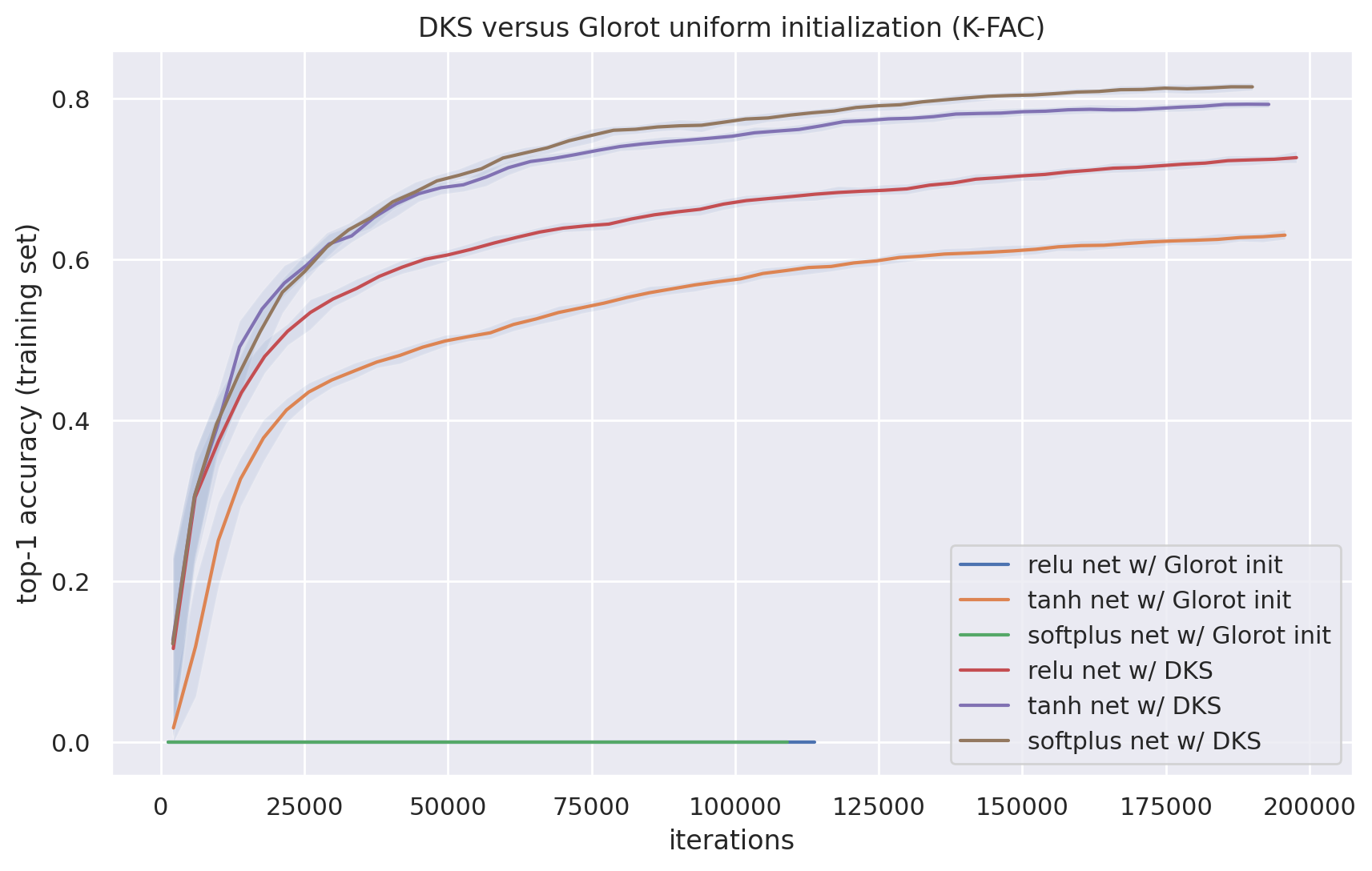}}

\resizebox{0.85\columnwidth}{!}{\includegraphics{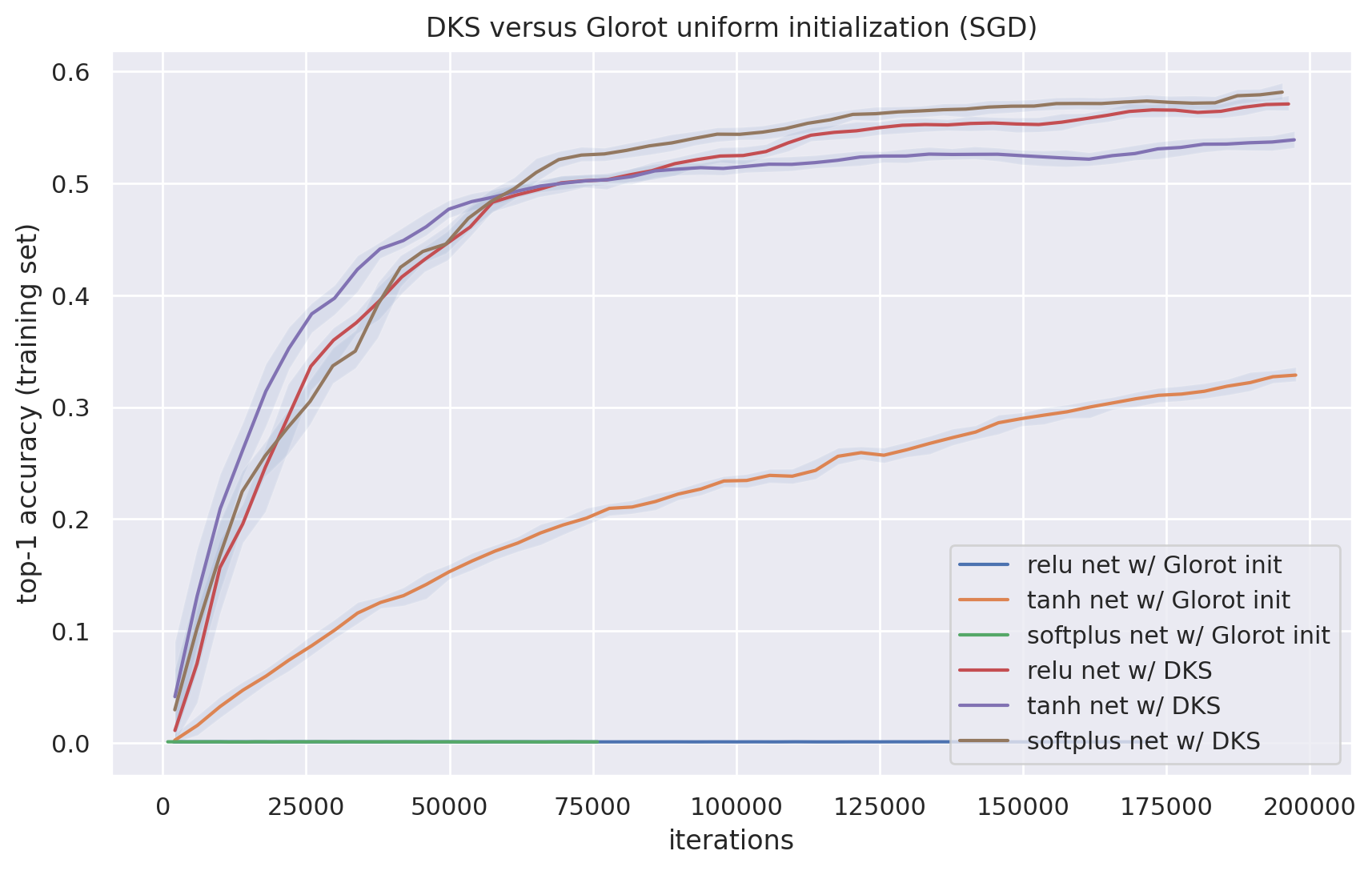}}

From these results we can see that the Glorot approach is significantly outperformed by DKS, and completely fails to produce a trainable network for both the RELU and softplus activation functions.

\subsubsection{LSUV and WLI}

The LSUV and WLI approaches, which are discussed in Section \ref{sec:LSUV}, represent the first generation of methods which attempt to capture the benefits of Batch Normalization through initialization. They are fairly similar in their implementation, which is why we consider them together here.

\resizebox{0.85\columnwidth}{!}{\includegraphics{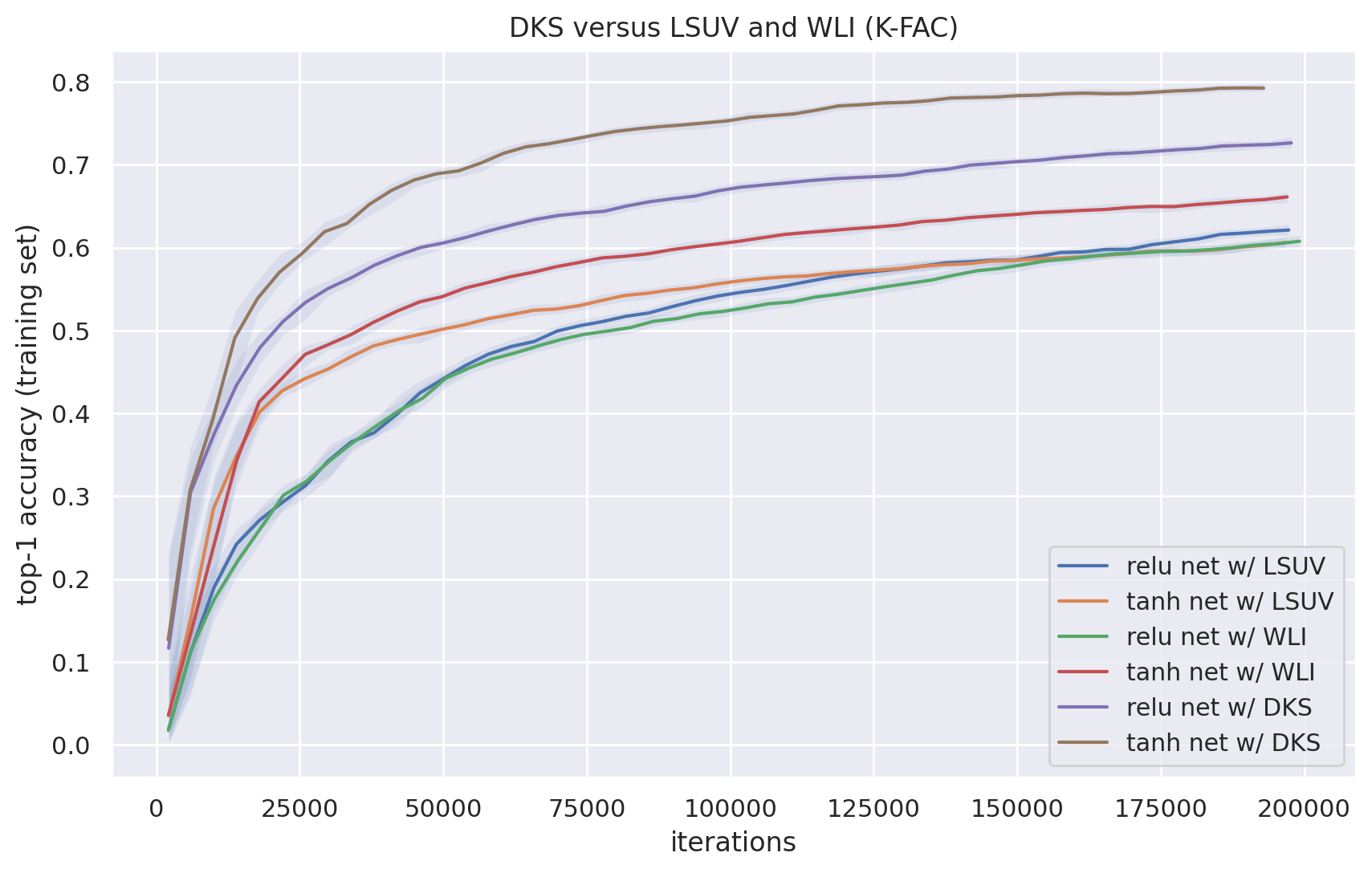}}

\resizebox{0.85\columnwidth}{!}{\includegraphics{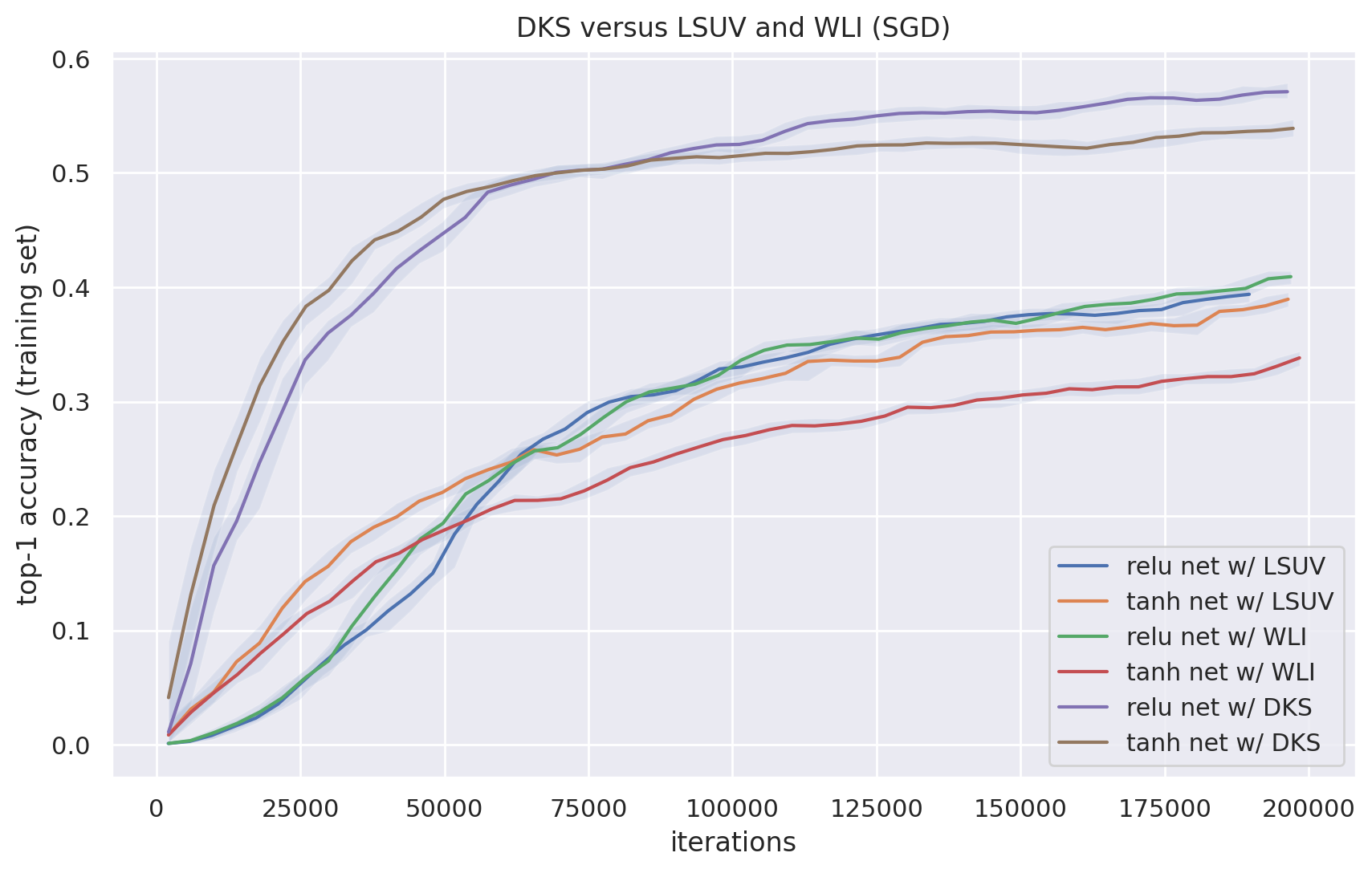}}

From these plots we can see that these methods outperform simple initializations schemes like fan-in and Glorot, but are still significantly outperformed by DKS.

\subsubsection{Self-normalizing neural networks}

Self-normalizing neural networks (which we discuss in Section \ref{sec:SELU-nets}) use SELU activation functions, together with a standard Gaussian fan-in initialization, to achieve certain conditions under variance propagation which are essentially equivalent to two of the four conditions enforced by DKS.

\resizebox{0.85\columnwidth}{!}{\includegraphics{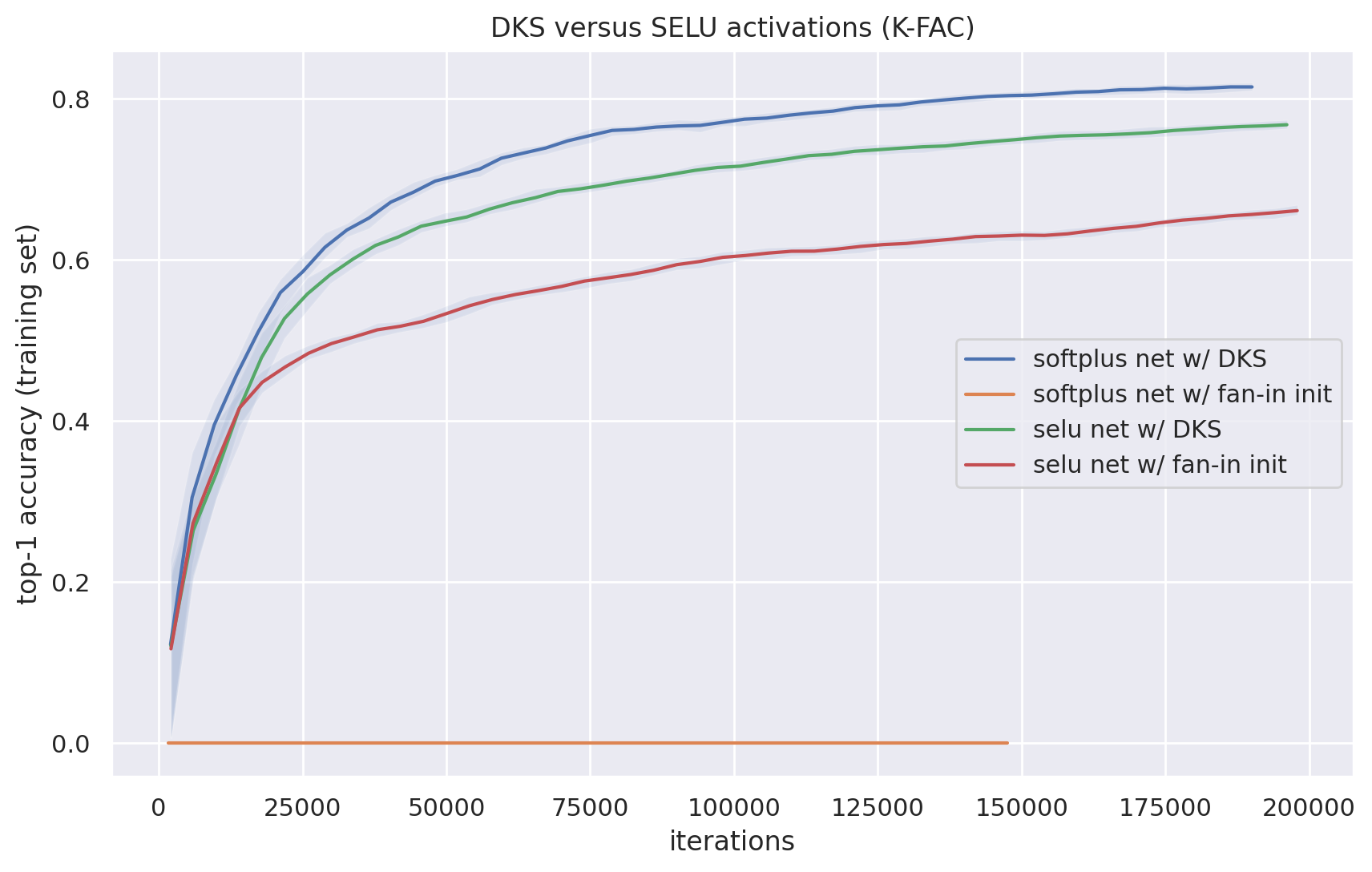}}

\resizebox{0.85\columnwidth}{!}{\includegraphics{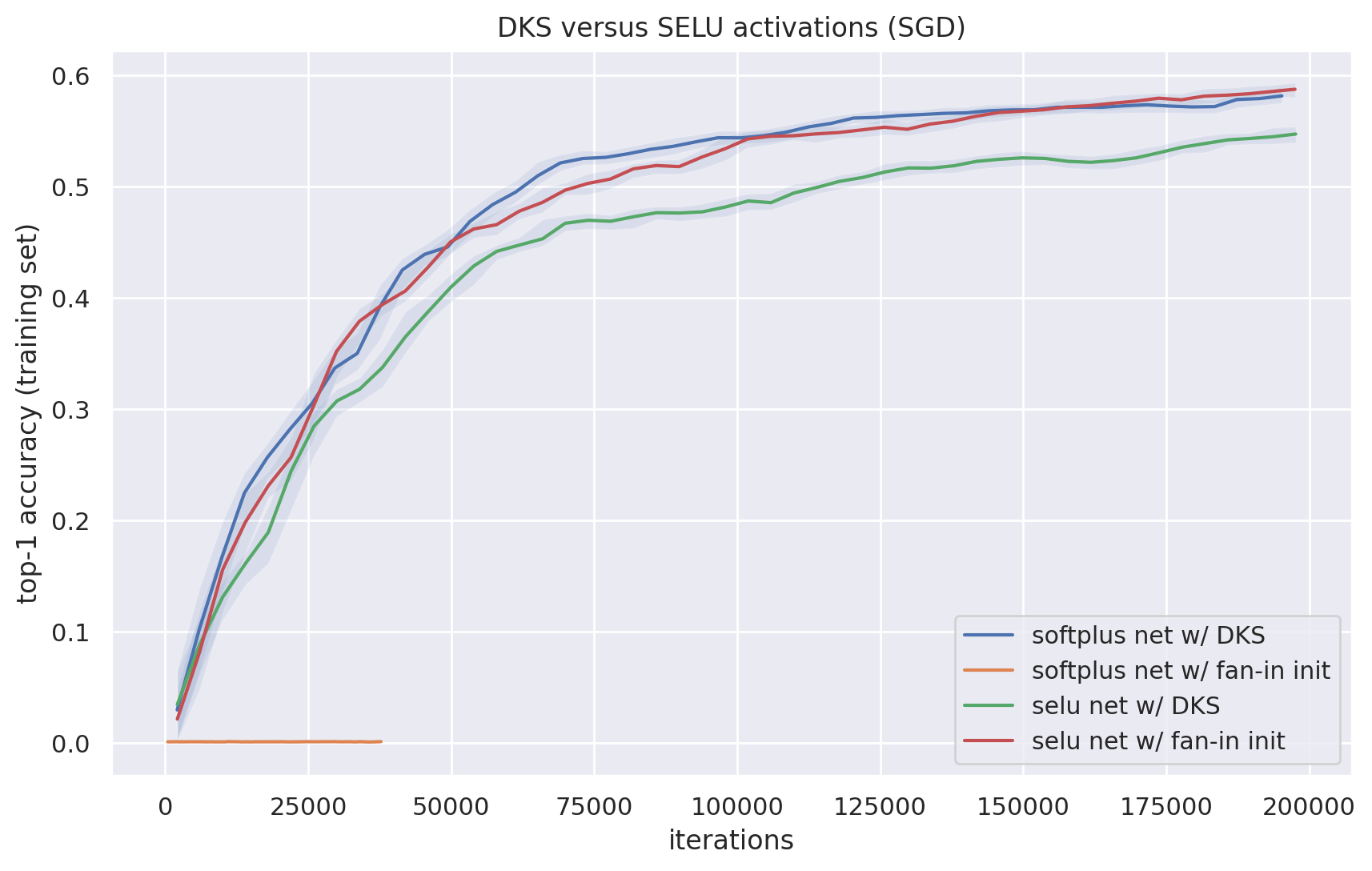}}

From these results we see that DKS applied to a softplus network matches or exceeds the optimization performance of a self-normalizing network. DKS also improves the performance of a SELU network optimized with K-FAC, although slightly degrades it for SGD.

\subsubsection{Looks linear method}\label{sec:looks-linear-experiments}

The Looks Linear method, which is discussed in Section \ref{sec:looks-linear-method}, is an approach for constructing and initializing RELU networks which makes them behave like perfectly linear functions at initialization time, without the use of skip connections. The method is somewhat difficult to fairly compare to other ones, as it involves doubling the channel dimension of each layer, while using a form of weight sharing which makes the resulting network less expressive than a standard one of the same dimensions. Our imperfect solution to this problem is to use the original dimensions when constructing networks with DKS, which will disadvantage DKS in the comparison.

We had some trouble optimizing the networks constructed with the Looks Linear method. K-FAC would quickly diverge for all the hyperparameter settings we tried, perhaps because it broke the delicate symmetry of the initial weights too quickly, leading to extreme nonlinear behavior. We had more luck with Adam and SGD, although we found that it was necessary to threshold the maximum update magnitude at 1 to achieve stable optimization (which is an approach known as ``clipping'' \citep{pascanu2013difficulty}).

\

\resizebox{0.85\columnwidth}{!}{\includegraphics{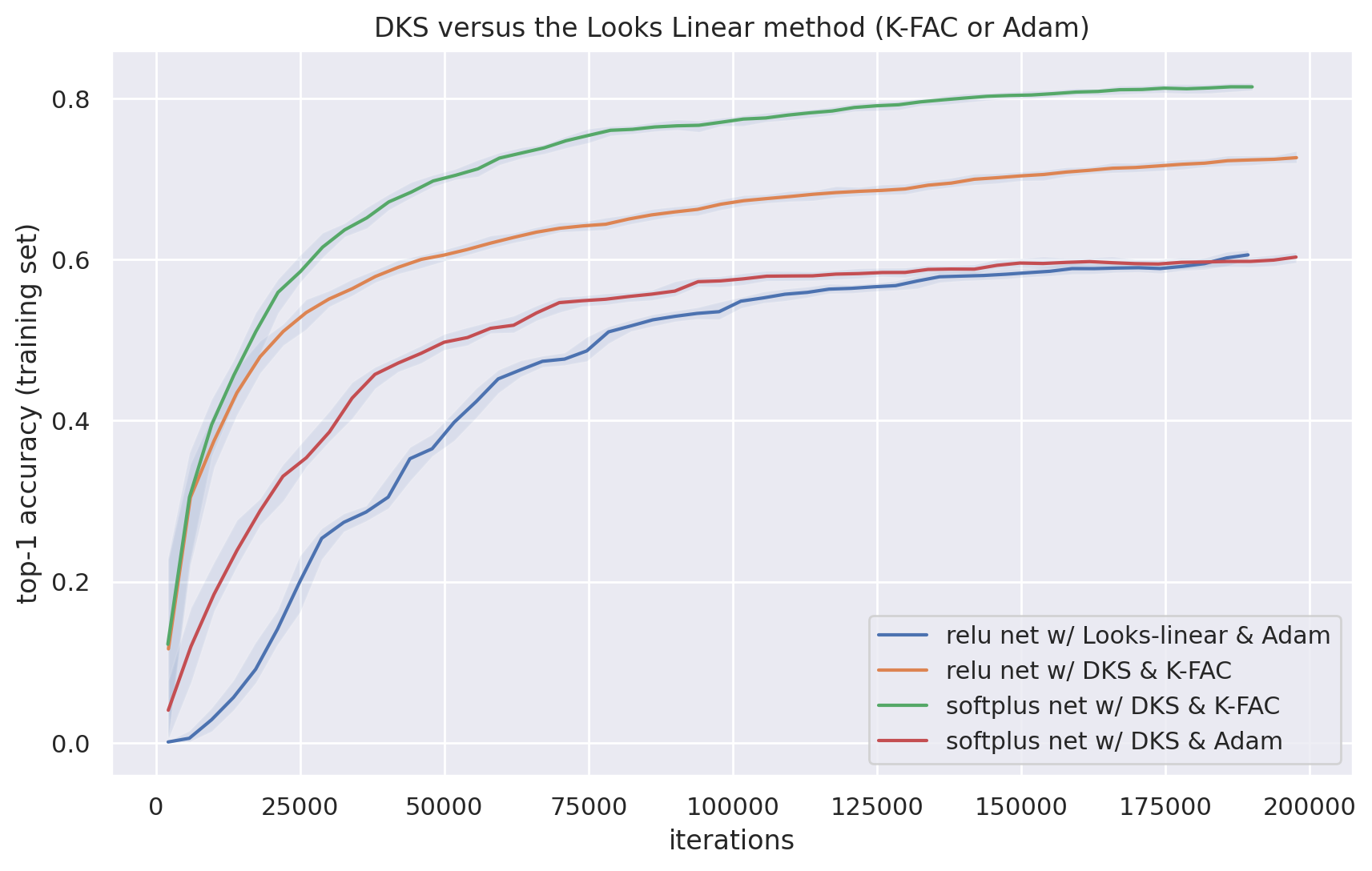}}

Because we couldn't get K-FAC to work well with the Looks Linear method, we used it with Adam instead in our first comparison. We note that with Adam, DKS performs similarly to the Looks Linear method, but when used with K-FAC, DKS significantly outperforms it.

\resizebox{0.85\columnwidth}{!}{\includegraphics{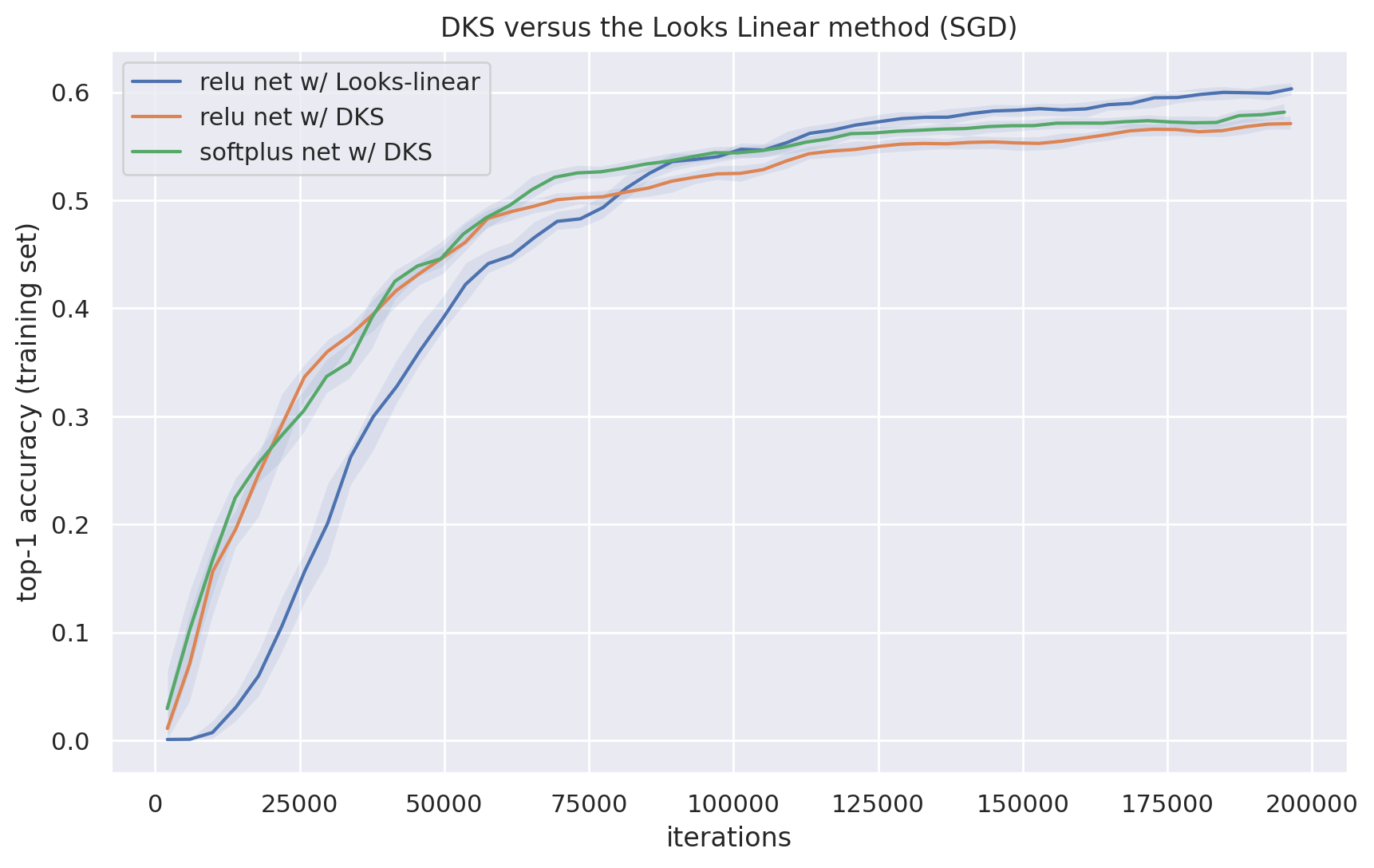}}

For SGD both methods seem to perform similarly, and notably better than both the fan-in/Glorot initializations, and also the LSUV/WLI methods.

We also conducted experiments with CIFAR-10, which yielded similar results. These are given below without commentary.

\

\resizebox{0.85\columnwidth}{!}{\includegraphics{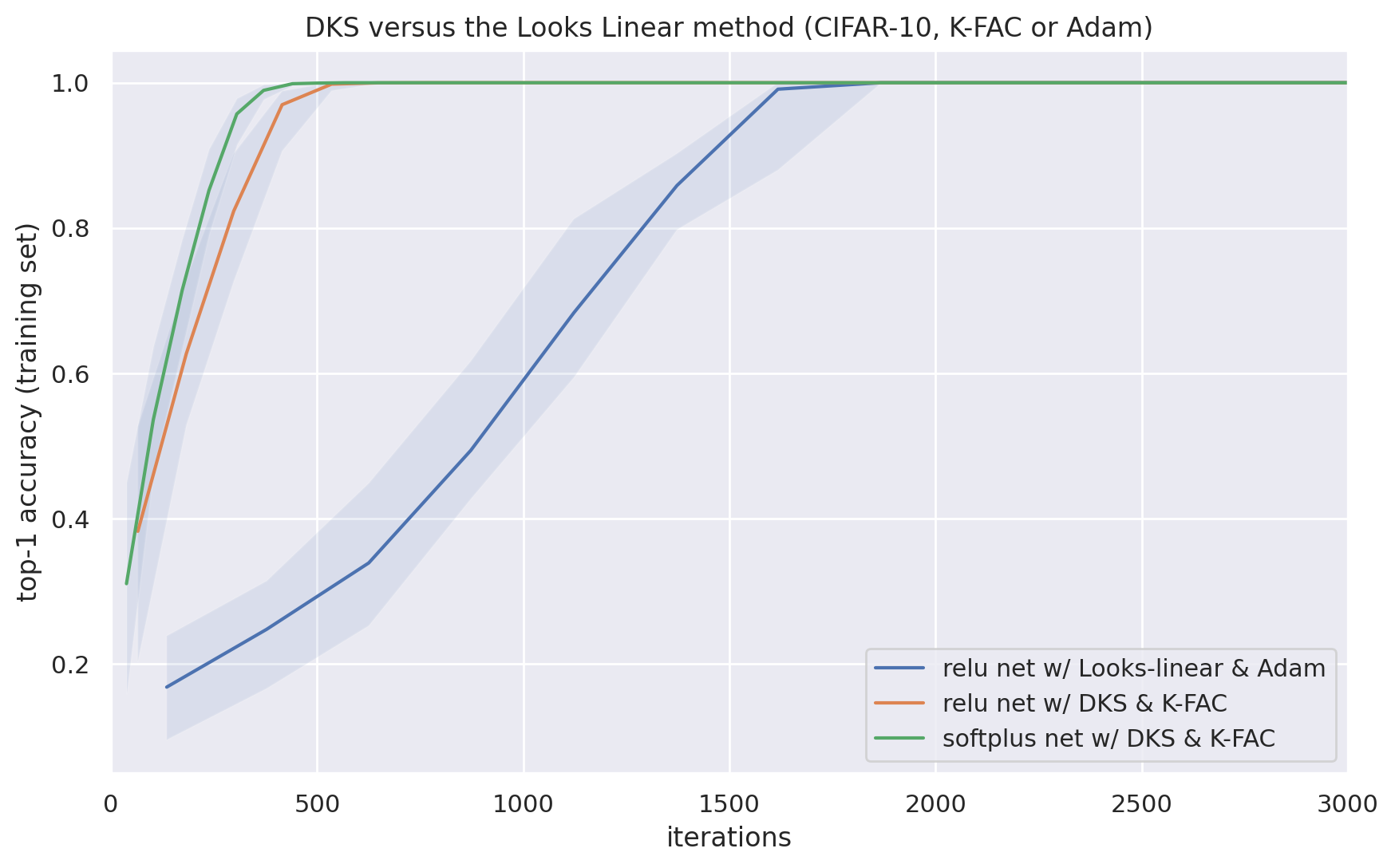}}

\resizebox{0.85\columnwidth}{!}{\includegraphics{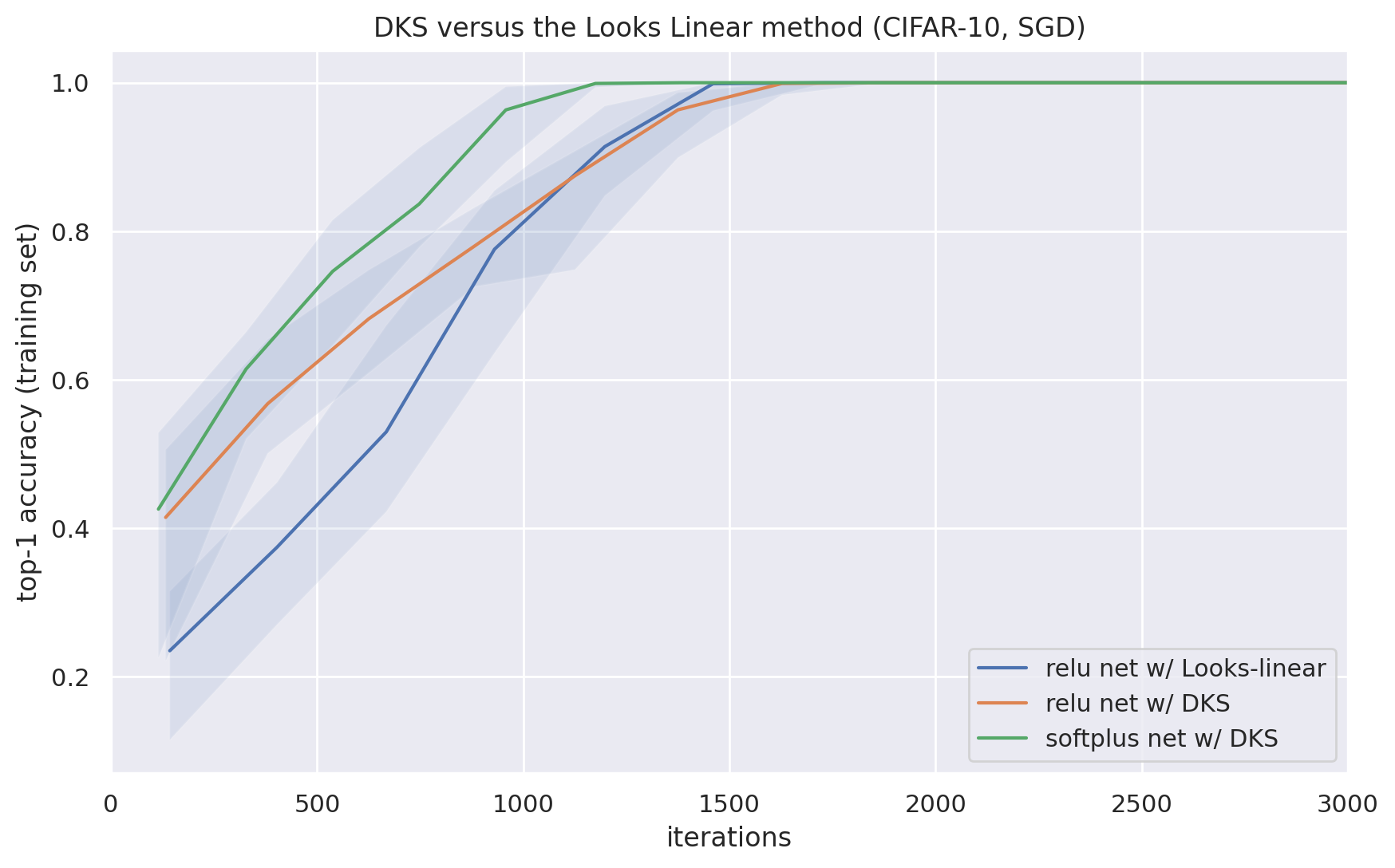}}

\subsubsection{Edge of Chaos (EOC) method}

The Edge of Chaos (EOC) method (described in detail in Section \ref{sec:EOC-method}) is the closest approach to ours in the existing literature, and the one which directly inspired it. The version in \citet{xiao2018dynamical}, which we will use here, involves two ingredients: choosing variances for the weight and bias distributions so that $\mathcal{C}' (1) = 1$ for each local C map $\mathcal{C}$, and using the Delta Orthogonal initialization for the weights (which is rescaled to achieve the target variance).

A clean comparison to EOC is somewhat difficult, as it is not fully specified. In particular, for most activation functions there are infinitely many combinations of the two variances which achieve $\mathcal{C}' (1) = 1$. And for the RELU activation function, the condition $\mathcal{C}' (1) = 1$ holds for any weight variance (given zero bias variance), so that the method reduces to an Orthogonal Delta initialization. \citet{xiao2018dynamical} focused their experiments on tanh networks, and following their advice we will take the variance of the weights and biases to be $1.01 / k$ and $1.654355 \cdot 10^{- 7}$ (respectively) for tanh nets, where $k$ is the input channel dimension for the given layer. We will also consider RELU networks, with a weight variance of $2 / k$ and a bias variance of 0.

\

\resizebox{0.85\columnwidth}{!}{\includegraphics{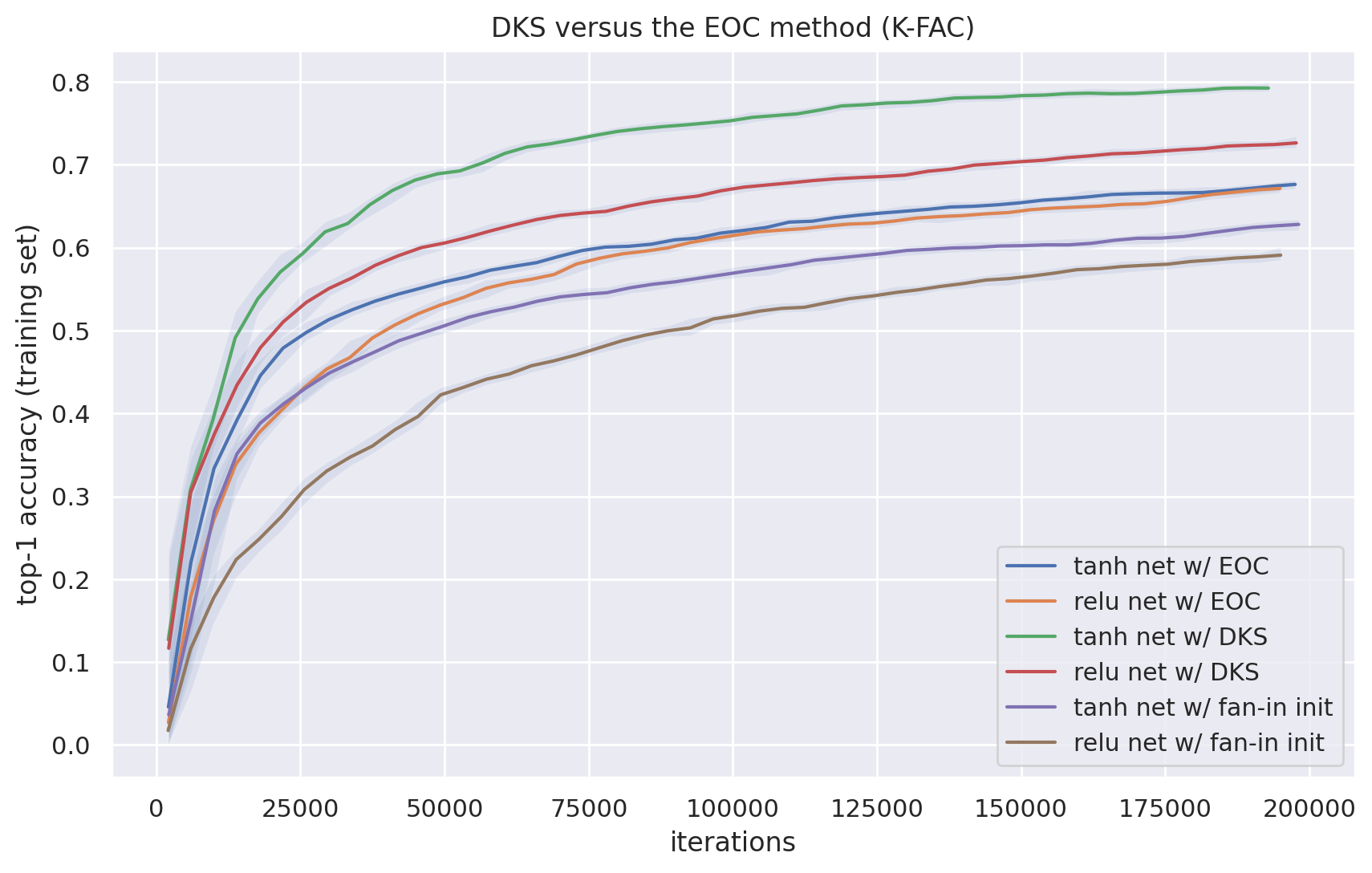}}

\resizebox{0.85\columnwidth}{!}{\includegraphics{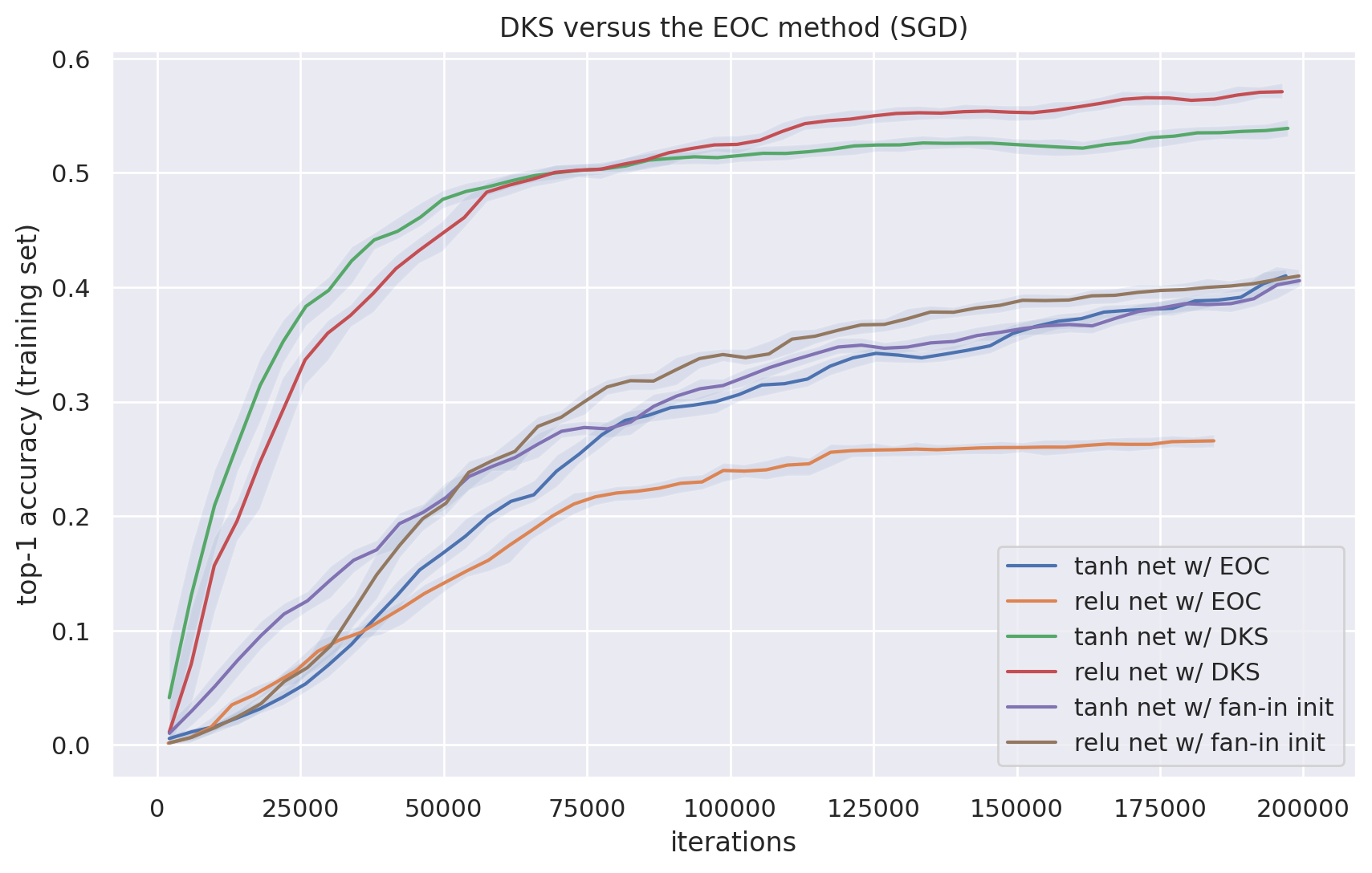}}

\resizebox{0.85\columnwidth}{!}{\includegraphics{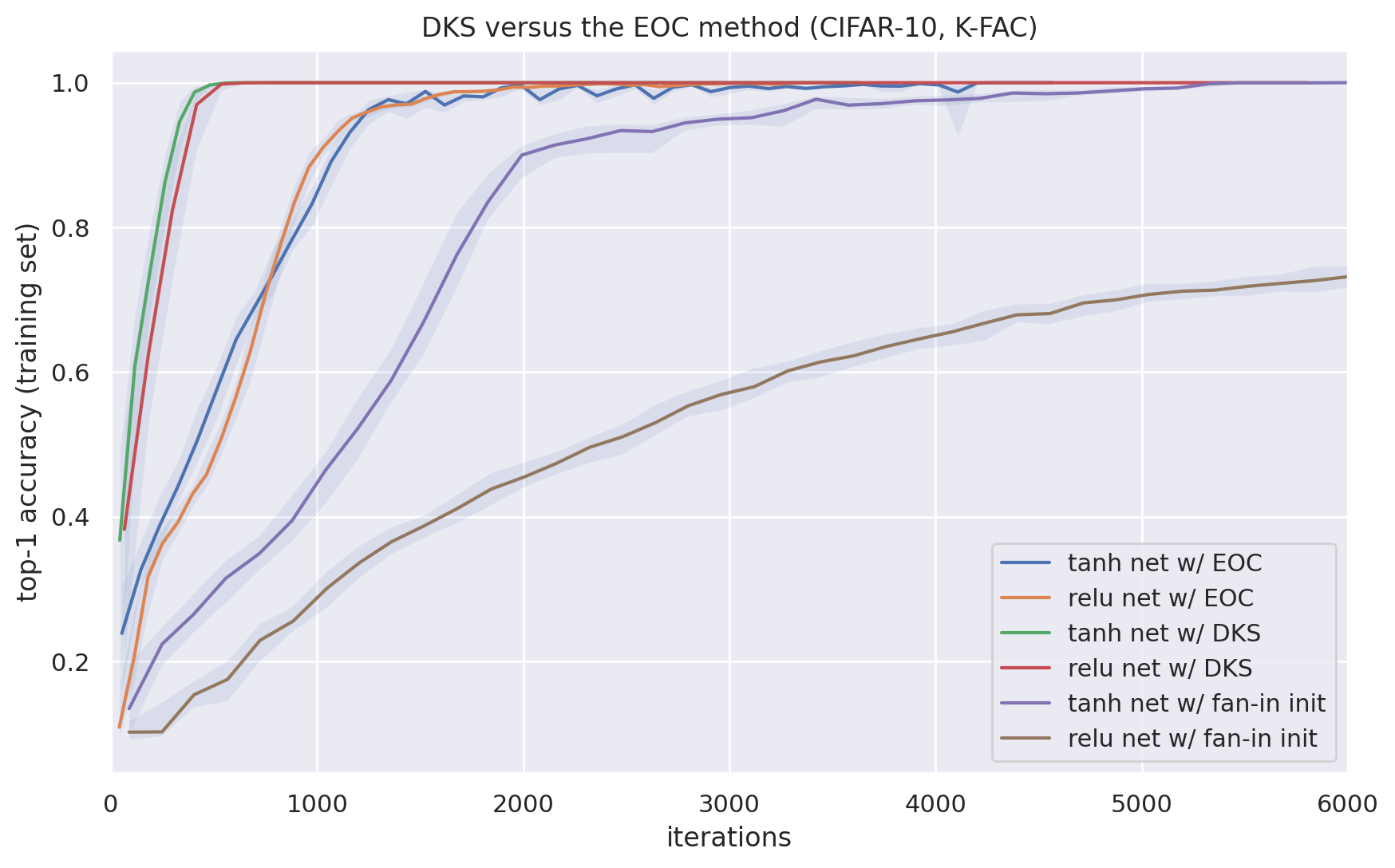}}

\resizebox{0.85\columnwidth}{!}{\includegraphics{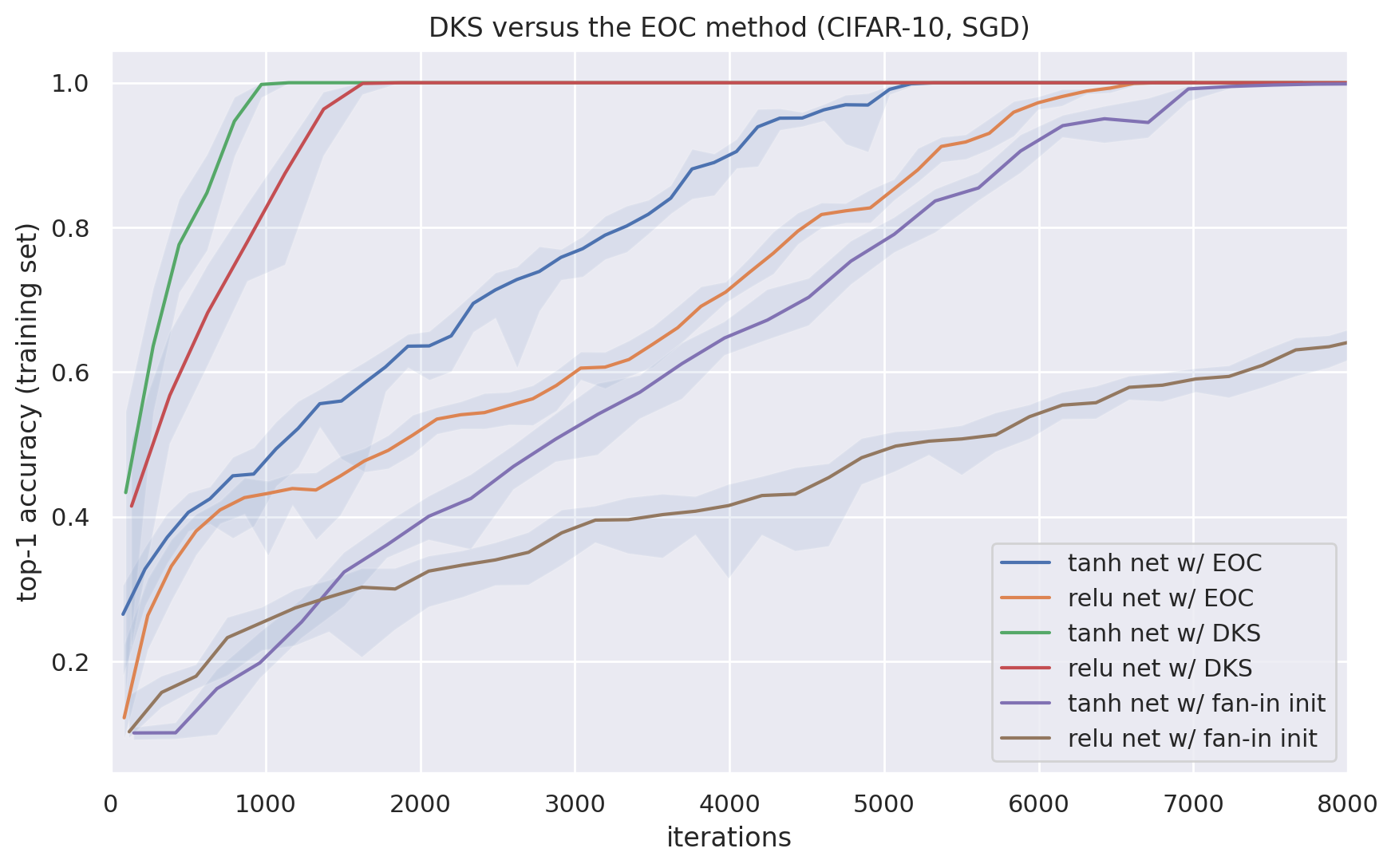}}

From these results we see that DKS significantly outperforms EOC in terms of optimization speed (for both Imagenet and CIFAR-10), which in turn outperforms the simple Fan-in initialization method.

\subsubsection{Fix-up}

Fix-up, which we briefly discuss in Section \ref{sec:resnet-related-discussion}, is a recent method for constructing and initializing networks with residual connections which is designed to eliminate the need for normalization layers. It involves initializing the weights of the final convolutional layer in each residual block to zero (so that the residual blocks behave like identity functions at initialization), using a special formula for the variance of the weights distribution, as well as introducing learnable scalar multiplication and bias operations throughout the network.

We were not able to get K-FAC to work well with Fix-up. This might have been due to a bad interaction with K-FAC and the extra parameters introduced by Fix-up (as K-FAC is designed specifically for the standard neural network parameters). Another possible explanation is that, like networks created with the Looks Linear method, the larger steps taken by K-FAC cause Fix-up networks to transition too quickly to extreme nonlinear behavior (after being essentially linear at initialization).

As Fix-up requires the use a skip connections, for the sake of fairness we compared it to BN-free networks constructed with DKS that also used skip connections. And because we couldn't get K-FAC to work well with Fix-up, we instead used Adam with Fix-up in our first comparison. (While this may seem unfair, we note that for networks with skip connections, K-FAC and SGD perform similarly, as shown in Subsection \ref{sec:experiment-DKS-skips}.)

\

\resizebox{0.85\columnwidth}{!}{\includegraphics{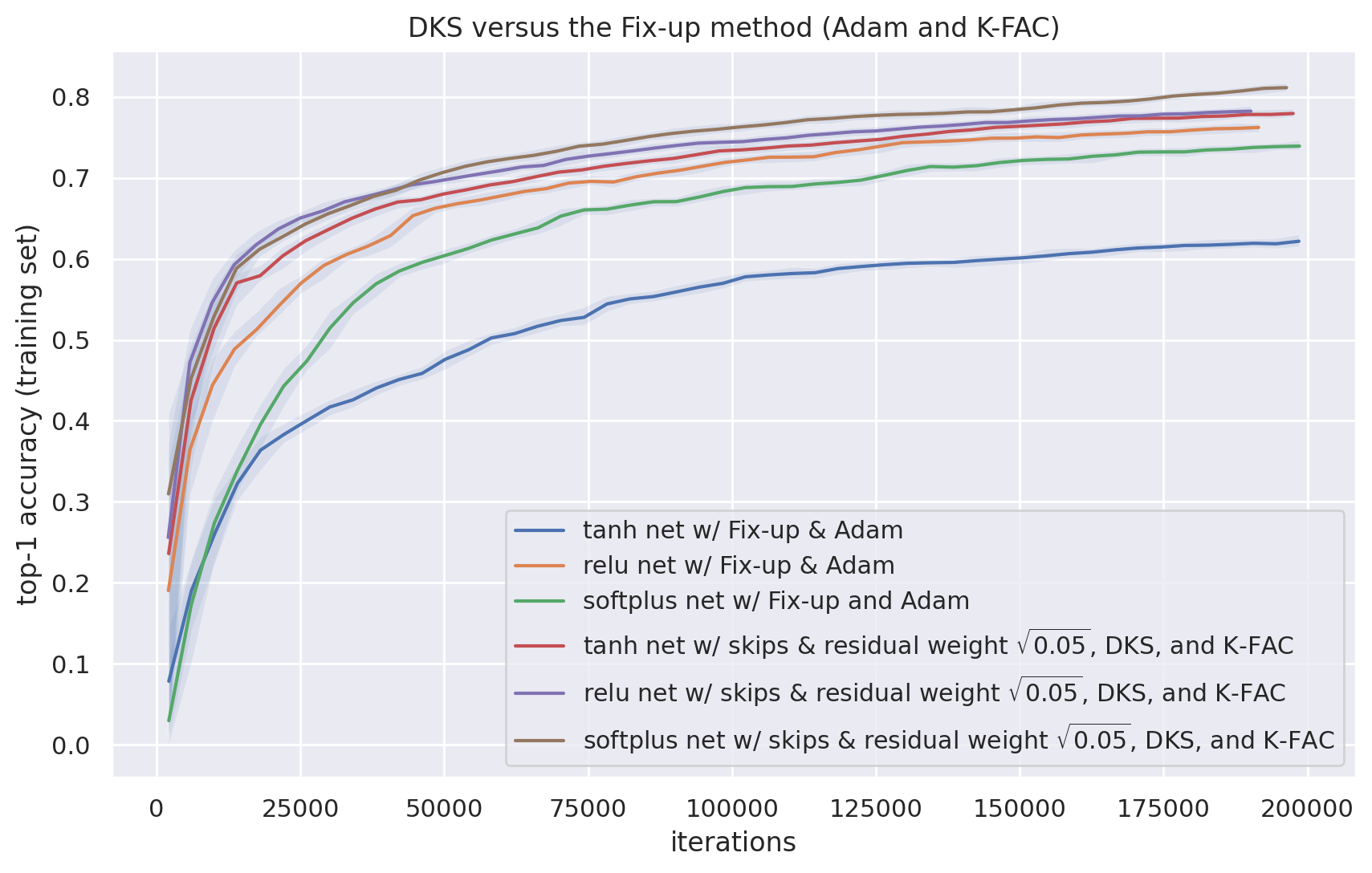}}

\resizebox{0.85\columnwidth}{!}{\includegraphics{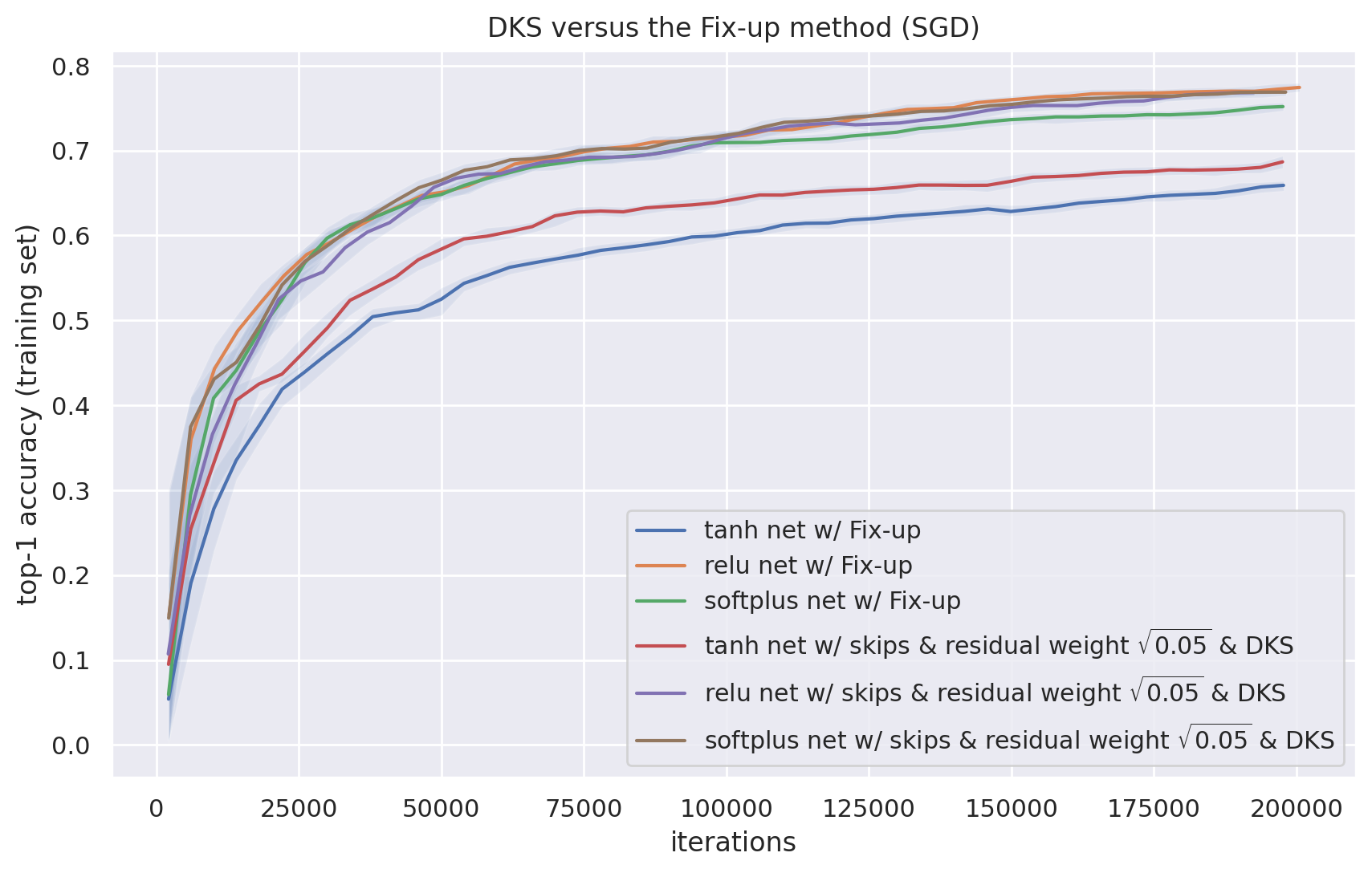}}

From these results we see that Fix-up performs similarly to DKS for RELU activation functions, but falls behind for tanh and softplus.

\subsection{Meta-parameter studies}

The influence of various training ``meta-parameters'' on the optimization and generalization performance of DKS networks is considered in Appendix \ref{app:meta-param-experiments}. These meta-parameters include the weight on the residual branch when using skip connections, DKS's $\zeta$ parameter, and the choice of optimizer. Our conclusions from these studies are summarized as follows:
\begin{itemizedot}
  \item When using DKS with skip connections, a weight of $\sqrt{0.05}$ on the residual branch works the best overall among several other sensible options, although this is likely to be contingent on details of the architecture (such as depth).
  
  \item In terms of optimization performance, $\zeta = 1.5$ typically works better than values that are much larger, or much closer to 1, although the difference isn't very big. In terms of generalization performance, somewhat smaller values (such as 1.1) may work slightly better.
  
  \item For networks without skip connections, K-FAC is the best optimizer in terms of speed, followed closely by Shampoo. Following that are Adam and then SGD, which both perform significantly worse than Shampoo in this setting. For networks with skip connections, the gap between K-FAC and SGD narrows substantially.
\end{itemizedot}

\subsection{Ablations and modifications of DKS}

Various ablations and modifications of DKS are considered in Appendix \ref{app:ablation-and-mod-experiments}. The overall conclusion of these studies is that each component of DKS, except perhaps for PLN (assuming reasonably well scaled input data), is required to achieve the highest optimization speed. When considering test error the conclusions are similar but somewhat muted, with the single exception that using weighted mean-pooling layers with K-FAC seems to improve test set performance while degrading training set performance.


\section{Conclusions}\label{sec:conclusions}

In this work we developed Deep Kernel Shaping (DKS), a method for making neural networks easier to train via model class preserving transformations. We showed how our method controls the shape of the network's initialization-time kernel, by way of our generalized Q/C map analysis, in order to prevent certain common pathologies associated with slow optimization and poor generalization. In our experiments we showed that DKS allows deep networks without skip connections or normalization layers to be trained at similar speeds to ResNets on Imagenet, assuming the use of K-FAC or Shampoo. To the best of our knowledge this is a unprecedented result. We also applied our generalized Q/C map analysis to explain the effectiveness of previously proposed methods for training deep networks, such as skip connections, normalization layers, and popular initialization schemes.

By demystifying trainability in deep networks, and disentangling it from model design, we hope that DKS will enable deep networks to reach new heights of performance, flexibility, and ease of use. There is even the potential that DKS may unlock a new class of neural models untrainable with standard tools like normalization layers and skip connections, possibly when used in combination with strong optimizers like K-FAC or Shampoo. 

Finally, because of their sensitivity to the strength of the optimizer, deep skip connection-free networks constructed with DKS have the potential to serve as new benchmark problem for neural network optimizers. This should be a welcome development to the area, as ResNets are often still used for benchmarking optimizers, despite the fact that it is known to be impossible to significantly outperform well-tuned SGD when training them at small/medium batch sizes \citep{zhang2019algorithmic}.

\section{Limitations and future directions}\label{sec:limitations-and-future-directions}

We end by discussing some limitations of DKS, along with possible ways to overcome them in future work.

\begin{itemize}

\item DKS currently doesn't support layers with multiplicative units, such as the self-attention layers in Transformers \citep{vaswani2017attention}. This is because we don't have a kernel approximation for such layers that would yield one-dimensional Q/C maps (or something similar). A possible way around this would be to generalize C maps to higher dimensional inputs, and develop new theory along the lines of Section \ref{sec:Cmap-analysis} to control their shape. Another possibility would be to find some weight initialization for self-attention layers which would give rise to one dimensional maps.

\item While DKS supports pooling layers in practice (based on our experiments), our theoretical treatment of these layers in Section \ref{sec:pooling-layers} is rudimentary and incomplete. Moreover, mean-pooling layers aren't really supported at all within our framework, since they make it impossible to achieve uniform q values. And while \emph{weighted} mean pooling layers can serve as a reasonable replacement (as discussed in Section \ref{sec:weighted-mean-pools}), their Q/C map interpretation, and the quality of their kernel function approximations, are both somewhat dubious. One possible way to improve this situation would be to develop better replacements for pooling layers that are compatible with DKS. Another would be to extend our analysis to handle non-uniform q values, possibly as part of a generalized higher-dimensional version of Q/C map analysis. The fact that pooling layers seem to work reasonably well with DKS in practice hints that this should be possible.

\item To match the training speed of standard ResNets on skip connection-free networks using DKS we were required to train with K-FAC or Shampoo. From a practical perspective this is somewhat unsatisfying, as those methods are significantly more complex than SGD, and introduce additional computational overheads (although these can be largely mitigated through various strategies such as those proposed in \citet{martens2015optimizing}, \citet{ba2017distributed}, and \citet{anil2020scalable}). An interesting direction for future work would be to try to achieve rapid training of skip connection-free networks with a much simpler optimizer like SGD or Adam (e.g.~by modifying DKS somehow), or to explain the importance of stronger optimizers for training such networks. Recent work arguing for the optimality of approximate natural gradient methods like K-FAC in in the NTK setting \citep{zhang2019fast, karakida2020understanding} may be a good starting point for the latter direction.

\item In our experiments with DKS we consistently observed increased overfitting compared to standard ResNets, resulting in top-1 test set accuracy on ImageNet that was lower by a few percent. This echoes similar observations made in related works such as \citet{zhang2019fixup}, and could by caused by a number of things, including the loss of noise from BN layers, or a subtle change in the inductive bias of the model. Addressing this remains an important direction for future work.

\item As discussed in Section \ref{sec:positively-homogeneous-problem}, while RELU activation functions can be used with DKS, one can only enforce three of the four Q/C map conditions. And while RELU networks with DKS perform well in most settings in our experiments despite this limitation, they perform poorly with skip connection-free networks trained using K-FAC (relative to other activation functions). Fortunately, since we have consistently strong performance for other activation functions, including RELU-like ones such as softplus, this arguably isn't a serious issue. Indeed, the main reason for using RELUs over other activation functions is that they are an important ingredient in the standard recipe for achieving fast and stable training of very deep networks, for which DKS is an alternative.

\item Recurrent neural networks (RNNs) are currently not supported by DKS. This is due to their sharing of parameters across time steps, which invalidates the kernel approximations that underlie our analysis. However, it is conceivable that a more advanced theory could be used to extend DKS to RNNs, and preliminary experiments we conducted with DKS on RNNs gave positive results, suggesting that it might already work well in practice.

\item Q/C map analysis is formally justified using kernel function approximations for neural networks at initialization time. The accuracy of these approximations is predicted by bounds such as those reviewed in Section \ref{sec:how-accurate-approx}. Currently, the best known bounds seem to be overly pessimistic, and in order to guarantee reasonable approximation error, require that the width (or channel dimension) grow \emph{exponentially} with the network's depth. We conjecture that much stronger bounds exist, although they might require the introduction of additional hypotheses, such as that the networks are constructed using DKS (or something similar).

\end{itemize}

\acks{We would like to thank Alex Botev, Alex Graves, Georg Ostrovski, Guodong Zhang, Ilja Kuzborskij, Koray Kavukcuoglu, Neil Rabinowitz, Soham De, Yann Dauphin, and Yee Whye Teh for their guidance, helpful discussions, and feedback on early drafts. We would also like to thank the entire team at DeepMind for supporting this project.}

\newpage

\part{Appendix}
\label{part:appendix}

\appendix

\section{Approximating average unit values}\label{app:average-unit-approx}

In addition to approximating the contents of the PKF $\Sigma_{f (Z), f (Z')}$, which can be thought of as entry-averages (across channels) of element-wise products between pairs of vectors in $f$'s output feature map, we may sometimes be interested in approximating the entry-averages of such vectors themselves. For a given $k$-dimensional vector $y$ with associated q value $q$ in $f$'s output, we have the intuitive approximation
\begin{equation}
  \frac{1}{k}  \mathbbm{1}^{\top} y \: \approx \: \mathbb{E}_{x \sim \mathcal{N} (0, 1)} \left[ \phi \left( \sqrt{q} x \right) \right], \label{eqn:average-unit-approx}
\end{equation}
where $\mathbbm{1}$ denotes a vector of 1's.  As we will show below, this is an accurate approximation with high probability in the same sense that the APKFs are. 


To begin, we define a modified version $g$ of $f$ with activation function $\psi (x) = \phi (x) + \mathbbm{1}$. With this definition we have $g (Z) = f (Z) + O$, where $O = \mathbbm{1}  \mathbbm{1}^{\top}$ is a matrix of 1's. Observing that
\[ g (Z)^{\top} g (Z) = (f (Z) + O)^{\top}  (f (Z) + O) = f (Z)^{\top} f (Z) + O^{\top} f (Z) + f (Z)^{\top} O + O^{\top} O \]
it follows that
\[ \Sigma_{g (Z), g (Z)} = \Sigma_{f (Z), f (Z)} + S + S^{\top} + \Sigma_{O, O} , \]
where
\[ A = \frac{1}{k}  \left[\begin{array}{cc}
     O^{\top} f (Z) & O^{\top} f (Z)\\
     O^{\top} f (Z) & O^{\top} f (Z)
   \end{array}\right] . \]
Meanwhile, under the APKF Condition we have that
\begin{eqnarray*}
  \Sigma_{g (Z), g (Z)} & = & \kappa_g (Z, Z') \\
  & \approx & \widetilde{\kappa_g} (\Sigma_{Z, Z'}) \\
  & = & \mathbb{E}_{u \sim \mathcal{N} (0, \Sigma_{Z, Z})} [\psi (u) \psi (u)^{\top}]\\
  & = & \mathbb{E}_{u \sim \mathcal{N} (0, \Sigma_{Z, Z})} [\phi (u) \phi (u)^{\top}] +\mathbb{E}_{u \sim \mathcal{N} (0, \Sigma_{Z, Z})} [\phi (u)  \mathbbm{1}^{\top}] +\mathbb{E}_{u \sim \mathcal{N} (0, \Sigma_{Z, Z})} [\mathbbm{1} \phi (u)^{\top}] \\
  & & + \; \mathbb{E}_{u \sim \mathcal{N} (0, \Sigma_{Z, Z})} [\mathbbm{1}  \mathbbm{1}^{\top}]\\
  & = & \widetilde{\kappa_f} (\Sigma_{Z, Z'}) + v \mathbbm{1}^{\top} + \mathbbm{1} v^{\top} + \mathbbm{1}  \mathbbm{1}^{\top}\\
  & \approx & \Sigma_{f (Z), f (Z)} + v \mathbbm{1}^{\top} + \mathbbm{1} v^{\top} + \mathbbm{1}  \mathbbm{1}^{\top},
\end{eqnarray*}
where $v =\mathbb{E}_{u \sim \mathcal{N} (0, \Sigma_{Z, Z})} [\phi (u)]$. Equating both expressions for $\Sigma_{g (Z), g (Z)}$, and using the fact that $\Sigma_{O, O} = \mathbbm{1}  \mathbbm{1}^{\top}$, we have
\[ v \mathbbm{1}^{\top} + \mathbbm{1} v^{\top} \approx A + A^{\top} . \]
Comparing diagonal entries of both sides implies
\[ v \approx \frac{1}{k}  (\mathbbm{1}^{\top} f (Z)) . \]
The RHS here is a vector (over locations) consisting of the vector-averages we are interested in estimating (i.e.~$\frac{1}{k}  \mathbbm{1}^{\top} y$). For any given one of them with associated q value $q$ we thus have the claimed approximation $\mathbb{E}_{x \sim \mathcal{N} (0, 1)} \left[ \phi \left( \sqrt{q} x \right) \right]$.

\section{Mathematical details for Section \ref{sec:local-C-map-derivatives}}\label{app:Gamma-derivative}

In this section we will derive the formula
\[ \Gamma'_{\phi} (c, q_1, q_2) = \sqrt{q_1 q_2} \Gamma_{\phi'} (c, q_1, q_2), \]
where
\[ \Gamma_{\phi} (c, q_1, q_2) \equiv \mathbb{E} \left[ \phi \left( \sqrt{q_1} x \right) \phi \left( \sqrt{q_2}  \left( cx + \sqrt{1 - c^2} y \right) \right) \right] \text{\quad for\quad} x, y \sim \mathcal{N} (0, 1), \]
and the derivative of $\Gamma_{\phi}$ is taken with respect to $c$.

Taking the derivative inside of the expectation we have
\begin{eqnarray*}
  \Gamma'_{\phi} (c, q_1, q_2) & = & \mathbb{E} \left[ \phi \left( \sqrt{q_1} x \right) \phi' \left( \sqrt{q_2}  \left( cx + \sqrt{1 - c^2} y \right) \right)  \sqrt{q_2}  \left( x - \frac{cy}{\sqrt{1 - c^2}} \right) \right]\\
  & = & \sqrt{q_2} \mathbb{E} \left[ \phi \left( \sqrt{q_1} x \right) \phi' \left( \sqrt{q_2}  \left( cx + \sqrt{1 - c^2} y \right) \right) x \right]\\
  &  & - \frac{\sqrt{q_2} c}{\sqrt{1 - c^2}} \mathbb{E} \left[ \phi \left( \sqrt{q_1} x \right) \phi' \left( \sqrt{q_2}  \left( cx + \sqrt{1 - c^2} y \right) \right) y \right] .
\end{eqnarray*}
By Stein's Lemma $\mathbb{E}_{u \sim \mathcal{N} (0, 1)} [f (u) u] =\mathbb{E}_{u \sim N (0, 1)} [f' (u)]$, so that
\begin{eqnarray*}
  \mathbb{E} \left[ \phi \left( \sqrt{q_1} x \right) \phi' \left( \sqrt{q_2}  \left( cx + \sqrt{1 - c^2} y \right) \right) x \right] & = & \sqrt{q_1} \mathbb{E} \left[ \phi' \left( \sqrt{q_1} x \right) \phi' \left( \sqrt{q_2}  \left( cx + \sqrt{1 - c^2} y \right) \right) \right]\\
  &  & - \hspace{1em} \sqrt{q_2} c\mathbb{E} \left[ \phi \left( \sqrt{q_1} x \right) \phi'' \left( \sqrt{q_2}  \left( cx + \sqrt{1 - c^2} y \right) \right) \right]
\end{eqnarray*}
and
\begin{eqnarray*}
  \mathbb{E} \left[ \phi \left( \sqrt{q_1} x \right) \phi' \left( \sqrt{q_2}  \left( cx + \sqrt{1 - c^2} y \right) \right) y \right] & = & \sqrt{q_2}  \sqrt{1 - c^2}\\
  &  & \cdot \hspace{1em} \mathbb{E} \left[ \phi \left( \sqrt{q_1} x \right) \phi'' \left( \sqrt{q_2}  \left( cx + \sqrt{1 - c^2} y \right) \right) \right] .
\end{eqnarray*}
Plugging these expressions into the above equation for $\Gamma'_{\phi} (c, q_1, q_2)$, and observing the cancellation of the $\sqrt{q_2} c\mathbb{E} \left[ \phi \left( \sqrt{q_1} x \right) \phi'' \left( \sqrt{q_2}  \left( cx + \sqrt{1 - c^2} y \right) \right) \right]$ term, we arrive at
\[ \Gamma'_{\phi} (c, q_1, q_2) = \sqrt{q_1 q_2} \mathbb{E} \left[ \phi' \left( \sqrt{q_1} x \right) \phi' \left( \sqrt{q_2}  \left( cx + \sqrt{1 - c^2} y \right) \right) \right] = \sqrt{q_1 q_2} \Gamma_{\phi'} (c, q_1, q_2) . \]

\section{A detailed analysis of C map convergence in deep networks}
\label{app:math-for-Cmaps_trainability}

\subsection{General purpose lemmas}

Unless specified otherwise, $f^n(x)$ will denote $\underbrace{f \circ \ldots \circ f}_{n}(x)$, i.e.~$f$ composed with itself $n$ times.

\begin{lemma}
\label{lem::fundamental_dynamical_systems}
Let $f : \RR \rightarrow \RR$ be given by $f(x) = \alpha x + \beta$, $\alpha \neq 0$. Then the explicit formula holds
\begin{equation}
f^n(x) = \alpha^n \left(x - \frac{\beta}{1-\alpha}\right) + \frac{\beta}{1-\alpha}.
\end{equation}
\end{lemma}

\begin{proof}
Let $\vp(x) = x + \frac{\beta}{1-\alpha}$. Let us consider $g = \vp^{-1} \circ f \circ \vp$. Then
\[
g(x) = \left[ f\left( x + \frac{\beta}{1-\alpha} \right) \right] - \frac{\beta}{1-\alpha} = \alpha x + \frac{\alpha \beta}{1-\alpha} + \beta - \frac{\beta}{1-\alpha} = 
\alpha x + \beta - \frac{\beta (1 - \alpha)}{1-\alpha} = \alpha x.
\]
Then $g^n(x) = \alpha^n x$. Moreover, $f^n = \vp \circ g^n \circ \vp^{-1}$, and thus 
\[
f^n(x) = \alpha^n \left(x - \frac{\beta}{1-\alpha}\right) + \frac{\beta}{1-\alpha}.
\]
\end{proof}

\begin{lemma}
\label{lem::iterated_inequality}
Let $x_1 < x_2$ and let the functions $f_1: [x_1, x_2] \rightarrow \RR$, $f_2: [x_1, x_2] \rightarrow \RR$ satisfy $f_1(x) \ge f_2(x)$ for every $x \in [x_1, x_2]$, and let $f_2$ be an increasing function. Then for every $n=1,2, \ldots$ such that $f^{n-1}(x) \in [x_1,x_2]$ there holds $f_1^n(x) \ge f_2^n(x)$
\end{lemma}

\begin{proof}
We use induction. The inequality holds for $n=1$ by assumption. Assume now that $f_1^k(x) \ge f_2^k(x)$ for some $k$ and that $f_1^k(x) \in [x_1, x_2]$. Then $f_1^{k+1}(x) = f_1(f_1^k(x)) \ge f_2(f_1^k(x)) \ge f_2(f_2^k(x))$, where the last inequality follows from the fact that $f_2$ is an increasing function.
\end{proof}

\begin{lemma}
\label{lem::first_smaller_than_1}
Assume that an increasing, strictly convex map $f$ satisfies $f(1) = 1$ and $f(x_0) = x_0$ for some $x_0 < 1$. Then $0 < \frac{f(b)-x_0}{b-x_0} < 1$ for every $b$ in $(x_0, 1)$.
\end{lemma}

\begin{proof}
The function $f$ is increasing, so $x_0 = f(x_0) < f(b)$, and it is strictly convex, so $f\left(t x_1 + (1-t) x_2\right) < t f(x_1) + (1-t) f(x_2)$ for all $x_1 < x_2$ and $0 < t < 1$. Using $x_1 = x_0$, $x_2 = 1$ and $t = \frac{1-b}{1-x_0}$ we obtain $f(b) = f\left( x_0 \cdot \frac{1-b}{1-x_0} + 1 \cdot \frac{b-x_0}{1-x_0} \right) <  \frac{1-b}{1-x_0} f\left( x_0 \right)  + \frac{b-x_0}{1-x0} f\left( 1 \right)  = \frac{1-b}{1-x_0} x_0  + \frac{b-x_0}{1-x_0} 1 = b$.
\end{proof}

\subsection{Definitions and preliminaries}

We will call a C map \textbf{\textit{nontrivial}} if and only if it is neither a constant function nor the identity function.

As we are assuming uniform q values, we have by Section \ref{sec:uniform-q-consequences} that C maps are positive definite functions, from which various properties immediately follow. The following proposition lists the ones we will use in this section:
\begin{proposition}
\label{prop::base_properties_cmap}
A C map $\CC$ satisfies the following properties:
\begin{itemize}
    \item[  i)] $\CC(1) = 1$ and $\CC(0) \ge 0$,
    \item[ ii)] for $c\ge 0$ is an increasing, convex function,
    \item[iii)] if $\CC$ is not a constant function it is strictly increasing for $c\ge 0$,
    \item[ iv)] if $\CC$ is nontrivial, it is strictly increasing and strictly convex for $c\ge 0$.
\end{itemize}
\end{proposition}


\begin{corollary}
\label{cor::linear_estimate_for_positive_x}
If $\CC$ is a nontrivial $c$-map, then for all $c \in [0,1]$ there holds $\CC(c) \ge  \CC'(1) c + 1-\CC'(1)$.
\end{corollary}

\begin{proof}
Suppose that for some $\tilde{c} \in [0,1]$ there holds $\CC(\tilde{c}) < \CC'(1) c + 1-\CC'(1)$. Then, by the Mean Value Theorem there exists $\vardbtilde{c} \in (\tilde{c}, 1)$ such that
\[ 
\CC'(\vardbtilde{c}) = \frac{1 - \CC(\tilde{c})}{1-\tilde{c}} > \frac{1 - \left( \CC'(1) \tilde{c} + 1-\CC'(1) \right)}{1-\tilde{c}} = \frac{\CC'(1) -  \CC'(1) \tilde{c}}{1-\tilde{c}}  = \CC'(1) , 
\]
which contradicts the strict convexity of $\CC$ on $[0,1]$ (Proposition~\ref{prop::base_properties_cmap}, iv)).
\end{proof}

\begin{lemma}
\label{lem::fundamental_bound_neg_x}
Suppose $\CC$ is a nontrivial C map. Then we have
\begin{equation}
\label{eq::base_estimate_neg_x}
\CC(-c) \ge 2 \CC(0) - \CC(c)     
\end{equation}
for every $c \in (0,1]$. Moreover, if this inequality becomes an equality for any $c \in (0,1]$, then $\CC$ is an odd function.
\end{lemma}

\begin{proof}

Because $\CC$ is positive definite, it can be written as
\begin{equation*}
\sum_{i = 0}^{\infty} b_i c^i
\end{equation*}
for $b_i \geq 0$. We can then decompose this as 
\begin{equation*}
\CC(c) = \CC(0) + \CC_e(c) + \CC_o(c),
\end{equation*} 
where $\CC_e(c) \equiv \sum\limits_{i=1}^{\infty} b_{2i} c^{2i}$ is an even function, and $\CC_o(c) \equiv \sum\limits_{i=0}^{\infty} b_{2i+1} c^{2i+1}$ is an odd function.

Thus, for $c \ge 0$
\[
\CC(-c) = \CC(0) + \CC_e(-c) + \CC_{o}(-c) = \CC(0) + \CC_e(c)  - \CC_{o}(c) \ge \CC(0) - \CC_e(c)  - \CC_{o}(c)
\]
as $\CC_{e}(c) = \CC_{e}(-c)$ and $\CC_{e}(c) \ge 0$. Therefore
\[
\CC(-c) \ge \CC(0) - \CC_e(c) - \CC_{o}(c) = 2 \CC(0) - (\CC(0) + \CC_e(c) + \CC_{o}(c)) = 2 \CC(0) - \CC(c).
\]
This inequality can become an equality for any nonzero $c$ if and only if $\CC_{e}(c)=0$. But if $\CC_{e}(c)=0$ for any such $c$, then $\CC_{e} \equiv 0$, which ends the proof.
\end{proof}

\begin{corollary}
\label{cor::linear_estimate_neg_general}
For $\gamma \in (0,1)$ there holds $\CC(-c) \ge - \frac{\CC(\gamma) - \CC(0)}{\gamma} c + \CC(0)$ for $c \in (0, \gamma)$.
\end{corollary}
\begin{proof}
By convexity (Proposition~\ref{prop::base_properties_cmap}),
\[
\CC(c) = \CC(0 \cdot \frac{\gamma - c}{\gamma}+ \gamma \cdot \frac{c}{\gamma}) \le \CC(0) \cdot \frac{\gamma - c}{\gamma} + \CC(\gamma) \cdot \frac{c}{\gamma} = \frac{\CC(\gamma) - \CC(0)}{\gamma} c + \CC(0).
\]
Combining this and Equation \ref{eq::base_estimate_neg_x} we get
\[
\CC(-c) \ge 2 \CC(0) - \CC(c)  \ge 2\CC(0) - \left( \frac{\CC(\gamma) - \CC(0)}{\gamma} c + \CC(0) \right) =  -\frac{\CC(\gamma) - \CC(0)}{\gamma} c + \CC(0).
\]
\end{proof}

\begin{corollary}
\label{cor::linear_estimate_neg}
There holds $\CC(-c) \ge -(1-\CC(0)) c + \CC(0)$ for $c \in (0, 1]$.
\end{corollary}
\begin{proof}
By convexity (Proposition~\ref{prop::base_properties_cmap}), $\CC(c) \le (1-\CC(0)) c + \CC(0)$. Combining this and Equation \ref{eq::base_estimate_neg_x} we get
\[
\CC(-c) \ge 2 \CC(0) - \CC(c)  \ge 2\CC(0) - (1-\CC(0)) c - \CC(0) = \CC(0) - (1-\CC(0)) c.
\]
\end{proof}

\begin{lemma}
\label{lem::technical_alpha}
Let $\CC$ be a C map and $c \in (-1, 0)$. Then
\[
0 < \frac{\CC(-c)-\CC(0)}{-c} < 1.
\]
\end{lemma}

\begin{proof}
The first inequality follows from the fact that $\CC$ is an increasing function on $[0,1]$, $\CC(0) \ge 0$, and that $-c \in (0,1)$. 

To show the second inequality, we consider two cases. 

If $\CC'(-c) \ge 1$, then $\CC(-c) < -c$, which is a consequence of the Mean Value Theorem. Indeed, $\CC'$ is an increasing function on the interval $[0,1]$, thus $\CC'(x) > 1$ for all $x \in [-c,1]$. By the Mean Value Theorem $\CC(1) - \CC(-c) > 1- (-c)$. But $\CC(1) = 1$, so this yields $-\CC(-c) > c$. Then $-c > \CC(-c) \ge \CC(-c) - \CC(0)$, as $\CC(0) \ge 0$, and therefore $ \frac{\CC(-c)-\CC(0)}{-c} < 1$, because $-c >0$. 

If, on the other hand, $\CC(-c) < 1$, then by monotonicity of $\CC'$ we have $\CC'(x) < 1$ for $x \in [0,-c]$, and the inequality $\frac{\CC(-c)-\CC(0)}{-c} < 1$ follows from the Mean Value Theorem.

\end{proof}

\begin{corollary}
\label{cor::no_negative_fix_points}
Let $\CC$ be a nontrivial C map. Then either $\CC$ has no fixed points on $[-1,0)$, or $\CC$ is an odd function.
\end{corollary}

\begin{proof}
By Corollary~\ref{cor::linear_estimate_neg}, $\CC(-c) \ge -(1-\CC(0)) c + \CC(0)$ for $c \in (0, 1]$. Thus, if $\CC(0) > 0$, then $\CC(c) \ge (1-\CC(0)) c + \CC(0) > c + f(0) > c$ for all $c \in [-1,0]$. In the case $\CC(0) = 0$ by Lemma~\ref{lem::fundamental_bound_neg_x} there holds $\CC(-c) \ge - \CC(c)$ for all $c \in [0,1]$. But, by strict concavity, $\CC(c) \le c$ for all $c \in [0,1]$. Thus if $\CC$ has a fixed point $c \in [-1,0)$, then $-c = \CC(-c) \ge -\CC(c) \ge -c$ and thus $\CC(-c) = 2\CC(0) - \CC(c)$, which (by Lemma~\ref{lem::fundamental_bound_neg_x}) implies that $\CC$ is an odd function.
\end{proof}

\begin{corollary}
\label{cor::fix_points_of_f}
Suppose $\CC$ is a nontrivial C map. Then one of the three following alternatives holds:
\begin{itemize}
    \item [  i)] The map $\CC$ has precisely one fixed point. This happens if and only if $\CC'(1) \le 1$, and the fixed point is $\fixpointofcmap = 1$.
    
    \item [ ii)] The map $\CC$ has precisely two fixed points. This happens if and only if $\CC'(1) > 1$ and $\CC$ is not an odd function, and the fixed points are $\fixpointofcmap \in [0,1)$ and $1$.
    
    \item [iii)] The function $\CC$ has precisely $3$ fixed points. This happens if and only if $\CC$ is an odd function and the fixed points are $-1$, $\fixpointofcmap = 0$ and $1$.
\end{itemize} 
\end{corollary}

\begin{proof}
We will treat each of i), ii) and iii) separately.
\begin{itemize}
    \item [  i)] By Proposition~\ref{prop::base_properties_cmap} iv), the function $\CC$ is strictly convex on $[0,1]$. We have $\CC(1) = 1$, so $\CC'(1) \le 1$ implies that $\CC(x) > x$ for $x \in [0,1)$. Indeed $g(x) := \CC(x) - x$ satisfies $g(1) = 0$ and $g$ is a decreasing function in [0,1], as $g'(x) = \CC'(x) - 1 < 0$ for $x \in [0,1)$. It remains to show, that $\CC$ has no fixed points on $[-1,0)$ interval. We showed that $\CC(0) > 0$, so $\CC$ cannot be an odd function, and thus, by  Corollary~\ref{cor::no_negative_fix_points} it has no fixed points on $[-1,0)$.
    
    \item[ ii)] Assume, that $\CC'(1) > 1$. Consider $g(x) = \CC(x) - x$. There holds $g(1) = 0$, and $g'(1) > 0$. Thus $g$ is an increasing function in some neighbourhood of $x=1$ (by assumption $\CC$ is an analytic function, hence all of its derivatives must be continuous). Then $g(1-\varepsilon) < 0$ for $\varepsilon$ sufficiently small. On the other hand, $g(0) \ge 0$, and thus by continuity of $g$ there exists $x \in [0,1)$ such that $g(x) = x$. The function $\CC$ is strictly convex on $[0,1]$, (by Proposition~\ref{prop::base_properties_cmap} iv)), so it can have at most two fixed points on $[0,1]$. By Corollary~\ref{cor::no_negative_fix_points}, these are the only fixed points of $\CC$, as we assumed that $f$ is not an odd function.
    
    \item [iii)]  We treat each of the implications in "if an only if" separately 
    \begin{itemize}
    \item["$\Leftarrow$"] By definition $\CC(1) = 1$, so if $\CC$ is an odd function, $f(-1) = -1$ and $\CC(0) = 0$. By Proposition~\ref{prop::base_properties_cmap} iv), the function $\CC$ is strictly convex on $[0,1]$, thus (because it is an odd function) it is strictly concave on $(-1,0)$, and thus the equation $\CC(x) = x$ has no solutions on $(-1,0) \cup (0,1)$.
    \item["$\Rightarrow$"] Is a direct consequence of Lemma~\ref{lem::fundamental_bound_neg_x}, as $\CC(-1) = -1 = 2 \CC(0) - f(1)$.
    \end{itemize}
\end{itemize}

\end{proof}

\begin{corollary}
\label{cor::x0_attractor}
The point $\fixpointofcmap$ from Corollary~\ref{cor::fix_points_of_f} ii) has to be an attractor. (In other words, $0 < \CC'(\fixpointofcmap) < 1$.)
\end{corollary}

\begin{proof}
As $\CC$ is positive definite, $\CC'(c) > 0$ for positive $c$'s. There holds $\frac{\CC(1) - \CC(\fixpointofcmap)}{1-\fixpointofcmap} = \frac{1-\fixpointofcmap}{1-\fixpointofcmap} = 1$ thus, by the Mean Value Theorem there exists a point $\bar{c}$ in the interval $(\fixpointofcmap, 1)$, such that $\CC'(\bar{c}) = 1$. By convexity, $\CC'$ is an increasing function on $(0, 1)$, so as $\fixpointofcmap < \bar{c}$, there holds $\CC'(\fixpointofcmap) < 1$.
\end{proof}

The following theorems give bounds on the convergence rate of c values under repeated applications of a C map $\CC$ satisfying i) or ii) of Corollary~\ref{cor::fix_points_of_f}.

\begin{theorem}[Linear global attractor for $\CC'(1) < 1$]
\label{thm::small_derivative}
Let a nontrivial C map $\CC$ satisfy $\CC'(1) < 1$. Then the unique (see Corollary~\ref{cor::fix_points_of_f}) fixed point $\fixpointofcmap = 1$ of $\CC$ is a linear global attractor of the whole set $[-1,1]$, and the following set inequalities holds
\[
\left( \CC'(1) \right)^n (c-1) + 1 \le \CC^n(c) \le 1
\]
for all $c\in [-1, 1]$.
\end{theorem}

\begin{proof}
By Corollary~\ref{cor::linear_estimate_for_positive_x} there holds $\CC(x) \ge \CC'(1) x + 1-\CC'(1)$ for $x \in [0,1]$. For values of $x \in [-1,0)$ we use Lemma~\ref{lem::fundamental_bound_neg_x}. This yields $\CC(x) \ge 2 \CC(0) - \CC(-x) \ge 2 \CC(0) - \left( -\CC'(1) x + 1 - \CC'(1) \right) \ge 2 - 2 \CC'(1) - \left( - \CC'(1) x + 1 - \CC'(1) \right) = 1-\CC'(1) + \CC'(1) x$. Thus we can apply Lemma~\ref{lem::fundamental_dynamical_systems} with $\alpha = \CC'(1)$ and $\beta = 1-\CC'(1)$ combined with Lemma~\ref{lem::iterated_inequality} on the whole set $[-1,1]$ and the Theorem follows.
\end{proof}

\begin{theorem}\label{thm::fixedpoint-convergence-detailed_appversion}
Let $\CC$ be a non-trivial C map satisfying $\CC'(1) > 1$ and $\CC$ is not an odd function (i.e. satisfying ii) of Corollary~\ref{cor::fix_points_of_f}), and let $\fixpointofcmap$ be the unique fixed point of $\CC$ in the interval $[0,1)$. Then for all $n=1,2,\ldots$
\begin{itemize}
    \item [  i)] for $c_0 \in (\fixpointofcmap, 1)$ there holds $\fixpointofcmap < \CC^n(c_0) \le \fixpointofcmap + \left( \frac{\CC(c_0)-\CC(\fixpointofcmap)}{c_0-\fixpointofcmap} \right)^n (c_0-\fixpointofcmap)$,
    \item [ ii)] for $c_0 \in [0, \fixpointofcmap)$ there holds $\fixpointofcmap + \left(\CC'(\fixpointofcmap) \right)^n (c_0-\fixpointofcmap) < \CC^n(c_0) < \fixpointofcmap$,
    \item [iii)] for $c_0 \in (-1, 0)$, if $\CC^{n-1}(c_0) \le 0$ there holds $\alpha^n\left(c_0 - \frac{\CC(0)}{1-\alpha}\right) + \frac{\CC(0)}{1-\alpha} \le \CC^n(c_0)$, where $\alpha = \frac{\CC(-c_0)-\CC(0)}{-c_0}$.
\end{itemize}
Moreover, we have that $\frac{\CC(c_0)-\CC(\fixpointofcmap)}{c_0-\fixpointofcmap}$, $\CC'(\fixpointofcmap)$, and $\alpha$ are all bounded strictly between 0 and 1.

\end{theorem}

\begin{proof}
We are going to use an appropriate linear estimate in each of the three dynamical regimes, and then invoke Lemma~\ref{lem::fundamental_dynamical_systems}

\begin{itemize}
\item [  i)] Note that $0 < \frac{\CC(c_0)-\CC(\fixpointofcmap)}{c_0-\fixpointofcmap} < 1$ by Lemma~\ref{lem::first_smaller_than_1}.

There holds $\CC(c) \le \frac{\CC(c_0)-\CC(\fixpointofcmap)}{c_0-\fixpointofcmap}(c-\fixpointofcmap) + \fixpointofcmap$ for $c \in [\fixpointofcmap, c_0]$, as $\CC(\fixpointofcmap) = \fixpointofcmap$ and $\CC$ is a convex function in $[0,1]$ (by Proposition~\ref{prop::base_properties_cmap}). First, let us note that the function $\CC$ is increasing by Proposition~\ref{prop::base_properties_cmap}. Let us take $\CC_1(c) \equiv \frac{\CC(c_0)-\CC(\fixpointofcmap)}{c_0-\fixpointofcmap}(c-\fixpointofcmap) + \fixpointofcmap$ and $\CC_2 \equiv \CC$. Such choice satisfies the assumptions of Lemma~\ref{lem::iterated_inequality}. Indeed, they are both increasing functions, both satisfying $\CC_1(\fixpointofcmap) = \CC_2(\fixpointofcmap) = \fixpointofcmap$, and $\CC_1(c_0) = \CC_2(c_0) < c_0$, so we get the inequality $\CC^n_1(c_0) \ge \CC^n_2(c_0)$ for all $n = 1,2, \ldots$. We apply Lemma~\ref{lem::fundamental_dynamical_systems} to function $\CC_1$, with $\alpha = \frac{\CC(c_0)-\CC(\fixpointofcmap)}{c_0-\fixpointofcmap}$ and $\beta = (1 -\alpha)\fixpointofcmap$, so that $\frac{\beta}{1-\alpha} = \fixpointofcmap$. This yields $\CC^n_1(c_0) = \fixpointofcmap + \left( \frac{\CC(c_0)-\CC(\fixpointofcmap)}{c_0-\fixpointofcmap} \right)^n (c_0-\fixpointofcmap)$ for every $n=1,2,\ldots$, which together with $\CC(c) \ge \fixpointofcmap$ for $\fixpointofcmap \le c \le c_0$ finishes the proof in this case. 

\item [ ii)] Note that $0 < \CC'(\fixpointofcmap) < 1$ by Corollary~\ref{cor::x0_attractor}.

There holds $\CC(c) \ge \CC'(\fixpointofcmap)(c-\fixpointofcmap) + \fixpointofcmap$ for $c \in [0,\fixpointofcmap)$, as $\CC(\fixpointofcmap) = \fixpointofcmap$ and $\CC$ is a convex function in $[0,1]$ (by Proposition~\ref{prop::base_properties_cmap}). Indeed, convexity implies, that $\CC'(c) < \CC'(\fixpointofcmap)$ for $c \in [0, \fixpointofcmap]$. By Mean Value Theorem $\CC(\fixpointofcmap) - \CC(c) \le \CC'(c)(\fixpointofcmap-c)$. We have $\CC(\fixpointofcmap) = \fixpointofcmap$, so the inequality becomes 
$\fixpointofcmap - \CC(c) \le \CC'(c)(\fixpointofcmap-c)$, and finally, $\CC(c) \ge \CC'(c)(c-\fixpointofcmap) + \fixpointofcmap$.
Let us take $\CC_1 \equiv \CC$ and $\CC_2(c) \equiv \CC'(\fixpointofcmap)(c-\fixpointofcmap) + \fixpointofcmap$. Such functions satisfy the assumptions of Lemma~\ref{lem::iterated_inequality} (note, that $\CC_2$ is increasing, because $\CC'(c_0) > 0$). Therefore $\CC^n_1(c_0) \ge \CC^n_2(c_0)$ for all $n = 1,2, \ldots$. We apply Lemma~\ref{lem::fundamental_dynamical_systems} with $\alpha = \CC'(\fixpointofcmap)$, $\beta = \left(1-\CC'(\fixpointofcmap)\right)\fixpointofcmap$, so that $\frac{\beta}{1-\alpha} = \fixpointofcmap$, and obtain  $\CC^n_2(c_0) = \left(\CC'(\fixpointofcmap)\right)^n (c_0-\fixpointofcmap) + \fixpointofcmap$, which finishes the proof in this case. 

\item [iii)] 
Note, that $0 < \alpha < 1$ by Lemma~\ref{lem::technical_alpha}.

By applying Corollary~\ref{cor::linear_estimate_neg_general} with $\gamma = -c_0$ we obtain $\CC(-c) \ge -\frac{\CC(-c_0) - \CC(0)}{-c_0} c + \CC(0)$ for $c \in [0, -c_0]$, and therefore $\CC(c) \ge \frac{\CC(-c_0) - \CC(0)}{-c_0} c + \CC(0)$ for $c \in [c_0, 0]$. Similarly to the previous two regimes, we apply Lemma~\ref{lem::iterated_inequality} to $\CC_1 \equiv \CC$ and $\CC_2(c) \equiv \frac{\CC(-c_0) - \CC(0)}{-c_0} c + \CC(0)$, and we get $\CC^n_1(c_0) \ge \CC^n_2(c_0)$ for all $n = 1,2, \ldots$ such that $\CC^{n-1}(c_0) \le 0$. Note, that $\CC_2$ is an increasing function, as $\frac{\CC(-c_0) - \CC(0)}{-c_0} > 0$. We apply Lemma~\ref{lem::fundamental_dynamical_systems} to function $\CC_2$ with $\alpha = \frac{\CC(-c_0)-\CC(0)}{-c_0}$ and $\beta = \CC(0)$, which yields $\CC^n_2(c_0) = \alpha^n\left(c_0 - \frac{\CC(0)}{1-\alpha}\right) + \frac{\CC(0)}{1-\alpha}$ for every $n=1,2,\ldots$. Thus we have $\CC^n_1(c_0) \ge \alpha^n\left(c_0 - \frac{\CC(0)}{1-\alpha}\right) + \frac{\CC(0)}{1-\alpha}$ whenever $\CC^{n-1}(c_0) \le 0$, which finishes the proof in this case.
\end{itemize}
\end{proof}

\begin{theorem}[The case of an odd map]\label{thm:odd-case}
Let $\CC$ be an odd nontrivial C map. Let $0 < \hat{c} < 1$. Then $\fixpointofcmap = 0$ is a linear global attractor of the whole set $[-\hat{c}, \hat{c}]$ and for any $c \in [-\hat{c}, \hat{c}]$ there holds
\[
-\hat{c}^n \le \CC^n(c) \le \hat{c}^n.
\]
\end{theorem}

\begin{proof}
We only need to prove the inequality $\CC^n(c) \le \hat{c}^n$ for positive $c$-s, due to symmetry and the fact that $c > 0 \Rightarrow \CC(c) > 0$ (the positive semi-axis is invariant under the map $\CC$). By Proposition~\ref{prop::base_properties_cmap}, iv) the map $\CC$ is strictly convex for $c \ge 0$, thus $\CC(c) \le \hat{c} \cdot c$ for $c \in [0, \hat{c}]$ and the theorem follows from Lemma~\ref{lem::fundamental_dynamical_systems} combined with Lemma~\ref{lem::iterated_inequality}.
\end{proof}


\section{Mathematical details for Section \ref{sec:Cmap-analysis}}

\subsection{Proof of Theorem \ref{thm:deviation-bound}}\label{app:thm-dev-proof}

Given our running assumption of uniform q values, C maps are positive definite functions (as established in Section \ref{sec:uniform-q-consequences}). This means that we can write
\[ C_f (c) = \sum^{\infty}_{i = 0} b_i c^i, \]
for some coefficients $b_i \geqslant 0$, so that the derivative of $C_f$ can similarly be written as
\[ C_f' (c) = \sum^{\infty}_{i = 1} ib_i c^{i - 1} . \]
Note that under this notation we have
\begin{enumeratenumeric}
  \item $b_1 = C'_f (0)$ ,
  
  \item $\sum^{\infty}_{i = 1} b_i = C_f (1) = 1$ , and
  
  \item $0 \leqslant b_i \leqslant 1$ for all $i$.
\end{enumeratenumeric}
Using these facts, it follows that
\begin{eqnarray*}
  \max_{c \in [- 1, 1]} | C_f (c) - c | & = & \max_{c \in [- 1, 1]} \left| b_0 + \sum_{i = 2}^{\infty} b_i c^i - (1 - b_1) c \right|\\
  & \leqslant & \max_{c \in [- 1, 1]} \left[ b_0 + \sum_{i = 2}^{\infty} b_i  | c |^i + (1 - b_1)  | c | \right]\\
  & \leqslant & b_0 + \sum_{i = 2}^{\infty} b_i \max_{c \in [- 1, 1]} | c |^i + (1 - b_1) \max_{c \in [- 1, 1]} | c |\\
  & = & b_0 + \sum_{i = 2}^{\infty} b_i + 1 - b_1 \hspace{0.8em} = \hspace{0.8em} 2 (1 - b_1)\\
  & = & 2 (1 - C'_f (0)) .
\end{eqnarray*}
Observing that $\left| b_0 - \frac{1}{2} \right| \leqslant \frac{1}{2}$ (which follows from $0 \leqslant b_0 \leqslant 1$) we also have that
\begin{eqnarray*}
  \max_{c \in [- 1, 1]} | C_f (c) - c | & \geqslant & \left| C_f \left( b_0 - \frac{1}{2} \right) - \left( b_0 - \frac{1}{2} \right) \right|\\
  & = & \left| \sum_{i = 2}^{\infty} b_i  \left( b_0 - \frac{1}{2} \right)^i + b_0 + b_1  \left( b_0 - \frac{1}{2} \right) - \left( b_0 - \frac{1}{2} \right) \right|\\
  & = & \left| \sum_{i = 2}^{\infty} b_i  \left( b_0 - \frac{1}{2} \right)^i + b_1 b_0 + \frac{1}{2}  (1 - b_1) \right|\\
  & \geqslant & b_1 b_0 + \frac{1}{2}  (1 - b_1) - \sum_{i = 2}^{\infty} b_i  \left| b_0 - \frac{1}{2} \right|^i\\
  & \geqslant & \frac{1}{2}  (1 - b_1) - \sum_{i = 2}^{\infty} b_i  \left( \frac{1}{2} \right)^i \hspace{0.8em} = \hspace{0.8em} \frac{1}{2}  (1 - b_1) - \frac{1}{4}  \sum_{i = 2}^{\infty} b_i  \left( \frac{1}{2} \right)^{i - 2}\\
  & \geqslant & \frac{1}{2}  (1 - b_1) - \frac{1}{4}  \sum_{i = 2}^{\infty} b_i \hspace{0.8em} = \hspace{0.8em} \frac{1}{2}  (1 - b_1) - \frac{1}{4}  (1 - b_0 - b_1)\\
  & = & \frac{1}{4}  (1 - b_1) + \frac{1}{4} b_0 \hspace{0.8em} \geqslant \hspace{0.8em} \frac{1}{4}  (1 - b_1)\\
  & = & \frac{1}{4}  (1 - C'_f (0)) .
\end{eqnarray*}
Similarly, for the second measure of deviation we have
\begin{eqnarray*}
  \max_{c \in [- 1, 1]} | C'_f (c) - 1 | = \max_{c \in [- 1, 1]} \left| \sum_{i = 2}^{\infty} ib_i c^{i - 1} - (1 - b_1) \right| & \leqslant & \max_{c \in [- 1, 1]} \left[ \sum_{i = 2}^{\infty} ib_i  | c |^{i - 1} + (1 - b_1) \right]\\
  & \leqslant & \sum_{i = 2}^{\infty} ib_i \hspace{0.8em} + 1 - b_1\\
  & = & C_f' (1) - b_1 + 1 - b_1\\
  & = & 2 (1 - C'_f (0)) + (C_f' (1) - 1) .
\end{eqnarray*}
Finally, if $b_0 = C_f (0) = 0$, we can relate the two key quantities $C_f' (1) - 1$ and $1 - C'_f (0)$ as follows:
\begin{eqnarray}
  C_f' (1) - 1 & = & C_f' (1) - C_f (1) \; = \; \sum_{i = 1}^{\infty} ib_i - \sum_{i = 1}^{\infty} b_i \nonumber\\
  & = & \sum_{i = 2}^{\infty} (i - 1) b_i \; \geqslant \; \sum_{i = 2}^{\infty} b_i = 1 - b_1 \nonumber\\
  & = & 1 - C'_f (0) .  \label{eqn:C-slope-ineq}
\end{eqnarray}
The theorem then follows directly from the above inequalities.

\subsection{Proof of Proposition \ref{prop:nonlinear-measure}}\label{app:prop-nonlin-proof}

It is well known that $H$ has a basis $h_0, h_1, h_2, \ldots$ known as the Hermite polynomials \citep[e.g.][]{wiki_hermite}, which is orthonormal (i.e.~$\langle h_i, h_j \rangle_H = 0$ for $i \neq j$ and $\| h_i \| = 1$), and has many other properties useful properties. Two of which we will make use of here is that $h_1$ is the identity function (i.e.~$h_1 (x) = x$), and that $h_0$ is constant and equal to 1 (i.e.~$h_0 (x) = 1$).

Because $h_0, h_1, h_2, \ldots$ form an orthonormal basis of $H$, we can represent $\phi$ in terms of this basis as
\[ \phi (x) = \sum_{i = 1}^{\infty} a_i h_i (x) \text{\quad where\quad} a_i = \langle \phi, h_i \rangle_H . \]
\citet{daniely2016toward} showed that when its input $q$ value is 1, the unnormalized C map $\tilde{\phi} (c) = \Gamma_{\phi} (c, 1, 1)$ of $f$ (aka the dual activation function of $\phi$; see Section \ref{sec:dual-activations}) can be obtained from this representation as
\begin{equation}
  \tilde{\phi} (c) = \sum_{i = 1}^{\infty} a_i^2 c^i = \sum_{i = 1}^{\infty} \langle \phi, h_i \rangle_H^2 c^i . \label{eqn:dual-activation-rep}
\end{equation}

Using the bilinearity of inner products, and the fact that $\| h_1 \|_H = 1$, we have
\begin{eqnarray*}
  \tmop{nl} (\phi)^2 & = & \frac{\| \phi - \langle \phi, h_1 \rangle h_1 \|^2_H}{\| \phi \|^2_H} \; = \; \frac{\| \phi \|^2_H - 2 \langle \phi, h_1 \rangle_H  \langle \phi, h_1 \rangle_H + \langle \phi, h_1 \rangle_H^2  \| h_1 \|^2}{\| \phi \|^2_H}\\
  & = & 1 - \frac{\langle \phi, h_1 \rangle_H^2}{\| \phi \|_H^2} .
\end{eqnarray*}
Using the facts that  $\tilde{\phi}' (0) = \langle \phi, h_1 \rangle_H^2$ (from Equation \ref{eqn:dual-activation-rep}) and that $C_f (c) = \tilde{\phi} (c) / Q_f (1) = \tilde{\phi} (c) / \| \phi \|_H^2$, we have
\[ 1 - \frac{\langle \phi, h_1 \rangle_H^2}{\| \phi \|_H^2} = 1 - \frac{\tilde{\phi}' (0)}{\| \phi \|_H^2} = 1 - C'_f (0), \]
and thus $\tmop{nl} (\phi)^2 = 1 - C'_f (0)$ as claimed.

Plugging this into Theorem \ref{thm:deviation-bound} it further follows that
\begin{eqnarray*}
  \frac{1}{4} \tmop{nl} (\phi)^2 \: \leqslant & \max_{c \in [- 1, 1]} | C_f (c) - c | & \leqslant \hspace{0.8em} 2 \tmop{nl} (\phi)^2 .
\end{eqnarray*}
\subsection{Proof of Proposition \ref{prop:nonaffine-measure}}\label{app:prop-nonaff-proof}

Analogously to $\tmop{nl} (\phi)$, $\tmop{na} (\phi)$ can be written as
\[ \tmop{na} (\phi)^2 = 1 - \frac{\langle \phi, h_1 \rangle_H^2}{\| \phi \|_H^2} . \]
Using Equation \ref{eqn:C-map-gen-derivative}, and the identities from the previous subsection, we have
\[ \frac{\| \phi' \|_H^2}{\| \phi \|_H^2} = \frac{\Gamma_{\phi'} (1, 1, 1)}{Q_f (1)} = C_f' (1) . \]
Combining this with $\langle \phi', h_0 \rangle_H^2 = \tilde{\phi}' (0)$ (from Equation \ref{eqn:dual-activation-rep}), it follows that
\[ \frac{\langle \phi', h_0 \rangle_H^2}{\| \phi' \|_H^2} = \frac{\tilde{\phi}' (0)}{\| \phi' \|_H^2} = \frac{\| \phi \|_H^2}{\| \phi' \|_H^2}  \frac{\tilde{\phi}' (0)}{\| \phi \|_H^2} = \frac{C'_f (0)}{C'_f (1)}, \]
where $f$ is a combined layer with $\phi$ as its activation function.

Thus,
\[ \tmop{na} (\phi)^2 = 1 - \frac{C'_f (0)}{C'_f (1)} \]
as claimed.

\subsection{Proof Proposition \ref{prop:F-neg-c-interp}}\label{app:F-neg-c-interp-proof}

Because $C_f$ is positive definite we can write it as $C_f (c) = \sum^{\infty}_{i = 0} b_i c^i$ for some $b_i \geqslant 0$ with $\sum^{\infty}_{i = 0} b_i = C_f (1) = 1$. Using this we can rewrite $F_f (c)$ as follows:
\[ F_f (c) = \frac{C_f (| c |) - C_f (0)}{| c |} = \frac{\sum^{\infty}_{i = 0} b_i  | c |^i - b_0}{| c |} = \sum^{\infty}_{i = 1} b_i  | c |^{i - 1} . \]
We then have
\begin{eqnarray*}
  \left| \frac{C_f (c) - C_f (0)}{c} \right| & = & \left| \frac{\sum^{\infty}_{i = 0} b_i c^i - b_0}{c} \right|\\
  & = & \left| \sum^{\infty}_{i = 1} b_i c^{i - 1} \right|\\
  & \leqslant & \sum^{\infty}_{i = 1} b_i  | c |^{i - 1}\\
  & = & F_f (c)
\end{eqnarray*}
as claimed.

\subsection{Proof of Proposition \ref{prop:degen-deriv-bound}}\label{app:degen-deriv-bound-proof}

As in Appendix \ref{app:F-neg-c-interp-proof} we have $C_f (c) = \sum^{\infty}_{i = 0} b_i c^i$ for some $b_i \geqslant 0$ with $\sum^{\infty}_{i = 0} b_i = C_f (1) = 1$, and $F_f (c) = \sum^{\infty}_{i = 1} b_i  | c |^{i - 1}$.

Let $h (x) = x | c |^{x - 1}$, so that $h' (x) = | c |^{x - 1} + x | c |^{x - 1} \log | c | = | c |^{x - 1}  (x \log | c | + 1)$. We observe that $h' (x) \leqslant 0$ for $x \geqslant - 1 / \log | c |$, and thus $h (x)$ is a decreasing function for $x \geqslant - 1 / \log | c |$.

Let $y = \log F_f (c) / \log | c |$. Since $F_f (c) \leqslant 1$, we have $y \geqslant - 1 / \log | c |$. Using this fact, and that $b_i \geqslant 0$ for all $i$, it thus follows that
\begin{eqnarray*}
  | C_f' (c) | & = & \left| \sum^{\infty}_{i = 1} ib_i c^{i - 1} \right|\\
  & \leqslant & \sum^{\infty}_{i = 1} ib_i  | c |^{i - 1}\\
  & = & \sum^{\lceil y \rceil - 1}_{i = 1} ib_i  | c |^{i - 1} + \sum^{\infty}_{i = \lceil y \rceil} ib_i  | c |^{i - 1}\\
  & \leqslant & y \sum^{\lceil y \rceil - 1}_{i = 1} b_i  | c |^{i - 1} + \left( \sum^{\infty}_{i = i_0 + 1} b_i \right) y | c |^{y - 1} \\
  & \leqslant & y \sum^{\infty}_{i = 1} b_i  | c |^{i - 1} + \left( \sum^{\infty}_{i = 0} b_i \right) y | c |^{y - 1}\\
  & = & y (F_f (c) + | c |^{y - 1})\\
  & = & \frac{\log F_f (c)}{\log | c |}  (F_f (c) + | c |^{\log F_f (c) / \log | c | - 1})\\
  & = & \frac{\log F_f (c)}{\log | c |}  \left( F_f (c) + \frac{F_f (c)}{| c |} \right) \: = \: \frac{F_f (c) \log F_f (c)}{| c | \log | c |}  (1 + | c |) .
\end{eqnarray*}
\subsection{Proof of Proposition \ref{prop:degen-C-map-deriv}}\label{app:degen-C-map-deriv-proof}

Let $\mathcal{C}$ be the C map for each of the $D$ subnetworks. Because $\mathcal{C}$ is positive definite we can write it as $\mathcal{C} (c) = \sum^{\infty}_{i = 0} b_i c^i$ for some $b_i \geqslant 0$ with $\sum^{\infty}_{i = 0} b_i = \mathcal{C} (1) = 1$.

Since $\fixpointofcmap = 1$, we have that $1$ is an attractive fixed point of $\mathcal{C}$, and thus $0 \leqslant \mathcal{C}' (1) \leqslant 1$. (noting that $\mathcal{C}' (1) \geqslant 0$ is true because $\mathcal{C}$ is positive definite and thus convex on $[0, 1]$). Meanwhile, since $\mathcal{C} (1) = 1$, we have by the chain rule that $C'_f (1) = \mathcal{C}' (1)^D$.

Now because $\mathcal{C}' (1) \leqslant 1$, we have by Corollary \ref{cor::fix_points_of_f} that $-1$ cannot be a fixed point of $\mathcal{C}$, and thus $c_1 \equiv \mathcal{C} (-1) > -1$. There are two cases for $c_1$ to consider.

In that case that $c_1 = 1$ we have
\[ \sum_{i = 0}^{\infty} b_i = \mathcal{C} (1) = 1 = \mathcal{C} (- 1) = \sum_{i = 0}^{\infty} b_i  (- 1)^i = \sum_{i = 0}^{\infty} b_i - 2 \sum_{i \text{ odd}}^{\infty} b_i \]
and thus $\sum_{i \text{ odd}} b_i = 0$. Because $b_i \geqslant 0$ for all $i$ it thus follows that $b_i = 0$ for odd $i$, and so $\mathcal{C}$ is an even function and therefore $\mathcal{C}' (- 1) = - \mathcal{C}' (1)$.

It remains to consider the case $c_1 \neq 1$. Let $g$ be a subnetwork of $f$ consisting of $D - 1$ compositions of the subnetworks that define $f$. Under this definition we have $C_f (c) = C_g (\mathcal{C} (c))$. Since $\fixpointofcmap = 1$ we have that $C_g (c) \rightarrow 1$ as $D \rightarrow \infty$ for any $c \in (- 1, 1)$ so that $F_g (c) \rightarrow 0$. Thus by Proposition \ref{prop:degen-deriv-bound} (and Remark \ref{rem:degen-deriv-bound} to handle the case $c_1 = 0$) it follows that $| C'_g (c_1) | \rightarrow 0$ (since $c_1 \in (- 1, 1)$). By the chain run we therefore have that $| C'_f (- 1) | = | C'_g (c_1) |  | \mathcal{C}' (- 1) | \rightarrow 0$ as $D \rightarrow \infty$.

\subsection{Proof of Proposition \ref{prop:degen-C-map-deriv-2}}\label{app:degen-C-map-deriv-2-proof}

Since $1 - \epsilon > 0$, we have by definition that
\[ F_f (1 - \epsilon) = \frac{C_f (1 - \epsilon) - C_f (0)}{1 - \epsilon} . \]
Rearranging and bounding, this becomes
\[ C_f (1 - \epsilon) = C_f (0) + (1 - \epsilon) F_f (1 - \epsilon) \leqslant C_f (0) + F_f (1 - \epsilon) . \]
Because $C_f$ is positive definite and thus convex on $[0, 1]$, we have for any $c \in [0, 1]$ that
\[ C_f' (1) \geqslant \frac{C_f (1) - C_f (c)}{1 - c} . \]
Taking $c = 1 - \epsilon$ and using using $C_f (1) = 1$ and the above inequality, we thus have
\[ C_f' (1) \geqslant \frac{1 - C_f (1 - \epsilon)}{\epsilon} \geqslant \frac{1 - (C_f (0) + F_f (1 - \epsilon))}{\epsilon} \]
as claimed.

\section{Mathematical details for Section \ref{sec:weighted-mean-pools}}\label{app:weighted-mean-pool}

A weighted mean-pooling layer $f$ over $\ell$ locations is defined by
\[ f (Z) = Zw, \]
where $w$ is an $\ell$-dimensional vector of parameters initialized using $\mathcal{N} (0, (1 / \ell) I)$.

In this section we will show that
\[ \widetilde{\kappa_f} (\Sigma_{Z, Z'}) = \frac{1}{\ell}  \left[\begin{array}{cc}
     \tmop{tr} \left( \frac{1}{k} Z^{\top} Z \right) & \tmop{tr} \left( \frac{1}{k} Z^{\top} Z' \right)\\
     \tmop{tr} \left( \frac{1}{k} {Z'}^{\top} Z \right) & \tmop{tr} \left( \frac{1}{k} {Z'}^{\top} Z' \right)
   \end{array}\right] \]
is an approximation of $\kappa_f$ at initialization-time (with high probability), and that it becomes more precise as $\ell$ grows, but only if the average absolute input c value to $f$ (across pairs of locations) simultaneously goes to zero.

Noting that $\mathbb{E} [ww^{\top}] = (1 / \ell) I$, we have for arbitrary matrix $M$ that $\mathbb{E} [w^{\top} Mw] =\mathbb{E} [\tmop{tr} (w^{\top} Mw)] =\mathbb{E} [\tmop{tr} (ww^{\top} M)] = \tmop{tr} (\mathbb{E} [ww^{\top}] M) = \tmop{tr} (M) / \ell$. Thus, when conditioned on $Z$ and $Z'$, we have $\mathbb{E} [f (Z)^{\top} f (Z')] =\mathbb{E} [w^{\top} Z^{\top} Z' w] = \tmop{tr} (Z^{\top} Z') / \ell$. It then follows that
\[ \mathbb{E} [\kappa_f (Z, Z')] = \frac{1}{k}  \left[\begin{array}{cc}
     f (Z)^{\top} f (Z) & f (Z)^{\top} f (Z')\\
     f (Z')^{\top} f (Z) & f (Z')^{\top} f (Z')
   \end{array}\right] = \frac{1}{k \ell}  \left[\begin{array}{cc}
     \tmop{tr} (Z^{\top} Z) & \tmop{tr} (Z^{\top} Z')\\
     \tmop{tr} \left( {Z'}^{\top} Z \right) & \tmop{tr} \left( {Z'}^{\top} Z' \right)
   \end{array}\right], \]
where $k$ is the number of channels. This is the claimed formula for $\widetilde{\kappa_f} (\Sigma_{Z, Z'})$.

It remains to establish the conditions under which $\kappa_f (Z, Z')$ will concentrate around its expectation. This is more difficult than in the combined layer case, as the different output channels of $f (Z)$ are not independent given $Z$, meaning that the variance of $\kappa_f (Z, Z')$ will not necessarily shrink as $k$ grows.

Instead, our strategy going forward will be to compute 
\[\tmop{Var} ([\kappa_f (Z, Z')]_{1, 2}) = \frac{1}{k^2} \tmop{Var} (f (Z)^{\top} f (Z'))\] 
as a function of $Z$ and $Z'$, and argue that this goes to zero as $\ell$ grows, given certain conditions on the value of the product $Z^{\top} Z'$ . (This will be sufficient to handle the other 3 entries of $\kappa_f (Z, Z')$ by taking $Z = Z'$.) Given a shrinking variance, concentration then follows via Chebyshev's inequality, although a faster rate could likely be obtained by observing that $[\kappa_f (Z, Z')]_{1, 2}$ is an average of independent sub-exponential random variables (for which better concentration bounds exists).

By Appendix A of \citet{cooijmans2019variance} we have for an arbitrary matrix $M$ that
\[ \mathbb{E}_{u \sim \mathcal{N} (0, \sigma^2 I)} [uu^{\top} Muu^{\top}] = \sigma^4  (\tmop{tr} (M) I + 2 M) . \]
Thus, by taking $M = Z^{\top} Z'$ it follows that
\begin{eqnarray*}
  \mathbb{E} [(f (Z)^{\top} f (Z'))^2] & = & \mathbb{E} [w^{\top} Z^{\top} Z' ww^{\top} Z^{\top} Z' w]\\
  & = & \mathbb{E} [\tmop{tr} (w^{\top} Z^{\top} Z' ww^{\top} Z^{\top} Z' w)]\\
  & = & \tmop{tr} (\mathbb{E} [ww^{\top} Z^{\top} Z' ww^{\top}] Z^{\top} Z')\\
  & = & \frac{1}{\ell^2} \tmop{tr} (Z^{\top} Z')^2 + \frac{2}{\ell^2} \tmop{tr} ((Z^{\top} Z')^2) .
\end{eqnarray*}
Given this, and the above fact that $\mathbb{E} [f (Z)^{\top} f (Z')] = \tmop{tr} (Z^{\top} Z')$, we have
\[ \tmop{Var} (f (Z)^{\top} f (Z')) =\mathbb{E} [(f (Z)^{\top} f (Z'))^2] -\mathbb{E} [(f (Z)^{\top} f (Z'))]^2 = \frac{2}{\ell^2} \tmop{tr} ((Z^{\top} Z')^2), \]
and so we conclude that
\[ \tmop{Var} ([\kappa_f (Z, Z')]_{1, 2}) = \frac{2}{k^2 \ell^2} \tmop{tr} ((Z^{\top} Z')^2) . \]
Note that the entries of $Z^{\top} Z'$ will grow in proportion to $k$, so that $\tmop{tr} ((Z^{\top} Z')^2)$ will grow in proportion to $k^2$, and thus the variance {\tmem{won't}} shrink as $k$ grows. Moreover, by taking $Z$ and $Z'$ to both be matrices of 1's (so that their associated q values are 1 for each location), we have that $\tmop{tr} ((Z^{\top} Z')^2) = k^2 \ell^2$, and so the variance won't always shrink as $\ell$ grows either. Nonetheless, we can identify situations where it will shrink with $\ell$, as we will explain next.

Let $q_{\max}$ be the maximum dimension-normalized squared norm of the columns of $Z$ and $Z'$, or in other words, the maximal input q value to $f$. By the Cauchy-Schwarz inequality, the absolute values of the entries of $Z^{\top} Z'$ are upper bounded by $q_{\max}$. If $Z^{\top} Z'$ is diagonal it is easy to show that $\tmop{tr} ((Z^{\top} Z')^2) \leqslant k^2 \ell q^2_{\max}$, and so the variance will indeed shrink at a rate of $1 / \ell$. More generally, if $c_{\tmop{avg}}$ is the average absolute cosine similarity between columns of $Z$ and $Z'$ (i.e.~the average absolute input c value between $Z$ and $Z'$ across all pairs of locations), then the average absolute value of the off-diagonal entries of $Z^{\top} Z'$ will be bounded by $q_{\max} c_{\tmop{avg}}$. In this case it follows that $\tmop{tr} ((Z^{\top} Z')^2) \leqslant k^2 \ell q^2_{\max}  (1 + \ell c^2_{\tmop{avg}})$, and so the variance will shrink as $\ell$ grows, provided that $c_{\tmop{avg}}$ simultaneously goes to zero.

\section{Mathematical details for Section \ref{sec:when-solutions-exist}}

\subsection{Proof of Proposition \ref{prop:alpha-exist}}\label{app:alpha-exist-proof}

By Equations \ref{eqn:C-map-gen-derivative} and \ref{eqn:special-Q1-formula} we have
\begin{eqnarray*}
  C_f' (1) & = & \frac{1}{Q_f (1)} \mathbb{E}_{x \sim \mathcal{N} (0, 1)} [ \hat{\phi}' (x)^2]\\
  & = & \frac{\mathbb{E}_{x \sim \mathcal{N} (0, 1)} [ (\alpha \gamma \phi' (\alpha x + \beta))^2]}{\gamma^2 \tmop{Var}_{x \sim \mathcal{N} (0, 1)}  [\phi (\alpha x + \beta) ]}\\
  & = & \frac{\mathbb{E}_{x \sim \mathcal{N} (0, 1)} [\phi' (\alpha x + \beta)^2]}{\frac{1}{\alpha^2} \tmop{Var}_{x \sim \mathcal{N} (0, 1)}  [\phi (\alpha x + \beta) ]} .
\end{eqnarray*}
Taking $\alpha \rightarrow 0$ in the numerator clearly gives $\phi' (\beta)^2$. To handle the denominator, we will make use of the ``delta method'' from statistics (which is derived using a Taylor series argument), which says that if $\phi' (\alpha x + \beta)$ is a continuous with respect to $x$, $\phi' (\beta) \neq 0$, and $\frac{1}{\alpha}  (\alpha x - 0) = x$ is distributed as $\mathcal{N} (0, 1)$, then $\frac{1}{\alpha}  (\phi (\alpha x + \beta) - \phi (\beta))$ converges in distribution to $\mathcal{N} (0, \phi' (\beta)^2)$ as $\alpha \rightarrow 0$.

It thus follows that
\[ \phi' (\beta)^2 = \lim_{\alpha \rightarrow 0} \tmop{Var}_{x \sim \mathcal{N} (0, 1)}  \left[ \frac{1}{\alpha}  (\phi (\alpha x + \beta) - \phi (\beta)) \right] = \lim_{\alpha \rightarrow 0} \frac{1}{\alpha^2} \tmop{Var}_{x \sim \mathcal{N} (0, 1)}  [\phi (\alpha x + \beta) ], \]
and so we conclude that $\lim_{\alpha \rightarrow 0} \: C_f' (1) = 1$.

\section{Mathematical details for Section \ref{sec:max-pooling}}\label{app:max-pool}

In this section we will argue that a max pooling layer $f$ approximately preserves its input q values, under the condition the vectors within a given patch of its input feature map always fall into two tight clusters.

To simplify the argument, we will assume that these clusters have zero variance (or in other words that there are only two distinct input vectors in each patch), which means that our conclusions will only hold \tmtextit{approximately}. In addition, we will assume uniform input q values to $f$ (which is guaranteed by DKS), and that $f$ is directly preceded by a convolutional layer $g$ initialized with a Gaussian Delta or Gaussian fan-in scheme.

To begin, we observe that $f \left( \sqrt{q} Z \right) = \sqrt{q} f (Z)$ for all $q \geqslant 0$, and thus we may assume without loss of generality that the input q values are 1. We may also assume without loss of generality that the number of locations in each patch is two, since extra vectors that are duplicates of the first two will not affect the maximum.

Consider a single patch in $f$'s input. For each channel $i$, we denote by $x^{(i)} = \left[\begin{array}{cc}
  x_1^{(i)} & x_2^{(i)}
\end{array}\right]$ the two inputs in said patch for that channel. When conditioned on the input to $g$, we have that the $x^{(i)}$ are iid Gaussian vectors with mean zero and variance matrix $\Sigma = \left[\begin{array}{cc}
  1 & c\\
  c & 1
\end{array}\right]$ for some $c \in [- 1, 1]$. This follows from the fact the $x^{(i)}$'s are linear combinations of iid mean zero Gaussian random variables (the weights), and that the $g$'s output q values are 1 by assumption.

The dimension-normalized squared norm of $f$'s output vector associated with this patch (which is what the corresponding output q value approximates) is given by
\[ s_k = \frac{1}{k}  \sum_{i = 1}^k y_i^2, \]
where $y_i = \max \{ x_1^{(i)}, x_2^{(i)} \}$, and where $k$ is the number of channels (which is the same both for $f$'s input and output). So, to argue that $f$'s output q value is 1, it suffices to show that the mean of $s_k$ is indeed 1, and that it concentrates around this mean as $k$ grows.

To compute the mean we will make use of the following result from \citet{nadarajah2008exact}:

\begin{lemma}[Adapted from \citet{nadarajah2008exact}]
  Suppose $u_1$ and $u_2$ are Gaussian random variables wth means $\mu_1$ and $\mu_2$, variances $\sigma_1^2$ and $\sigma_2^2$, and correlation coefficient $\rho$. Then we have
  \[ \mathbb{E} [\max \{ u_1, u_2 \}^2] = (\sigma_1^2 + \mu_1^2) \Phi \left( \frac{\mu_1 - \mu_2}{\theta} \right) + (\sigma_2^2 + \mu_2^2) \Phi \left( \frac{\mu_2 - \mu_1}{\theta} \right) + (\mu_1 + \mu_2) \theta \phi \left( \frac{\mu_1 - \mu_2}{\theta} \right), \]
  where $\phi$ and $\Phi$ are the pdf and cdf of the standard normal distribution, and where
  \[ \theta = \sqrt{\sigma_1^2 + \sigma_2^2 - 2 \rho \sigma_1 \sigma_2} . \]
\end{lemma}

In our case we have $\mu_1 = \mu_2 = 0$, $\sigma_1^2 = \sigma_2^2 = 1$, and $\rho = c$. Substituting these into the above expression yields
\[ \mathbb{E} [y_i^2] = 2 \Phi (0) = 1. \]
As $s_k$ is the average of the $y_i^2$'s, it thus follows that $\mathbb{E} [s_k] = 1$. It remains to show that $s_k$ concentrates around its mean as $k$ grows.

Observe that $y_i^4 = \max \{ x_1^{(i)}, x_2^{(i)} \}^4 \leqslant (x_1^{(i)})^4 + (x_2^{(i)})^4$. We thus have
\[ \tmop{Var} (y_i^2) \leqslant \mathbb{E} [y_i^4] \leqslant \mathbb{E} [(x_1^{(i)})^4] +\mathbb{E} [(x_2^{(i)})^4] = 3 + 3 = 6, \]
where we have used the fact that $\mathbb{E}_{u \sim \mathcal{N} (0, 1)} [u^4] = 3$.

Because the $y_i$'s are independent we have $\tmop{Var} (s_k) = (1 / k^2) 6 k = 6 / k$, and so the variance shrinks as $k$ grows. Concentration of $s_k$ around its mean then follows by Chebyshev's inequality. (Note that one could possibly obtain a faster rate than Chebyshev's inequality gives by arguing that the $y_i^2$'s are sub-exponential random variables, and then applying concentration bounds for averages of such variables.)

\section{Mathematical details for Section \ref{sec:NTK-analysis}}

\subsection{Estimating per-layer NTK matrices for deep nets with degenerate C maps}\label{app:est-per-layer-NTK-mat}

In this section we will estimate the per-layer NTK matrices $K_i$ for each of the different cases described in Section \ref{sec:NTK-degen-C-map}. For the sake of simplicity we will argue in a semi-rigorous manner, employing fuzzy notions like ``very large'', ``small'', ``(not) approximately equal'', ``not too close'', ``reasonably smooth'', etc. Note that in infinite depth limit these fuzzy notions all become precise, and in particular, ``approximately equal'' becomes ``equal'', ``reasonably smooth'' becomes ``smooth'', etc. Note that all of our conclusions from this analysis have been verified by our numerical studies.

Let $c_0 \equiv x^{\top} x' / d_0$ for some pair of inputs $x$ and $x'$ taken from the training set. Note that by assumption we have $c_0 \not\approx \pm 1$.

Following the notation of Section \ref{sec:NTK-via-C-map}, for the $i$-th combined layer $f_i$ of the network $f$ we will denote by $g_i$ the subnetwork that maps $f$'s input to the input of $f_i$, and by $h_i$ the subnetwork that maps the input of $f_i$ to $f$'s final output.

Suppose that $f_i$ is early in the network, so that $i$ is small and $D - i$ is large. This means that $g_i$ is a shallow subnetwork of $f$, and so $C_{g_i}$ is well-behaved, while $h_i$ is a deep subnetwork, and so $C_{h_i}$ is degenerate. Because $C_{g_i}$ is well-behaved it is reasonably smooth (e.g.~by Theorem \ref{thm:deviation-bound}), and so $C_{g_i} (c_0) \not\approx \pm 1$ (since $c_0 \not\approx \pm 1$). Then since $C_{h_i}$ is degenerate, this implies by Proposition \ref{prop:degen-deriv-bound} that $C'_{h_i} (C_{g_i} (c_0)) \approx 0$. Meanwhile, we trivially have that $C'_{h_i} (C_{g_i} (1)) = C'_{h_i} (1)$ (since $C_{g_i} (1) = 1$). Thus, $\Theta_i (x, x') = C_{g_i} (c_0) C'_{h_i} (C_{g_i} (c_0)) \approx 0$ by Equation \ref{eqn:NTK-elegant}, and also $\Theta_i (x, x) = C_{g_i} (1) C'_{h_i} (C_{g_i} (1)) = 1 \cdot C'_{h_i} (1) = \mathcal{C}' (1)^{D - i}$. Since $x$ and $x'$ are general distinct inputs from the training set we thus have that $K_i \approx \mathcal{C}' (1)^{D - i} I$. And because $D - i$ is large this will be very small when $\mathcal{C}' (1) < 1$, very large when $\mathcal{C}' (1) > 1$, and equal to the identity matrix when $\mathcal{C}' (1) = 1$.

Now suppose $f_i$ is a layer later in the network, so that $i$ is large and $D - i$ is small. This means that $g_i$ is a deep subnetwork of $f$, and so $C_{g_i}$ is degenerate, while $h_i$ is a shallow subnetwork, and so $C_{h_i}$ is well-behaved. Because $C_{g_i}$ is degenerate we have $C_{g_i} (c_0) \approx \fixpointofcmap$ (as $c_0 \not\approx \pm 1$ by assumption). Since $C_{h_i}$ is well-behaved it is reasonably smooth, and thus $C_{h_i}' (C_{g_i} (c_0)) \approx C_{h_i}' (\fixpointofcmap)$. By Equation \ref{eqn:NTK-elegant} we therefore have that $\Theta_i (x, x') = C_{g_i} (c_0) C'_{h_i} (C_{g_i} (c_0)) \approx \fixpointofcmap C_{h_i}' (\fixpointofcmap)$, and $\Theta_i (x, x) = \mathcal{C}' (1)^{D - i}$ as in the previous case. Since $x$ and $x'$ are distinct general inputs from the training set we thus have that $K_i \approx \mathcal{C}' (1)^{D - i} I + \fixpointofcmap C_{h_i}' (\fixpointofcmap)  (E - I) = \mathcal{C}' (1)^{D - i} I + \fixpointofcmap  \mathcal{C}' (\fixpointofcmap)^{D - i}  (E - I)$, where $E$ denotes the matrix of 1's and we have used the fact that $\fixpointofcmap$ is a fixed point of $\mathcal{C}$ to get that $C_{h_i}' (\fixpointofcmap) = \mathcal{C}' (\fixpointofcmap)^{D - i}$. If $\fixpointofcmap = 1$, then the estimate for $K_i$ simplifies to $K_i \approx \mathcal{C}' (1)^{D - i} E$.

Finally, suppose $f_i$ is a layer in the middle of the network, so that both $i$ and $D - i$ are large. This means that both $g_i$ and $h_i$ are deep subnetworks of $f$, and so $C_{g_i}$ and $C_{h_i}$ are degenerate. Because $C_{g_i}$ is degenerate we have $C_{g_i} (c_0) \approx \fixpointofcmap$ (as $c_0 \not\approx \pm 1$ by assumption). Thus, by Equation \ref{eqn:NTK-elegant} we have that $\Theta_i (x, x') = C_{g_i} (c_0) C'_{h_i} (C_{g_i} (c_0)) \approx \fixpointofcmap C_{h_i}' (C_{g_i} (c_0))$, and $\Theta_i (x, x) = \mathcal{C}' (1)^{D - i}$ as in the previous case. There are three scenarios to consider.

If $0 \leqslant \fixpointofcmap < 1$ and $\mathcal{C}' (1) > 1$, then since $C_{g_i} (c_0) \approx \fixpointofcmap \not\approx 1$ by assumption and $C_{h_i}' (c) \approx 0$ for all $c \not\approx \pm 1$ by Proposition \ref{prop:degen-deriv-bound}, it follows that $C_{h_i}' (C_{g_i} (c_0)) \approx 0$. So in this case we have $K_i \approx \mathcal{C}' (1)^{D - i} I$, which will be very large since $\mathcal{C}' (1) > 1$ and $D - i$ is large.

If $\fixpointofcmap = 1$ and $\mathcal{C}' (1) < 1$, then since as $D - i$ is large and $C_{h_i}'$ is non-negative and non-decreasing (by Section \ref{sec:pd-functions}) we have $0 \leqslant C_{h_i}' (C_{g_i} (c_0)) \leqslant C_{h_i}' (1) = \mathcal{C}' (1)^{D - i} \approx 0$, and thus $C_{h_i}' (C_{g_i} (c_0)) \approx 0$. So in this scenario we have $K_i \approx \mathcal{C}' (1)^{D - i} I \approx 0$ since $\mathcal{C}' (1) < 1$ and $D - i$ is large.

If $\fixpointofcmap = 1$ and $\mathcal{C}' (1) = 1$, then we have $C_{g_i} (c_0) \approx \fixpointofcmap = 1$, but this doesn't help us estimate $C_{h_i}' (C_{g_i} (c_0))$, since the value of \ $C_{h_i}' (c)$ may be {\tmem{highly}} sensitive to the distance of $c$ from 1 (because $C_{h_i}$ is degenerate and can thus have extreme behavior near $c = 1$). Since we have that $C_{h_i}' (1) = \mathcal{C}' (1)^{D - i} = 1$, and that $C_{h_i}'$ is non-negative and non-decreasing on $[0, 1]$ (which contains $C_{g_i} (c_0)$), we do at least know that the entries of $K_i$ are bounded between 0 and 1, and that the diagonal entries are 1. From numerical studies we conducted of the C map of deep RELU networks (which by Section \ref{sec:relu-network-degen} have $\fixpointofcmap = 1$ and $\mathcal{C}' (1) = 1$) we observe that $K_i \approx I + \alpha_i  (E - I)$, where $0 = \alpha_1 \leqslant \alpha_2 \leqslant \cdots \leqslant \alpha_D = 1$ are constants, and we conjecture that this holds in general. \james{Can someone improve this? I feel like it should be possible to derive an estimate in the limit, but it's not obvious. In the deep RELU net example it seems like the different values of $C_{g_i} (c)$ for $c < 0.999$ are closer to each other than they are to the fixed point 1. I don't really understand how that happens, but perhaps the very slow convergence to the fixed point that you get when $\mathcal{C}' (1) = 1$ results in this effect.}

\subsection{Estimating the overall NTK matrix}\label{app:est-overall-NTK-mat}

Having computed an estimate of the per-layer NTK matrix $K_i$ in each case for $\fixpointofcmap$ and $i$, it remains to estimate the full NTK matrix for each $\fixpointofcmap$.

When $\fixpointofcmap = 1$ and $\mathcal{C}' (1) < 1$ we have $K_i \approx 0$ for early layers and middle layers, and $K_i \approx \mathcal{C}' (1)^{D - i} E$ for later layers. If $L$ is the number of later layers we thus have
\[ K = \sum_{i = 1}^D K_i \approx \sum_{i = D - L + 1}^D \mathcal{C}' (1)^{D - i} E = \sum_{i = 0}^{L - 1} \mathcal{C}' (1)^i E = \frac{1 - \mathcal{C}' (1)^L}{1 - \mathcal{C}' (1)} E \approx \frac{1}{1 - \mathcal{C}' (1)} E. \]

When $0 \leqslant \fixpointofcmap < 1$ and $\mathcal{C}' (1) > 1$ we have $K_i \approx \mathcal{C}' (1)^{D - i} I$ for early and middle layers, and $K_i \approx \mathcal{C}' (1)^{D - i} I + \fixpointofcmap  \mathcal{C}' (\fixpointofcmap)^{D - i}  (E - I)$ for later layers. Moreover, since $\fixpointofcmap$ is an attractive fixed point of $\mathcal{C}$ we have $\mathcal{C}' (\fixpointofcmap) < 1$. If $L$ is the number of later layers we thus have
\begin{eqnarray*}
  K & = & \sum_{i = 1}^D K_i\\
  & \approx & \sum_{i = 1}^D \mathcal{C}' (1)^{D - i} I + \sum_{i = D - L + 1}^D \fixpointofcmap  \mathcal{C}' (\fixpointofcmap)^{D - i}  (E - I)\\
  & = & \sum_{i = 0}^{D - 1} \mathcal{C}' (1)^i I + \sum_{i = 0}^{L - 1} \fixpointofcmap  \mathcal{C}' (\fixpointofcmap)^i  (E - I)\\
  & = & \frac{1 - \mathcal{C}' (1)^D}{1 - \mathcal{C}' (1)} I + \fixpointofcmap  \frac{1 - \mathcal{C}' (\fixpointofcmap)^L}{1 - \mathcal{C}' (\fixpointofcmap)}  (E - I)\\
  & \approx & \frac{1 - \mathcal{C}' (1)^D}{1 - \mathcal{C}' (1)} I + \frac{\fixpointofcmap}{1 - \mathcal{C}' (\fixpointofcmap)}  (E - I) .
\end{eqnarray*}

Finally, when $\fixpointofcmap = 1$ and $\mathcal{C}' (1) < 1$, we have $K_i \approx I$ for early layers, $K_i \approx I + \alpha_i  (E - I)$ for some $0 \leqslant \alpha_i < 1$ for middle layers (which is only a conjectured formula), and $K_i \approx E$ for later layers. Or in general, we have $K_i \approx I + \alpha_i  (E - I)$ for some $0 \leqslant \alpha_i \leqslant 1$ for all layers. This gives
\[ K = \sum_{i = 1}^D K_i \approx \sum_{i = 1}^D I + \alpha_i  (E - I) = D (I + \bar{\alpha}  (E - I)), \]
where $\bar{\alpha} = \frac{1}{D}  \sum_{i = 1}^D \alpha_i$ (so that $0 \leqslant \bar{\alpha} \leqslant 1$). From our empirical studies of the C map of deep RELU networks we observe that $\bar{\alpha} = 1 / 4$, and for some other example C maps we observe $\bar{\alpha} = 1 / 3$ (the latter of which is consistent with the estimate given in \citet{xiao2020disentangling}).

\subsection{Proof of Theorem \ref{thm:NTK-our-approach}}\label{app:NTK-our-approach-proof}

Let $g_i$ and $h_i$ for $i = 1, 2, \ldots, D$ be defined as in Section \ref{sec:NTK-via-C-map} for the network in question, and let $x$ and $x'$ be two input data vectors.

Under DKS we have $C_{g_i} (0) = C_{h_i} (0) = 0$, $C'_{g_i} (1) \leqslant \zeta^{(i - 1) / (D - 1)}$, and $C'_{h_i} (1) \leqslant \zeta^{(D - i) / (D - 1)}$. By Theorem \ref{thm:deviation-bound} this implies that
\[ | C_{g_i} (c_0) - c_0 | \leqslant 2 (\zeta^{(i - 1) / (D - 1)} - 1) \]
and
\[ | C_{h_i}' (C_{g_i} (c_0)) - 1 | \leqslant 3 (\zeta^{(D - i) / (D - 1)} - 1), \]
where $c_0 = x^{\top} x' / d_0$.

As $q_D = 1$ under DKS, it then follows from Equation \ref{eqn:NTK-elegant} that
\begin{eqnarray*}
  \left| \Theta_i (x, x') - \frac{1}{d_0} x^{\top} x' \right| & = & | C_{g_i} (c_0) C'_{h_i} (C_{g_i} (c_0)) - c_0 |\\
  & \leqslant & | C'_{h_i} (C_{g_i} (c_0)) - 1 |  (| c_0 | + | C_{g_i} (c_0) - c_0 |) + | C_{g_i} (c_0) - c_0 |\\
  & \leqslant & 3 (\zeta^{(D - i) / (D - 1)} - 1)  (1 + 2 (\zeta^{(i - 1) / (D - 1)} - 1)) + 2 (\zeta^{(i - 1) / (D - 1)} - 1)\\
  & \leqslant & 6 (\zeta^{(D - i) / (D - 1)} \zeta^{(i - 1) / (D - 1)} - 1) + 3 (\zeta^{(D - i) / (D - 1)} - 1)\\
  &  & + \hspace{1em} 2 (\zeta^{(i - 1) / (D - 1)} - 1)\\
  & = & 6 (\zeta - 1) + 3 (\zeta^{(D - i) / (D - 1)} - 1) + 2 (\zeta^{(i - 1) / (D - 1)} - 1)\\
  & \leqslant & 11 (\zeta - 1),
\end{eqnarray*}
where we have used the general facts that $| ab - c | \leqslant | b - 1 |  (| c | + | a - c |) + | a - c |$ on the second line, and $| a - 1 |  | b - 1 | \leqslant | ab - 1 |$ for $a, b \geqslant 1$ on the third line.

\section{Path-weight analysis and its relationship to approximate kernel analysis}\label{app:path-weight-analysis}

Path-weight analysis is a method for analyzing deep fully-connected RELU networks developed in a series works \citep{balduzzi2015kickback, balduzzi2016deep, balduzzi2017shattered}. It is capable of approximating some of the same quantities computed by Q/C maps. Unlike those methods, it is not based on the ``propagation'' of anything through the layers of the network, but instead exploits the special structure of RELU networks to decompose their computation in terms of a collection of simple ``paths'' that are easier to analyze. In this section we will give a quick derivation of path-weight analysis, identify a possible issue with it, and then discuss how its predictions relate to those made by Q/C maps.

For simplicity, we will assume that the sub-network $f$ has only 1 output unit, and biases of zero. We will start by defining a \tmtextbf{\tmtextit{path}} as a sequence of (scalar-valued) units chosen from each layer of the network (which includes the input layer). (So for example, we might choose the 4-th unit from the input layer, the 2-nd unit from the next layer, etc). An \tmtextbf{\tmtextit{active path}} $p$ is one where every RELU unit in $p$ is \tmtextbf{\tmtextit{active}} in the sense of having a non-negative input value (i.e.~so that the RELU function is in its ``linear region''). We then define a \tmtextbf{\tmtextit{path-weight}} $W_p$ for a path $p$ as the product of the weights that connect the units of $p$ in $f$'s graph representation. Given these definitions, its not hard to see that $f (x) = \sum_p W_p x_p$, where $x_p$ is the entry of $x$ corresponding to the (single) input unit in $p$, and the sum is taken over all active paths in $f$.

Observe that the total number of paths $P$ is simply the product of the input dimensions for each fully-connected layer. Thus, if the weights are chosen according to a standard Gaussian fan-in initialization (so that they are iid with mean zero and variance $\sigma^2 / m$, where $m$ is the layer's input dimension) it follows that
\[ \mathbb{E} [W_p W_q] = \left\{ \begin{array}{ll}
     \sigma^{2 D} / P & \text{if } p = q\\
     0 & \text{otherwise}
   \end{array} \right., \]
where $D$ is the number of weights along $p$. Using the fact that the expected number of active paths is the same starting from any input unit (due to symmetry), we thus have that
\begin{equation}
  \mathbb{E} [f (x) f (x')] =\mathbb{E} \left[ \left( \sum_p W_p x_p \right)  \left( \sum_q W_q x'_q \right) \right] = \sigma^{2 D} \mathbb{E} \left[ \sum_{p \text{ co-active}} x_p x'_p \right] = \sigma^{2 D}  \frac{\mathbb{E} [A_{x, x'}]}{P}  \frac{1}{k} x^{\top} x', \label{eqn:path-analysis-kernel}
\end{equation}
where $A_{x, x'}$ is the number of paths that are active for both $x$ and $x'$ (which are called \tmtextbf{\tmtextit{co-active paths}}), and $k$ is the dimension of $x$. So, if one can estimate the expected fraction of paths which are co-active (i.e.~$\mathbb{E} [A_{x, x'}] / P$), then one can effectively estimate per-unit expected squared values (or products) under the weight distribution, which are precisely the quantities estimated by signal propagation (and thus Q/C maps, as discussed in Section \ref{sec:var/sig-prop-apk-relationship}).

In order to estimate expected numbers of active and co-active paths, \citet{balduzzi2017shattered} make the key approximating assumption that for any setting of the weights, exactly half of the RELU units in each nonlinear layer are active for a typical input, and exactly one quarter of these units are active (or co-active) for a pair of typical {\tmem{distinct}} inputs. Given this assumption, and that the number of nonlinear layers is $R$, the fraction of all possible paths ending at $y$ which are active for any $x$ is thus $(1 / 2)^R$, and the fraction which are co-active for any distinct $x$ and $x'$ is $(1 / 4)^R$.

While this assumption is clearly not true in general, one can possibly argue that the fraction of the units that are active or co-active in a given layer will concentrate around its expectation as the layer widths go to infinity. However, while the expected fraction is indeed 1/2 for active units (e.g.~because the input distribution to each unit is symmetric about 0), it won't in general be $1 / 4$ for co-active units. (For example, given two nearly identical inputs $x$ and $x'$, the fraction of co-active units will be much closer to $1 / 2$, especially in early layers.) Despite this potentially significant issue, \citet{balduzzi2017shattered} find that this assumption is a reasonable approximation in practice for certain well-behaved example networks with realistic inputs.

When $f$ is a standard feed-forward RELU networks with R nonlinear layers with $\sigma^2 = 2$ (which is \poscite{he2015delving} recommendation for RELU networks) and $D = R$, Equation \ref{eqn:path-analysis-kernel} thus gives $\mathbb{E} [f (x)^2] \approx \| x \|^2 / k$ and $\mathbb{E} [f (x) f (x')] \approx 1 / 2^R  (x^{\top} x') / k$ for $x \neq x'$. The former of these predictions agrees with Q/C maps in the sense that $f$ is equivalent to a network $g$ with $\sigma^2 = 1$ and normalized RELUs (given by multiplication of a RELU by $\sqrt{2}$), for which we have $Q_g (q) = q$. The latter prediction does not however agree, as we have by Section \ref{sec:relu-network-degen} that $C_g (c)$ rapidly approaches 1 instead of 0 for all $c$ as the depth grows. (This mismatch is likely due to the approximating assumption discussed above.)

As $f (x) = \sum_p W_p x_p$, the {\tmem{derivative}} of $f (x)$ with respect to some input unit is just the sum over path-weights for all active paths starting at that unit. Using this fact, it follows that $\mathbb{E} [f' (x)^{\top} f' (x')] = \sigma^{2 D} \mathbb{E} [A_{x, x'}] / P$ (using a derivation similar to Equation \ref{eqn:path-analysis-kernel}). For $f$ given as above, \citet{balduzzi2017shattered} use this identity to show that
\[ \mathbb{E} [f' (x)^{\top} f' (x')] \approx \left\{ \begin{array}{ll}
     1 & \text{if } x = x'\\
     1 / 2^R & \text{otherwise}
   \end{array} \right. \]
for standard feed-forward RELU networks with R nonlinear layers. They then argue that this decorrelation (or ``shattering'') of the input gradients will make the network difficult or impossible to successfully train. 

As we saw in Section \ref{sec:NTK-analysis}, Q/C map analysis can be used to compute the NTK matrix $K$ (and per-layer NTK matrices $K_i$ for $i = 1, 2, \ldots, D$), which is a matrix of estimates of the inner products between parameter gradients for different training inputs (assuming a squared error loss). For deep normalized RELU networks such as $g$ we are in the collapsing case with $\fixpointofcmap = 1$ and $\mathcal{C}' (1) = 1$ (by Section \ref{sec:relu-network-degen}), and so assuming that $D$ is large we have by Section \ref{sec:NTK-degen-C-map} that $K_i$ is approximately the identity for early layers, and approximately equal to a matrix of ones for later layers. This predicts that substantial optimization will only occur in early layers, and that this is unlikely to yield any significant generalization. Note that this prediction is somewhat different than the one made by \citet{balduzzi2017shattered} insofar as the training loss will in fact be minimized given sufficiently wide layers.

The per-layer NTK matrix for the first linear layer, denoted $K_1$, is closely related to the matrix of input gradients estimated above under path-weight analysis. Reassuringly, the predictions agree in the sense that they are both (close to) an identity matrix.

\section{Analyzing (nearly) standard ResNets using Q/C maps}\label{app:resnet-map-analysis}

As discussed in Section \ref{sec:resnet-related-discussion}, ResNets  
represent a very popular solution to the same problem that DKS is aimed at solving: how to construct a very deep network that can be trained with a gradient-based optimizer. It thus worth understanding how the effectiveness of ResNets can be explained within the framework of Q/C maps.

Since our analysis framework doesn't handle BN layers, we can't apply it directly to standard ResNets. As a compromise, we will instead analyze a ResNet which is modified to use Layer Normalization (LN) layers in their place, as these are handled within our framework, and perform a somewhat similar function. (See Section \ref{sec:layer-norm} for a discussion of normalization layers.) Note that Transformer models \citep{vaswani2017attention}, which also employ a residual structure, already use LN layers in place of BN layers, although BN layers remain the more popular option for convolutional residual networks.

In order to apply our Q/C map analysis to convolutional layers, we will assume the use of a Delta initialization (as opposed to a Gaussian fan-in initialization). And to simulate a weight variance of 2 as used in \poscite{he2015delving} initialization scheme we will use ``normalized RELU'' activations, which are obtained from standard RELU activation functions by multiplication of their input (or output) by $\sqrt{2}$.

The Q map for a normalized RELU layer is the identity function, which can be straightforwardly derived from Equation \ref{eqn:Q-map}. So, since each (normalized) RELU nonlinear layer is immediately preceded by an LN layer, we have that its output q value will always be 1 (as the output q value of a LN layer is always 1). It thus follows that the output q value of a residual branch will always be 1 (since affine layers preserve q values). Meanwhile, for non-transition blocks (i.e.~those with equal input and output channel dimensions), shortcut branches compute the identity function, so that their Q map is the identity. And for transition blocks, shortcut branches consist of an LN layer, a RELU layer, and an affine layer, so that their output q value is always 1. By Equation \ref{eqn:sum-q-formula}, the output q value of residual block is the sum of the q values for the two branches, which will therefore be 1 plus the block's input q value for non-transition blocks (mirroring \poscite{de2020batch} variance propagation analysis), and $1 + 1 = 2$ for transition blocks. From these we observations it follows that q values will grow as a sequence $1, 2, \ldots$ with each successive block, until a transition block is encountered, at which point the q value is reset to 2.

For the standard values 50, 101, and 152 of the ResNet-V2 ``depth'' parameter $D$, there is a sequence of $(D - 2) / 3$ residual blocks, 4 of which are transition blocks (which includes the first block). Thus, the sequence of input q values for the blocks are $1, 2, \ldots, q_1, q_1 + 1, 2, \ldots, q_2, q_2 + 1, 2, \ldots, q_3, q_3 + 1, 2, \ldots, q_4$, for some integers $q_i$ (where the input q values to the transition blocks are $q_1 + 1$, $q_2 + 1$, $q_3 + 1$, and $q_4 + 1$). For $D = 50$ we have $(q_1, q_2, q_3, q_4) = (3, 4, 6, 3)$, for $D = 101$ we have $(q_1, q_2, q_3, q_4) = (3, 4, 23, 3)$, and for $D = 152$ we have $(q_1, q_2, q_3, q_4) = (3, 8, 36, 3)$.

Let $\mathcal{C}$ denote the local C map for a RELU nonlinear layer, which is given in Equation \ref{eqn:C-map-RELU}, and notably doesn't depend on the input q value (which will always be the case for positively homogeneous activation functions). Further, let $B_q$ denote the C map for non-transition blocks with input q value $q$, and $T$ denote the C map for transition blocks (which doesn't depend on the input q value). From the above analysis we have
\[ C_f = \mathcal{C} \circ E_4 \circ E_3 \circ E_2 \circ E_1, \]
where
\[ E_i = B_{q_i} \circ B_{q_i - 1} \circ \cdots \circ B_3 \circ B_2 \circ T \quad \text{for} \quad i \in \{ 1, 2, 3, 4 \} . \]
As discussed in Section \ref{sec:layer-norm}, LN layers always output q values of 1. And because their inputs always comes directly from an affine layer in this network (or a sum over these), we have that their C maps are the identity function. Thus, by Equation \ref{eqn:sum-c-formula} it follows that
\begin{equation}
  B_q (c) = \frac{\mathcal{C} (\mathcal{C} (\mathcal{C} (c))) + qc}{1 + q} \quad \text{and} \quad T (c) = \frac{\mathcal{C} (\mathcal{C} (\mathcal{C} (c))) + \mathcal{C} (c)}{2} . \label{eqn:standard-Ai-and-B}
\end{equation}
Given these identities (and Equation \ref{eqn:C-map-RELU}), we can compute and plot $C_f$ for each possible $D$:
\begin{center}
\resizebox{4.5in}{3in}{\includegraphics{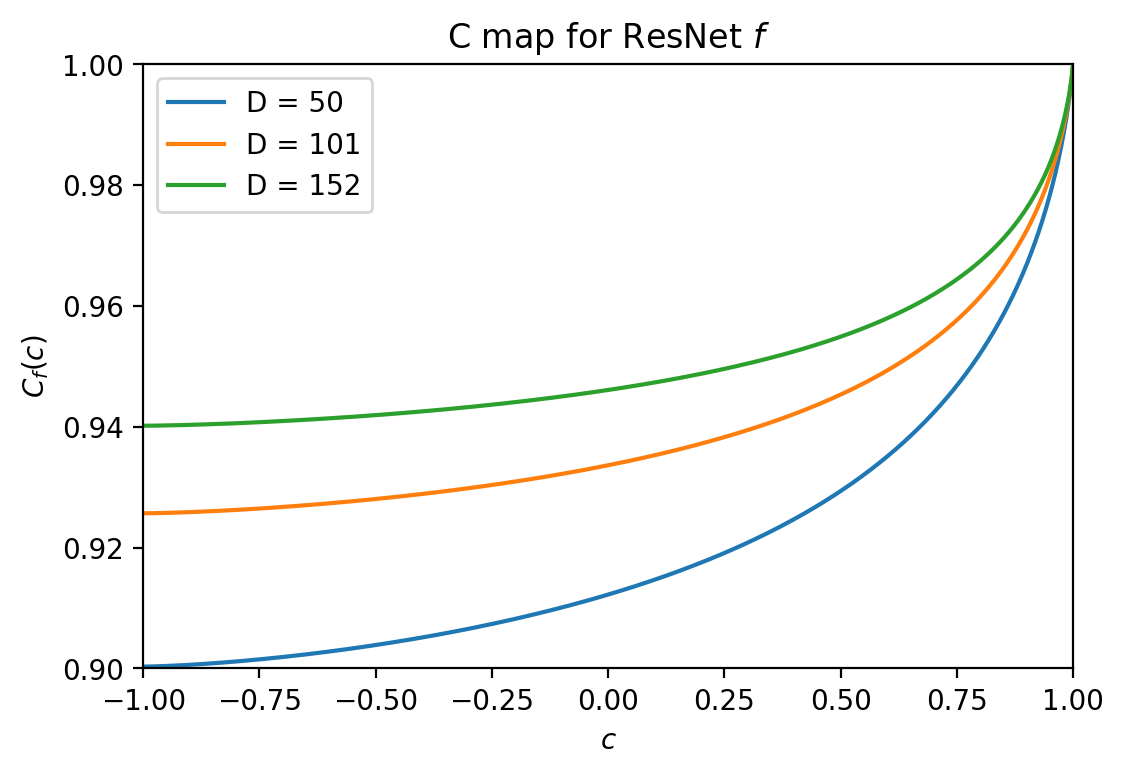}}
\end{center}
While there is some compression of the range of c values as depth increases in ResNets, it is {\tmem{much}} milder than what we see for standard deep RELU networks (e.g.~in the plots of Section \ref{sec:relu-network-degen}). For example, $[- 1, 1]$ is mapped to $[0.94, 1]$ for a ResNet with $D = 152$, versus $[0.996, 1]$ for a standard RELU network of depth 100. This observation is compatible the idea that ResNets are closer to linear/identity functions (which have identity C maps) than standard deep networks, at least at initialization time. 

It is also worth considering the C map behavior of a ResNet where the q values {\tmem{don't}} grow throughout the network. For example, instead of having simple sums at the end of each residual block, we could use normalized sums, with a weight of $1 / \sqrt{2}$ on both branches. Or, we could add an LN layer to the shortcut branch of every block. In either case, we have q values of 1 throughout the network, which leads to the following redefinition:
\[ B_q (c) = \frac{\mathcal{C} (\mathcal{C} (\mathcal{C} (c))) + c}{2} . \]
Intuitively, this definition places more weight onto the nonlinear contribution than we had previously.

Plotting $C_f$ in this scenario for different $D$'s yields the following:
\begin{center}
\resizebox{4.5in}{3in}{\includegraphics{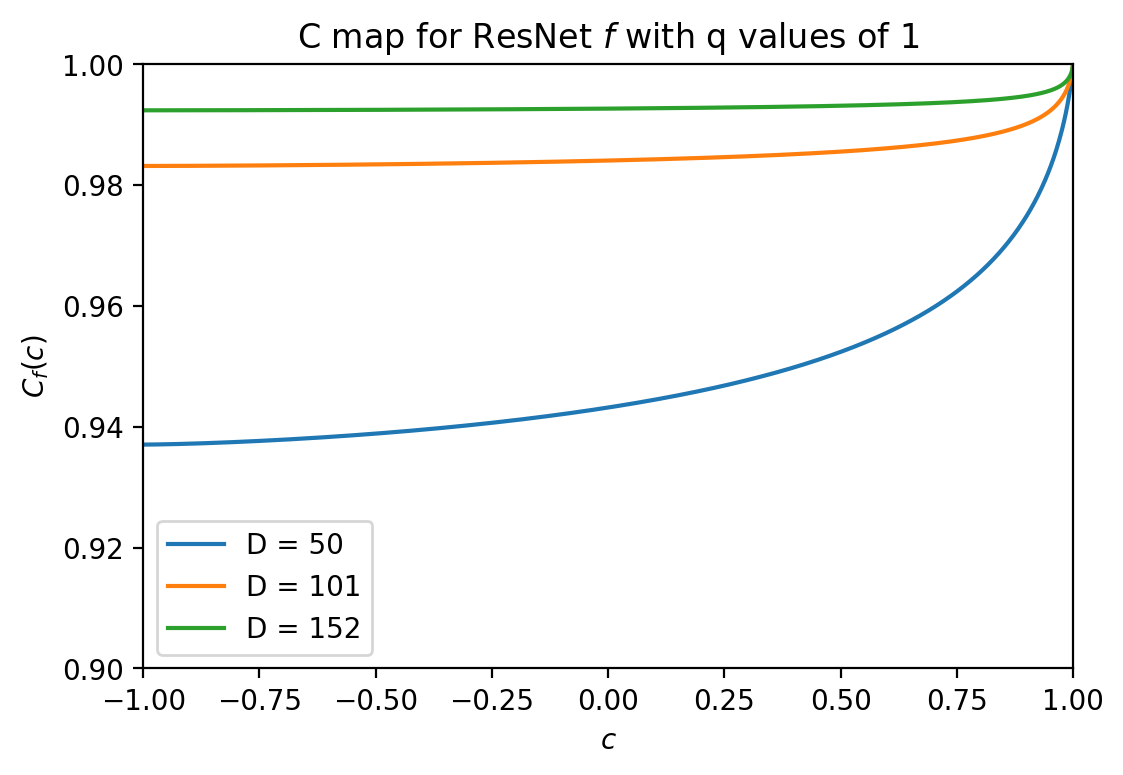}}
\end{center}
While less extreme than standard (non-residual) RELU networks, we still see significantly more compression than before, thus reinforcing the importance of growing q values for the trainability of ResNets.

It's worth noting one can achieve the same effect in a network with constant q values equal to 1 (such as the ones constructed with DKS) by careful choice of weights on the sums at the end of each residual block. For example, one can recover the original form of $B_q$ (seen in Equation \ref{eqn:standard-Ai-and-B}) by using a weight of $w = 1 / \sqrt{q + 1}$ for the residual branch of the corresponding block (and a weight of $\sqrt{1 - w^2}$ on the shortcut branch). Doing this for all blocks exactly recovers the C map of a standard ResNet, as neither $T$ nor $\mathcal{C}$ depend on the q values.

\section{Empirical evidence for the relationship between Q map derivatives and kernel approximation error} \label{app:Q-map-error}

\james{move this to later in the appendix?}

In this section we will provide empirical evidence for the relationship between Q map derivatives and kernel approximation error that we posited in Section \ref{sec:error-and-Qmaps}. To do this, we will examine the effect of changes to the local Q map conditions used in DKS on the accuracy of the predictions made by Q maps for an example network.

In particular, we will consider the skip-free BN-free modified ResNet used in our main experiments (from Section \ref{sec:main-experiments}) with the softplus activation function, and a depth parameter of 50. The local Q maps of this network's combined layers are equal to the same function $\mathcal{Q}$, and we will consider the effect of using values of $0.95$, $1.0$, and $1.01$ for $\mathcal{Q}'(1)$ in DKS's local map conditions. As we still have $\mathcal{Q}(1) = 1$ with this change, and the network's input q values are still 1 (due to our use of PLN), we thus have constant q values of 1 for all layers. We can empirically estimate the accuracy of this prediction by measuring how much $\|v\|^2 / \dim(v)$ deviates from 1 for location vectors $v$ from the network's feature maps (computed at initialization time). 

The following plots show these ``empirical q values", averaged across locations and 192 training examples, versus the layer index for which they are computed. Vertical lines indicate the standard deviation.

\begin{center}
\resizebox{4.5in}{3in}{\includegraphics{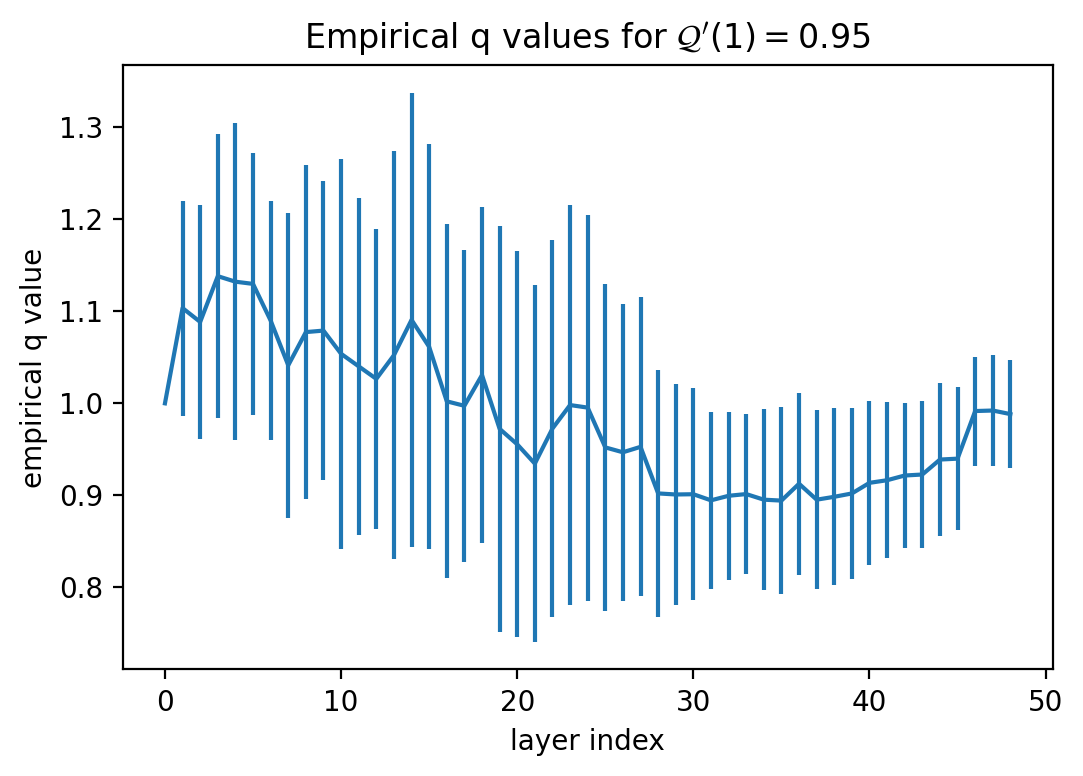}}
\end{center}

\begin{center}
\resizebox{4.5in}{3in}{\includegraphics{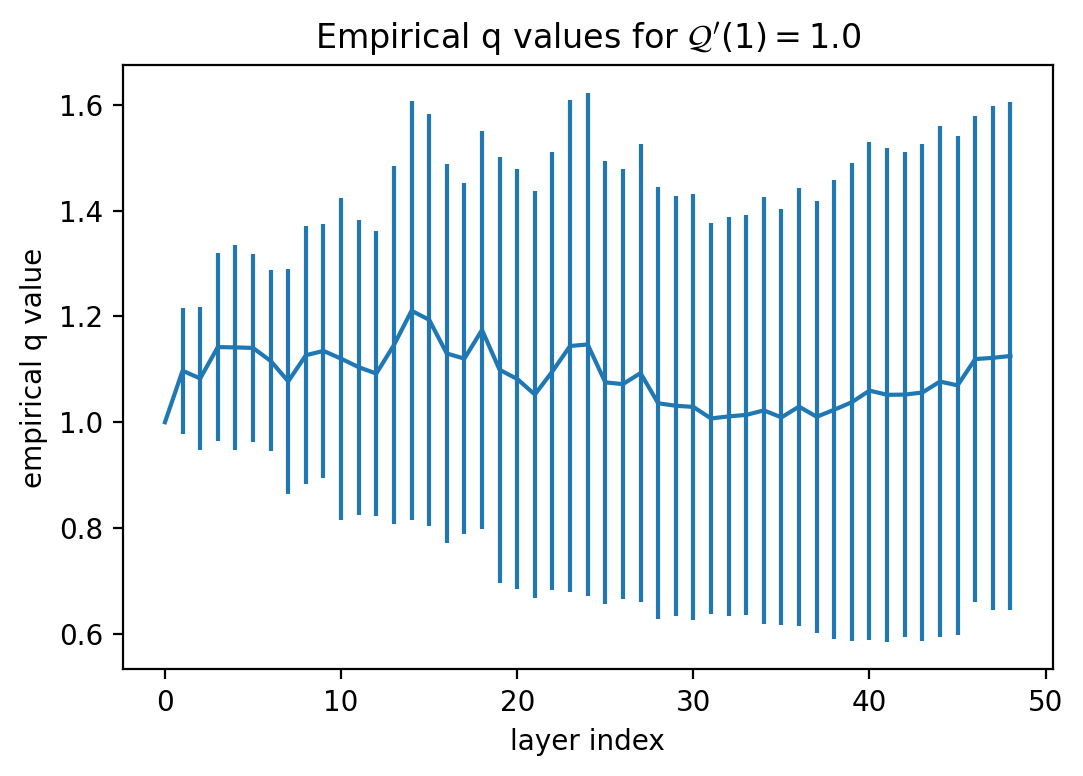}}
\end{center}

From these first two plots we can see that average empirical q value remains close to 1 for both the $\mathcal{Q}'(1) = 0.95$ and $\mathcal{Q}'(1) = 1.0$ cases. Although in the latter case we see higher variance, especially for deeper layers. In the next plot, we see that the empirical q values rapidly diverge from 1 when $\mathcal{Q}'(1) = 1.01$, thus confirming the intuitions given in Section \ref{sec:error-and-Qmaps}.

\begin{center}
\resizebox{4.5in}{3in}{\includegraphics{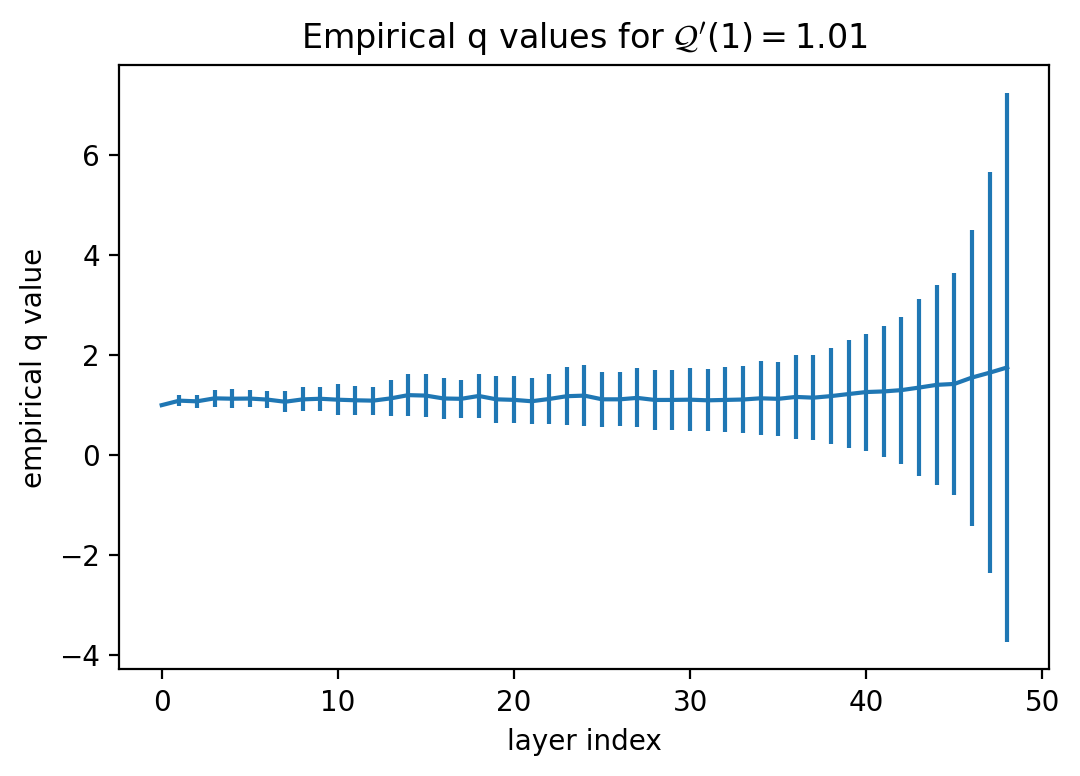}}
\end{center}

\section{Example learning rate schedules from FIRE PBT}\label{app:PBT-LR-example}

In this section we present the learning rate schedules that were found by FIRE PBT for our main Imagenet experiments from Section \ref{sec:main-experiments}.

\resizebox{0.85\columnwidth}{!}{\includegraphics{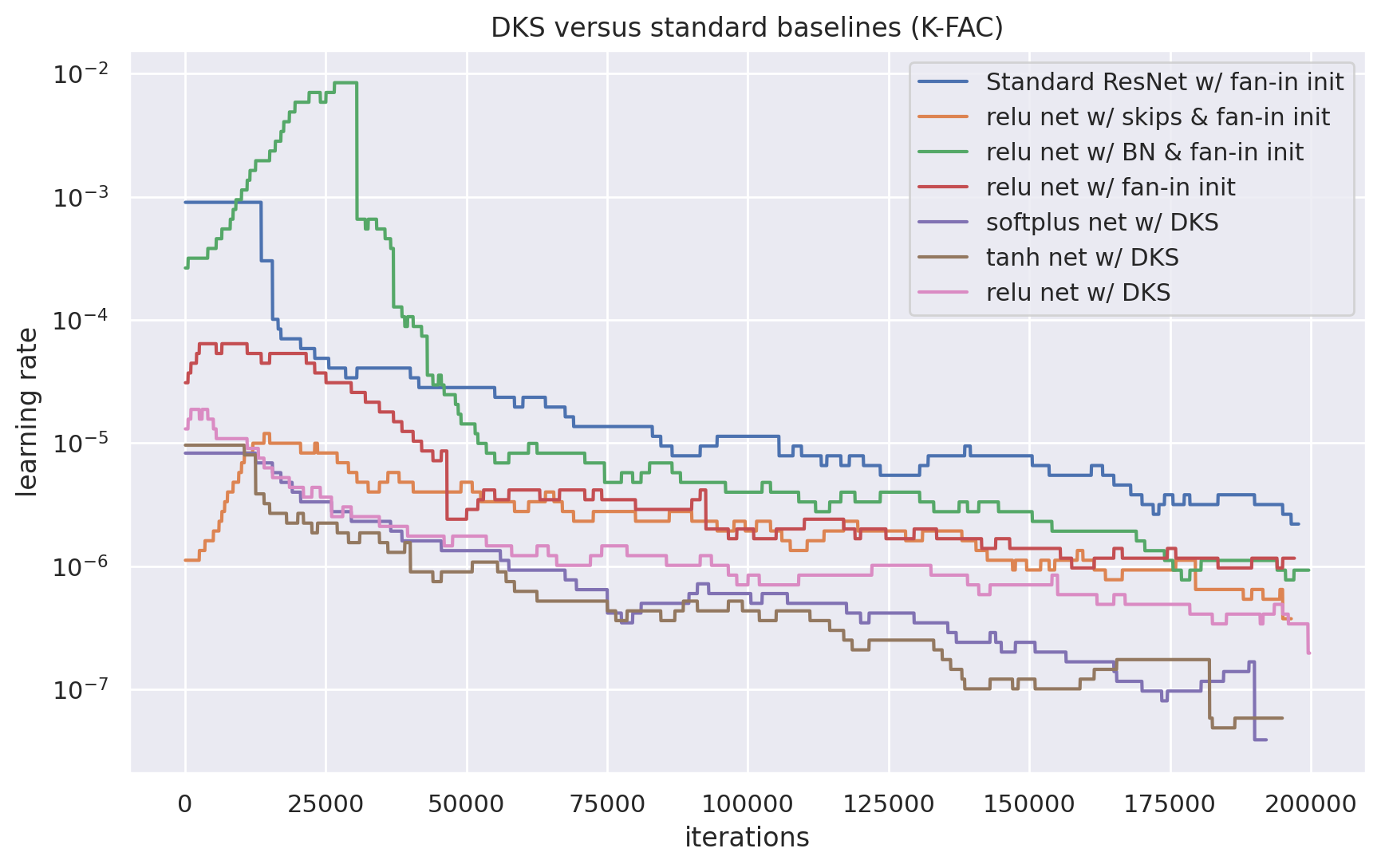}}

\resizebox{0.85\columnwidth}{!}{\includegraphics{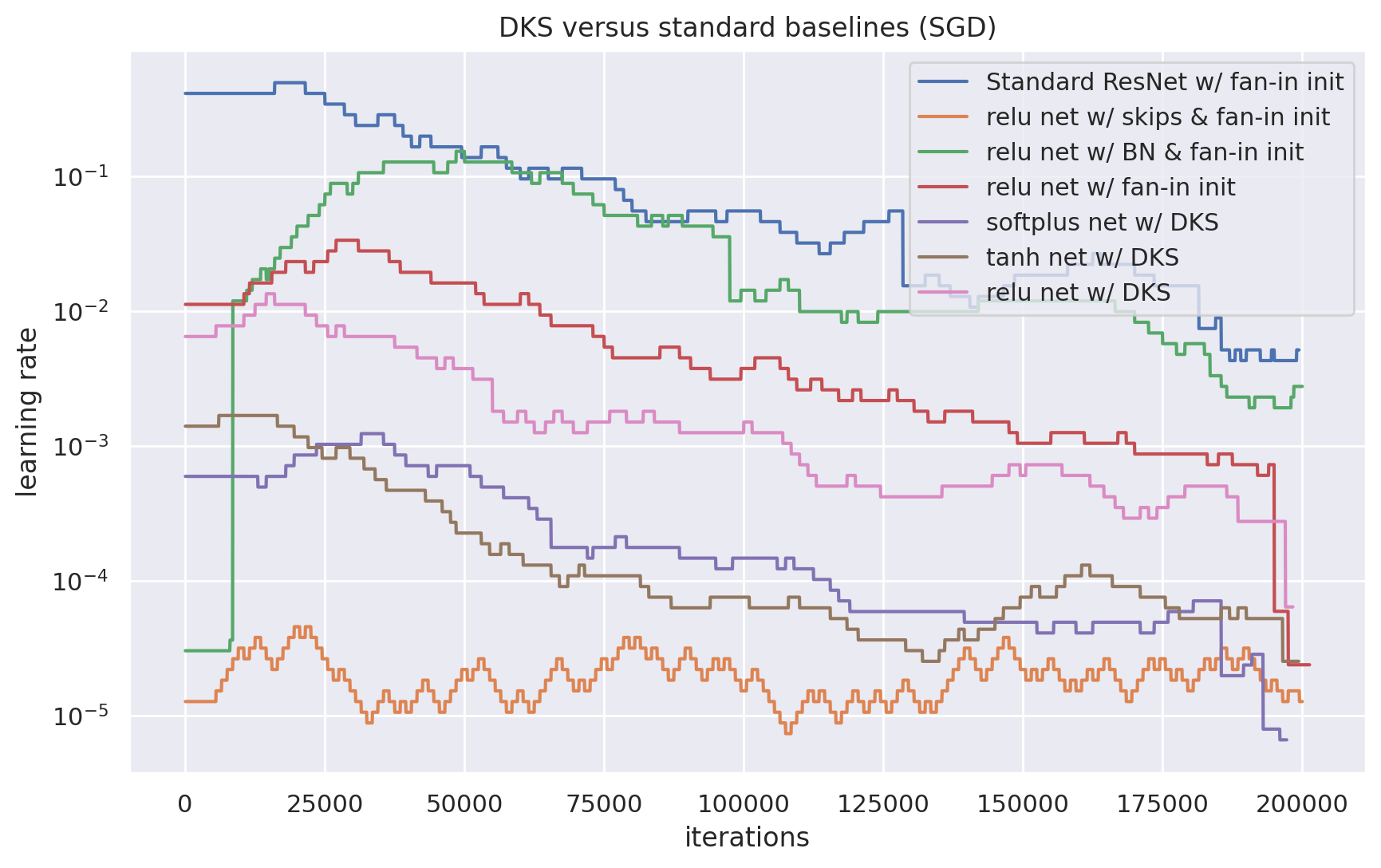}}

\section{Meta-parameter studies}\label{app:meta-param-experiments}

In this section we will experimentally study the effect of various training ``meta-parameters'' on the optimization and generalization performance of networks constructed with DKS. These will include the weight on the residual branch when using skip connections, DKS's global slope bound parameter $\zeta$, and the choice of optimizer.

Except when otherwise indicated, all experiments will use the same default settings (as stated in Section \ref{sec:default-experiment-settings}) as our main set of experiments.

\subsection{Sweeping residual weights}\label{app:residual-weight-sweep}

In this subsection we compare different values for the weights of the residual branches BN-free networks with skip connections constructed using DKS. To satisfy the condition that the branch sums at the end of each residual block are normalized, we set their weights to $w$ and $\sqrt{1 - w^2}$ for the residual and shortcut branches (respectively).

In addition to running experiments using the same weights for all blocks, we also tried setting the weights individually according to the recipe at the end of Appendix \ref{app:resnet-map-analysis}, so as to recover the C map of an (almost) standard ResNet. (Note that this requires a generalized version of the maximal slope functions given for our modified ResNets in Section \ref{sec:maximal-slope-mod-resnet}, but is otherwise a straightforward change.)

\

\resizebox{0.85\columnwidth}{!}{\includegraphics{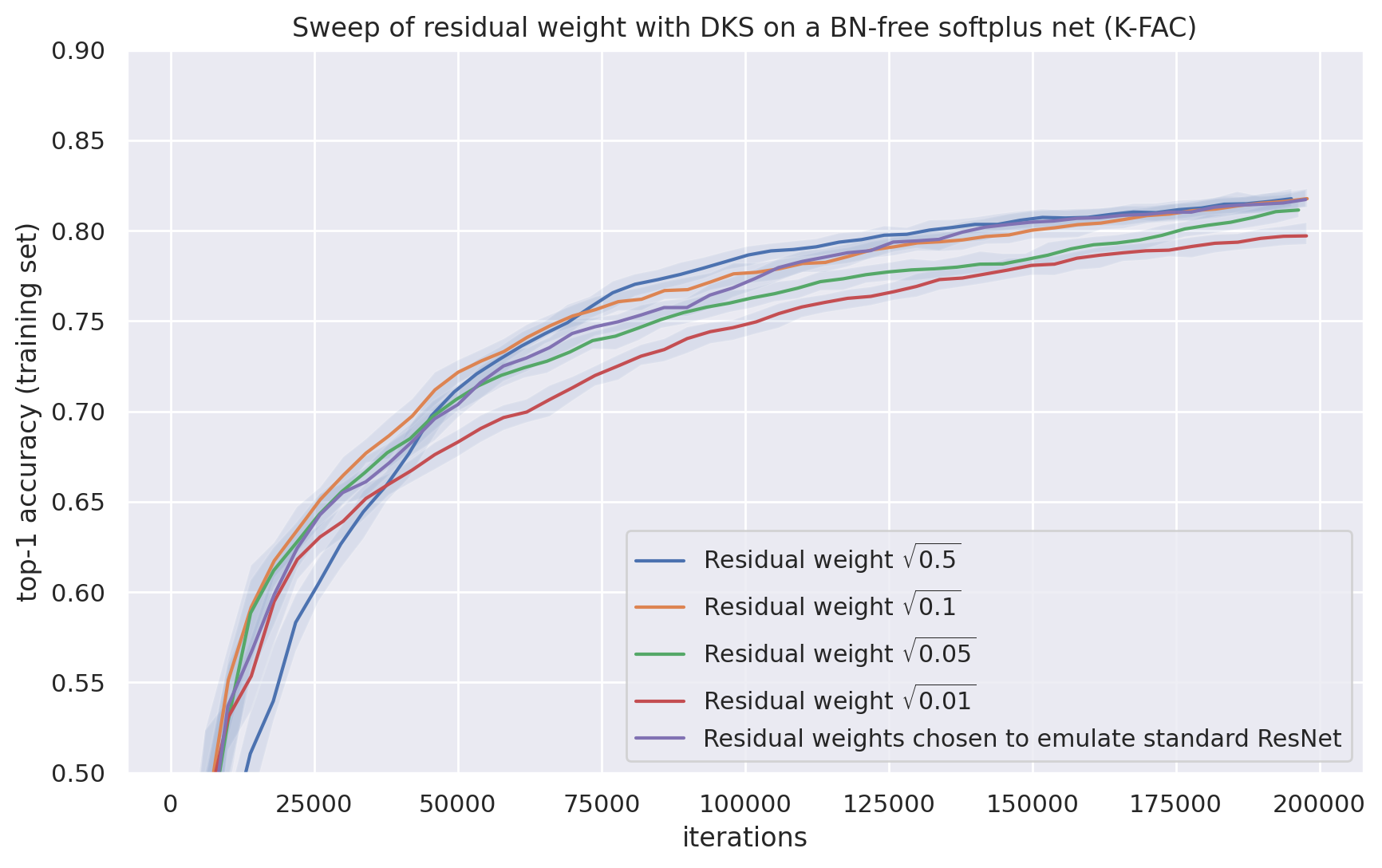}}

\resizebox{0.85\columnwidth}{!}{\includegraphics{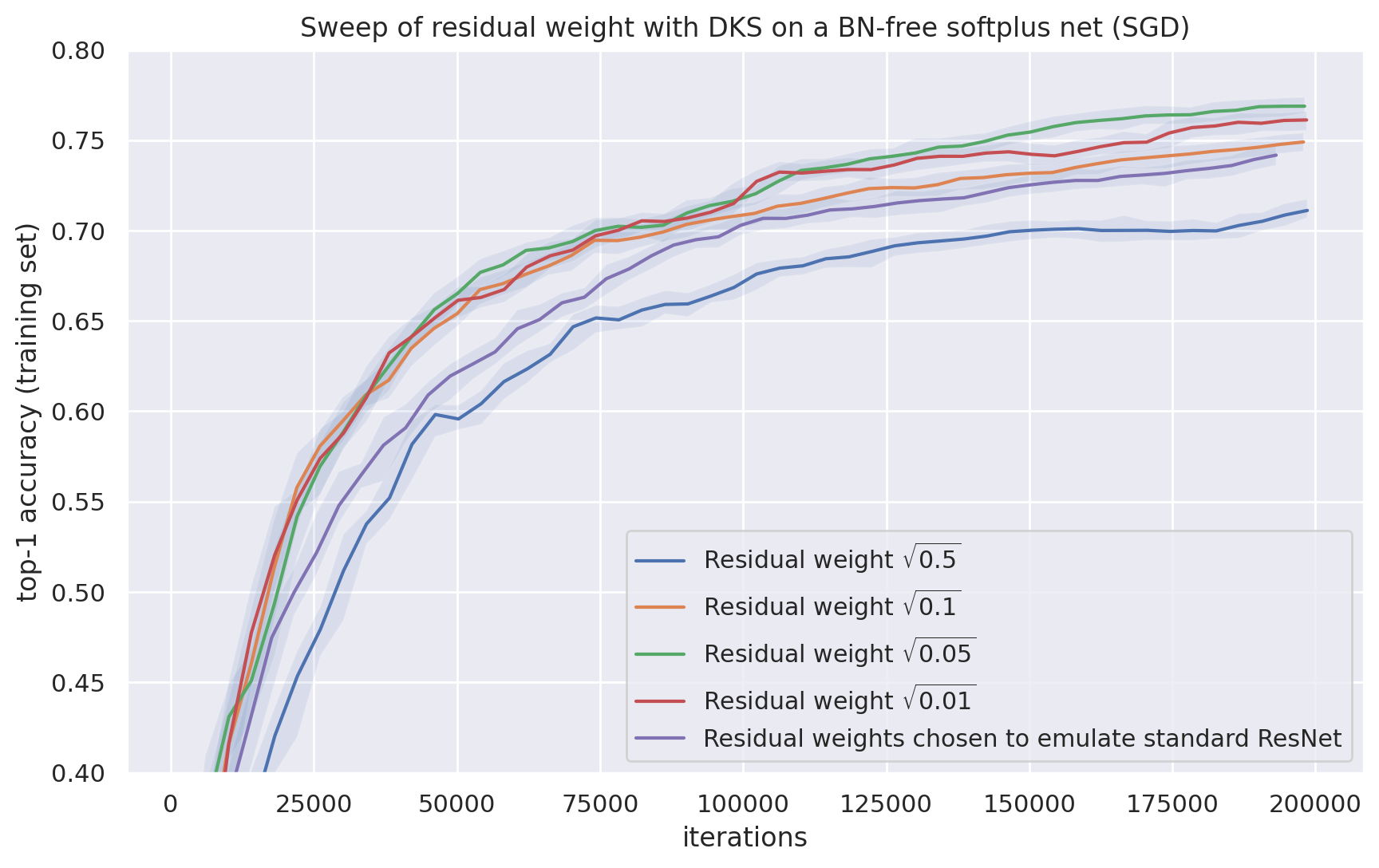}}

From these results we see that the value $\sqrt{0.05}$ seems to work best when using SGD. When using K-FAC, the difference in optimization speed between the three largest options is much smaller, and so we will use $\sqrt{0.05}$ as the default value for all optimizers.

Note that while $\sqrt{0.05}$ is (arguably) the best amoung the values we tried for this network, there is no reason to think that this value will be the best choice for other residual architectures (or the same architecture for a different depth parameter).

\subsection{DKS with different optimizers}

In this subsection we compare the optimization performance of different optimizers on skip-free BN-free networks constructed with DKS.

\resizebox{0.85\columnwidth}{!}{\includegraphics{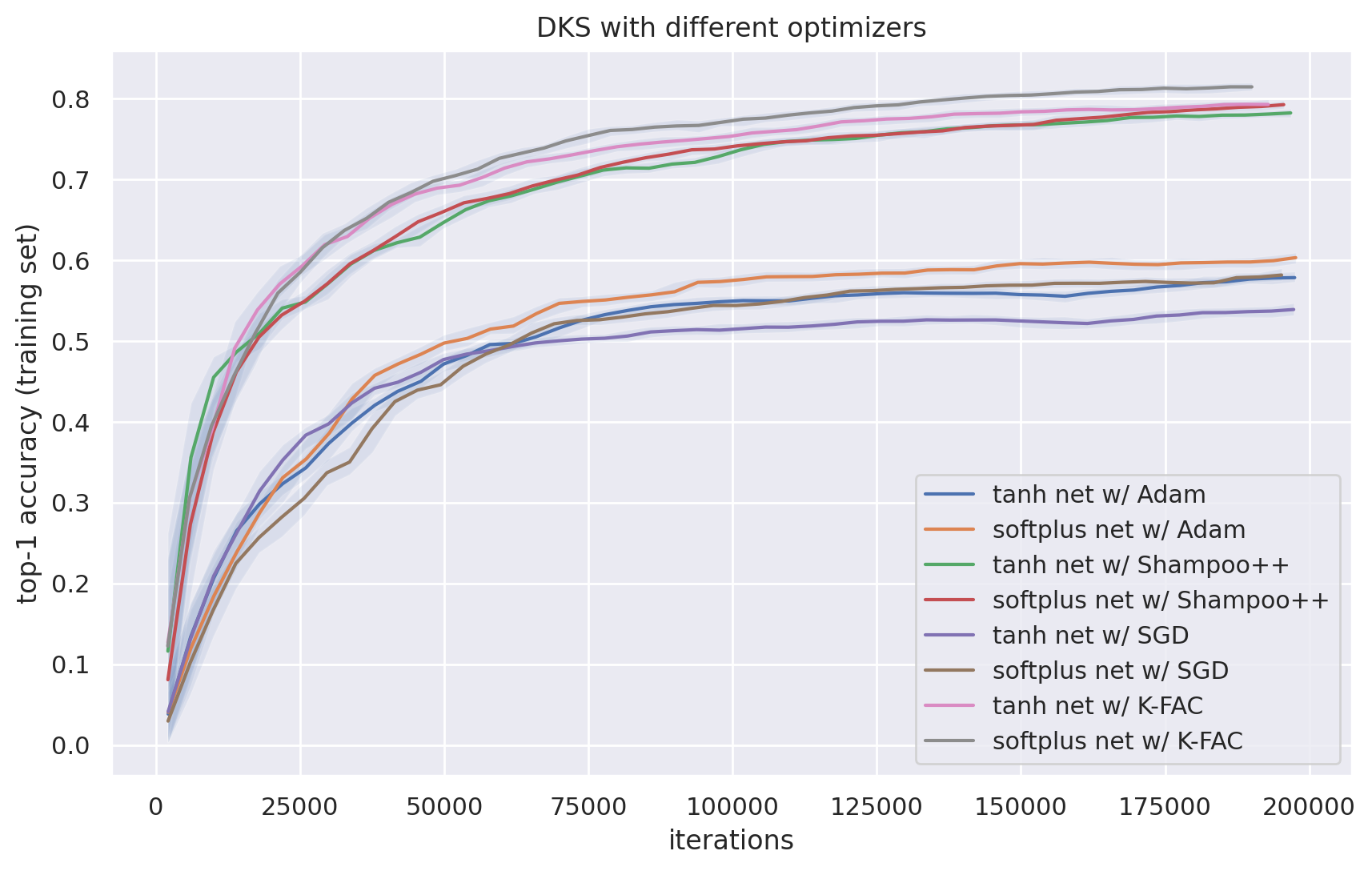}}

These results show that K-FAC and Shampoo have a large advantage over Adam and SGD in this setting. Moreover, Adam has a small advantage over SGD, and K-FAC has a small advantage over Shampoo.

\resizebox{0.85\columnwidth}{!}{\includegraphics{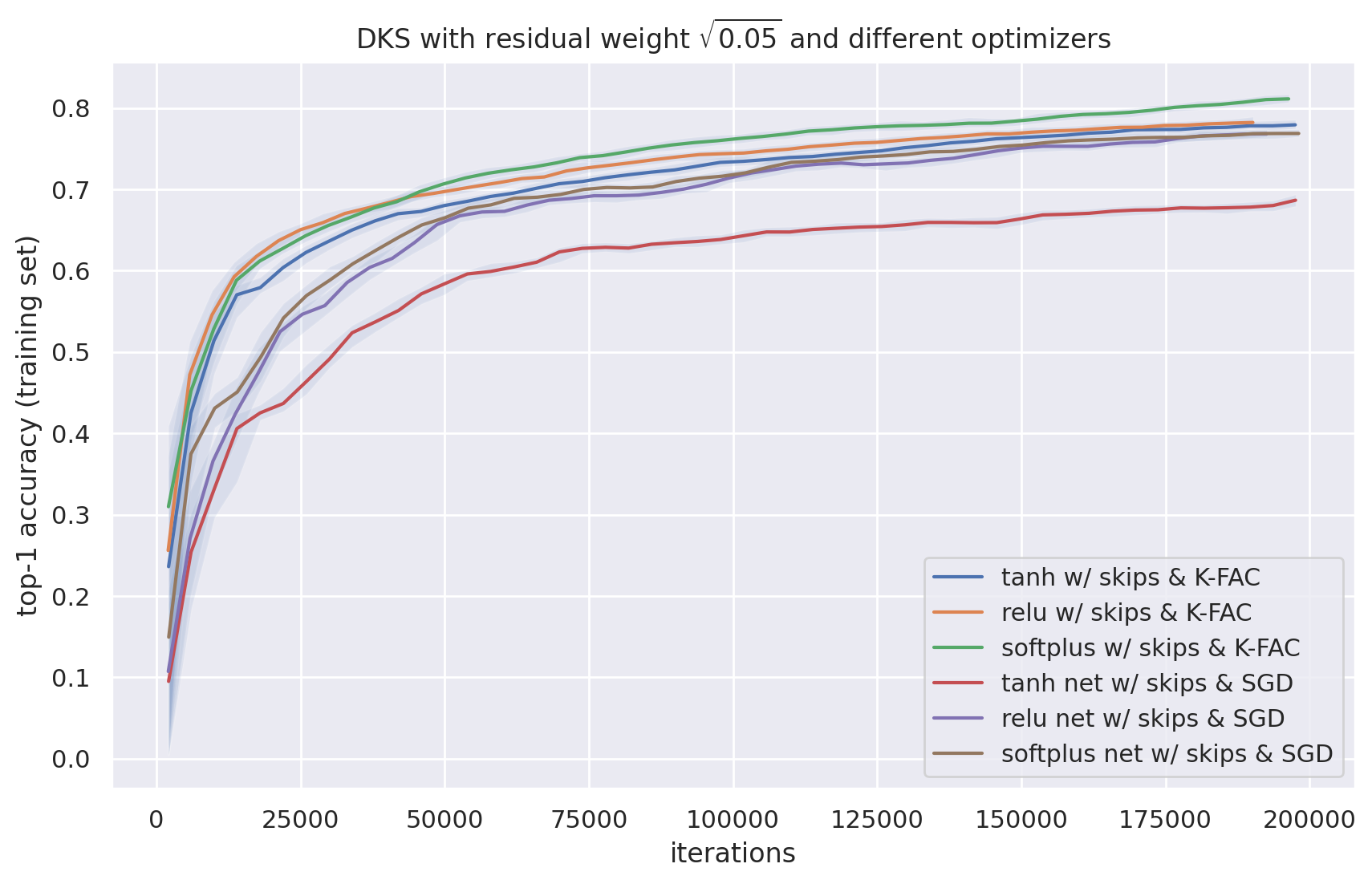}}

The picture looks different for networks with skip connections, and K-FAC and SGD yield fairly similar optimization speeds for two out of the three activation functions we tried.

\subsection{Sweeping $\zeta$ values}\label{app:sweep-zeta-1}

In this subsection we study the influence of the global slope bound $\zeta$ on the optimization speed of skip-free BN-free tanh networks constructed with DKS. In particular, we compare the default choice of $\zeta = 1.5$ to various ``extreme'' values, which are either very large (corresponding to highly nonlinear network behavior), or are very close to 1 (corresponding to very linear behavior).

\resizebox{0.85\columnwidth}{!}{\includegraphics{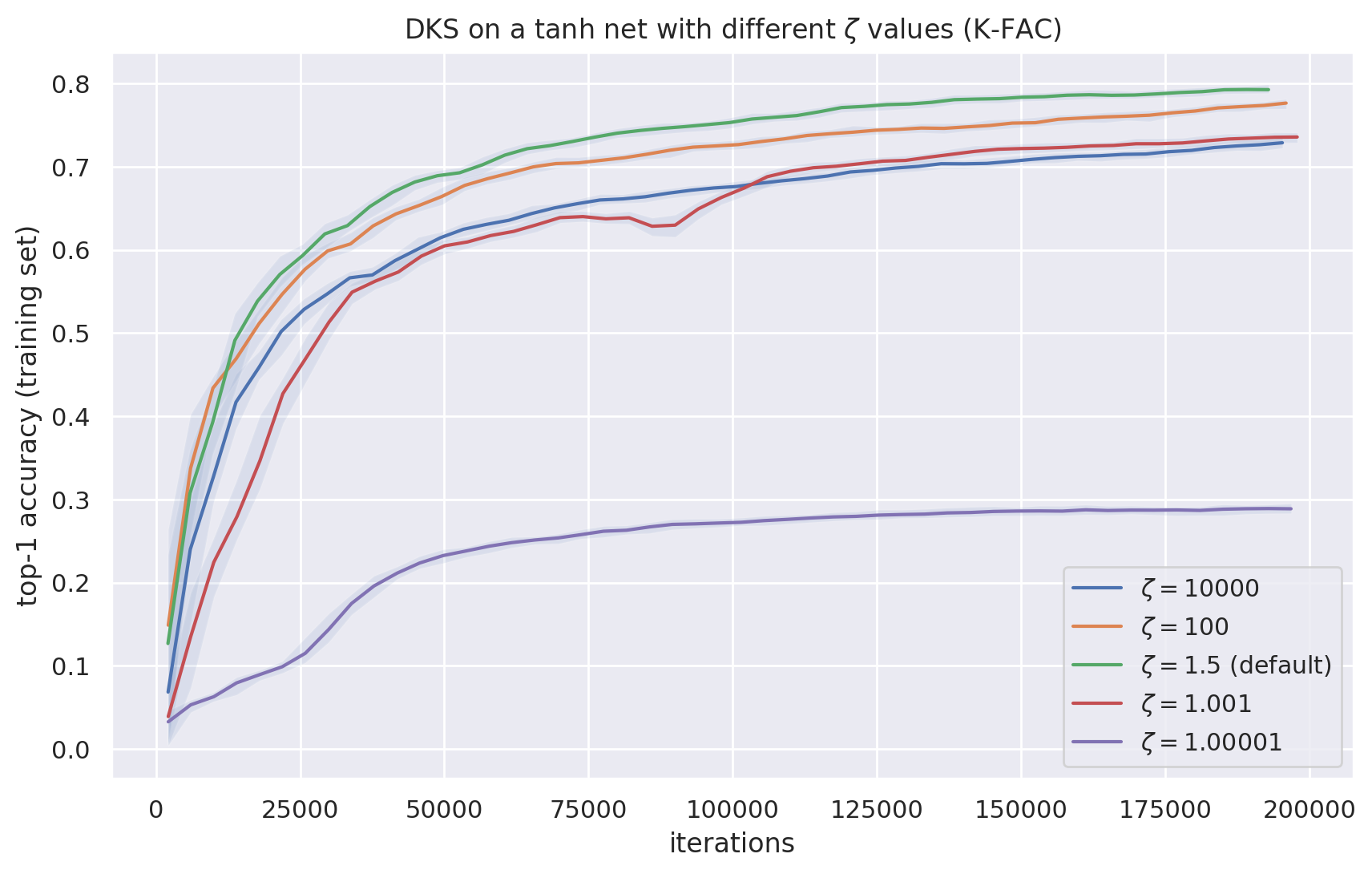}}

\resizebox{0.85\columnwidth}{!}{\includegraphics{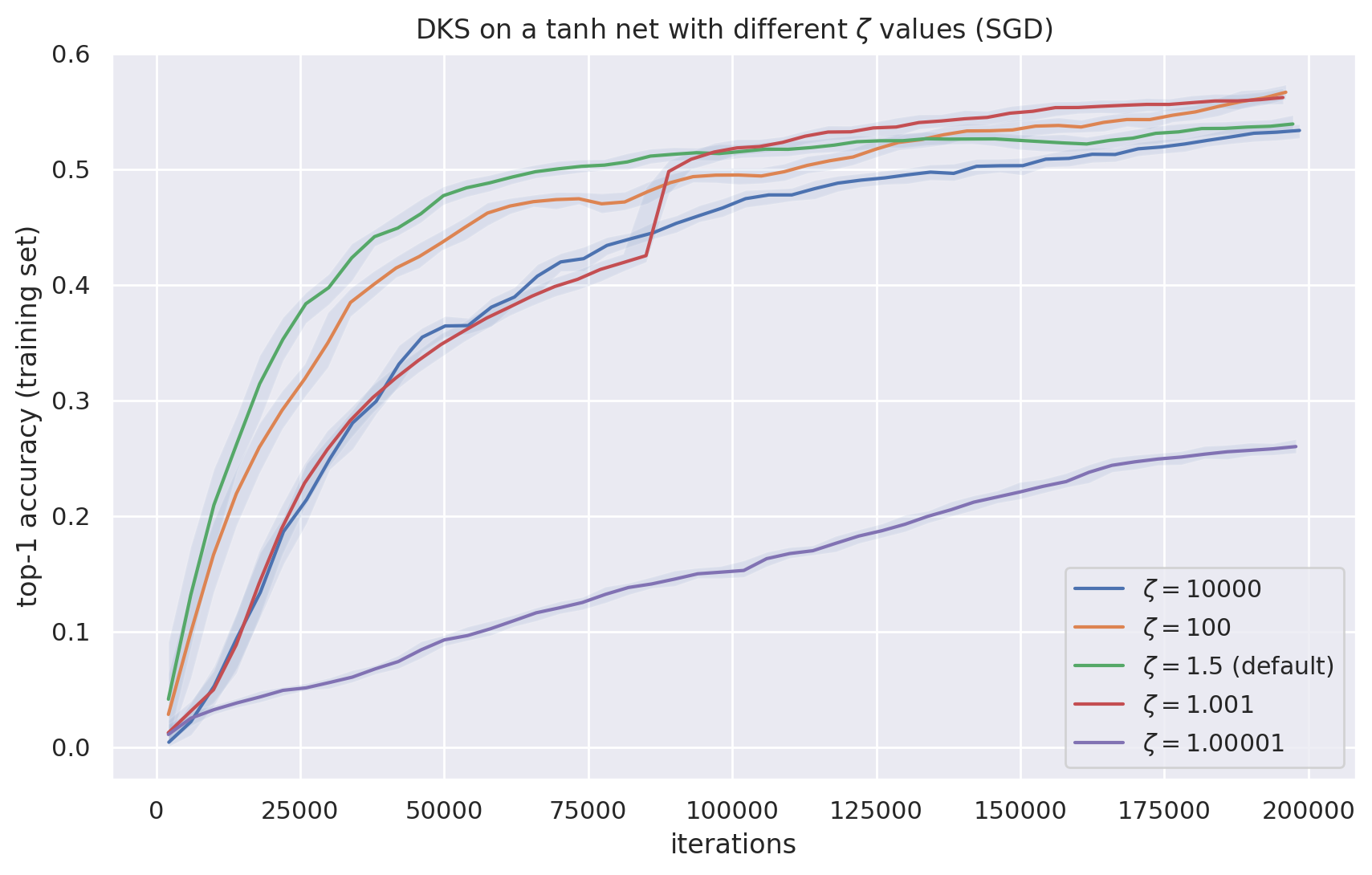}}

From these results we can see that $\zeta = 1.5$ gives the fastest optimization performance for K-FAC and the fastest short-term optimization performance for SGD. $\zeta = 1.00001$, which corresponds to a very linear network, gives very slow optimization performance, perhaps for the reasons discussed in Section \ref{sec:too-linear}. Somewhat surprisingly, the choice $\zeta = 10000$ yields quite respectable (although still suboptimal) performance.

\subsection{The influence of $\zeta$ on generalization}\label{app:sweep-zeta-2}

In this subsection we study the influence of the global slope bound $\zeta$ on the generalization performance for networks trained with K-FAC.

\

\resizebox{0.85\columnwidth}{!}{\includegraphics{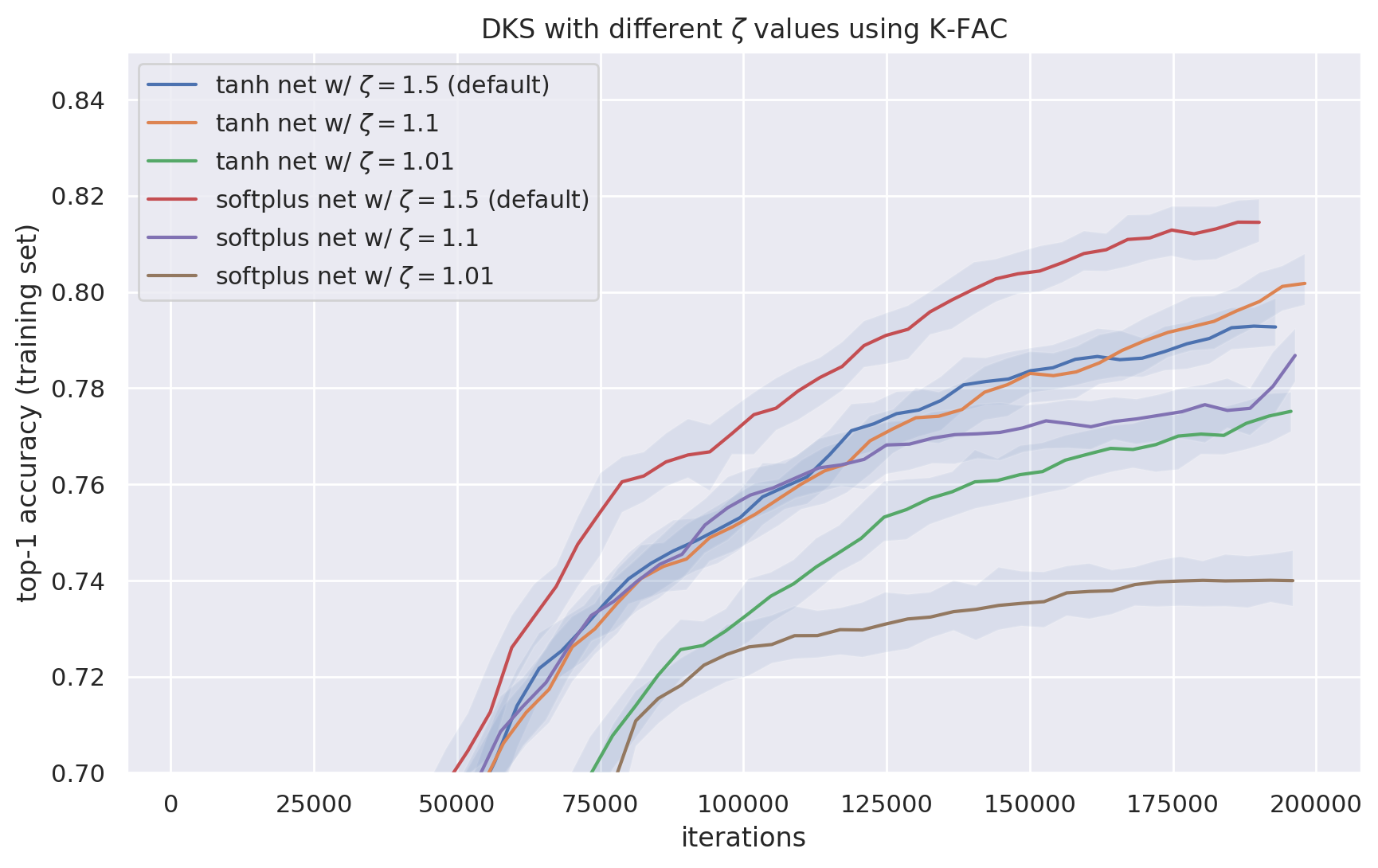}}

\resizebox{0.85\columnwidth}{!}{\includegraphics{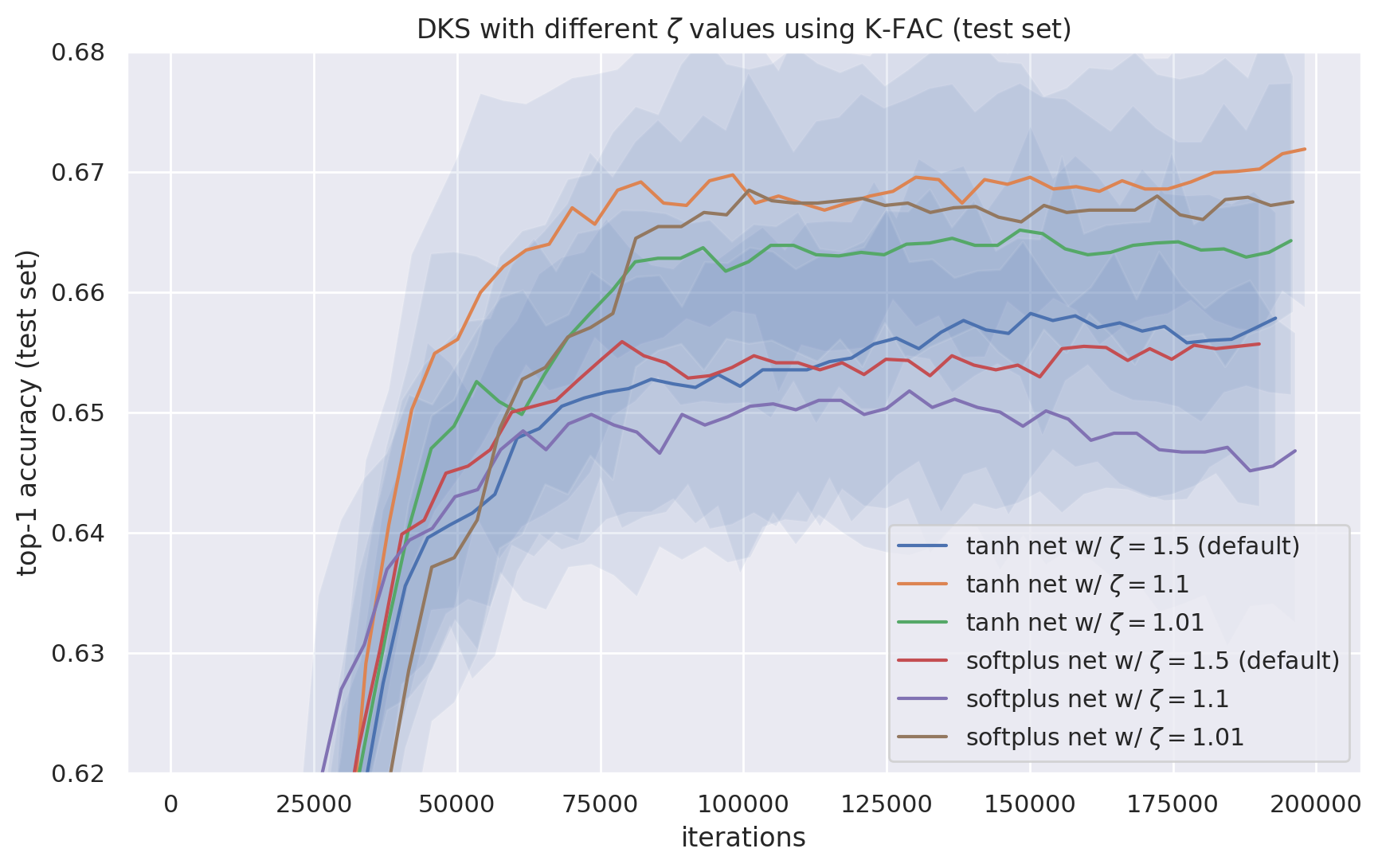}}

From these results we can see that while the value $\zeta = 1.5$ tends to give faster optimization, slightly lower values are associated with improved generalization.

\section{Experiments with ablations and modifications of DKS}\label{app:ablation-and-mod-experiments}

In this section we consider various ablations and modifications of DKS. The overall conclusion of these studies is that each component of DKS, except perhaps for PLN (assuming reasonably well scaled input data), is required to achieve the highest optimization speed. When considering test error the conclusions are similar but somewhat muted, with the single exception that using weighted mean-pooling layers with K-FAC seems to improve test set performance while degrading training set performance.

For plots that contain solid and dotted lines of the same color, solid lines will correspond to the default unmodified version of DKS, while dotted lines will correspond to the ablated/modified version. Except when otherwise indicated, all experiments will use the same default settings (as stated in Section \ref{sec:default-experiment-settings}) as our main set of experiments. We will omit results for test error, except in those cases where it gives qualitatively different results from the training error.

\subsection{Alternative initializations for weights / filter banks}

In this subsection we consider replacing the Orthogonal Delta initialization used in DKS with various alternative weight initialization schemes. Note that while the use of Delta weight initializations is required by the Q/C map theory that underlies DKS, in practice one can still try DKS with any other weight initialization scheme.

\

\resizebox{0.85\columnwidth}{!}{\includegraphics{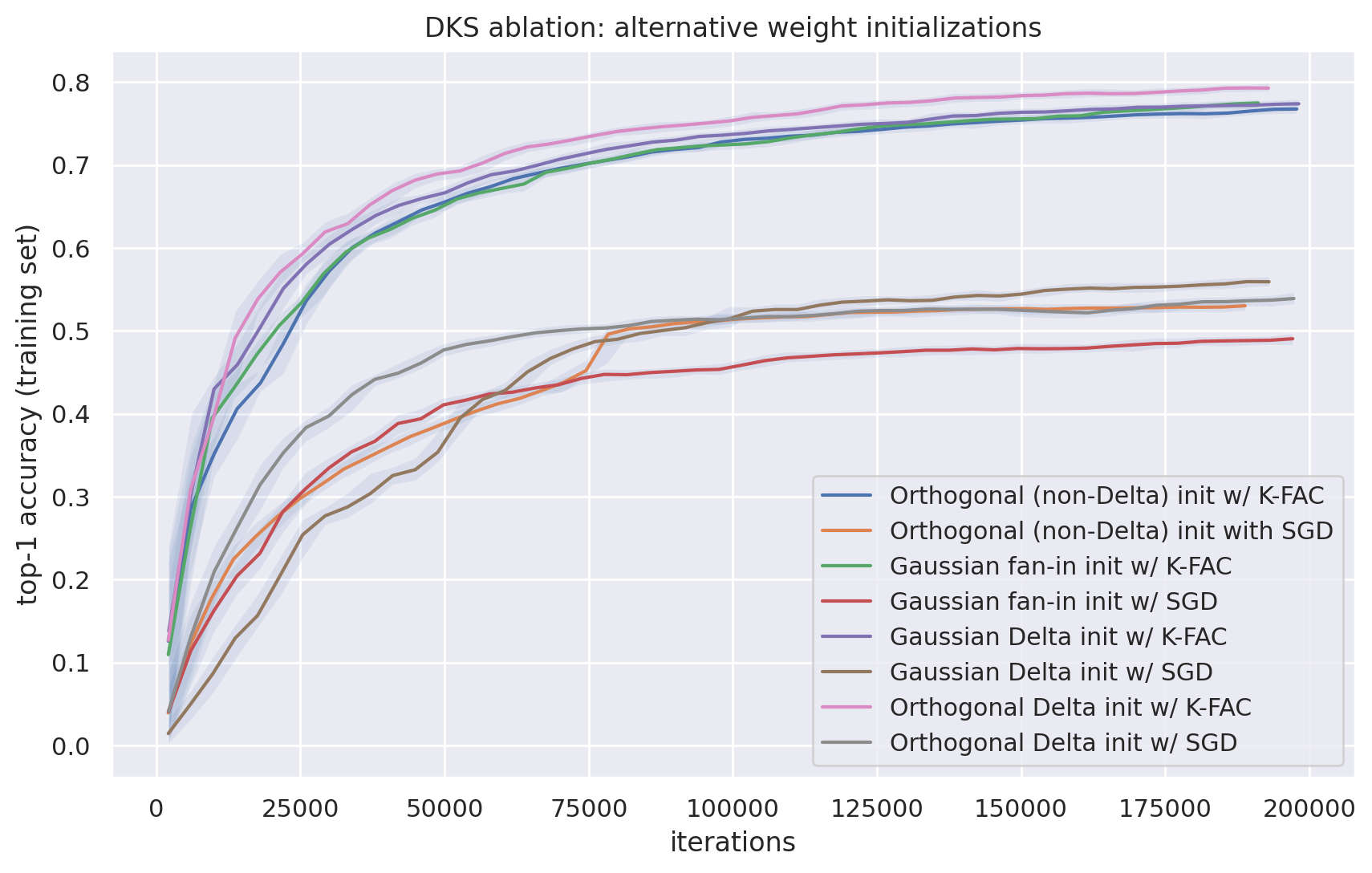}}

From these results we can see that while the Orthogonal Delta initialization gives the best results for both K-FAC and SGD, the results for the Orthogonal (non-Delta) and Gaussian Delta initializations are very close.

One might be tempted to conclude from these findings that the weight initialization is relatively unimportant for skip-free BN-free networks in general. However, as we can see from the following plot, there is a much larger gap in performance between the different options when we don't use DKS's activation function transformations:

\resizebox{0.85\columnwidth}{!}{\includegraphics{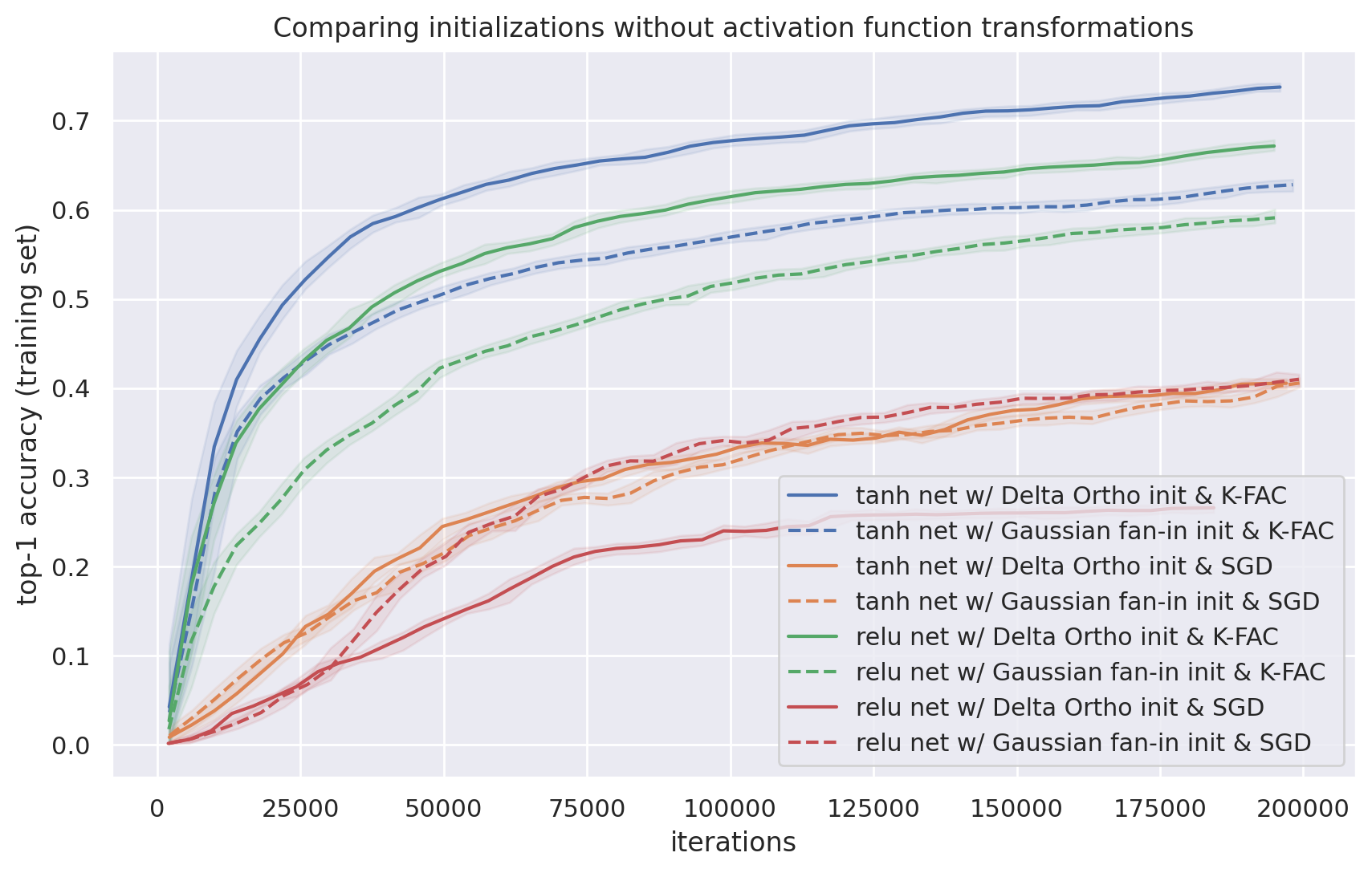}}

\subsection{Only enforcing the $Q_f (1) = 1$ condition}\label{sec:only_Q1=1_cond}

In this subsection we consider modifying DKS to only enforce the condition that $Q_f (1) = 1$ for all subnetworks $f$ (which is equivalent to doing the same for all nonlinear layers $f$). Note that this conditions is roughly analogous to what normalization layers, the ``He initialization method'' for RELUs, and the LSUV/WLI initializations are trying to achieve. 

To achieve this condition we use only the output scale parameter (denoted $\gamma$ in Section \ref{sec:activation-transform}) in the transformed activation functions.

\resizebox{0.85\columnwidth}{!}{\includegraphics{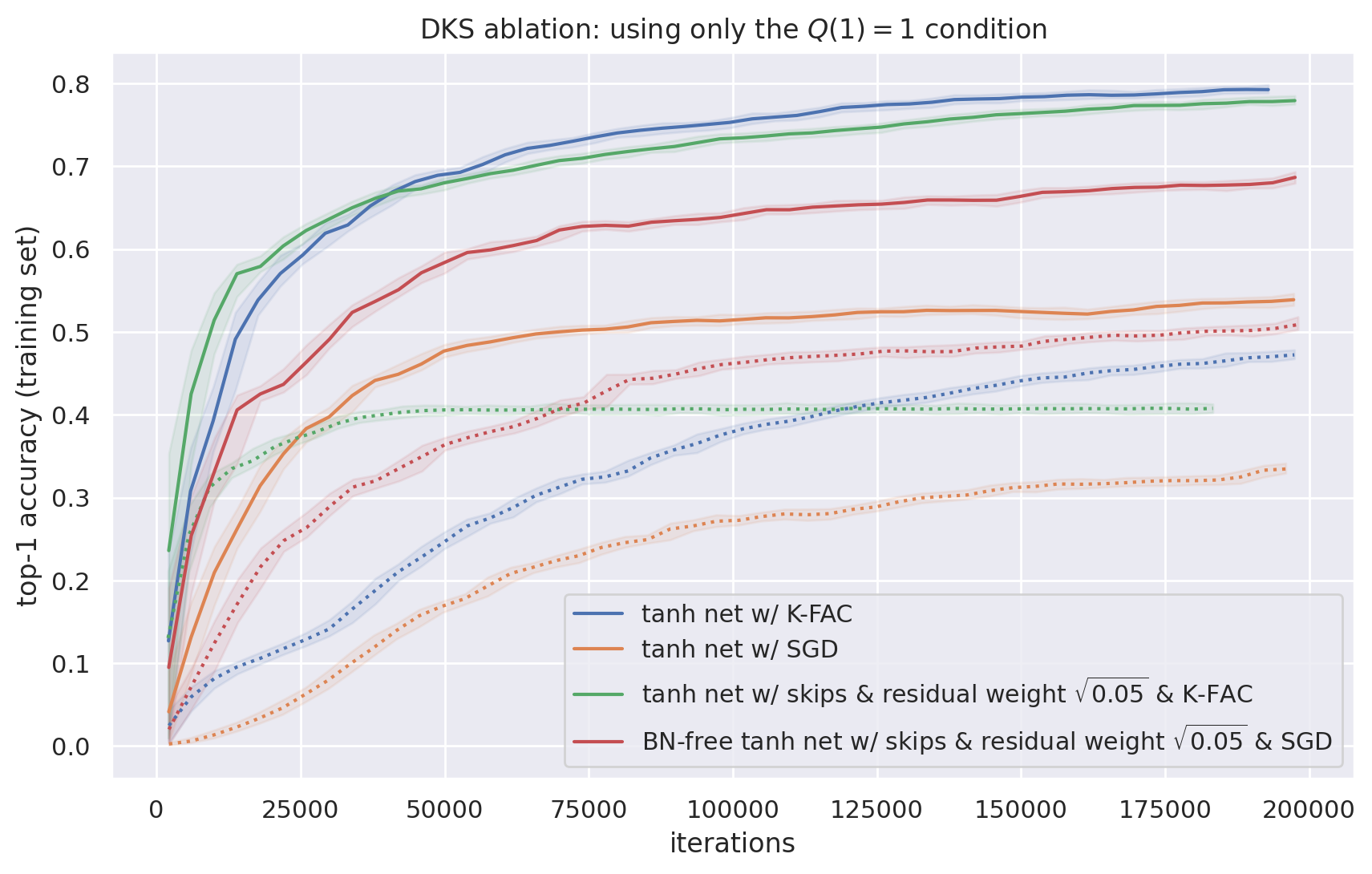}}

From these results we see that the condition $Q_f (1) = 1$ is clearly not enough by itself to achieve fast optimization in BN-free networks, with or without skip connections. (And as we will see in Section \ref{app:removing-act-transforms-experiment}, enforcing this condition {\tmem{by itself}} may actually do more harm than good.)

\subsection{Removing the condition $Q_f' (1) = 1$}

In this subsection we consider removing the condition $Q_f' (1) = 1$ from the set of four conditions that we enforce in DKS. The remaining three conditions are achieved by setting three of the four activation function parameters (defined in Section \ref{sec:activation-transform}), with the input shift parameter $\beta$ being left out (which is equivalent to taking $\beta = 0$).

\resizebox{0.85\columnwidth}{!}{\includegraphics{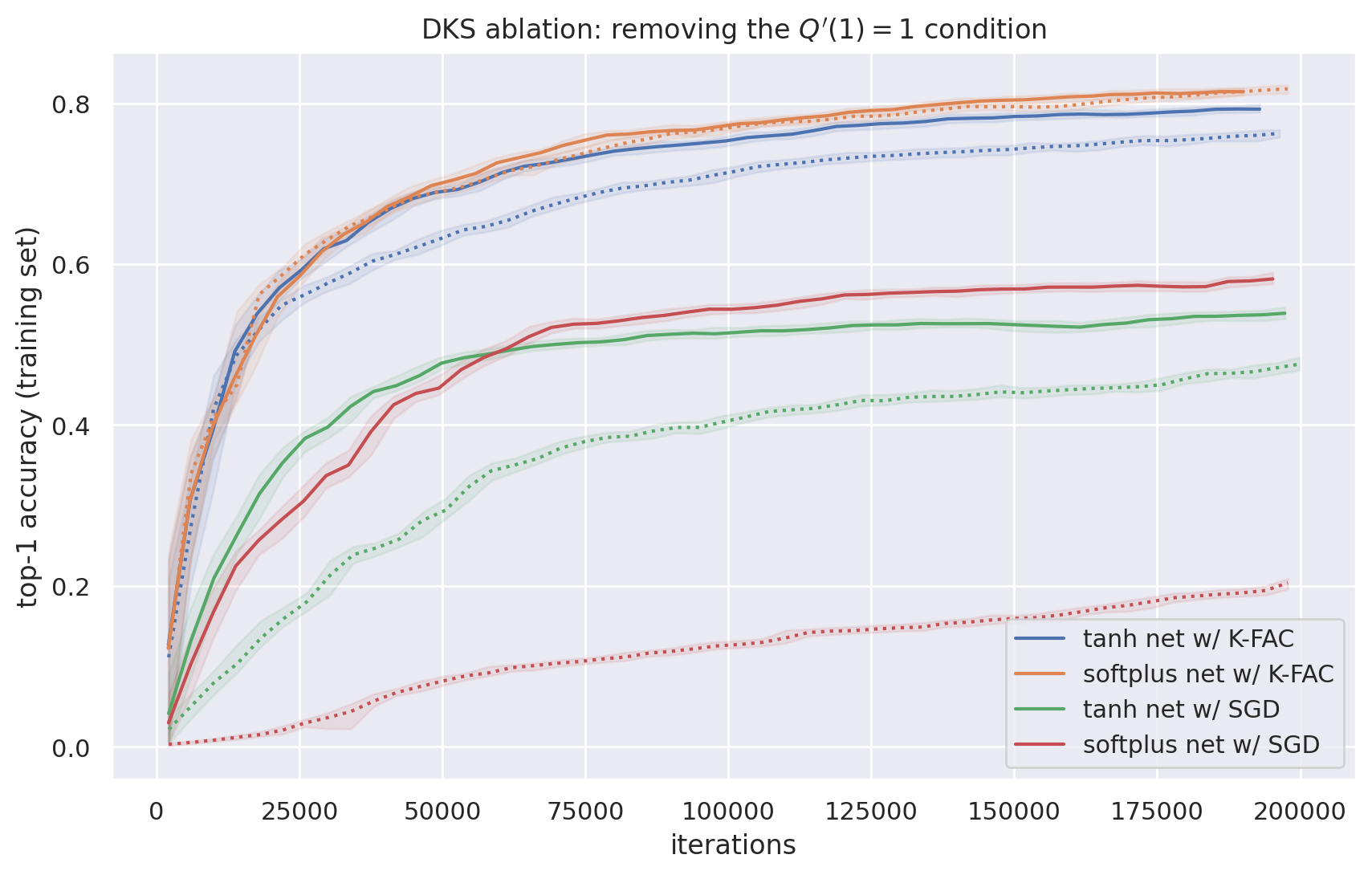}}

From these results we see that this condition appears to be important in most training scenarios, but not all of them.

\subsection{Minimizing $Q_f' (1)$ instead of enforcing $Q_f' (1) = 1$}\label{app:min-qslope-experiment}

In this subsection we consider the effect of minimizing $Q_f' (1)$ for each subnetwork $f$ in DKS instead of enforcing the condition $Q_f' (1) = 1$. This is accomplished by minimizing $Q_f' (1)$ for each nonlinear layer $f$ in the network. This modification is motivated by the observation that minimizing $Q_f' (1)$ should, according to the reasoning of Section \ref{sec:qslope-condition}, minimize the total kernel approximation error.

\resizebox{0.85\columnwidth}{!}{\includegraphics{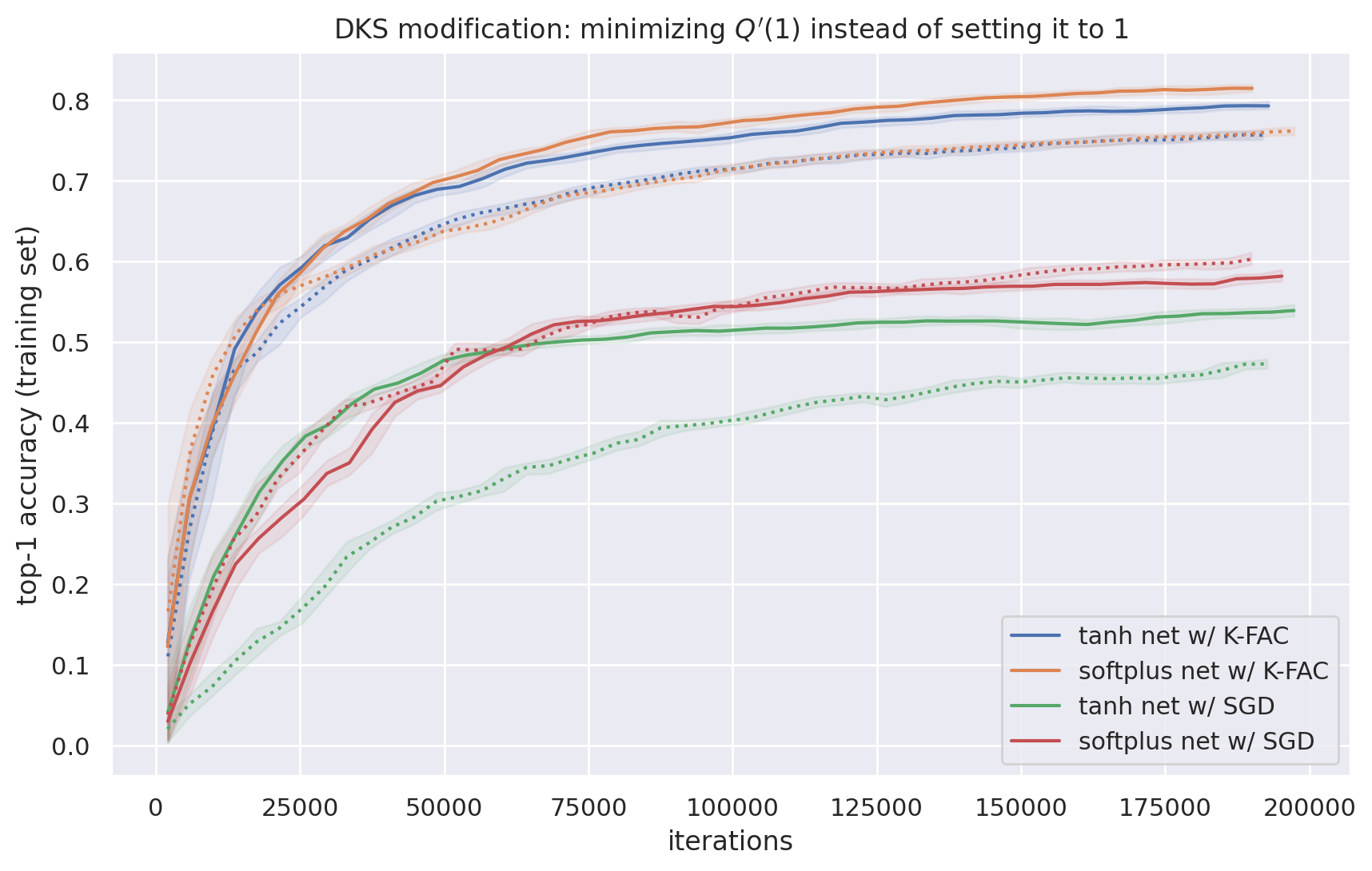}}

From these results we can see that minimizing $Q_f' (1)$ works overall worse than simply setting $Q_f' (1) = 1$. The reasons for this remain unclear.

\subsection{Removing activation function transformations completely}\label{app:removing-act-transforms-experiment}

In this subsection we consider completely removing the activation function transformations from DKS. What remains is the Delta Orthogonal initialization for the weights (and a zero initialization of the biases), normalized sums between residual and shortcut branches, and the use of PLN.

\resizebox{0.85\columnwidth}{!}{\includegraphics{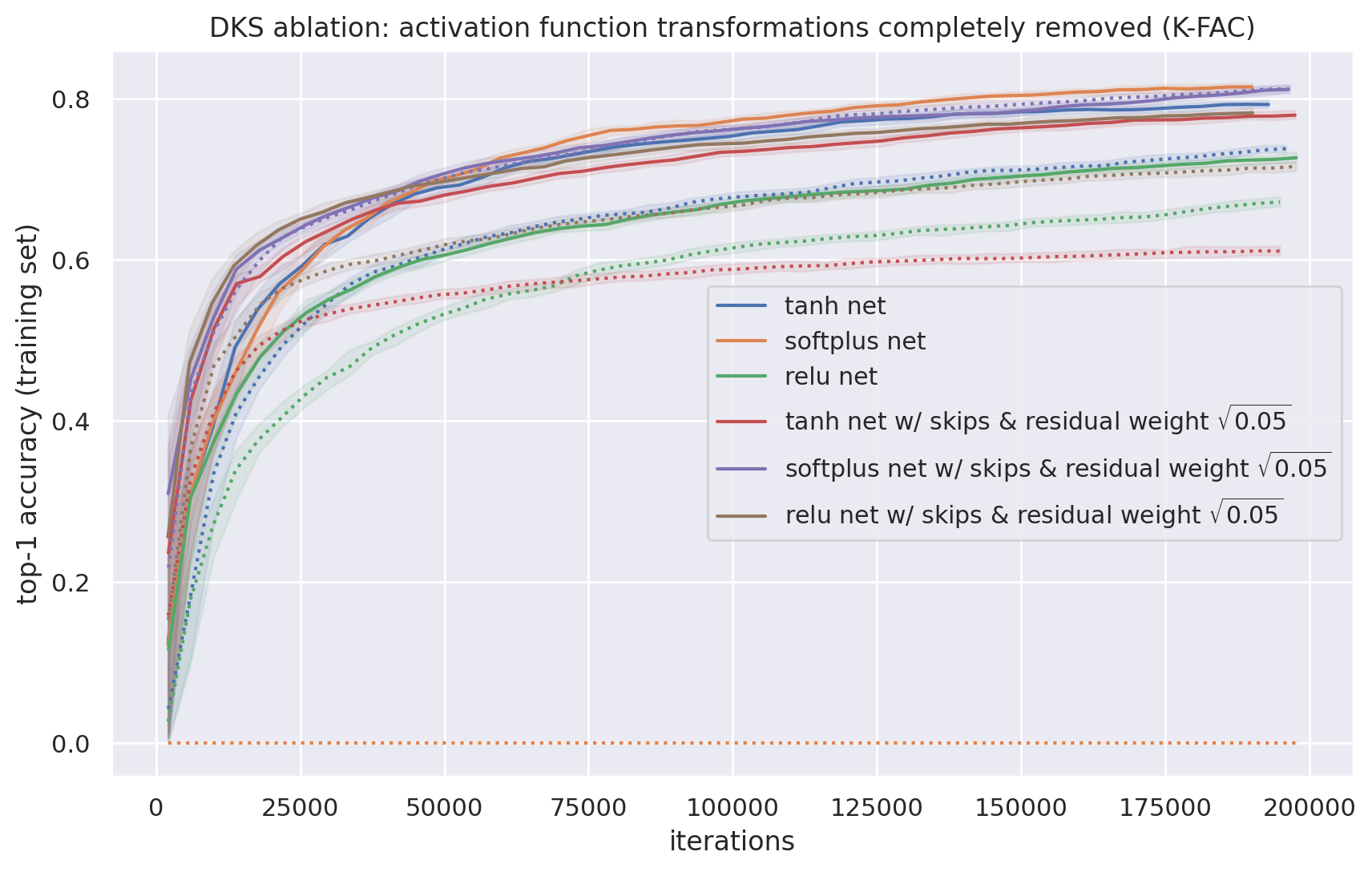}}

\resizebox{0.85\columnwidth}{!}{\includegraphics{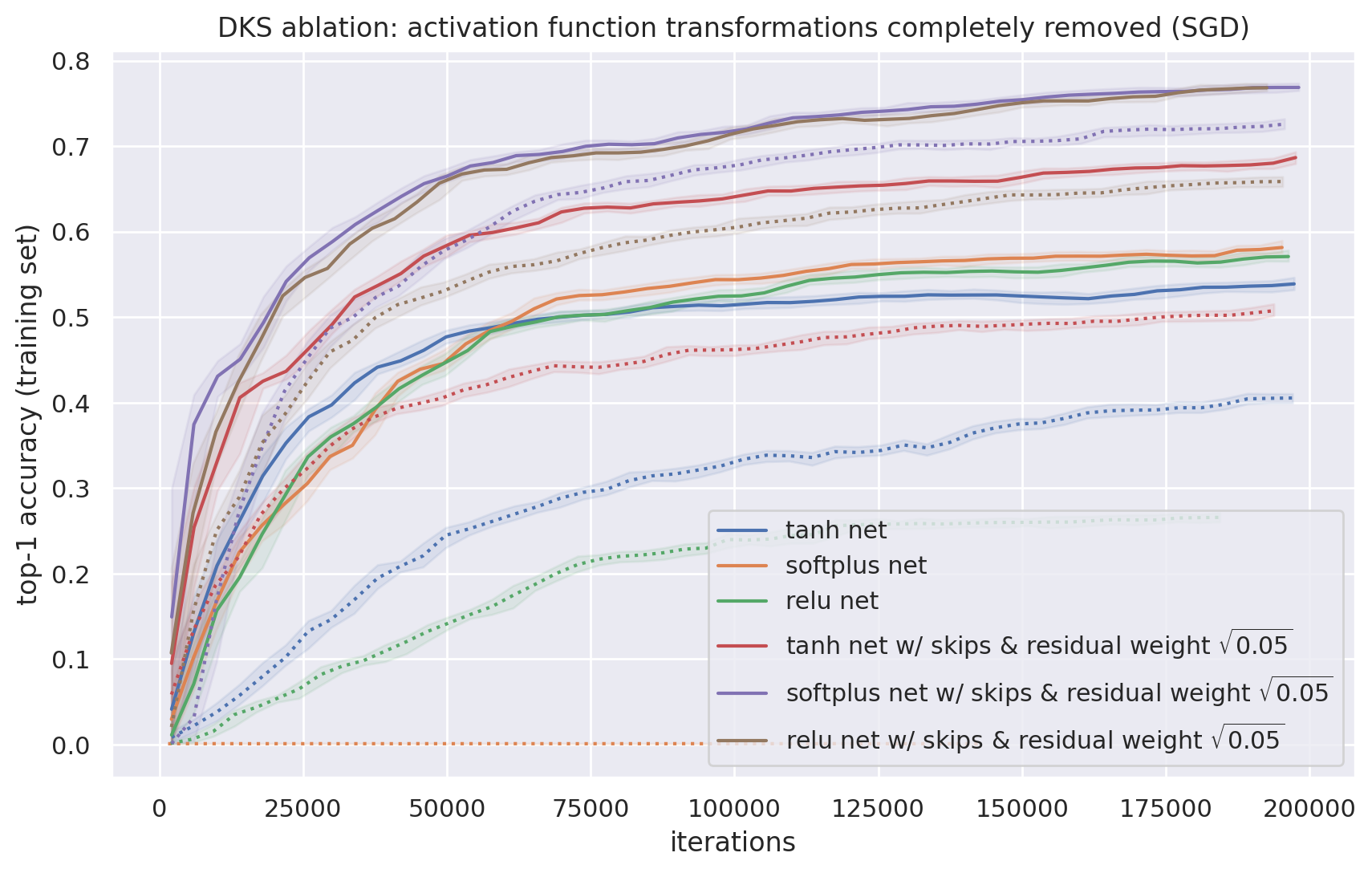}}

From these results we see that activation function transformations are important in all the scenarios we tested, except when training softplus networks with skip connections using K-FAC. Moreover, they seem to be especially important when training with SGD.

Comparing the results here for tanh networks to those given in Subsection \ref{sec:only_Q1=1_cond}, we can see that training speed becomes {\tmem{worse}} if we enforce the condition $Q_f (1) = 1$ by itself (versus enforcing no conditions at all). While potentially counterintuitive, this isn't actually surprising. Indeed, there is no reason to think that the local C map for an untransformed tanh layer will be more favorable given an input q value of 1 compared to some other value. (In this case, the ``other value" is the fixed point of a tanh layer's local Q map.) Note that for the fully transformed tanh layers generated by DKS this consideration is moot, since any fixed input q value is essentially equivalent to all other choices due to the flexibility afforded by the activation function's input scale parameter (denoted $\alpha$ in Section \ref{sec:activation-transform}).

We also have results for skip-free BN-free networks trained on CIFAR-10, which tell a similar story. These are given below:

\resizebox{0.85\columnwidth}{!}{\includegraphics{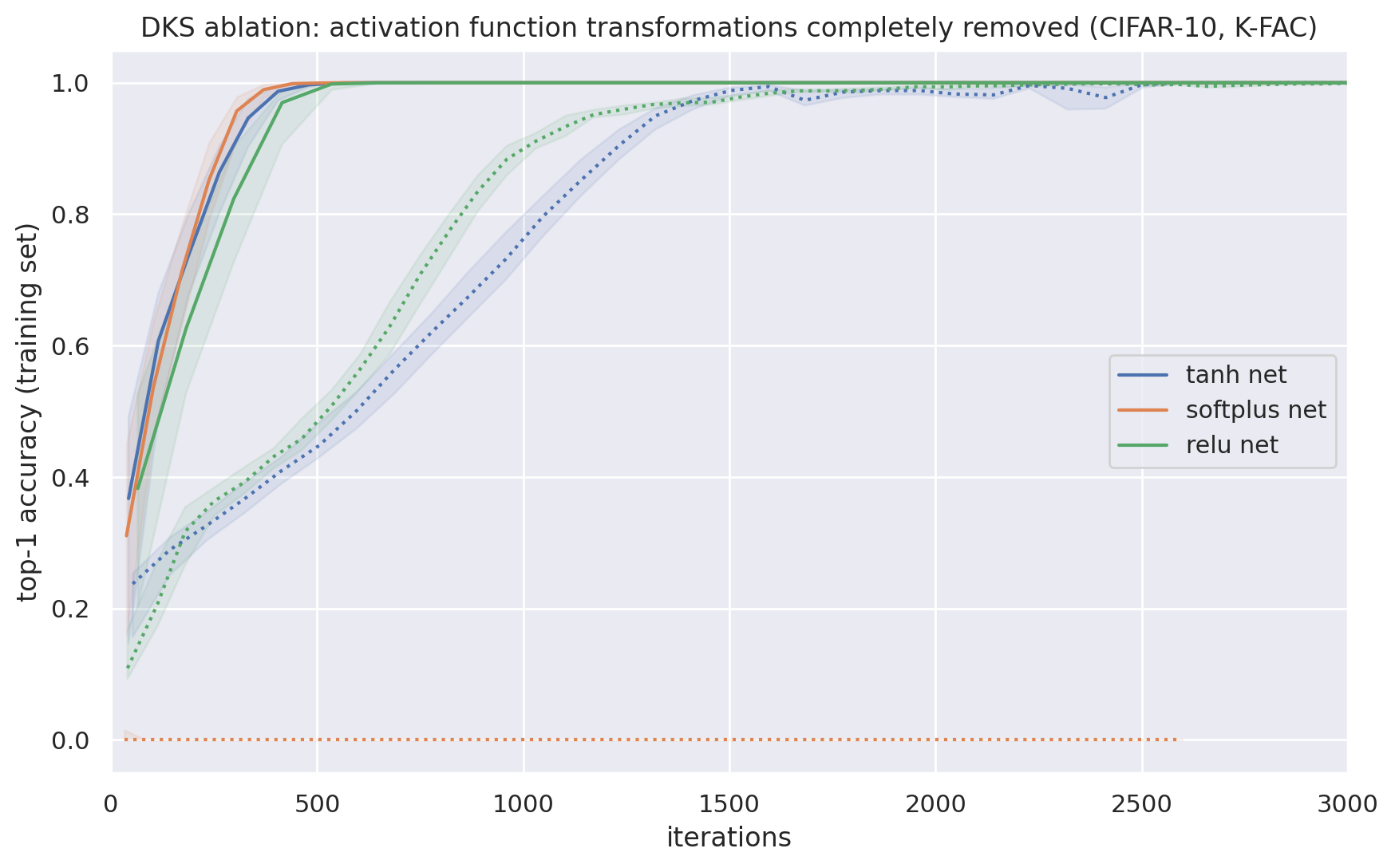}}

\resizebox{0.85\columnwidth}{!}{\includegraphics{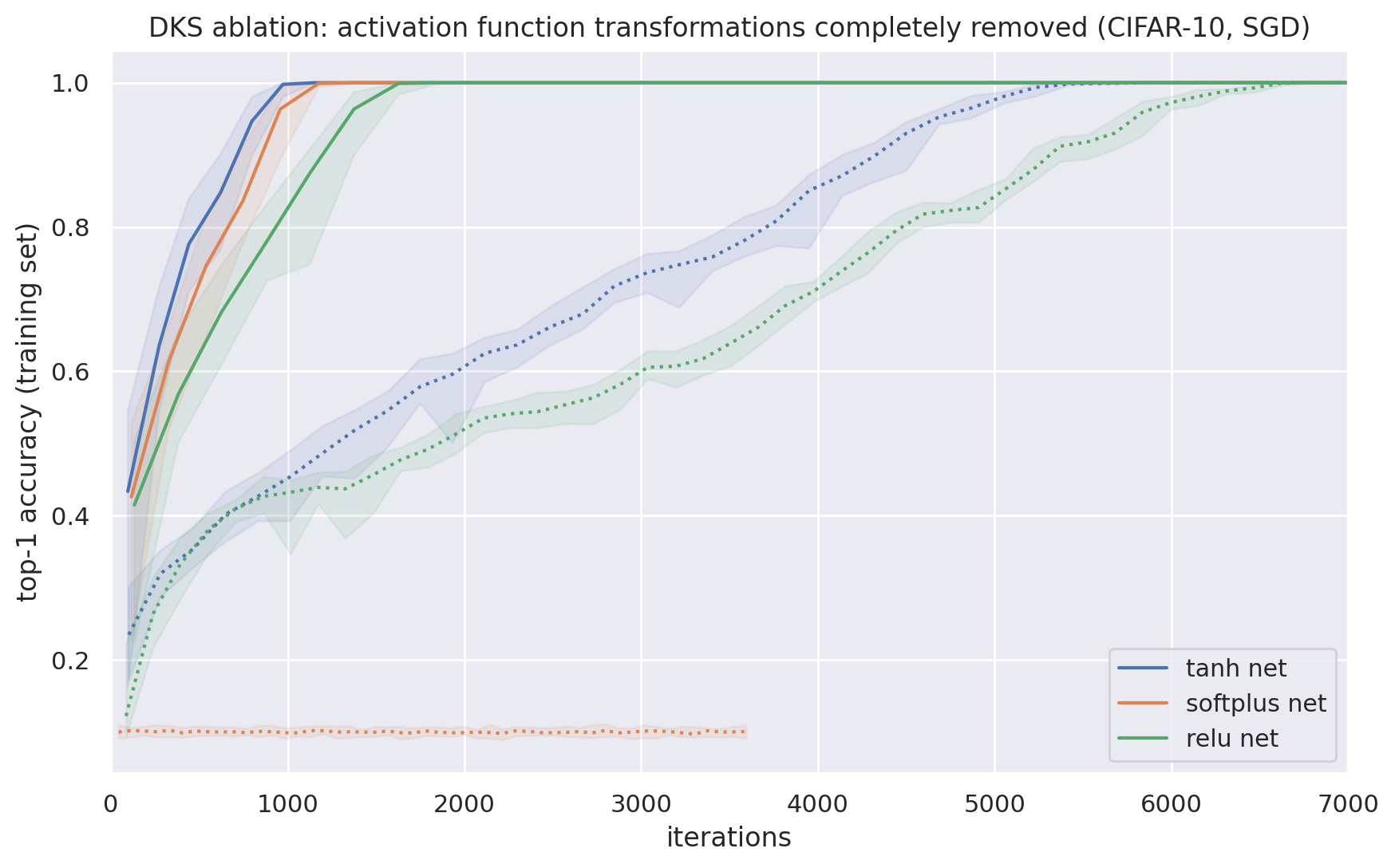}}

\subsection{Removing Per-Location Normalization (PLN)}\label{app:ablation-PLN}

In this subsection we consider using DKS without the Per-Location Normalization (PLN) data pre-processing step described in Section \ref{sec:PLN}.

\resizebox{0.85\columnwidth}{!}{\includegraphics{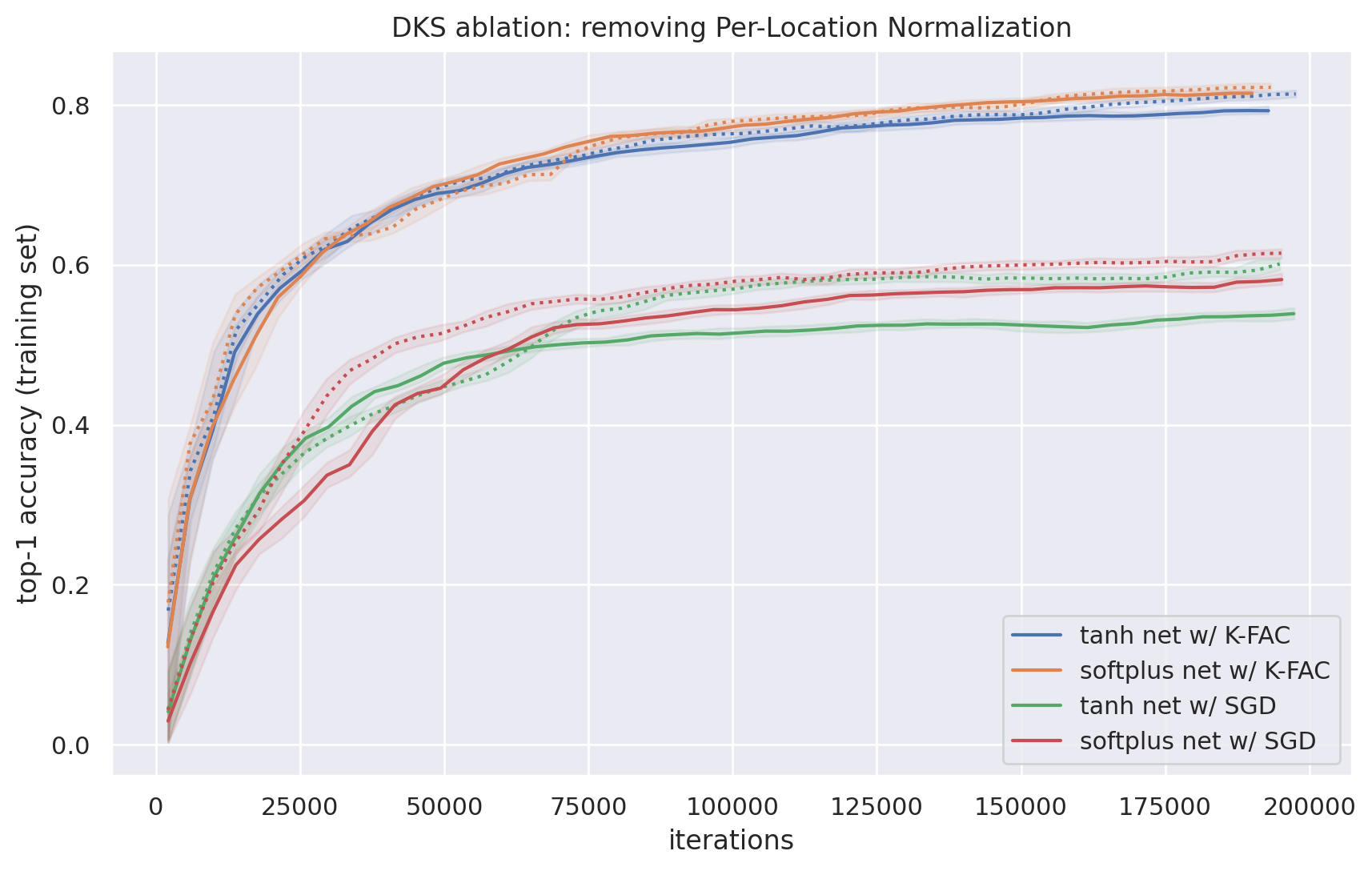}}

From these results we can see that, somewhat surprisingly, the use of PLN actually harms optimization performance, especially for SGD. We speculate about possible reasons for this in Section \ref{sec:PLN}.

Whatever the reasons, we know that they are contingent on the default properties of the input training data. We can demonstrate this by rerunning the same experiment without PLN for input data that is scaled by a factor of 100 ({\tmem{after}} the usual Imagenet pre-processing and augmentation).

\resizebox{0.85\columnwidth}{!}{\includegraphics{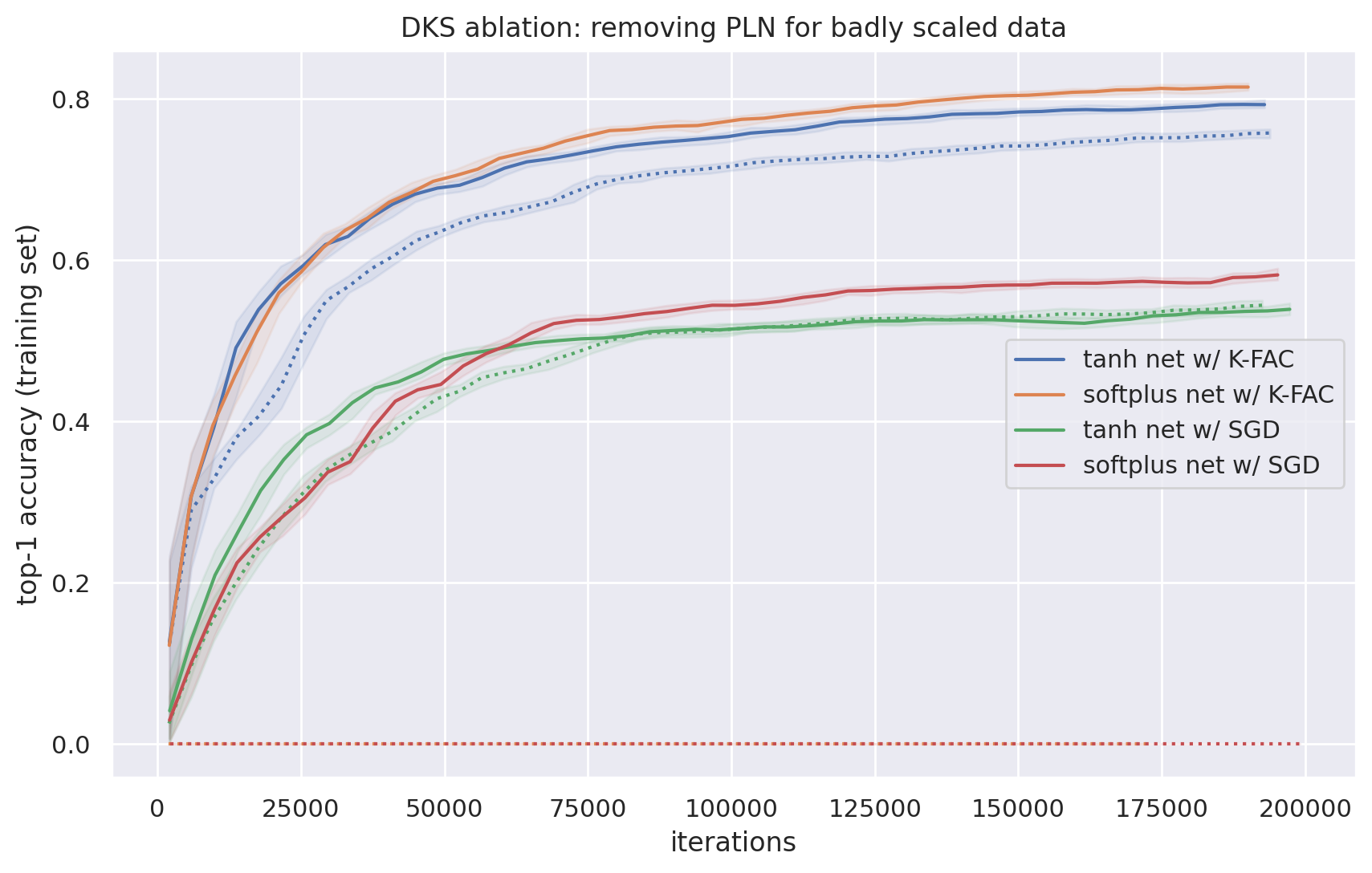}}

From these results we can see that if the original input data is badly scaled, using PLN will have a positive effect on optimization performance.

\subsection{Using equivalent parameters instead of activation transformations}\label{app:equiv-params-experiments}

In this subsection we consider the effect of achieving the four conditions of DKS via ``equivalent parameters'' (as defined in Section \ref{sec:method-as-pure-init}), instead of explicit transformations to the activation functions. This change brings DKS much closer to being a pure initialization approach, but introduces a reparameterization which can have implications for optimization.

\

\resizebox{0.85\columnwidth}{!}{\includegraphics{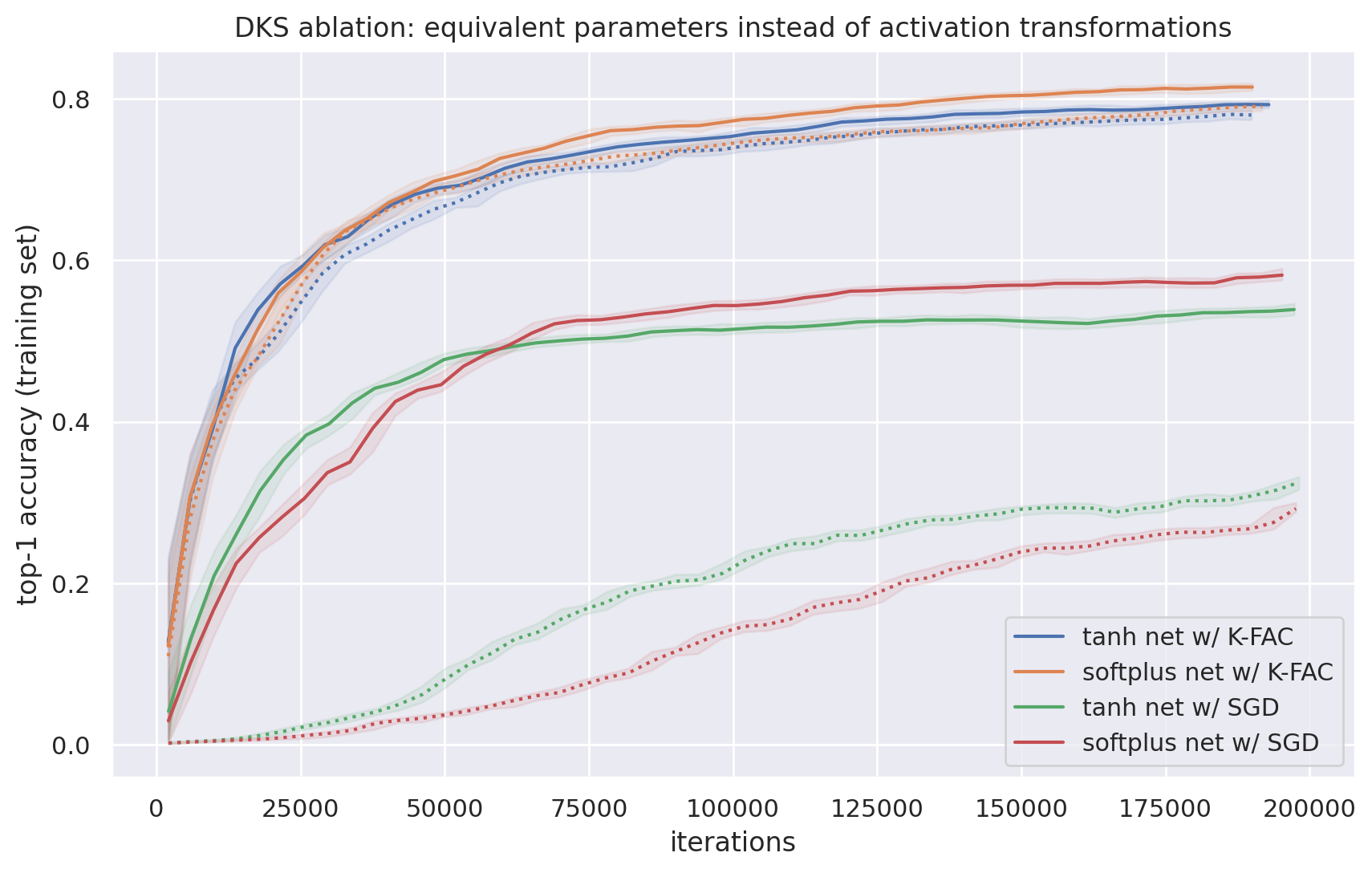}}

From these results we can see that optimization performance with K-FAC is mostly unaffected by this change, while with SGD it becomes {\tmem{much}} worse. This difference in behavior between the two optimizers isn't surprising, since as discussed in Section \ref{sec:method-as-pure-init}, K-FAC is essentially invariant to the kind of reparameterization being performed here, while SGD is not.

\subsection{Replacing max-pooling layers with convolutions}

As discussed in Section \ref{sec:max-pooling}, max-pooling layers are not fully compatible with the Q/C map theory underlying DKS, and so we must be cautious when applying DKS to networks containing them. The networks used in our main experiments contain a max-pooling layer (near the beginning), and in this subsection we will justify this decision by considering the effect of replacing that layer with a standard convolutional layer of same kernel size, stride, etc.

\resizebox{0.85\columnwidth}{!}{\includegraphics{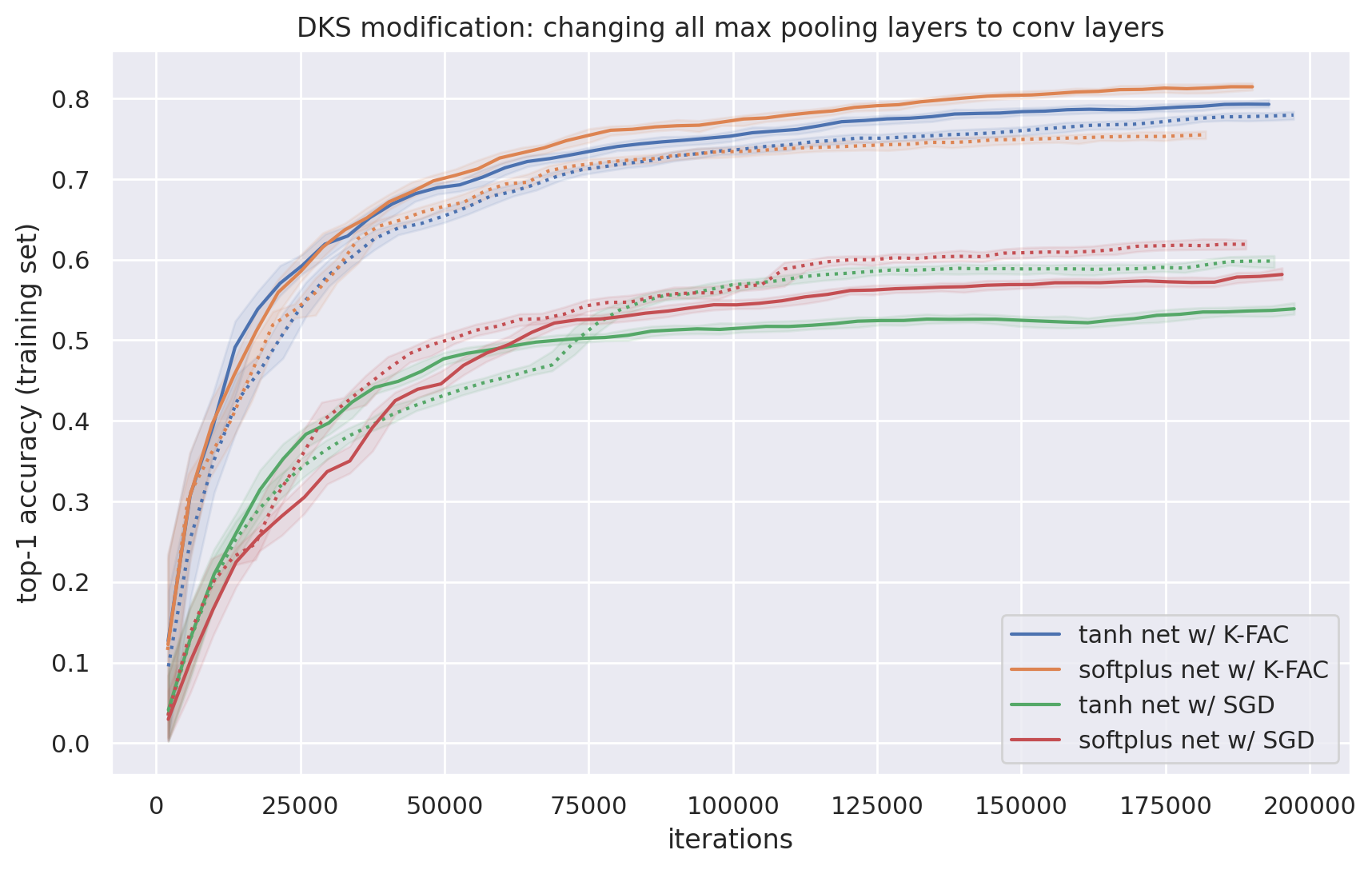}}

From these results we see roughly similar optimization performance with and without this replacement, with K-FAC becoming slightly slower, and SGD becoming slightly faster.

\subsection{Replacing mean pooling layers with weighted mean-pooling layers}\label{app:weighted-mean-pool-experiment}

In this subsection we consider the effect using a weighted mean-pooling layer, as defined in Section \ref{sec:weighted-mean-pools}, in place of the standard mean-pooling layer normally present near the end of our modified ResNet architecture. As discussed in Section \ref{sec:weighted-mean-pools}, standard mean-pooling layers are not compatible with the theory underlying DKS, while weighted mean-pooling layers are, at least to some extent.

\resizebox{0.85\columnwidth}{!}{\includegraphics{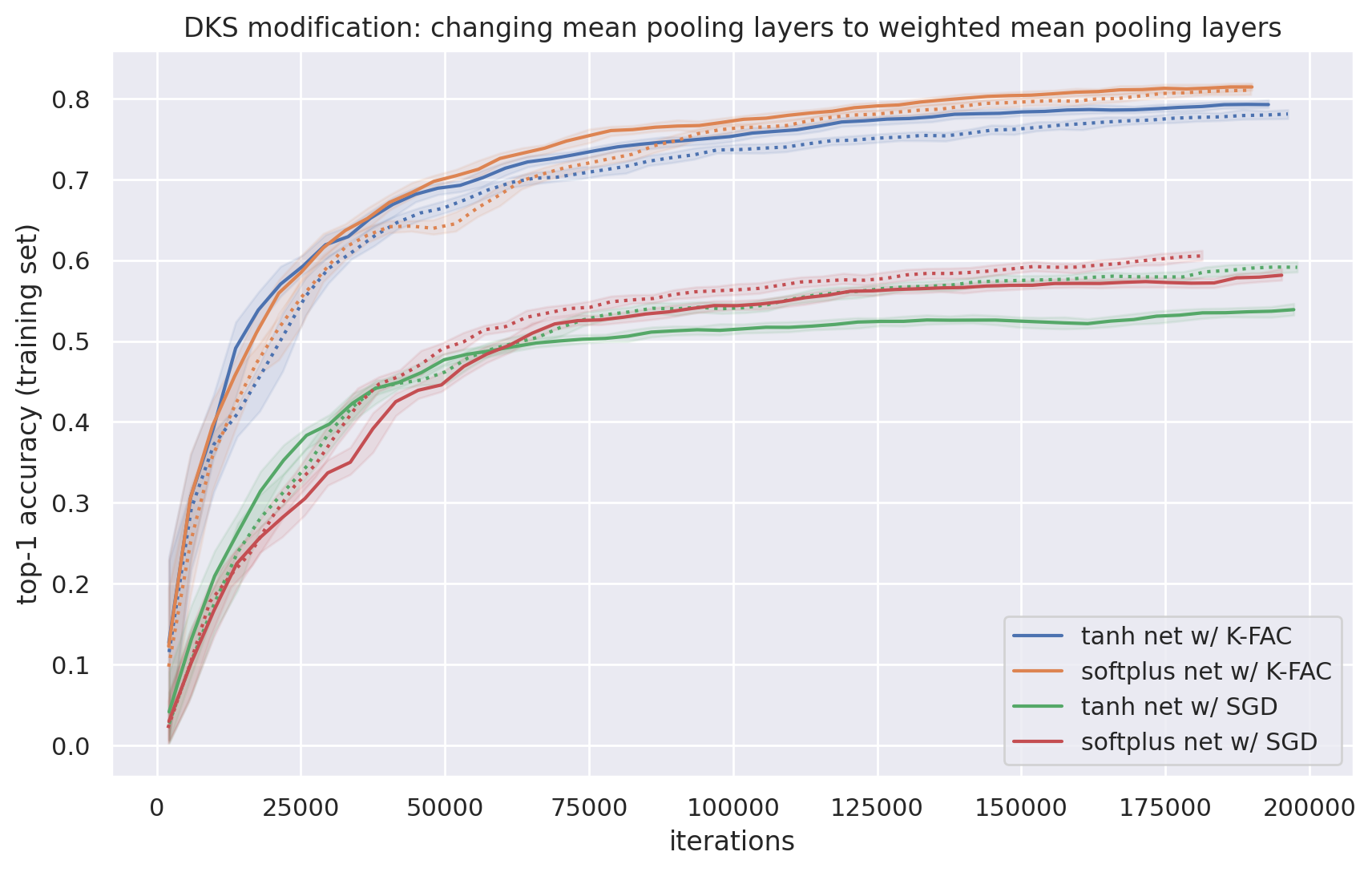}}

\resizebox{0.85\columnwidth}{!}{\includegraphics{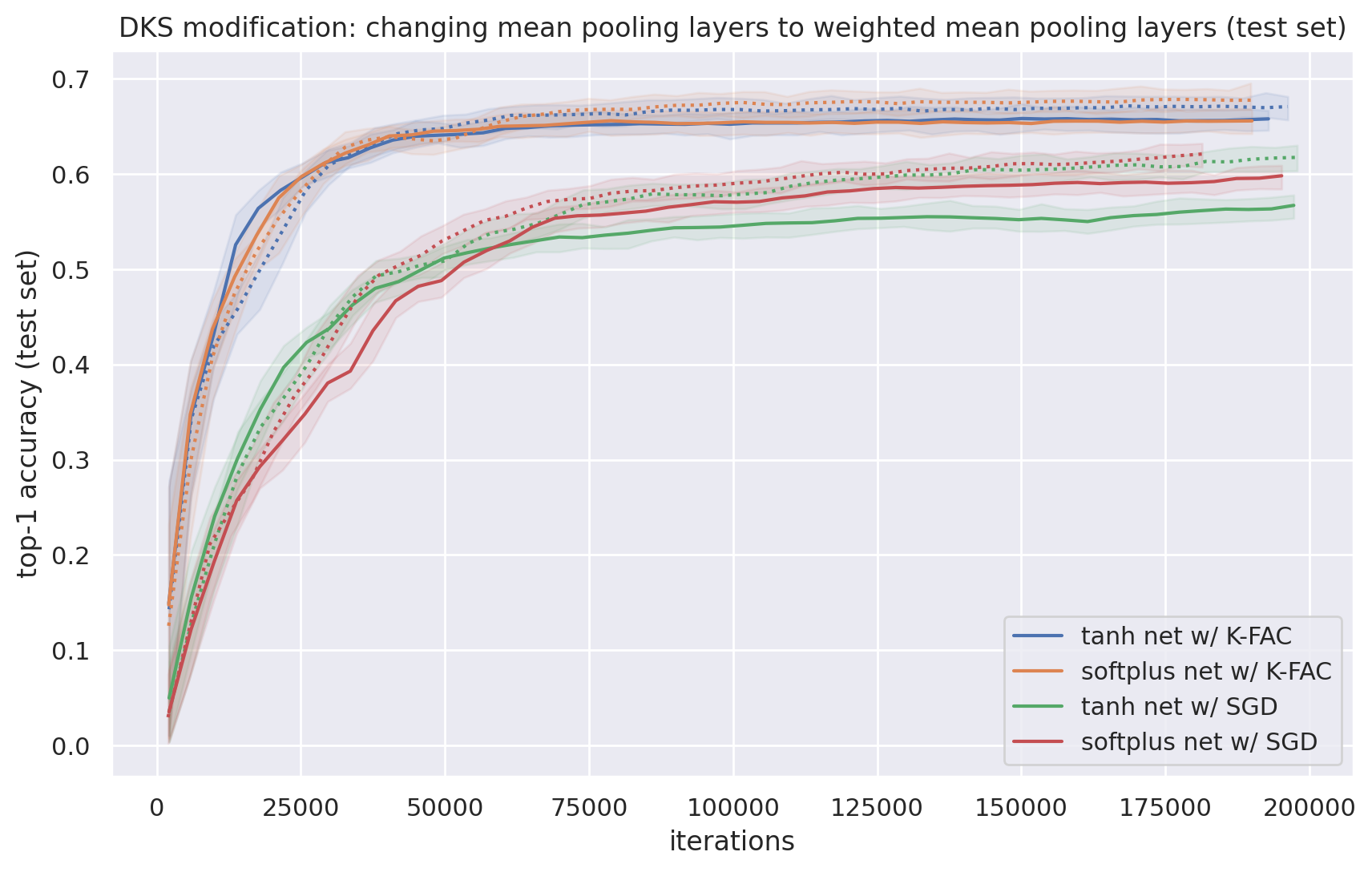}}

From these results we see that optimization performance gets slightly worse with K-FAC, and slightly better with SGD, and that generalization improves slightly as well for both optimizers. Thus, even though we didn't use weighted mean-pooling layers in our main set of experiments, they are probably worth trying when using DKS.






\bibliography{bibliography}

\begin{thebibliography}{95}
\providecommand{\natexlab}[1]{#1}
\providecommand{\url}[1]{\texttt{#1}}
\expandafter\ifx\csname urlstyle\endcsname\relax
  \providecommand{\doi}[1]{doi: #1}\else
  \providecommand{\doi}{doi: \begingroup \urlstyle{rm}\Url}\fi

\bibitem[Abadi et~al.(2015)Abadi, Agarwal, Barham, Brevdo, Chen, Citro,
  Corrado, Davis, Dean, Devin, Ghemawat, Goodfellow, Harp, Irving, Isard, Jia,
  Jozefowicz, Kaiser, Kudlur, Levenberg, Man\'{e}, Monga, Moore, Murray, Olah,
  Schuster, Shlens, Steiner, Sutskever, Talwar, Tucker, Vanhoucke, Vasudevan,
  Vi\'{e}gas, Vinyals, Warden, Wattenberg, Wicke, Yu, and
  Zheng]{tensorflow2015-whitepaper}
M.~Abadi, A.~Agarwal, P.~Barham, E.~Brevdo, Z.~Chen, C.~Citro, G.~S. Corrado,
  A.~Davis, J.~Dean, M.~Devin, S.~Ghemawat, I.~Goodfellow, A.~Harp, G.~Irving,
  M.~Isard, Y.~Jia, R.~Jozefowicz, L.~Kaiser, M.~Kudlur, J.~Levenberg,
  D.~Man\'{e}, R.~Monga, S.~Moore, D.~Murray, C.~Olah, M.~Schuster, J.~Shlens,
  B.~Steiner, I.~Sutskever, K.~Talwar, P.~Tucker, V.~Vanhoucke, V.~Vasudevan,
  F.~Vi\'{e}gas, O.~Vinyals, P.~Warden, M.~Wattenberg, M.~Wicke, Y.~Yu, and
  X.~Zheng.
\newblock {TensorFlow}: Large-scale machine learning on heterogeneous systems,
  2015.
\newblock URL \url{https://www.tensorflow.org/}.
\newblock Software available from tensorflow.org.

\bibitem[Allen-Zhu et~al.(2019)Allen-Zhu, Li, and Song]{allen2019convergence}
Z.~Allen-Zhu, Y.~Li, and Z.~Song.
\newblock A convergence theory for deep learning via over-parameterization.
\newblock In \emph{International Conference on Machine Learning}, pages
  242--252, 2019.

\bibitem[Anil et~al.(2020)Anil, Gupta, Koren, Regan, and
  Singer]{anil2020scalable}
R.~Anil, V.~Gupta, T.~Koren, K.~Regan, and Y.~Singer.
\newblock Scalable second order optimization for deep learning.
\newblock \emph{arXiv preprint arXiv:2002.09018}, 2020.

\bibitem[Anselmi et~al.(2015)Anselmi, Rosasco, Tan, and
  Poggio]{anselmi2015deep}
F.~Anselmi, L.~Rosasco, C.~Tan, and T.~Poggio.
\newblock Deep convolutional networks are hierarchical kernel machines.
\newblock \emph{arXiv preprint arXiv:1508.01084}, 2015.

\bibitem[Arora et~al.(2019)Arora, Du, Hu, Li, Salakhutdinov, and
  Wang]{arora2019exact}
S.~Arora, S.~S. Du, W.~Hu, Z.~Li, R.~Salakhutdinov, and R.~Wang.
\newblock On exact computation with an infinitely wide neural net.
\newblock \emph{Advances in Neural Information Processing Systems}, 2019.

\bibitem[Ba et~al.(2017)Ba, Grosse, and Martens]{ba2017distributed}
J.~Ba, R.~Grosse, and J.~Martens.
\newblock Distributed second-order optimization using kronecker-factored
  approximations.
\newblock In \emph{International Conference on Learning Representations}, 2017.

\bibitem[Ba et~al.(2016)Ba, Kiros, and Hinton]{ba2016layer}
J.~L. Ba, J.~R. Kiros, and G.~E. Hinton.
\newblock Layer normalization.
\newblock \emph{arXiv preprint arXiv:1607.06450}, 2016.

\bibitem[Bachlechner et~al.(2020)Bachlechner, Majumder, Mao, Cottrell, and
  McAuley]{bachlechner2020rezero}
T.~Bachlechner, B.~P. Majumder, H.~H. Mao, G.~W. Cottrell, and J.~McAuley.
\newblock Rezero is all you need: Fast convergence at large depth.
\newblock \emph{arXiv preprint arXiv:2003.04887}, 2020.

\bibitem[Balduzzi(2016)]{balduzzi2016deep}
D.~Balduzzi.
\newblock Deep online convex optimization with gated games.
\newblock \emph{arXiv preprint arXiv:1604.01952}, 2016.

\bibitem[Balduzzi et~al.(2015)Balduzzi, Vanchinathan, and
  Buhmann]{balduzzi2015kickback}
D.~Balduzzi, H.~Vanchinathan, and J.~Buhmann.
\newblock Kickback cuts backprop's red-tape: Biologically plausible credit
  assignment in neural networks.
\newblock In \emph{Proceedings of the AAAI Conference on Artificial
  Intelligence}, 2015.

\bibitem[Balduzzi et~al.(2017)Balduzzi, Frean, Leary, Lewis, Ma, and
  McWilliams]{balduzzi2017shattered}
D.~Balduzzi, M.~Frean, L.~Leary, J.~Lewis, K.~W.-D. Ma, and B.~McWilliams.
\newblock The shattered gradients problem: If resnets are the answer, then what
  is the question?
\newblock In \emph{International Conference on Machine Learning}, pages
  342--350. PMLR, 2017.

\bibitem[Bergstra et~al.(2009)Bergstra, Desjardins, Lamblin, and
  Bengio]{bergstra2009quadratic}
J.~Bergstra, G.~Desjardins, P.~Lamblin, and Y.~Bengio.
\newblock Quadratic polynomials learn better image features.
\newblock \emph{Technical report, 1337}, 2009.

\bibitem[Bradley(1981)]{bradley1981central}
R.~C. Bradley.
\newblock Central limit theorems under weak dependence.
\newblock \emph{Journal of Multivariate Analysis}, 11\penalty0 (1):\penalty0
  1--16, 1981.

\bibitem[Brock et~al.(2021)Brock, De, Smith, and Simonyan]{brock2021high}
A.~Brock, S.~De, S.~L. Smith, and K.~Simonyan.
\newblock High-performance large-scale image recognition without normalization.
\newblock \emph{arXiv preprint arXiv:2102.06171}, 2021.

\bibitem[Brown et~al.(2020)Brown, Mann, Ryder, Subbiah, Kaplan, Dhariwal,
  Neelakantan, Shyam, Sastry, Askell, Agarwal, Herbert-Voss, Krueger, Henighan,
  Child, Ramesh, Ziegler, Wu, Winter, Hesse, Chen, Sigler, Litwin, Gray, Chess,
  Clark, Berner, McCandlish, Radford, Sutskever, and Amodei]{gpt3}
T.~Brown, B.~Mann, N.~Ryder, M.~Subbiah, J.~D. Kaplan, P.~Dhariwal,
  A.~Neelakantan, P.~Shyam, G.~Sastry, A.~Askell, S.~Agarwal, A.~Herbert-Voss,
  G.~Krueger, T.~Henighan, R.~Child, A.~Ramesh, D.~Ziegler, J.~Wu, C.~Winter,
  C.~Hesse, M.~Chen, E.~Sigler, M.~Litwin, S.~Gray, B.~Chess, J.~Clark,
  C.~Berner, S.~McCandlish, A.~Radford, I.~Sutskever, and D.~Amodei.
\newblock Language models are few-shot learners.
\newblock In \emph{Advances in Neural Information Processing Systems}, 2020.

\bibitem[Cai et~al.(2019)Cai, Gao, Hou, Chen, Wang, He, Zhang, and
  Wang]{cai2019gram}
T.~Cai, R.~Gao, J.~Hou, S.~Chen, D.~Wang, D.~He, Z.~Zhang, and L.~Wang.
\newblock Gram-gauss-newton method: Learning overparameterized neural networks
  for regression problems.
\newblock \emph{arXiv preprint arXiv:1905.11675}, 2019.

\bibitem[Cho and Saul(2009)]{cho2012kernel}
Y.~Cho and L.~Saul.
\newblock Kernel methods for deep learning.
\newblock In \emph{Advances in Neural Information Processing Systems},
  volume~22, 2009.

\bibitem[Clevert et~al.(2016)Clevert, Unterthiner, and
  Hochreiter]{clevert2016fast}
D.-A. Clevert, T.~Unterthiner, and S.~Hochreiter.
\newblock Fast and accurate deep network learning by exponential linear units
  (elus).
\newblock In \emph{International Conference on Learning Representations}, 2016.

\bibitem[Cooijmans and Martens(2019)]{cooijmans2019variance}
T.~Cooijmans and J.~Martens.
\newblock On the variance of unbiased online recurrent optimization.
\newblock \emph{arXiv preprint arXiv:1902.02405}, 2019.

\bibitem[Dalibard and Jaderberg(2021)]{firepbt}
V.~Dalibard and M.~Jaderberg.
\newblock Faster improvement rate population based training.
\newblock \emph{arXiv preprint arXiv:2109.13800}, 2021.

\bibitem[Daniely et~al.(2016)Daniely, Frostig, and Singer]{daniely2016toward}
A.~Daniely, R.~Frostig, and Y.~Singer.
\newblock Toward deeper understanding of neural networks: The power of
  initialization and a dual view on expressivity.
\newblock \emph{Advances In Neural Information Processing Systems},
  29:\penalty0 2253--2261, 2016.

\bibitem[De and Smith(2020)]{de2020batch}
S.~De and S.~Smith.
\newblock Batch normalization biases residual blocks towards the identity
  function in deep networks.
\newblock \emph{Advances in Neural Information Processing Systems}, 33, 2020.

\bibitem[Deng et~al.(2009)Deng, Dong, Socher, Li, Li, and
  Fei-Fei]{deng2009imagenet}
J.~Deng, W.~Dong, R.~Socher, L.-J. Li, K.~Li, and L.~Fei-Fei.
\newblock Imagenet: A large-scale hierarchical image database.
\newblock In \emph{2009 IEEE conference on computer vision and pattern
  recognition}, pages 248--255. Ieee, 2009.

\bibitem[Du et~al.(2019{\natexlab{a}})Du, Lee, Li, Wang, and
  Zhai]{du2019gradient}
S.~Du, J.~Lee, H.~Li, L.~Wang, and X.~Zhai.
\newblock Gradient descent finds global minima of deep neural networks.
\newblock In \emph{International Conference on Machine Learning}, pages
  1675--1685, 2019{\natexlab{a}}.

\bibitem[Du et~al.(2019{\natexlab{b}})Du, Zhai, Poczos, and
  Singh]{du2018gradient}
S.~S. Du, X.~Zhai, B.~Poczos, and A.~Singh.
\newblock Gradient descent provably optimizes over-parameterized neural
  networks.
\newblock In \emph{International Conference on Learning Representations},
  2019{\natexlab{b}}.

\bibitem[Duchi et~al.(2011)Duchi, Hazan, and Singer]{duchi2011adaptive}
J.~Duchi, E.~Hazan, and Y.~Singer.
\newblock Adaptive subgradient methods for online learning and stochastic
  optimization.
\newblock \emph{Journal of machine learning research}, 12\penalty0 (7), 2011.

\bibitem[Eaton(1989)]{eaton1989group}
M.~L. Eaton.
\newblock Group invariance applications in statistics.
\newblock In \emph{Regional conference series in Probability and Statistics},
  pages i--133. JSTOR, 1989.

\bibitem[Elfwing et~al.(2018)Elfwing, Uchibe, and Doya]{elfwing2018sigmoid}
S.~Elfwing, E.~Uchibe, and K.~Doya.
\newblock Sigmoid-weighted linear units for neural network function
  approximation in reinforcement learning.
\newblock \emph{Neural Networks}, 107:\penalty0 3--11, 2018.

\bibitem[Espeholt et~al.(2018)Espeholt, Soyer, Munos, Simonyan, Mnih, Ward,
  Doron, Firoiu, Harley, Dunning, et~al.]{espeholt2018impala}
L.~Espeholt, H.~Soyer, R.~Munos, K.~Simonyan, V.~Mnih, T.~Ward, Y.~Doron,
  V.~Firoiu, T.~Harley, I.~Dunning, et~al.
\newblock Impala: Scalable distributed deep-rl with importance weighted
  actor-learner architectures.
\newblock In \emph{International Conference on Machine Learning}, pages
  1407--1416. PMLR, 2018.

\bibitem[Fukushima and Miyake(1982)]{fukushima1982neocognitron}
K.~Fukushima and S.~Miyake.
\newblock Neocognitron: A self-organizing neural network model for a mechanism
  of visual pattern recognition.
\newblock In \emph{Competition and cooperation in neural nets}, pages 267--285.
  Springer, 1982.

\bibitem[Garriga-Alonso et~al.(2018)Garriga-Alonso, Rasmussen, and
  Aitchison]{garriga2018deep}
A.~Garriga-Alonso, C.~E. Rasmussen, and L.~Aitchison.
\newblock Deep convolutional networks as shallow gaussian processes.
\newblock \emph{arXiv preprint arXiv:1808.05587}, 2018.

\bibitem[Glorot and Bengio(2010)]{glorot2010understanding}
X.~Glorot and Y.~Bengio.
\newblock Understanding the difficulty of training deep feedforward neural
  networks.
\newblock In \emph{Proceedings of the thirteenth international conference on
  artificial intelligence and statistics}, pages 249--256. JMLR Workshop and
  Conference Proceedings, 2010.

\bibitem[Goodfellow et~al.(2016)Goodfellow, Bengio, and
  Courville]{goodfellow2016deep}
I.~Goodfellow, Y.~Bengio, and A.~Courville.
\newblock \emph{Deep learning}.
\newblock MIT press, 2016.

\bibitem[Google, 2018()]{tpus}
Google, 2018.
\newblock Cloud tpu.
\newblock \url{https://cloud.google.com/tpu/}.
\newblock Accessed: 2021.

\bibitem[Grosse(2021)]{roger-notes}
R.~Grosse.
\newblock {University of Toronto CSC2541, Topics in Machine Learning: Neural
  Net Training Dynamics, Lecture Notes, Chapter 5: Adaptive Gradient Methods,
  Normalization, and Weight Decay}, 2021.
\newblock URL:
  \url{https://www.cs.toronto.edu/~rgrosse/courses/csc2541_2021/readings/L05_normalization.pdf}.
  Last visited on 09/2021.

\bibitem[Grosse and Martens(2016)]{grosse2016kronecker}
R.~Grosse and J.~Martens.
\newblock A kronecker-factored approximate fisher matrix for convolution
  layers.
\newblock \emph{arXiv preprint arXiv:1602.01407}, 2016.

\bibitem[Gupta et~al.(2018)Gupta, Koren, and Singer]{gupta2018shampoo}
V.~Gupta, T.~Koren, and Y.~Singer.
\newblock Shampoo: Preconditioned stochastic tensor optimization.
\newblock In \emph{International Conference on Machine Learning}, pages
  1842--1850, 2018.

\bibitem[Hazan and Jaakkola(2015)]{hazan2015steps}
T.~Hazan and T.~Jaakkola.
\newblock Steps toward deep kernel methods from infinite neural networks.
\newblock \emph{arXiv preprint arXiv:1508.05133}, 2015.

\bibitem[He et~al.(2015)He, Zhang, Ren, and Sun]{he2015delving}
K.~He, X.~Zhang, S.~Ren, and J.~Sun.
\newblock Delving deep into rectifiers: Surpassing human-level performance on
  imagenet classification.
\newblock In \emph{Proceedings of the IEEE international conference on computer
  vision}, pages 1026--1034, 2015.

\bibitem[He et~al.(2016{\natexlab{a}})He, Zhang, Ren, and Sun]{he2016deep}
K.~He, X.~Zhang, S.~Ren, and J.~Sun.
\newblock Deep residual learning for image recognition.
\newblock In \emph{Proceedings of the IEEE conference on computer vision and
  pattern recognition}, pages 770--778, 2016{\natexlab{a}}.

\bibitem[He et~al.(2016{\natexlab{b}})He, Zhang, Ren, and Sun]{he2016identity}
K.~He, X.~Zhang, S.~Ren, and J.~Sun.
\newblock Identity mappings in deep residual networks.
\newblock In \emph{European conference on computer vision}, pages 630--645.
  Springer, 2016{\natexlab{b}}.

\bibitem[Hochreiter et~al.(2001)Hochreiter, Bengio, Frasconi, Schmidhuber,
  et~al.]{hochreiter2001gradient}
S.~Hochreiter, Y.~Bengio, P.~Frasconi, J.~Schmidhuber, et~al.
\newblock Gradient flow in recurrent nets: the difficulty of learning long-term
  dependencies, 2001.

\bibitem[Ioffe and Szegedy(2015)]{ioffe2015batch}
S.~Ioffe and C.~Szegedy.
\newblock Batch normalization: Accelerating deep network training by reducing
  internal covariate shift.
\newblock In \emph{International conference on machine learning}, pages
  448--456, 2015.

\bibitem[Jacot et~al.(2018)Jacot, Hongler, and Gabriel]{jacot2018neural}
A.~Jacot, C.~Hongler, and F.~Gabriel.
\newblock Neural tangent kernel: Convergence and generalization in neural
  networks.
\newblock In \emph{Advances in neural information processing systems}, 2018.

\bibitem[Jaderberg et~al.(2017)Jaderberg, Dalibard, Osindero, Czarnecki,
  Donahue, Razavi, Vinyals, Green, Dunning, Simonyan,
  et~al.]{jaderberg2017population}
M.~Jaderberg, V.~Dalibard, S.~Osindero, W.~M. Czarnecki, J.~Donahue, A.~Razavi,
  O.~Vinyals, T.~Green, I.~Dunning, K.~Simonyan, et~al.
\newblock Population based training of neural networks.
\newblock \emph{arXiv preprint arXiv:1711.09846}, 2017.

\bibitem[Jones et~al.(2001--)Jones, Oliphant, Peterson, et~al.]{scipy2001}
E.~Jones, T.~Oliphant, P.~Peterson, et~al.
\newblock {SciPy}: Open source scientific tools for {Python}, 2001--.
\newblock URL \url{http://www.scipy.org/}.

\bibitem[Jumper et~al.(2021)Jumper, Evans, Pritzel, Green, Figurnov,
  Ronneberger, Tunyasuvunakool, Bates, {\v{Z}}{\'\i}dek, Potapenko,
  et~al.]{jumper2021highly}
J.~Jumper, R.~Evans, A.~Pritzel, T.~Green, M.~Figurnov, O.~Ronneberger,
  K.~Tunyasuvunakool, R.~Bates, A.~{\v{Z}}{\'\i}dek, A.~Potapenko, et~al.
\newblock Highly accurate protein structure prediction with alphafold.
\newblock \emph{Nature}, pages 1--11, 2021.

\bibitem[Karakida and Osawa(2020)]{karakida2020understanding}
R.~Karakida and K.~Osawa.
\newblock Understanding approximate fisher information for fast convergence of
  natural gradient descent in wide neural networks.
\newblock \emph{Advances in Neural Information Processing Systems}, 33, 2020.

\bibitem[Kingma and Ba(2014)]{kingma2014adam}
D.~P. Kingma and J.~Ba.
\newblock Adam: A method for stochastic optimization.
\newblock In \emph{International Conference on Learning Representations}, 2014.

\bibitem[Klambauer et~al.(2017)Klambauer, Unterthiner, Mayr, and
  Hochreiter]{klambauer2017self}
G.~Klambauer, T.~Unterthiner, A.~Mayr, and S.~Hochreiter.
\newblock Self-normalizing neural networks.
\newblock In \emph{Proceedings of the 31st international conference on neural
  information processing systems}, pages 972--981, 2017.

\bibitem[Kr{\"a}henb{\"u}hl et~al.(2016)Kr{\"a}henb{\"u}hl, Doersch, Donahue,
  and Darrell]{krahenbuhl2016data}
P.~Kr{\"a}henb{\"u}hl, C.~Doersch, J.~Donahue, and T.~Darrell.
\newblock Data-dependent initializations of convolutional neural networks.
\newblock In \emph{International Conference on Learning Representations}, 2016.

\bibitem[Krizhevsky and Hinton(2009)]{krizhevsky2009learning}
A.~Krizhevsky and G.~Hinton.
\newblock Learning multiple layers of features from tiny images.
\newblock Technical report, University of Toronto, 2009.

\bibitem[LeCun et~al.(1998{\natexlab{a}})LeCun, Bottou, Bengio, and
  Haffner]{lecun1998gradient}
Y.~LeCun, L.~Bottou, Y.~Bengio, and P.~Haffner.
\newblock Gradient-based learning applied to document recognition.
\newblock \emph{Proceedings of the IEEE}, 86\penalty0 (11):\penalty0
  2278--2324, 1998{\natexlab{a}}.

\bibitem[LeCun et~al.(1998{\natexlab{b}})LeCun, Bottou, Orr, and
  M{\"u}ller]{lecun1998efficient}
Y.~A. LeCun, L.~Bottou, G.~B. Orr, and K.-R. M{\"u}ller.
\newblock Efficient backprop.
\newblock In \emph{Neural networks: Tricks of the trade}. Springer,
  1998{\natexlab{b}}.

\bibitem[Lee et~al.(2018)Lee, Bahri, Novak, Schoenholz, Pennington, and
  Sohl-Dickstein]{lee2018deep}
J.~Lee, Y.~Bahri, R.~Novak, S.~S. Schoenholz, J.~Pennington, and
  J.~Sohl-Dickstein.
\newblock Deep neural networks as gaussian processes.
\newblock In \emph{International Conference on Learning Representations}, 2018.

\bibitem[Li and Liang(2018)]{li2018learning}
Y.~Li and Y.~Liang.
\newblock Learning overparameterized neural networks via stochastic gradient
  descent on structured data.
\newblock \emph{Advances in neural information processing systems}, 2018.

\bibitem[Li and Arora(2019)]{li2019exponential}
Z.~Li and S.~Arora.
\newblock An exponential learning rate schedule for deep learning.
\newblock In \emph{International Conference on Learning Representations}, 2019.

\bibitem[Luk and Grosse(2018)]{luk2018coordinate}
K.~Luk and R.~Grosse.
\newblock A coordinate-free construction of scalable natural gradient.
\newblock \emph{arXiv preprint arXiv:1808.10340}, 2018.

\bibitem[Mairal et~al.(2014)Mairal, Koniusz, Harchaoui, and
  Schmid]{mairal2014convolutional}
J.~Mairal, P.~Koniusz, Z.~Harchaoui, and C.~Schmid.
\newblock Convolutional kernel networks.
\newblock \emph{Advances in neural information processing systems},
  27:\penalty0 2627--2635, 2014.

\bibitem[Martens(2021)]{martens2021validity}
J.~Martens.
\newblock On the validity of kernel approximations for orthogonally-initialized
  neural networks.
\newblock \emph{arXiv preprint arXiv:2104.05878}, 2021.

\bibitem[Martens and Grosse(2015)]{martens2015optimizing}
J.~Martens and R.~Grosse.
\newblock Optimizing neural networks with kronecker-factored approximate
  curvature.
\newblock \emph{arXiv preprint arXiv:1503.05671}, 2015.

\bibitem[Matthews et~al.(2018)Matthews, Hron, Rowland, Turner, and
  Ghahramani]{matthews2018gaussian}
A.~G. d.~G. Matthews, J.~Hron, M.~Rowland, R.~E. Turner, and Z.~Ghahramani.
\newblock Gaussian process behaviour in wide deep neural networks.
\newblock In \emph{International Conference on Learning Representations}, 2018.

\bibitem[Meckes(2019)]{meckes2019random}
E.~S. Meckes.
\newblock \emph{The random matrix theory of the classical compact groups},
  volume 218.
\newblock Cambridge University Press, 2019.

\bibitem[Mishkin and Matas(2015)]{mishkin2015all}
D.~Mishkin and J.~Matas.
\newblock All you need is a good init.
\newblock \emph{arXiv preprint arXiv:1511.06422}, 2015.

\bibitem[Nadarajah and Kotz(2008)]{nadarajah2008exact}
S.~Nadarajah and S.~Kotz.
\newblock Exact distribution of the max/min of two gaussian random variables.
\newblock \emph{IEEE Transactions on very large scale integration (VLSI)
  systems}, 16\penalty0 (2):\penalty0 210--212, 2008.

\bibitem[Neal(1996)]{neal1996bayesian}
R.~M. Neal.
\newblock Bayesian learning for neural networks.
\newblock \emph{Lecture notes in statistics}, 118, 1996.

\bibitem[Novak et~al.(2018)Novak, Xiao, Lee, Bahri, Yang, Hron, Abolafia,
  Pennington, and Sohl-Dickstein]{novak2018bayesian}
R.~Novak, L.~Xiao, J.~Lee, Y.~Bahri, G.~Yang, J.~Hron, D.~A. Abolafia,
  J.~Pennington, and J.~Sohl-Dickstein.
\newblock Bayesian deep convolutional networks with many channels are gaussian
  processes.
\newblock \emph{arXiv preprint arXiv:1810.05148}, 2018.

\bibitem[Pascanu et~al.(2013)Pascanu, Mikolov, and
  Bengio]{pascanu2013difficulty}
R.~Pascanu, T.~Mikolov, and Y.~Bengio.
\newblock On the difficulty of training recurrent neural networks.
\newblock In \emph{International conference on machine learning}, pages
  1310--1318, 2013.

\bibitem[Poole et~al.(2016)Poole, Lahiri, Raghu, Sohl-Dickstein, and
  Ganguli]{poole2016exponential}
B.~Poole, S.~Lahiri, M.~Raghu, J.~Sohl-Dickstein, and S.~Ganguli.
\newblock Exponential expressivity in deep neural networks through transient
  chaos.
\newblock \emph{Advances in neural information processing systems},
  29:\penalty0 3360--3368, 2016.

\bibitem[Powell(1964)]{powell1964efficient}
M.~J. Powell.
\newblock An efficient method for finding the minimum of a function of several
  variables without calculating derivatives.
\newblock \emph{The Computer Journal}, 7\penalty0 (2):\penalty0 155--162, 1964.

\bibitem[Prajit et~al.(2017)Prajit, Zoph, and Quoc]{prajit2017swish}
R.~Prajit, B.~Zoph, and V.~L. Quoc.
\newblock Swish: a self-gated activation function.
\newblock \emph{arXiv preprint arXiv:1710.059417}, 2017.

\bibitem[Rahimi and Recht(2008)]{rahimi2008weighted}
A.~Rahimi and B.~Recht.
\newblock Weighted sums of random kitchen sinks: replacing minimization with
  randomization in learning.
\newblock In \emph{Nips}, pages 1313--1320, 2008.

\bibitem[Ramachandran et~al.(2017)Ramachandran, Zoph, and
  Le]{ramachandran2017swish}
P.~Ramachandran, B.~Zoph, and Q.~V. Le.
\newblock Swish: a self-gated activation function.
\newblock \emph{arXiv preprint arXiv:1710.05941}, 7:\penalty0 1, 2017.

\bibitem[Santurkar et~al.(2018)Santurkar, Tsipras, Ilyas, and
  Madry]{santurkar2018does}
S.~Santurkar, D.~Tsipras, A.~Ilyas, and A.~Madry.
\newblock How does batch normalization help optimization?
\newblock In \emph{Proceedings of the 32nd international conference on neural
  information processing systems}, pages 2488--2498, 2018.

\bibitem[Saxe et~al.(2014)Saxe, McClelland, and Ganguli]{saxe2014exact}
A.~Saxe, J.~McClelland, and S.~Ganguli.
\newblock Exact solutions to the nonlinear dynamics of learning in deep linear
  neural networks.
\newblock In \emph{International Conference on Learning Representations}, 2014.

\bibitem[Schoenberg(1988)]{schoenberg1988positive}
I.~Schoenberg.
\newblock Positive definite functions on spheres.
\newblock \emph{Duke Math. J}, 1:\penalty0 172, 1988.

\bibitem[Schoenholz et~al.(2017)Schoenholz, Gilmer, Ganguli, and
  Sohl-Dickstein]{schoenholz2016deep}
S.~S. Schoenholz, J.~Gilmer, S.~Ganguli, and J.~Sohl-Dickstein.
\newblock Deep information propagation.
\newblock In \emph{International Conference on Learning Representations}, 2017.

\bibitem[Shao et~al.(2020)Shao, Hu, Wang, Xue, and Raj]{shao2020normalization}
J.~Shao, K.~Hu, C.~Wang, X.~Xue, and B.~Raj.
\newblock Is normalization indispensable for training deep neural network?
\newblock \emph{Advances in Neural Information Processing Systems}, 33, 2020.

\bibitem[Silver et~al.(2018)Silver, Hubert, Schrittwieser, Antonoglou, Lai,
  Guez, Lanctot, Sifre, Kumaran, Graepel, et~al.]{silver2018general}
D.~Silver, T.~Hubert, J.~Schrittwieser, I.~Antonoglou, M.~Lai, A.~Guez,
  M.~Lanctot, L.~Sifre, D.~Kumaran, T.~Graepel, et~al.
\newblock A general reinforcement learning algorithm that masters chess, shogi,
  and go through self-play.
\newblock \emph{Science}, 362\penalty0 (6419):\penalty0 1140--1144, 2018.

\bibitem[Simonyan and Zisserman(2015)]{simonyan2015very}
K.~Simonyan and A.~Zisserman.
\newblock Very deep convolutional networks for large-scale image recognition.
\newblock In \emph{International Conference on Learning Representations}, 2015.

\bibitem[Szegedy et~al.(2015)Szegedy, Liu, Jia, Sermanet, Reed, Anguelov,
  Erhan, Vanhoucke, and Rabinovich]{szegedy2015going}
C.~Szegedy, W.~Liu, Y.~Jia, P.~Sermanet, S.~Reed, D.~Anguelov, D.~Erhan,
  V.~Vanhoucke, and A.~Rabinovich.
\newblock Going deeper with convolutions.
\newblock In \emph{Proceedings of the IEEE conference on computer vision and
  pattern recognition}, pages 1--9, 2015.

\bibitem[Szegedy et~al.(2017)Szegedy, Ioffe, Vanhoucke, and
  Alemi]{szegedy2017inception}
C.~Szegedy, S.~Ioffe, V.~Vanhoucke, and A.~A. Alemi.
\newblock Inception-v4, inception-resnet and the impact of residual connections
  on learning.
\newblock In \emph{Thirty-first AAAI conference on artificial intelligence},
  2017.

\bibitem[Tan and Le(2019)]{tan2019efficientnet}
M.~Tan and Q.~Le.
\newblock Efficientnet: Rethinking model scaling for convolutional neural
  networks.
\newblock In \emph{International Conference on Machine Learning}, pages
  6105--6114, 2019.

\bibitem[Vaswani et~al.(2017)Vaswani, Shazeer, Parmar, Uszkoreit, Jones, Gomez,
  Kaiser, and Polosukhin]{vaswani2017attention}
A.~Vaswani, N.~Shazeer, N.~Parmar, J.~Uszkoreit, L.~Jones, A.~N. Gomez,
  {\L}.~Kaiser, and I.~Polosukhin.
\newblock Attention is all you need.
\newblock In \emph{Advances in neural information processing systems}, pages
  5998--6008, 2017.

\bibitem[Veit et~al.(2016)Veit, Wilber, and Belongie]{veit2016residual}
A.~Veit, M.~J. Wilber, and S.~Belongie.
\newblock Residual networks behave like ensembles of relatively shallow
  networks.
\newblock \emph{Advances in neural information processing systems},
  29:\penalty0 550--558, 2016.

\bibitem[{Wikipedia contributors}(2021)]{wiki_hermite}
{Wikipedia contributors}.
\newblock Hermite polynomials --- {Wikipedia}{,} the free encyclopedia, 2021.
\newblock URL
  \url{https://en.wikipedia.org/w/index.php?title=Hermite_polynomials&oldid=1022750425}.
\newblock [Online; accessed 30-August-2021].

\bibitem[Williams(1997)]{williams1997computing}
C.~K. Williams.
\newblock Computing with infinite networks.
\newblock \emph{Advances in neural information processing systems}, pages
  295--301, 1997.

\bibitem[Wu et~al.(2018)Wu, Ren, Liao, and Grosse]{wu2018understanding}
Y.~Wu, M.~Ren, R.~Liao, and R.~Grosse.
\newblock Understanding short-horizon bias in stochastic meta-optimization.
\newblock In \emph{International Conference on Learning Representations}, 2018.

\bibitem[Xiao et~al.(2018)Xiao, Bahri, Sohl-Dickstein, Schoenholz, and
  Pennington]{xiao2018dynamical}
L.~Xiao, Y.~Bahri, J.~Sohl-Dickstein, S.~Schoenholz, and J.~Pennington.
\newblock Dynamical isometry and a mean field theory of cnns: How to train
  10,000-layer vanilla convolutional neural networks.
\newblock In \emph{International Conference on Machine Learning}, pages
  5393--5402, 2018.

\bibitem[Xiao et~al.(2020)Xiao, Pennington, and
  Schoenholz]{xiao2020disentangling}
L.~Xiao, J.~Pennington, and S.~Schoenholz.
\newblock Disentangling trainability and generalization in deep neural
  networks.
\newblock In \emph{International Conference on Machine Learning}, pages
  10462--10472, 2020.

\bibitem[Yang et~al.(2019)Yang, Pennington, Rao, Sohl-Dickstein, and
  Schoenholz]{yang2019mean}
G.~Yang, J.~Pennington, V.~Rao, J.~Sohl-Dickstein, and S.~S. Schoenholz.
\newblock A mean field theory of batch normalization.
\newblock In \emph{International Conference on Learning Representations}, 2019.

\bibitem[Zagoruyko and Komodakis(2016)]{zagoruyko2016wide}
S.~Zagoruyko and N.~Komodakis.
\newblock Wide residual networks.
\newblock In \emph{British Machine Vision Conference 2016}. British Machine
  Vision Association, 2016.

\bibitem[Zhang et~al.(2019{\natexlab{a}})Zhang, Li, Nado, Martens, Sachdeva,
  Dahl, Shallue, and Grosse]{zhang2019algorithmic}
G.~Zhang, L.~Li, Z.~Nado, J.~Martens, S.~Sachdeva, G.~Dahl, C.~Shallue, and
  R.~B. Grosse.
\newblock Which algorithmic choices matter at which batch sizes? insights from
  a noisy quadratic model.
\newblock \emph{Advances in neural information processing systems},
  32:\penalty0 8196--8207, 2019{\natexlab{a}}.

\bibitem[Zhang et~al.(2019{\natexlab{b}})Zhang, Martens, and
  Grosse]{zhang2019fast}
G.~Zhang, J.~Martens, and R.~B. Grosse.
\newblock Fast convergence of natural gradient descent for over-parameterized
  neural networks.
\newblock In \emph{NeurIPS}, 2019{\natexlab{b}}.

\bibitem[Zhang et~al.(2019{\natexlab{c}})Zhang, Dauphin, and
  Ma]{zhang2019fixup}
H.~Zhang, Y.~N. Dauphin, and T.~Ma.
\newblock Fixup initialization: Residual learning without normalization.
\newblock In \emph{International Conference on Learning Representations},
  2019{\natexlab{c}}.

\end{thebibliography}

\end{document}